\documentclass{article}

\usepackage{iclr2027_conference,times}
\usepackage[T1]{fontenc}
\usepackage{float}
\usepackage{amssymb}
\usepackage{mathtools}
\usepackage{mathrsfs}
\usepackage{graphicx}
\usepackage{subcaption}
\usepackage{url}
\usepackage{booktabs}
\usepackage{amsthm}
\usepackage{booktabs,graphicx,multirow,bigstrut}
\usepackage{xcolor}
\usepackage{wrapfig}
\usepackage{hyperref}
\usepackage{changepage}
\newsavebox{\banditupperbox}

\newtheorem{proposition}{Proposition}

\newtheorem{corollary}{Corollary}

\newtheorem{definition}{Definition}

\newcommand{\piTF}{\pi_{\mathrm{TF}}}

\newcommand{\E}{\mathbb{E}}

\renewcommand{\headrulewidth}{0pt}

\title{Is In-Context Reinforcement Learning A Posterior Sampler? An Empirical Study}

\author{Weiqin Chen\thanks{Equal contribution}$~~$, Xinlong Du$^{*}$, Gesualdo Scutari, Harsha Honnappa \\
Edwardson School of Industrial Engineering, Purdue University\\
West Lafayette, IN 47907, USA \\
\texttt{\{chen6439, du339, gscutari, honnappa\}@purdue.edu} \\
}

\iclrfinalcopy 

\begin{document}
\maketitle

\fancyhead{}

\begin{abstract}
In-context reinforcement learning (ICRL) exploits the impressive in-context learning ability of large language models to zero-shot generalize to new, unseen tasks at inference time without any parameter updates. Among many approaches, the Decision-Pretrained Transformer (DPT) is particularly appealing for its theoretical grounding: under idealized assumptions such as a sufficiently expressive model and infinite pretraining data, DPT can perform Bayesian posterior sampling. Whether this characterization survives in practice remains open. Therefore, we ask: \textit{is ICRL, specifically DPT and our proposed DPT variants, a Bayesian posterior sampler in practice?} To approach this question, we establish two necessary conditions that a Bayesian posterior sampler must satisfy: \textit{sufficient-statistic invariance} and the \textit{martingale property}. We then construct statistical hypothesis tests to statistically assess if practically-trained ICRL models satisfy these conditions. Applying these tests in Gaussian, Bernoulli, and linear bandit settings, we find that both conditions are consistently violated, suggesting that ICRL may not behave as Bayesian posterior samplers in practice. Beyond this negative result, our tests provide a principled toolkit for testing the Bayesian behavior of ICRL.
\end{abstract}

\section{Introduction}
\label{sec:intro}
Reinforcement learning (RL) has demonstrated remarkable performance in a wide range of applications, such as strategy games~\citep{mnih2013playing, mnih2015human}, robotics~\citep{levine2016end, duan2016benchmarking}, and recommender systems~\citep{afsar2022reinforcement, lin2023survey}. However, a fundamental challenge in RL is the limited capability to few-shot or zero-shot generalize to new unseen tasks, as traditional RL methods learn a policy for a single task through an ocean of environment interactions and extensive trials-and-errors, and thus typically requires retraining a policy from scratch for each new task. In-context RL (ICRL) offers a compelling alternative attributed to the impressive capability of in-context learning by large language models (LLMs) \citep{dong2024survey, li2023transformers, wei2023larger}. Specifically, a transformer \citep{vaswani2017attention} pretrained on a diverse set of tasks/environments can adapt to and reinforcement learn in new, unseen tasks at test time, purely by conditioning on interactions with the new, unseen tasks (a.k.a. context) without any parameter updates. Several ICRL frameworks have been proposed to realize this vision. \emph{Algorithm Distillation} (AD) pretrains a transformer on the complete learning histories of an RL algorithm (e.g., PPO \citep{schulman2017proximal} or SAC \citep{haarnoja2018soft}), where the context consists of a sequence of episodes exhibiting progressive improvement and dynamics \citep{laskin2022context}. The transformer learns to extrapolate this improvement operator, effectively distilling the source RL algorithm's learning dynamics into its in-context behavior. \emph{Decision-Pretrained Transformer} (DPT) adopts a fundamentally different approach \citep{lee2024supervised}. Rather than imitating an evolving algorithm, DPT pretrains a transformer via supervised learning to predict optimal actions conditioned on contexts of varying lengths. Moreover, DPT can employ arbitrary context data, such as interactions collected under a random policy \citep{chen2024random}, by leveraging the optimal action labels.

Notably, a central theoretical grounding underpins DPT's appeal. Prior works \citep{lee2024supervised,lin2024transformers} establish a theoretical understanding that ``DPT is an efficient implementation of \textit{Bayesian posterior sampling}'' that achieves rate-optimal Bayesian regret \citep{russo2016information}. This, however, relies on strong assumptions such as a sufficiently expressive model and infinite pretraining data to exactly match or closely approximate the pretraining data distribution \citep{lee2024supervised,lin2024transformers}, which may not hold in practice.
\begin{wrapfigure}{r}{0.5\textwidth}
\centering
\begin{subfigure}{0.245\textwidth}
\includegraphics[width=\textwidth]{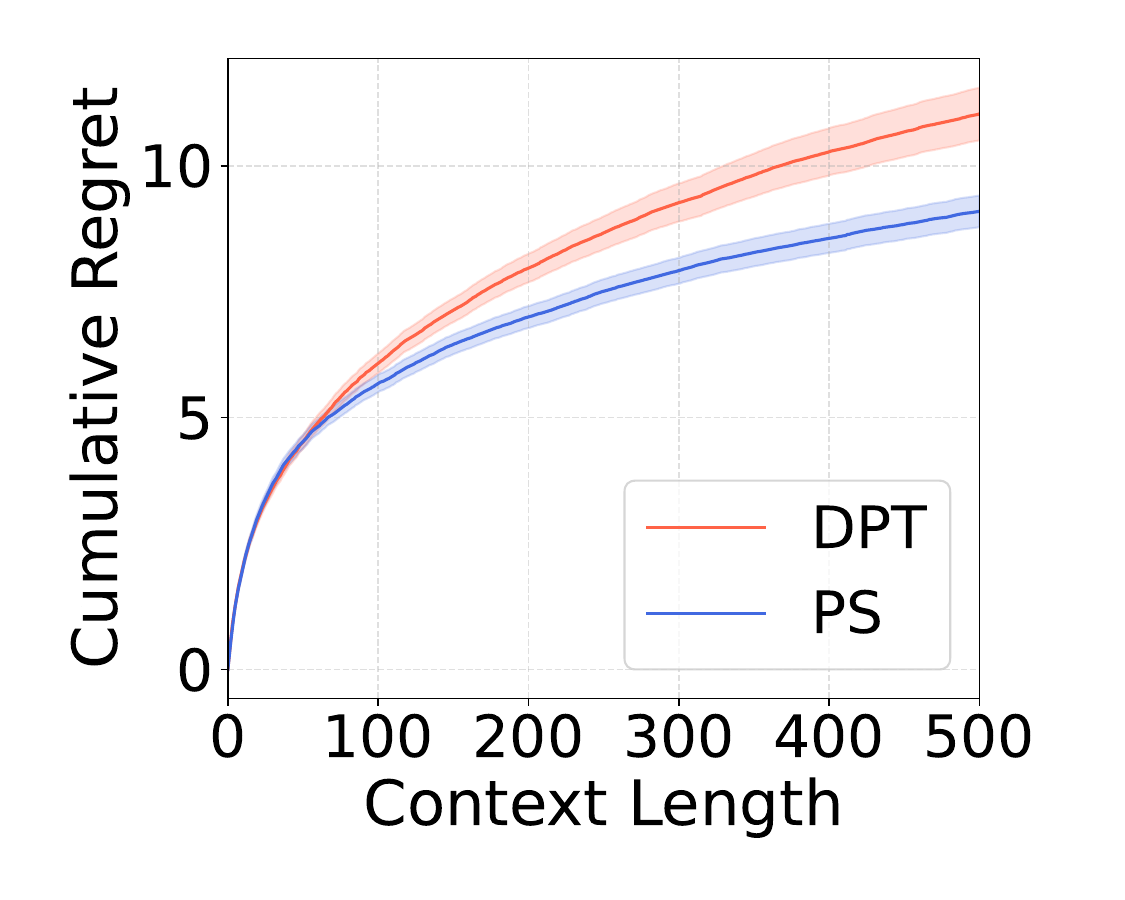}
\caption{Cumulative regret}
\end{subfigure}
\begin{subfigure}{0.245\textwidth}
\includegraphics[width=\textwidth]{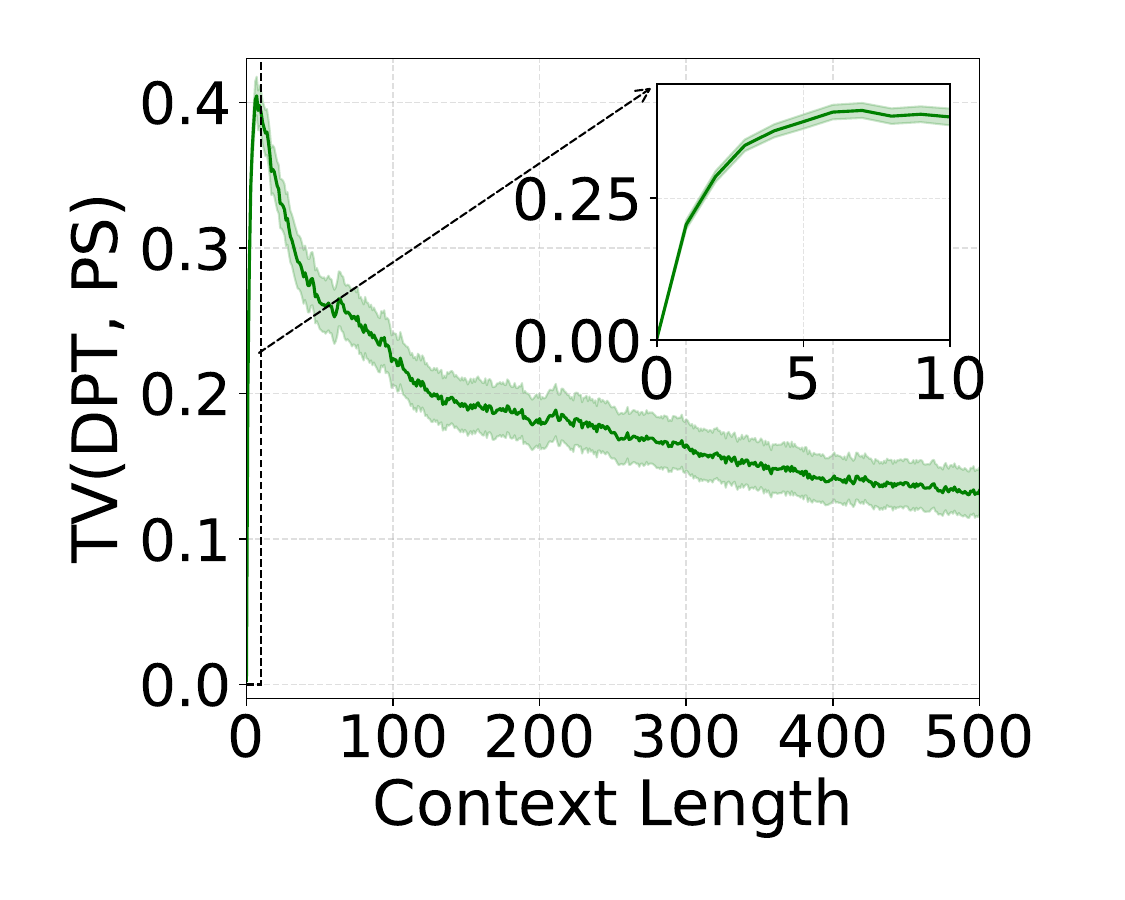}
\caption{TV distance}
\end{subfigure}
\caption{DPT versus Thompson Sampling in Gaussian bandits: cumulative regret and total variation (TV) distance between action distributions.}
\label{fig_motivation}
\end{wrapfigure}
In addition, an empirical gap between DPT and a Bayesian posterior sampler (PS) is observed from experiments (see Figure~\ref{fig_motivation}), where we compare the default setting of DPT as in \citep{lee2024supervised} and Thompson sampling \citep{thompson1933likelihood} (a Bayesian posterior sampler in multi-armed bandits) in Gaussian (multi-armed) bandits. Accordingly, this paper aims to explore and answer the following research question:

\begin{adjustwidth}{0.35cm}{0.35cm}
\emph{Is ICRL, including DPT and our proposed DPT variants, a Bayesian posterior sampler in practice?}
\end{adjustwidth}

Furthermore, it is worth highlighting that the research question above is also motivated by an empirical study \citep{falck2024bayesian}, where in-context learning of common LLMs like GPT-3.5 \citep{openaigpt35}, GPT-4 \citep{openai2023gpt4}, Llama-2 \citep{touvron2023llama2}, and Mistral \citep{jiang2023mistral} is demonstrated non-Bayesian from a martingale perspective (a necessary condition for Bayesian systems). Inspired by \citep{falck2024bayesian}, we take a statistical hypothesis testing approach to the research question above. Our \textbf{main contributions} are as follows.
\begin{itemize}
    \item \textbf{Necessary conditions.} We provide a behavioral characterization of a transformer as a Bayesian posterior sampler in terms of its action distribution, and derive two necessary conditions from it, namely \emph{sufficient-statistic invariance} and the \emph{martingale property}, which we instantiate for Gaussian, Bernoulli, and linear bandits.

    \item \textbf{Testing framework and DPT variants.} We turn each condition into a hypothesis test on action probabilities. Since DPT predicts only actions, we propose DPT$'$ and GDPT, which augment DPT with a reward head so that the martingale property becomes testable.

    \item \textbf{Empirical findings.} Across all bandit settings, deployment modes, context transformations, and context lengths, DPT, DPT$'$, and GDPT consistently violate sufficient-statistic invariance, and DPT$'$ and GDPT additionally violate the martingale property, even when their decision performance closely tracks Thompson Sampling. Ablation studies that enforce permutation invariance do not remove the violations. Beyond this negative result, our statistical tests provide a principled toolkit for auditing Bayesian claims about ICRL.
\end{itemize}

A more detailed discussion of related work is provided in Appendix~\ref{sec:related}.

\section{In-Context Reinforcement Learning Under DPT}
\label{sec:dpt}
This section describes the three main phases of ICRL under the DPT framework \citep{lee2024supervised}: pretraining dataset collection, pretraining, and deployment. We consider a finite-horizon decision-making setting where each environment (or task) $\tau = \langle \mathcal{A}, \mathcal{R}, H \rangle$ consists of an action space $\mathcal{A}$, reward kernel $\mathcal{R}: \mathcal{A} \to \Delta(\mathbb{R})$ evaluating the quality of the decision, and the horizon $H$.
\subsection{Pretraining Dataset Collection}
The pretraining dataset is collected by repeatedly sampling environments, collecting context data, and labeling optimal actions. More concretely, for each instance $i \in \{1,\dots,N\}$:
\begin{enumerate}
    \item[$i)$] \textbf{Sample an environment:} Draw $\tau_i \sim \mathcal{T}_{\mathrm{env}}$ from a distribution $\mathcal{T}_{\mathrm{env}}$ over environments.
    \item[$ii)$] \textbf{Collect context data:} Generate context data with $n$ transitions $D_n^{(i)} = \big((A_1^{(i)}, R_1^{(i)}),\ldots,(A_n^{(i)}, R_n^{(i)})\big)$ by deploying a context-collection policy in environment $\tau_i$, i.e., $D_n^{(i)} \sim \mathcal{D}_{\mathrm{context}}(\cdot\,; \tau_i) \in \Delta((\mathcal{A} \times \mathbb{R})^n)$. The context-collection policy may be arbitrary, ranging from a simple random policy to more sophisticated approaches such as PPO or posterior sampling. Throughout the paper, a context is an \emph{ordered} sequence of action--reward pairs. We use uppercase $(A_t, R_t)$ to denote the random action and reward at position $t$, and lowercase $(a, r)$ to denote their realizations.
    \item[$iii)$] \textbf{Label the optimal action and construct the pretraining dataset:}  Let $a^{*,(i)} \sim \pi^*_{\tau_i}$, where $\pi^*_{\tau_i}$ denotes the optimal policy for the environment $\tau_i$. Select the optimal action $a^{*,(i)} $ as the label, and add tuple $(D_n^{(i)}, a^{*,(i)})$ to the pretraining dataset $\mathcal{B}$.
\end{enumerate}
Together, the joint pretraining distribution is given by
\begin{equation}
\mathcal{P}_{\mathrm{pre}}(\tau, D, a^*) = \mathcal{T}_{\mathrm{env}}(\tau) \cdot \mathcal{D}_{\mathrm{context}}(D; \tau) \cdot \pi^*_\tau(a^*).
\end{equation}
\subsection{Pretraining}
Having collected the pretraining dataset $\mathcal{B}$, we use it to pretrain a decoder-only transformer such as GPT-2 \citep{radford2019language}, with $L$ layers, $H_{\mathrm{attn}}$ attention heads, and embedding dimension $d_{\mathrm{emb}}$.

\textbf{Architecture.}
Consider a generic instance $(D_n, a^*) \in \mathcal{B}$, where we drop the index $i$ for brevity. Following \citet{lee2024supervised}, each transition in the context $D_n=\big((A_1,R_1),\ldots,(A_n,R_n)\big)$ is stacked into a single token $\xi_t=(A_t,R_t)\in\mathbb{R}^{d_\xi}$ with $d_\xi=d_{\mathcal{A}}+1$, where the action $A_t$ is identified with its encoding in $\mathbb{R}^{d_{\mathcal{A}}}$. A fixed start token $\xi_0=\mathbf{1}\in\mathbb{R}^{d_\xi}$, which plays the role of the query-state token in \citep{lee2024supervised} (a constant in the multi-armed bandit setting), is prepended, giving the input sequence $\xi_{0:n}=(\xi_0,\xi_1,\ldots,\xi_n)\in\mathbb{R}^{d_\xi\times(n+1)}.$ Each token of $\xi_{0:n}$ is first mapped to a $d_{\mathrm{emb}}$-dimensional embedding by a linear layer, and a learned positional encoding is added, yielding $h^{(0)}_t\in\mathbb{R}^{d_{\mathrm{emb}}}$ for position $t=0,\ldots,n$. The embedded sequence then passes through $L$ standard transformer blocks. Each block performs two operations, each wrapped in a residual connection: $(i)$ causal multi-head self-attention, which updates the representation at position $t$ by aggregating information from positions $\le t$, and $(ii)$ a position-wise feed-forward network, e.g., an MLP applied independently to each position with weights shared across positions. A linear head $W_{\mathrm{out}}$ and a Softmax operation map the final-layer representation $h^{(L)}_t$ at each position to a distribution over actions
\begin{equation}
\pi_{\mathrm{TF}}(\cdot\mid D_t):=\mathrm{Softmax}\big(W_{\mathrm{out}}\,h^{(L)}_t\big)\in\Delta(\mathcal{A}),\qquad t=0,\ldots,n,
\end{equation}
where $D_t:=\big((A_1,R_1),\ldots,(A_t,R_t)\big)$ is the length-$t$ prefix of the context and $D_0$ is the empty context. Due to the causal attention, $\pi_{\mathrm{TF}}(\cdot\mid D_t)$ depends only on the first $t$ transitions, not on future context tokens. In particular, $\pi_{\mathrm{TF}}(\cdot\mid D_0)$ acts without any information about the environment, and can be interpreted as the model's prior.

\textbf{Loss function.}
The transformer is trained to minimize the expected cross-entropy loss over the pretraining distribution
\begin{equation}
\min_\psi \; \mathbb{E}_{\mathcal{P}_{\mathrm{pre}}}\left[\sum_{t=1}^n\ell\big(\pi_{\mathrm{TF}}(\cdot\mid D_t),a^*\big)\right],
\quad
\ell\big(\pi_{\mathrm{TF}}(\cdot\mid D_t),a^*\big):=-\log \pi_{\mathrm{TF}}(a^*\mid D_t),
\end{equation}
where $\psi$ denotes the transformer's parameters.

\subsection{Deployment}

After pretraining, the model is deployed in a new environment $\tau\sim\mathcal{T}_{\mathrm{test}}$ without any parameter updates. Given a context $D_t$, the model is queried via a forward pass on $\xi_{0:t}$, and the action distribution $\pi_{\mathrm{TF}}(\cdot\mid D_t)$ at the \emph{last position} is read out. DPT supports two types of deployment, offline and online, which differ in how the input context is constructed.

\textbf{Offline deployment.}
The model does not collect the context itself. Instead, a static context $D_n\sim\mathcal{D}_{\mathrm{test}}(\cdot\,;\tau)$ of $n$ transitions is provided, and the model acts according to $\pi_{\mathrm{TF}}(\cdot\mid D_n)$.

\textbf{Online deployment.}
The context is initialized as the empty context $D_0$, and the model collects its own context. At each round $t\in[n]$, the model samples an action from its predicted distribution, $A_t\sim \pi_{\mathrm{TF}}(\cdot\mid D_{t-1})$, observes a reward $R_t$ from the test environment, and appends the transition to the context, i.e., $D_t=\big(D_{t-1},(A_t,R_t)\big)$.

\section{Are DPT and Its Variants Bayesian Posterior Samplers?}
\label{sec:posterior_sampling_behavior}
Having introduced the DPT framework, we emphasize that the idealized assumptions under which DPT is expected to behave as a Bayesian posterior sampler, such as a sufficiently expressive model and infinite pretraining data, may not hold in practice. To investigate whether DPT exhibits Bayesian posterior sampling behavior in practice, this section establishes two necessary conditions that any Bayesian posterior sampler must satisfy and develops a statistical testing framework to quantitatively assess if DPT and its variants satisfies these conditions. All proofs in this section are provided in Appendix~\ref{sec:proposition_proofs}. We begin by characterizing the transformer as a Bayesian posterior sampler, motivated by its role as the common decision-making component across most ICRL algorithms, including DPT.

\subsection{Transformer as a Bayesian Posterior Sampler}
We first formalize what it means for the transformer to behave as a Bayesian posterior sampler.
\begin{definition}[Transformer as a Bayesian Posterior Sampler]
\label{def:transformer_bayesian_posterior}
Consider a decision-making problem with some unknown latent parameter $\theta\in\Theta$. Let $D_n=\big((A_1,R_1),\ldots,(A_n,R_n)\big)$ denote a context, and let
$\pi_{\mathrm{TF}}(a\mid D_n)$ represent the
transformer's action distribution given the context $D_n$. We say that the transformer behaves as a Bayesian posterior sampler if there exist a prior distribution $\pi_0(\theta)$ and a context-independent decision kernel $\mathcal{K}(a\mid \theta)$ such that
\begin{equation}
\pi_{\mathrm{TF}}(a\mid D_n)
=
\int_{\Theta}
\mathcal{K}(a\mid \theta) \,
p_{\pi_0}(\theta\mid D_n)
\,d\theta,
\label{eq:transformer_environment_posterior}
\end{equation}
where $p_{\pi_0}(\theta\mid D_n)$ is the Bayesian posterior induced by the prior $\pi_0(\theta)$ and the likelihood model $p(\cdot\mid \theta)$, namely $p_{\pi_0}(\theta\mid D_n)
\propto
\pi_0(\theta)\,
p(D_n\mid \theta)$.
\end{definition}
Note that Definition~\ref{def:transformer_bayesian_posterior} describes a generalized form of Bayesian posterior samplers. It reduces to Thompson sampler \citep{thompson1933likelihood} when the decision kernel $\mathcal{K}(a\mid\theta)$ is the Dirac point mass at the optimal action under the sampled $\theta$, and it represents Bayesian imitation when $\mathcal{K}(a\mid\theta)$ is a behavior policy associated with the sampled $\theta$, such as the source RL algorithm's policy at the corresponding training stage in AD \citep{laskin2022context}. Moreover, \eqref{eq:transformer_environment_posterior} of Definition~\ref{def:transformer_bayesian_posterior} encompasses the action predictions of arbitrary Bayesian posterior samplers, going beyond the prediction in the pretraining dataset required by \citep{lee2024supervised}. In addition,
Definition~\ref{def:transformer_bayesian_posterior} provides a purely behavioral characterization in the sense that it only requires the transformer policy to ``behave'' as a Bayesian posterior sampler in terms of the action distribution. In particular, the definition does not require the transformer to explicitly infer a latent parameter $\theta$ or maintain an explicit posterior distribution over $\theta$. Building on Definition~\ref{def:transformer_bayesian_posterior}, we establish the two necessary conditions that any Bayesian posterior sampler must satisfy: sufficient-statistic invariance and the martingale property.
\subsubsection{Sufficient-Statistic Invariance}
Let $\mathcal{X}=\{x_1,\ldots,x_K\}$ be the known action features and $\theta\in\Theta$ be the unknown parameter for the decision-making problem in Definition~\ref{def:transformer_bayesian_posterior}. Recall that the finite-dimensional exponential-family reward likelihoods take the form
\begin{equation}
p(r\mid x,\theta)
=
h(r,x)
\exp\left\{
\eta(\theta)^\top T(r,x)
-
\lambda(\theta)^\top \phi(x)
\right\},
\label{eq:general_exponential_family}
\end{equation}
where $h(r,x)$ is independent of $\theta$, $T(r,x)\in\mathbb{R}^{q}$ and $\phi(x)\in\mathbb{R}^{p}$ are fixed functions, and $\eta(\theta)$ and $\lambda(\theta)$ are parameter-dependent vectors. For a context $D_n=\big((A_1,R_1),\ldots,(A_n,R_n)\big)$, define
\[
S_n^{(1)}=\sum_{t=1}^n T(R_t,x_{A_t}),
\quad
S_n^{(2)}=\sum_{t=1}^n \phi(x_{A_t}),
\quad \text{and} \quad
S(D_n)=\left(S_n^{(1)},S_n^{(2)}\right).
\]
A statistic $S(D_n)$ is sufficient for $\theta$ if the distribution of the context $D_n$ is independent of $\theta$ conditional on $S(D_n)$. Hence, the posterior depends on the context only through the sufficient statistic. 
\begin{proposition}[Sufficient-Statistic Invariance]
\label{prop:posterior_sufficient_statistics}
Consider the transformer in Definition~\ref{def:transformer_bayesian_posterior} with known action features $\mathcal{X}$. At each time $t$, an action $A_t$ is selected and a reward $R_t$ is observed. $i)$ If the transformer's latent reward likelihood is context-independent, namely $p(R_t=r\mid D_{t-1},x_{A_t},\theta)=p(R_t=r\mid x_{A_t}, \theta)$, then its action distribution remains invariant to any permutation of the context,
\begin{equation}\label{eq:transformer_sufficient_statistic_invariance_permute}
\piTF(a\mid D_n)=\piTF(a\mid \sigma_1(D_n)),
~ \forall a\in\mathcal{A},\forall n\in\mathbb{N},
\end{equation}
where $\sigma_1(\cdot)$ denotes a permutation transformation. $ii)$ If, in addition, the transformer's latent reward likelihood takes the exponential-family form \eqref{eq:general_exponential_family}, then its action distribution depends on the context only through $S(D_n)$. Accordingly,
\begin{equation}
\piTF(a\mid D_n)=\piTF(a\mid \sigma_2(D_n)),
~ \forall a\in\mathcal{A},\forall n\in\mathbb{N},
\label{eq:transformer_sufficient_statistic_invariance}
\end{equation}
where $\sigma_2(\cdot)$ denotes a transformation of the contexts that preserves the sufficient statistic.
\end{proposition}
Note that permutation transformations preserve the sufficient statistic, so \eqref{eq:transformer_sufficient_statistic_invariance_permute} is a special case of \eqref{eq:transformer_sufficient_statistic_invariance}. Proposition~\ref{prop:posterior_sufficient_statistics} implies that if a transformer behaves as a Bayesian posterior sampler as in Definition~\ref{def:transformer_bayesian_posterior}, its action distribution given any context $D$ should coincide with that given another context $D'$ whenever $D$ and $D'$ share the same sufficient statistic. 
We now turn to the martingale property.
\subsubsection{Martingale Property}
\begin{proposition}[Martingale Property]
\label{prop:martingale}
Consider the decision-making problem and the transformer in
Definition~\ref{def:transformer_bayesian_posterior}. Given a context $D_n=\big((A_1,R_1),\ldots,(A_n,R_n)\big)$ and the sigma-algebra $\mathcal{F}_n$ generated by $D_n$, let $P_{\mathrm{TF}}^{(m)}( \cdot \mid D_n)$ denote the distribution of length-$m$ continuations generated autoregressively by the transformer's own future action and reward predictions. It then follows that
\begin{equation}
\mathbb E_{D_{n+1 : n+m} \sim  P_{\mathrm{TF}}^{(m)}( \cdot \mid D_n)}
\left[
\pi_{\mathrm{TF}}(a\mid D_{n+m})
\mid
\mathcal F_n
\right]
=
\pi_{\mathrm{TF}}(a\mid D_n),
~ \forall a\in\mathcal{A},\forall m,n\in\mathbb{N}.
\label{eq:coherent_martingale}
\end{equation}
In other words, $\{\pi_{\mathrm{TF}}(a\mid D_n)\}_{n\geq0}$ is a martingale with respect to $\{\mathcal F_n\}_{n\geq0}$.
\end{proposition}
Given a context, Proposition~\ref{prop:martingale} implies that if the transformer behaves as a Bayesian posterior sampler and its future continuations are generated according to its own Bayesian posterior-predictive distribution, then the action distribution conditioned on these future continuations coincides in expectation with the action distribution conditioned on the current given context. Notably, Proposition~\ref{prop:posterior_sufficient_statistics} relies on specific conditions about the transformer's latent Bayesian observational model, such as a context-independent and finite-dimensional exponential-family likelihood, whereas Proposition~\ref{prop:martingale} requires access to the transformer's Bayesian posterior-predictive distribution. While Proposition~\ref{prop:posterior_sufficient_statistics} can be tested for DPT provided that its underlying observational model satisfies the required conditions, the martingale property in Proposition~\ref{prop:martingale} cannot be directly tested for DPT. This is because DPT is pretrained to predict actions only and therefore cannot generate future continuations containing both actions and rewards. To address this limitation, we introduce two DPT variants in the next subsection that satisfy the conditions of Proposition~\ref{prop:posterior_sufficient_statistics} by construction and can generate future continuations according to their own Bayesian posterior-predictive distributions.
\subsection{DPT Variants}
\label{sec:dpt_variants}
The two proposed DPT variants, DPT$'$ and GDPT, retain DPT's transformer backbone and incorporate an additional reward head to jointly predict actions \emph{and} their associated rewards. DPT$'$ and GDPT differ only in their supervision labels, as described below.

\textbf{DPT$'$.} This variant is identical to DPT except that it extends DPT's optimal action prediction to a joint prediction of the optimal action and its associated reward. The training objective for DPT$'$ is the cross-entropy loss for both the action head and the reward head.

\textbf{GDPT.} This variant employs the same architecture as DPT$'$. However, rather than predicting the optimal action and its associated reward as in DPT$'$, GDPT is trained to autoregressively predict the next action-reward pair in the context. Consequently, GDPT does not require optimal actions and rewards as supervision labels. 
This is similar to the training procedure of AD \citep{laskin2022context}, which, however, lacks a reward head and predicts only actions.
On the other hand, GDPT requires expert contexts for training (same as AD), such as Thompson sampling trajectories, whereas DPT and DPT$'$ can be trained on arbitrary contexts, including both random and expert contexts~\citep{lee2024supervised}. This distinction arises from the availability of optimal supervision labels for DPT and DPT$'$.

It is worth noting that both DPT$'$ and GDPT satisfy the conditions of Proposition~\ref{prop:posterior_sufficient_statistics} by construction. Specifically, both variants employ Gaussian reward heads whose structures mirror those of the action heads, which are supposed to be context-independent under Definition~\ref{def:transformer_bayesian_posterior}. Moreover, given a context, both DPT$'$ and GDPT can generate future continuations according to their respective posterior-predictive distributions, as required by Proposition~\ref{prop:martingale}.
\subsection{Statistical Hypothesis Testing Framework}
\label{sec_testing_framework}
Having established the two necessary conditions for a Bayesian posterior sampler (Propositions~\ref{prop:posterior_sufficient_statistics} and \ref{prop:martingale}) and introduced the transformer-based ICRL schemes including DPT, DPT$'$, and GDPT, this subsection develops a statistical hypothesis testing framework to assess whether these ICRL schemes satisfy the two conditions in practice.
\subsubsection{Sufficient Statistic}
Under the conditions of Proposition~\ref{prop:posterior_sufficient_statistics}, a Bayesian posterior sampler should produce identical action distributions for contexts with the same sufficient statistics. Fix a context length $n$. Let $\mathcal{T}$ denote the distribution over environments and let $P_n(\cdot\mid\tau)$ denote the distribution of length-$n$ contexts collected in environment $\tau$. We sample $\tau\sim\mathcal{T}$ and $D_n\sim P_n(\cdot\mid\tau)$, and independently draw a transformation $\sigma\sim\Sigma$, where $\Sigma$ is supported on transformations satisfying $S(\sigma(D_n))=S(D_n)$ (including $\sigma_1(\cdot)$ and $\sigma_2(\cdot)$ in Proposition~\ref{prop:posterior_sufficient_statistics}). For each action $a\in\mathcal{A}$, define the paired difference
\begin{equation}
X_a
:=
\pi_{\mathrm{TF}}(a\mid D_n)
-
\pi_{\mathrm{TF}}(a\mid \sigma(D_n)).
\end{equation}
Sufficient-statistic invariance requires $X_a=0$ for every action. However, testing the exact equality may yield a trivial rejection whenever there exists a nonzero discrepancy, regardless of its magnitude. To avoid the effects of noise and assess systematic deviations, we instead test a relaxed null hypothesis
\begin{equation}
\label{eq:statistics_test}
H_0:
\mathbb{E}_{\substack{
\tau\sim\mathcal{T},\;
D_n\sim P_n(\cdot\mid\tau), \;\sigma\sim\Sigma
}}
\left[X_a\right]
=0,
~
\forall a\in\mathcal{A}, \forall n\in\mathbb{N}.
\end{equation}
Importantly, we must choose $\Sigma$ such that $D$ and $\sigma(D)$ induce distinct distributions.  Otherwise, the expectation in \eqref{eq:statistics_test} trivially equals to $0$ by symmetry. For details on how we guarantee this, see Appendix~\ref{sec:context_transformations}.
\subsubsection{Martingale Property}
Given a fixed starting context $D_n$ and a continuation length $m\geq0$ in DPT$'$ and GDPT, let $P_{\mathrm{TF}}^{(m)}(\cdot\mid D_n)$ denote the distribution of length-$m$ continuations generated autoregressively by their transformer's own future action and reward predictions. For each action $a\in\mathcal{A}$, define
\begin{equation}
Y_a
:=
\pi_{\mathrm{TF}}(a\mid D_{n+m})
-
\pi_{\mathrm{TF}}(a\mid D_n),
\end{equation}
where the randomness in $Y_a$ arises solely from the generated continuations $D_{n+1:n+m}$. Following Proposition~\ref{prop:martingale}, the null hypothesis is 
\begin{equation}
\label{eq:martingale_test}
H_0:
\mathbb{E}_{
D_{n+1:n+m}\sim
P_{\mathrm{TF}}^{(m)}(\cdot\mid D_n)
}
\left[Y_a\right]
=0,
~
\forall a\in\mathcal{A}, \forall n\in\mathbb{N}.
\end{equation}
We generate independent continuations from the same $D_n$ to obtain conditionally i.i.d.\ increment vectors. Since $\sum_{a\in\mathcal{A}}Y_a=0$, we apply the paired Hotelling test as described in Appendix~\ref{sec:hotelling}. 
\section{Experiments}
\label{sec:case_studies}
This section presents experimental results for the two statistical tests developed in Section~\ref{sec_testing_framework}. We first describe the shared experimental setup, then present results for DPT, DPT$'$, and GDPT in Gaussian, Bernoulli, and linear bandits settings. Following these, we investigate the effect of the transformer's architecture through an ablation study.
\subsection{Experimental Setup}
\label{sec:exp_gauss}
\textbf{Environments (Tasks).}
We consider three bandit problems: Gaussian, Bernoulli, and linear bandits. In Gaussian bandits, rewards satisfy $r \mid a \sim \mathcal{N}(\mu_a, \sigma^2)$ with unknown arm means $\{\mu_a\}$ and known variance $\sigma^2$. In Bernoulli bandits, rewards satisfy $r \mid a \sim \operatorname{Bernoulli}(p_a)$ with unknown success probabilities $\{p_a\}$. In linear bandits, each arm $a$ is associated with a fixed, known feature vector $x_a \in \mathbb{R}^d$, and rewards satisfy $r \mid a \sim \mathcal{N}(x_a^\top \theta, \sigma^2)$ with an unknown shared parameter $\theta \in \mathbb{R}^d$ and known variance $\sigma^2$. Gaussian and Bernoulli bandits can be viewed as special cases of linear bandits.

\textbf{Pretraining and Deployment.}
The dataset collection, pretraining, and deployment procedures for DPT, DPT$'$, and GDPT follow Section~\ref{sec:dpt}. For online deployment, the contexts are collected by the model itself. For offline deployment, the contexts are pre-specified by either a Dirichlet-mixture random policy or Thompson Sampling depending on the bandit setting, following ~\citep{lee2024supervised}.

\textbf{Tests and Metrics.}
Under Proposition~\ref{prop:posterior_sufficient_statistics}, a Bayesian posterior sampler with the assumed reward likelihood must be invariant to transformations preserving the sufficient statistic. For a context $D_n$, let $N_a$ and $S_a$ be the number of pulls and the total rewards of arm $a$. The sufficient statistic is $\{(N_a,S_a)\}_{a=1}^K$ for Gaussian and Bernoulli bandits, and $\big(V=\sum_{a}N_a x_a x_a^{\top}, b=\sum_{a}x_a S_a\big)$ for linear bandits (details in Appendix~\ref{sec:bandit_sufficient_statistics_proofs}). We test this invariance with two transformations: \textit{sorting} reorders the action-reward pairs in ascending order of reward, \textit{reward replacement} replaces each arm's rewards with its empirical mean plus centered uniform noise. Since Bernoulli rewards only take values in $\{0,1\}$, we only use \textit{sorting}. Each test evaluates the hypothesis in Eq.~\eqref{eq:statistics_test} on $N=5000$ independently sampled environments and their contexts. For the martingale test, we fix one starting context $D_n$, generate independent continuations of length $m$ using the model's own future action and reward predictions, and test Eq.~\eqref{eq:martingale_test}. Both tests are carried out with the paired Hotelling tests with significance level $\alpha=0.05$ (see Appendix~\ref{sec:hotelling}). In addition, we compare cumulative regret, offline suboptimality, and the total variation (TV) distance of action distributions against Thompson Sampling, a canonical posterior sampling method.

\textbf{Hyperparameters.}
Appendix~\ref{sec:experimental_hyperparameters} provides a complete list of the hyperparameters used in our experiments, including those specifying the transformer's architecture and the experimental setup.

\subsection{Experimental Results}
\label{sec:case_dpt}
\begin{table}[htbp]
\centering
\small
\setlength{\tabcolsep}{6pt}
\renewcommand{\arraystretch}{1.08}
\caption{$p$-values for sufficient-statistic invariance tests on DPT in Gaussian, Bernoulli, and linear bandits ($N=5000$, $\alpha=0.05$). Here, $n$ denotes the context length.}
\label{tab:ss_summary_dpt}
%
\sbox{\banditupperbox}{
\begin{tabular}{@{}r rrrr rr@{}}
\toprule
& \multicolumn{4}{c}{Gaussian bandits}
& \multicolumn{2}{c}{Bernoulli bandits} \\
\cmidrule(lr){2-5}\cmidrule(lr){6-7}
& \multicolumn{2}{c}{Offline}
& \multicolumn{2}{c}{Online}
& \multicolumn{1}{c}{Offline}
& \multicolumn{1}{c}{Online} \\
$n$
& \multicolumn{1}{c}{Sort}
& \multicolumn{1}{c}{Replace}
& \multicolumn{1}{c}{Sort}
& \multicolumn{1}{c}{Replace}
& \multicolumn{1}{c}{Sort}
& \multicolumn{1}{c}{Sort} \\
\midrule
$50$  & $2.88\times 10^{-112}$ & $6.95\times 10^{-230}$ & $2.44\times 10^{-259}$ & $<10^{-300}$ & $1.41\times 10^{-141}$ & $4.79\times 10^{-80}$ \\
$100$ & $5.68\times 10^{-108}$ & $<10^{-300}$ & $5.22\times 10^{-265}$ & $<10^{-300}$ & $2.23\times 10^{-158}$ & $5.83\times 10^{-133}$ \\
$200$ & $5.62\times 10^{-178}$ & $<10^{-300}$ & $1.44\times 10^{-282}$ & $<10^{-300}$ & $5.43\times 10^{-184}$ & $4.36\times 10^{-161}$ \\
$300$ & $3.10\times 10^{-155}$ & $<10^{-300}$ & $1.02\times 10^{-256}$ & $<10^{-300}$ & $9.22\times 10^{-196}$ & $1.63\times 10^{-167}$ \\
$400$ & $1.25\times 10^{-148}$ & $<10^{-300}$ & $1.93\times 10^{-229}$ & $<10^{-300}$ & $3.17\times 10^{-209}$ & $4.13\times 10^{-177}$ \\
$500$ & $8.53\times 10^{-154}$ & $<10^{-300}$ & $5.96\times 10^{-224}$ & $<10^{-300}$ & $2.63\times 10^{-206}$ & $1.33\times 10^{-187}$ \\
\midrule
\end{tabular}%
}
%
\resizebox{\linewidth}{!}{%
\begin{minipage}{\wd\banditupperbox}
\noindent\usebox{\banditupperbox}\par\nointerlineskip
\noindent
\begin{tabular*}{\linewidth}{@{\extracolsep{\fill}}r rrrr@{}}
& \multicolumn{4}{c}{Linear bandits} \\
\cmidrule(lr){2-5}
& \multicolumn{2}{c}{Offline}
& \multicolumn{2}{c}{Online} \\
$n$
& \multicolumn{1}{c}{Sort}
& \multicolumn{1}{c}{Replace}
& \multicolumn{1}{c}{Sort}
& \multicolumn{1}{c}{Replace} \\
\midrule
$25$  & $1.56\times 10^{-18}$ & $7.60\times 10^{-7}$ & $4.08\times 10^{-22}$ & $2.31\times 10^{-4}$ \\
$50$  & $1.76\times 10^{-49}$ & $1.65\times 10^{-12}$ & $2.95\times 10^{-59}$ & $1.31\times 10^{-8}$ \\
$100$ & $4.88\times 10^{-76}$ & $1.21\times 10^{-21}$ & $1.05\times 10^{-87}$ & $4.79\times 10^{-26}$ \\
$150$ & $1.81\times 10^{-67}$ & $1.74\times 10^{-42}$ & $8.36\times 10^{-69}$ & $7.27\times 10^{-33}$ \\
$200$ & $6.43\times 10^{-93}$ & $6.38\times 10^{-45}$ & $2.33\times 10^{-82}$ & $6.94\times 10^{-40}$ \\
\bottomrule
\end{tabular*}
\end{minipage}%
}
\end{table}

\textbf{DPT.} Table~\ref{tab:ss_summary_dpt} reports the $p$-values of the sufficient-statistic tests for DPT on Gaussian, Bernoulli, and linear bandits, under both offline and online deployments and both the ``Sort'' and ``Replace'' transformations. Since the sufficient-statistic invariance should hold for arbitrary context lengths, we consider context lengths $n \in \{50, 100, 200, 300, 400, 500\}$ for Gaussian and Bernoulli bandits and $n \in \{25, 50, 100, 150, 200\}$ for linear bandits. Notably, all tests in Table~\ref{tab:ss_summary_dpt} reject the null hypothesis in \eqref{eq:statistics_test} at $\alpha = 0.05$. This provides evidence that DPT does not satisfy Proposition~\ref{prop:posterior_sufficient_statistics} and hence cannot be characterized as a Bayesian posterior sampler under Definition~\ref{def:transformer_bayesian_posterior}, provided that the latent reward likelihood of the transformer is context-independent and belongs to the exponential family. We further compare DPT with Thompson Sampling in Figures~\ref{fig_evaluation_gaussian_bandit}, \ref{fig_evaluation_bernoulli_bandit}, and~\ref{fig_evaluation_linear_bandit} for Gaussian, Bernoulli, and linear bandits, respectively.

\begin{table}[htbp]
\centering
\small
\setlength{\tabcolsep}{3pt}
\renewcommand{\arraystretch}{1.15}
\setlength{\bigstrutjot}{\dimexpr(5.75\baselineskip+0.6ex)/20\relax}
\caption{$p$-values for martingale tests on DPT$'$ in Gaussian, Bernoulli, and linear bandits ($N=5000$, $\alpha=0.05$). Here, $n$ and $m$ denote the starting-context and continuation lengths, respectively.}
\label{tab:martingale_dpt_prime_grouped}
\resizebox{\linewidth}{!}{
\begin{tabular}[t]{@{}c r rrrr@{}}
\toprule
& & \multicolumn{2}{c}{Gaussian} & \multicolumn{2}{c}{Bernoulli} \\
\cmidrule(lr){3-4}\cmidrule(lr){5-6}
$n$ & $m$
& \multicolumn{1}{c}{Offline} & \multicolumn{1}{c}{Online}
& \multicolumn{1}{c}{Offline} & \multicolumn{1}{c}{Online} \\
\midrule
\multirow[c]{5}{*}{$50$} & $100$ & $<10^{-300}$ & $<10^{-300}$ & $<10^{-300}$ & $<10^{-300}$ \\
 & $200$ & $<10^{-300}$ & $<10^{-300}$ & $<10^{-300}$ & $<10^{-300}$ \\
 & $300$ & $<10^{-300}$ & $<10^{-300}$ & $<10^{-300}$ & $<10^{-300}$ \\
 & $400$ & $<10^{-300}$ & $<10^{-300}$ & $<10^{-300}$ & $<10^{-300}$ \\
 & $450$ & $<10^{-300}$ & $<10^{-300}$ & $<10^{-300}$ & $<10^{-300}$ \\
\addlinespace[0.6ex]
\multirow[c]{4}{*}{$100$} & $100$ & $<10^{-300}$ & $<10^{-300}$ & $<10^{-300}$ & $<10^{-300}$ \\
 & $200$ & $<10^{-300}$ & $<10^{-300}$ & $<10^{-300}$ & $<10^{-300}$ \\
 & $300$ & $<10^{-300}$ & $<10^{-300}$ & $<10^{-300}$ & $<10^{-300}$ \\
 & $400$ & $<10^{-300}$ & $<10^{-300}$ & $<10^{-300}$ & $<10^{-300}$ \\
\addlinespace[0.6ex]
\multirow[c]{3}{*}{$200$} & $100$ & $<10^{-300}$ & $1.45\times 10^{-55}$ & $<10^{-300}$ & $1.51\times 10^{-157}$ \\
 & $200$ & $<10^{-300}$ & $1.91\times 10^{-82}$ & $<10^{-300}$ & $1.29\times 10^{-130}$ \\
 & $300$ & $<10^{-300}$ & $4.09\times 10^{-119}$ & $<10^{-300}$ & $1.46\times 10^{-113}$ \\
\addlinespace[0.6ex]
\multirow[c]{2}{*}{$300$} & $100$ & $<10^{-300}$ & $1.83\times 10^{-204}$ & $<10^{-300}$ & $1.24\times 10^{-91}$ \\
 & $200$ & $<10^{-300}$ & $3.75\times 10^{-12}$ & $<10^{-300}$ & $1.12\times 10^{-76}$ \\
\addlinespace[0.6ex]
$400$ & $100$ & $<10^{-300}$ & $1.67\times 10^{-250}$ & $<10^{-300}$ & $<10^{-300}$ \\
\bottomrule
\end{tabular}%
\hspace{1.5em}%
\begin{tabular}[t]{@{}c r rr@{}}
\toprule
& & \multicolumn{2}{c}{Linear bandits} \\
\cmidrule(lr){3-4}
$n$ & $m$
& \multicolumn{1}{c}{Offline} & \multicolumn{1}{c}{Online} \\
\midrule
\multirow[c]{4}[tb8]{*}{$25$} & $50$ & $<10^{-300}$ & $<10^{-300}$ \bigstrut \\
 & $100$ & $<10^{-300}$ & $<10^{-300}$ \bigstrut \\
 & $150$ & $<10^{-300}$ & $<10^{-300}$ \bigstrut \\
 & $175$ & $<10^{-300}$ & $<10^{-300}$ \bigstrut \\
\addlinespace[0.6ex]
\multirow[c]{3}[tb6]{*}{$50$} & $50$ & $<10^{-300}$ & $<10^{-300}$ \bigstrut \\
 & $100$ & $<10^{-300}$ & $<10^{-300}$ \bigstrut \\
 & $150$ & $<10^{-300}$ & $<10^{-300}$ \bigstrut \\
\addlinespace[0.6ex]
\multirow[c]{2}[tb4]{*}{$100$} & $50$ & $<10^{-300}$ & $<10^{-300}$ \bigstrut \\
 & $100$ & $<10^{-300}$ & $<10^{-300}$ \bigstrut \\
\addlinespace[0.6ex]
$150$ & $50$ & $<10^{-300}$ & $<10^{-300}$ \bigstrut \\
\bottomrule
\end{tabular}%
}
\end{table}

\textbf{DPT$'$.} As described in Section~\ref{sec:dpt_variants}, DPT$'$ satisfies the conditions of Proposition~\ref{prop:posterior_sufficient_statistics} by construction. In addition, it can generate future continuations from its own action--reward predictions, so both tests apply. For space reasons, we present the martingale tests in Table~\ref{tab:martingale_dpt_prime_grouped} and defer the sufficient-statistic tests to Table~\ref{tab:ss_summary_dpt_prime} in Appendix~\ref{sec:dpt_prime_tables_and_plots}. The sufficient-statistic tests follow the same protocol as for DPT, and all of them reject the null hypothesis in \eqref{eq:statistics_test} at $\alpha=0.05$. For the martingale tests, we fix one starting context $D_n$ for each bandit and deployment mode, and generate $N = 5000$ independent length-$m$ continuations from DPT$'$. Since the martingale property in \eqref{eq:coherent_martingale} should hold for arbitrary $n$ and $m$, we consider starting-context lengths $n \in \{50, 100, 200, 300, 400\}$ with continuation lengths $m$ satisfying $n + m \leq 500$ for Gaussian and Bernoulli bandits, and $n \in \{25, 50, 100, 150\}$ with $n + m \leq 200$ for linear bandits. All martingale tests in Table~\ref{tab:martingale_dpt_prime_grouped} reject the null hypothesis in \eqref{eq:martingale_test} at $\alpha = 0.05$. In other words, the current action probabilities of DPT$'$ differ systematically from their expected values after model-generated continuations, in violation of Proposition~\ref{prop:martingale}. Per-arm martingale increment curves are provided in Figures~\ref{fig_martingale_increment_offline_gaussian_DPT_prime}, \ref{fig_martingale_increment_online_gaussian_DPT_prime}, \ref{fig_martingale_increment_offline_bernoulli_DPT_prime}, \ref{fig_martingale_increment_online_bernoulli_DPT_prime}, \ref{fig_martingale_increment_offline_linear_DPT_prime}, and~\ref{fig_martingale_increment_online_linear_DPT_prime}. Performance comparisons with Thompson Sampling are presented in Figures~\ref{fig_evaluation_gaussian_bandit_DPT_prime}, \ref{fig_evaluation_bernoulli_bandit_DPT_prime}, and~\ref{fig_evaluation_linear_bandit_DPT_prime}. In Gaussian bandits, offline DPT$'$ attains suboptimality close to that of Thompson Sampling while its TV distance remains positive.

\textbf{GDPT.} GDPT, likewise, satisfies the conditions of Proposition~\ref{prop:posterior_sufficient_statistics} by construction, and both tests follow the same protocol as for DPT$'$. All sufficient-statistic tests in Table~\ref{tab:ss_summary_gdpt} in Appendix~\ref{sec:gdpt_tables_and_plots} reject the null hypothesis in \eqref{eq:statistics_test} at $\alpha=0.05$. Hence, GDPT also cannot be characterized as a Bayesian posterior sampler under Definition~\ref{def:transformer_bayesian_posterior}. Table~\ref{tab:martingale_gdpt_grouped} reports the martingale tests. All Gaussian and linear tests, for both offline and online starting contexts, and all Bernoulli tests with offline starting contexts reject at $\alpha = 0.05$. For Bernoulli bandits with online starting contexts, $13$ of the $15$ tests reject the null hypothesis, with two exceptions $(n, m) = (50, 100)$ and $(100, 400)$ at $p$-values of $0.099$ and $0.094$, respectively. Thus, the null hypothesis in \eqref{eq:martingale_test} is rejected in all but two of the evaluated settings, indicating systematic violations of Proposition~\ref{prop:martingale} at the tested contexts. Per-arm martingale increment curves are provided in Figures~\ref{fig_martingale_increment_offline_gaussian_GDPT}, \ref{fig_martingale_increment_online_gaussian_GDPT}, \ref{fig_martingale_increment_offline_bernoulli_GDPT}, \ref{fig_martingale_increment_online_bernoulli_GDPT}, \ref{fig_martingale_increment_offline_linear_GDPT}, and~\ref{fig_martingale_increment_online_linear_GDPT}. Performance comparisons with Thompson Sampling are presented in Figures~\ref{fig_evaluation_gaussian_bandit_GDPT}, \ref{fig_evaluation_bernoulli_bandit_GDPT}, and~\ref{fig_evaluation_linear_bandit_GDPT}. In Gaussian bandits, the cumulative regret and offline suboptimality of GDPT closely follow those of Thompson Sampling, whereas its offline TV distance remains positive. Together with the results for DPT$'$, this demonstrates that decision performance comparable to Thompson Sampling can coexist with violations of the necessary conditions for Bayesian posterior sampling, and therefore does not by itself imply posterior-sampling behavior.

\textbf{Ablation Study.}
%
%
%
The ablation study investigates whether the observed violations result primarily from the transformer's sensitivity to the order of the context. To this end, we modify the DPT architecture to be permutation-invariant by removing the positional encoding, replacing causal attention with bi-directional attention, and applying mean pooling at the output layer. This modification applies to DPT and DPT$'$, whose supervision labels are the optimal action (reward) given the entire context, but not to GDPT, whose autoregressive objective requires causal attention to prevent each prediction from accessing its own target. We therefore evaluate DPT and DPT$'$ with the modified architecture in the Gaussian bandit under the same pretraining, deployment, and testing procedures as before. The sufficient-statistic tests are reported in Tables~\ref{tab:ss_ablation_dpt} and~\ref{tab:ss_ablation_dpt_prime}, and the martingale tests for DPT$'$ are in Table~\ref{tab:martingale_ablation_dpt_prime}. For both models, the \textit{sort} transformation now yields identical action-probability vectors at every evaluated context length, enforced by the modified architecture. These entries are marked as `Equal'' with no $p$-value reported. However, all \textit{reward replacement} tests and all martingale tests still reject the null hypotheses. Thus, enforcing permutation invariance restores sorting invariance by construction but still fails the sufficient-statistic invariance and the martingale property, indicating that the sensitivity to the context order is not the main reason for the observed violations. Performance comparisons for the ablation studies are provided in Figures~\ref{fig_evaluation_gaussian_bandit_DPT_ablation} and~\ref{fig_evaluation_gaussian_bandit_DPT_prime_ablation}. Per-arm martingale increment curves for ablated DPT$'$ are provided in Figures~\ref{fig_martingale_increment_offline_gaussian_DPT_prime_ablation} and~\ref{fig_martingale_increment_online_gaussian_DPT_prime_ablation}.

\begin{table}[htbp]
\centering
\small
\setlength{\tabcolsep}{3pt}
\renewcommand{\arraystretch}{1.15}
\setlength{\bigstrutjot}{\dimexpr(5.75\baselineskip+0.6ex)/20\relax}
\caption{$p$-values for martingale tests on GDPT in Gaussian, Bernoulli, and linear bandits ($N=5000$, $\alpha=0.05$). Here, $n$ and $m$ denote the starting-context and continuation lengths, respectively.}
\label{tab:martingale_gdpt_grouped}
\resizebox{\linewidth}{!}{%
\begin{tabular}[t]{@{}c r rrrr@{}}
\toprule
& & \multicolumn{2}{c}{Gaussian} & \multicolumn{2}{c}{Bernoulli} \\
\cmidrule(lr){3-4}\cmidrule(lr){5-6}
$n$ & $m$
& \multicolumn{1}{c}{Offline} & \multicolumn{1}{c}{Online}
& \multicolumn{1}{c}{Offline} & \multicolumn{1}{c}{Online} \\
\midrule
\multirow[c]{5}{*}{$50$} & $100$ & $6.85\times 10^{-110}$ & $2.89\times 10^{-248}$ & $2.02\times 10^{-5}$ & $0.099$ \\
 & $200$ & $1.33\times 10^{-192}$ & $<10^{-300}$ & $0.004$ & $0.014$ \\
 & $300$ & $1.20\times 10^{-61}$ & $2.04\times 10^{-157}$ & $3.09\times 10^{-4}$ & $0.028$ \\
 & $400$ & $1.90\times 10^{-53}$ & $9.77\times 10^{-121}$ & $0.002$ & $0.042$ \\
 & $450$ & $5.02\times 10^{-57}$ & $8.08\times 10^{-246}$ & $0.001$ & $0.013$ \\
\addlinespace[0.6ex]
\multirow[c]{4}{*}{$100$} & $100$ & $5.86\times 10^{-177}$ & $<10^{-300}$ & $8.13\times 10^{-11}$ & $3.01\times 10^{-26}$ \\
 & $200$ & $1.09\times 10^{-20}$ & $<10^{-300}$ & $2.15\times 10^{-11}$ & $0.021$ \\
 & $300$ & $5.85\times 10^{-78}$ & $2.39\times 10^{-30}$ & $7.22\times 10^{-6}$ & $4.49\times 10^{-4}$ \\
 & $400$ & $2.89\times 10^{-79}$ & $1.08\times 10^{-94}$ & $3.91\times 10^{-10}$ & $0.094$ \\
\addlinespace[0.6ex]
\multirow[c]{3}{*}{$200$} & $100$ & $<10^{-300}$ & $<10^{-300}$ & $<10^{-300}$ & $3.08\times 10^{-34}$ \\
 & $200$ & $2.68\times 10^{-10}$ & $<10^{-300}$ & $<10^{-300}$ & $1.27\times 10^{-79}$ \\
 & $300$ & $2.56\times 10^{-46}$ & $<10^{-300}$ & $<10^{-300}$ & $4.92\times 10^{-24}$ \\
\addlinespace[0.6ex]
\multirow[c]{2}{*}{$300$} & $100$ & $<10^{-300}$ & $<10^{-300}$ & $<10^{-300}$ & $<10^{-300}$ \\
 & $200$ & $<10^{-300}$ & $<10^{-300}$ & $<10^{-300}$ & $2.02\times 10^{-106}$ \\
\addlinespace[0.6ex]
$400$ & $100$ & $<10^{-300}$ & $<10^{-300}$ & $<10^{-300}$ & $<10^{-300}$ \\
\bottomrule
\end{tabular}%
\hspace{1.5em}%
\begin{tabular}[t]{@{}c r rr@{}}
\toprule
& & \multicolumn{2}{c}{Linear bandits} \\
\cmidrule(lr){3-4}
$n$ & $m$
& \multicolumn{1}{c}{Offline} & \multicolumn{1}{c}{Online} \\
\midrule
\multirow[c]{4}[tb8]{*}{$25$} & $50$ & $3.07\times 10^{-11}$ & $4.73\times 10^{-4}$ \bigstrut \\
 & $100$ & $2.37\times 10^{-30}$ & $1.51\times 10^{-7}$ \bigstrut \\
 & $150$ & $1.79\times 10^{-15}$ & $5.65\times 10^{-11}$ \bigstrut \\
 & $175$ & $3.77\times 10^{-35}$ & $1.38\times 10^{-16}$ \bigstrut \\
\addlinespace[0.6ex]
\multirow[c]{3}[tb6]{*}{$50$} & $50$ & $1.27\times 10^{-170}$ & $<10^{-300}$ \bigstrut \\
 & $100$ & $<10^{-300}$ & $<10^{-300}$ \bigstrut \\
 & $150$ & $<10^{-300}$ & $<10^{-300}$ \bigstrut \\
\addlinespace[0.6ex]
\multirow[c]{2}[tb4]{*}{$100$} & $50$ & $<10^{-300}$ & $<10^{-300}$ \bigstrut \\
 & $100$ & $<10^{-300}$ & $<10^{-300}$ \bigstrut \\
\addlinespace[0.6ex]
$150$ & $50$ & $<10^{-300}$ & $<10^{-300}$ \bigstrut \\
\bottomrule
\end{tabular}%
}
\end{table}

\section{Conclusion}
\label{sec:conclusion}
In this paper, we investigate whether ICRL, specifically DPT and two variants DPT$'$ and GDPT, behave as Bayesian posterior samplers in practice, where the idealized assumptions underpinning its theoretical characterization may not hold. We establish two necessary conditions that any Bayesian posterior sampler must satisfy: sufficient-statistic invariance and the martingale property. Then we design statistical hypothesis tests for both conditions. Experiments on Gaussian, Bernoulli, and linear bandits show that DPT, DPT$'$, and GDPT consistently violate the sufficient-statistic invariance. Additionally, DPT$'$ and GDPT violate the martingale property under their own model-generated continuations, even when their decision performance is comparable to that of Thompson Sampling. The ablation study further exhibits that enforcing permutation invariance in the architecture restores invariance to the context order but recovers neither condition, indicating that the sensitivity to the context order is not the main source of these violations. Beyond these negative findings, our tests rely only on observable action predictions and thus provide a principled toolkit for examining Bayesian claims about ICRL. Extending the tests to decision-making problems with state dependence, and understanding what architectural and training choices are needed to satisfy the two conditions, are important directions for future work.


\clearpage

\bibliography{iclr2027_conference}

@inproceedings{falck2024bayesian,
  title={Is In-Context Learning in Large Language Models {B}ayesian? {A} Martingale Perspective},
  author={Falck, Fabian and Wang, Ziyu and Holmes, Christopher C.},
  booktitle={Proceedings of the 41st International Conference on Machine Learning},
  series={Proceedings of Machine Learning Research},
  volume={235},
  pages={12784--12805},
  publisher={PMLR},
  year={2024},
  url={https://proceedings.mlr.press/v235/falck24a.html}
}

@inproceedings{lee2024supervised,
  title={Supervised Pretraining Can Learn In-Context Reinforcement Learning},
  author={Lee, Jonathan N. and Xie, Annie and Pacchiano, Aldo and Chandak, Yash and Finn, Chelsea and Nachum, Ofir and Brunskill, Emma},
  booktitle={Advances in Neural Information Processing Systems},
  volume={36},
  year={2023},
  url={https://proceedings.neurips.cc/paper_files/paper/2023/hash/8644b61a9bc87bf7844750a015feb600-Abstract-Conference.html}
}

@inproceedings{lin2024transformers,
  title={Transformers as Decision Makers: Provable In-Context Reinforcement Learning via Supervised Pretraining},
  author={Lin, Licong and Bai, Yu and Mei, Song},
  booktitle={International Conference on Learning Representations},
  year={2024},
  url={https://arxiv.org/abs/2310.08566}
}

@article{wakayama2025icl,
  title={In-Context Learning Is Provably {B}ayesian Inference: A Generalization Theory for Meta-Learning},
  author={Wakayama, Tomoya and Suzuki, Taiji},
  journal={arXiv preprint arXiv:2510.10981},
  year={2025},
  url={https://arxiv.org/abs/2510.10981}
}

@article{thompson1933likelihood,
  title={On the Likelihood that One Unknown Probability Exceeds Another in View of the Evidence of Two Samples},
  author={Thompson, William R.},
  journal={Biometrika},
  volume={25},
  number={3/4},
  pages={285--294},
  year={1933}
}

@article{russo2016information,
  title={An Information-Theoretic Analysis of {T}hompson Sampling},
  author={Russo, Daniel and Van Roy, Benjamin},
  journal={Journal of Machine Learning Research},
  volume={17},
  number={68},
  pages={1--30},
  year={2016}
}

@techreport{radford2019language,
  title={Language Models are Unsupervised Multitask Learners},
  author={Radford, Alec and Wu, Jeffrey and Child, Rewon and Luan, David and Amodei, Dario and Sutskever, Ilya},
  institution={OpenAI},
  year={2019},
  url={https://cdn.openai.com/better-language-models/language_models_are_unsupervised_multitask_learners.pdf}
}

@inproceedings{laskin2022context,
  title={In-Context Reinforcement Learning with Algorithm Distillation},
  author={Laskin, Michael and Wang, Luyu and Oh, Junhyuk and Parisotto, Emilio and Spencer, Stephen and Steigerwald, Richie and Strouse, DJ and Hansen, Steven and Filos, Angelos and Brooks, Ethan and Gazeau, Maxime and Sahni, Himanshu and Singh, Satinder and Mnih, Volodymyr},
  booktitle={International Conference on Learning Representations},
  year={2023},
  url={https://openreview.net/forum?id=hy0a5MMPUv}
}

@inproceedings{xie2022explanation,
  title={An Explanation of In-context Learning as Implicit {B}ayesian Inference},
  author={Xie, Sang Michael and Raghunathan, Aditi and Liang, Percy and Ma, Tengyu},
  booktitle={International Conference on Learning Representations},
  year={2022},
  url={https://arxiv.org/abs/2111.02080}
}

@article{schulman2017proximal,
  title={Proximal policy optimization algorithms},
  author={Schulman, John and Wolski, Filip and Dhariwal, Prafulla and Radford, Alec and Klimov, Oleg},
  journal={arXiv preprint arXiv:1707.06347},
  year={2017}
}

@inproceedings{osband2013more,
  title={(More) Efficient Reinforcement Learning via Posterior Sampling},
  author={Osband, Ian and Russo, Daniel and Van Roy, Benjamin},
  booktitle={Advances in Neural Information Processing Systems},
  volume={26},
  year={2013}
}

@article{mnih2013playing,
  title={Playing atari with deep reinforcement learning},
  author={Mnih, Volodymyr and Kavukcuoglu, Koray and Silver, David and Graves, Alex and Antonoglou, Ioannis and Wierstra, Daan and Riedmiller, Martin},
  journal={arXiv preprint arXiv:1312.5602},
  year={2013}
}

@article{mnih2015human,
  title={Human-level control through deep reinforcement learning},
  author={Mnih, Volodymyr and Kavukcuoglu, Koray and Silver, David and Rusu, Andrei A and Veness, Joel and Bellemare, Marc G and Graves, Alex and Riedmiller, Martin and Fidjeland, Andreas K and Ostrovski, Georg and others},
  journal={nature},
  volume={518},
  number={7540},
  pages={529--533},
  year={2015},
  publisher={Nature Publishing Group UK London}
}

@article{levine2016end,
  title={End-to-end training of deep visuomotor policies},
  author={Levine, Sergey and Finn, Chelsea and Darrell, Trevor and Abbeel, Pieter},
  journal={Journal of Machine Learning Research},
  volume={17},
  number={39},
  pages={1--40},
  year={2016}
}

@inproceedings{duan2016benchmarking,
  title={Benchmarking deep reinforcement learning for continuous control},
  author={Duan, Yan and Chen, Xi and Houthooft, Rein and Schulman, John and Abbeel, Pieter},
  booktitle={International conference on machine learning},
  pages={1329--1338},
  year={2016},
  organization={PMLR}
}

@article{afsar2022reinforcement,
  title={Reinforcement learning based recommender systems: A survey},
  author={Afsar, M Mehdi and Crump, Trafford and Far, Behrouz},
  journal={ACM Computing Surveys},
  volume={55},
  number={7},
  eid={145},
  pages={1--38},
  year={2023},
  doi={10.1145/3543846},
  note={38 pages}
}

@article{lin2023survey,
  title={A survey on reinforcement learning for recommender systems},
  author={Lin, Yuanguo and Liu, Yong and Lin, Fan and Zou, Lixin and Wu, Pengcheng and Zeng, Wenhua and Chen, Huanhuan and Miao, Chunyan},
  journal={IEEE Transactions on Neural Networks and Learning Systems},
  volume={35},
  number={10},
  pages={13164--13184},
  year={2024},
  publisher={IEEE},
  doi={10.1109/TNNLS.2023.3280161}
}

@inproceedings{dong2024survey,
  title={A survey on in-context learning},
  author={Dong, Qingxiu and Li, Lei and Dai, Damai and Zheng, Ce and Ma, Jingyuan and Li, Rui and Xia, Heming and Xu, Jingjing and Wu, Zhiyong and Chang, Baobao and others},
  booktitle={Proceedings of the 2024 conference on empirical methods in natural language processing},
  pages={1107--1128},
  year={2024}
}

@inproceedings{li2023transformers,
  title={Transformers as algorithms: Generalization and stability in in-context learning},
  author={Li, Yingcong and Ildiz, Muhammed Emrullah and Papailiopoulos, Dimitris and Oymak, Samet},
  booktitle={International conference on machine learning},
  pages={19565--19594},
  year={2023},
  organization={PMLR}
}

@article{wei2023larger,
  title={Larger language models do in-context learning differently},
  author={Wei, Jerry and Wei, Jason and Tay, Yi and Tran, Dustin and Webson, Albert and Lu, Yifeng and Chen, Xinyun and Liu, Hanxiao and Huang, Da and Zhou, Denny and others},
  journal={arXiv preprint arXiv:2303.03846},
  year={2023}
}

@article{chen2024random,
  title={Random policy enables in-context reinforcement learning within trust horizons},
  author={Chen, Weiqin and Paternain, Santiago},
  journal={Transactions on Machine Learning Research},
  year={2025},
  url={https://openreview.net/forum?id=mAiMKnr9r5}
}

@inproceedings{vaswani2017attention,
  title={Attention Is All You Need},
  author={Vaswani, Ashish and Shazeer, Noam and Parmar, Niki and Uszkoreit, Jakob and Jones, Llion and Gomez, Aidan N. and Kaiser, {\L}ukasz and Polosukhin, Illia},
  booktitle={Advances in Neural Information Processing Systems},
  volume={30},
  year={2017},
  url={https://proceedings.nips.cc/paper_files/paper/2017/hash/3f5ee243547dee91fbd053c1c4a845aa-Abstract.html}
}

@inproceedings{haarnoja2018soft,
  title={Soft Actor-Critic: Off-Policy Maximum Entropy Deep Reinforcement Learning with a Stochastic Actor},
  author={Haarnoja, Tuomas and Zhou, Aurick and Abbeel, Pieter and Levine, Sergey},
  booktitle={Proceedings of the 35th International Conference on Machine Learning},
  series={Proceedings of Machine Learning Research},
  volume={80},
  pages={1861--1870},
  year={2018},
  publisher={PMLR},
  url={https://proceedings.mlr.press/v80/haarnoja18b.html}
}

@misc{openaigpt35,
  title={{GPT-3.5 Turbo} Model},
  author={{OpenAI}},
  year={n.d.},
  howpublished={OpenAI API documentation},
  note={Accessed September 19, 2026},
  url={https://developers.openai.com/api/docs/models/gpt-3.5-turbo}
}

@article{openai2023gpt4,
  title={{GPT-4} Technical Report},
  author={{OpenAI}},
  journal={arXiv preprint arXiv:2303.08774},
  year={2023},
  url={https://arxiv.org/abs/2303.08774}
}

@article{touvron2023llama2,
  title={{Llama 2}: Open Foundation and Fine-Tuned Chat Models},
  author={Touvron, Hugo and Martin, Louis and Stone, Kevin and Albert, Peter and Almahairi, Amjad and Babaei, Yasmine and Bashlykov, Nikolay and Batra, Soumya and Bhargava, Prajjwal and Bhosale, Shruti and others},
  journal={arXiv preprint arXiv:2307.09288},
  year={2023},
  url={https://arxiv.org/abs/2307.09288}
}

@article{jiang2023mistral,
  title={{Mistral 7B}},
  author={Jiang, Albert Q. and Sablayrolles, Alexandre and Mensch, Arthur and Bamford, Chris and Chaplot, Devendra Singh and de las Casas, Diego and Bressand, Florian and Lengyel, Gianna and Lample, Guillaume and Saulnier, Lucile and Renard Lavaud, L{\'e}lio and Lachaux, Marie-Anne and Stock, Pierre and Le Scao, Teven and Lavril, Thibaut and Wang, Thomas and Lacroix, Timoth{\'e}e and El Sayed, William},
  journal={arXiv preprint arXiv:2310.06825},
  year={2023},
  url={https://arxiv.org/abs/2310.06825}
}

@inproceedings{chen2021decision,
  title={Decision Transformer: Reinforcement Learning via Sequence Modeling},
  author={Chen, Lili and Lu, Kevin and Rajeswaran, Aravind and Lee, Kimin and Grover, Aditya and Laskin, Misha and Abbeel, Pieter and Srinivas, Aravind and Mordatch, Igor},
  booktitle={Advances in Neural Information Processing Systems},
  volume={34},
  year={2021},
  url={https://papers.nips.cc/paper/2021/hash/7f489f642a0ddb10272b5c31057f0663-Abstract.html}
}

@inproceedings{janner2021offline,
  title={Offline Reinforcement Learning as One Big Sequence Modeling Problem},
  author={Janner, Michael and Li, Qiyang and Levine, Sergey},
  booktitle={Advances in Neural Information Processing Systems},
  volume={34},
  year={2021},
  url={https://papers.nips.cc/paper/2021/hash/099fe6b0b444c23836c4a5d07346082b-Abstract.html}
}

@book{casella2002statistical,
  title={Statistical Inference},
  author={Casella, George and Berger, Roger L.},
  edition={2},
  publisher={Duxbury},
  year={2002},
  isbn={9780534243128},
  url={https://www.cengage.com/c/statistical-inference-2e-casella/berger/9780534243128/}
}
\bibliographystyle{iclr2027_conference}
\clearpage

\appendix

\section{Related Work}
\label{sec:related}
\paragraph{In-context reinforcement learning.}
Transformers have been successfully applied to RL by formulating decision-making as a sequence modeling problem. Two pioneering works, Decision Transformer and Trajectory Transformer, take different approaches. Decision Transformer predicts actions conditioned on desired returns and past interactions~\citep{chen2021decision}, while Trajectory Transformer models trajectories autoregressively and performs planning via beam search~\citep{janner2021offline}.
In-context reinforcement learning (ICRL) takes this paradigm a step further by studying how pretrained transformers can adapt to new, unseen tasks without any parameter updates. \citet{laskin2022context} introduced Algorithm Distillation (AD), which pretrains a causal transformer on the learning histories generated by a source RL algorithm, with the goal of distilling its policy improvement operator across episodes. \citet{lee2024supervised} subsequently proposed the Decision-Pretrained Transformer (DPT), which predicts optimal actions given a query state and arbitrary contexts of varying lengths. DPT typically relies on optimal policies, such as well-trained PPO policies or human-designed policies, to provide supervision labels. A central theoretical grounding for DPT is that, under idealized assumptions such as a sufficiently expressive
model and infinite pretraining data, it behaves as a Bayesian posterior sampler. Building on this perspective, we investigate whether DPT and our proposed DPT variants satisfy necessary conditions for the Bayesian posterior sampler in practice.

\paragraph{Is in-context learning Bayesian?}
A prominent line of work examines whether in-context learning (ICL) can be characterized as Bayesian inference. \citet{xie2022explanation} model pretraining data as a mixture of hidden Markov models indexed by a latent concept and demonstrate that the pretrained predictor implicitly infers the concept shared across the prompt, becoming asymptotically Bayes-optimal as the number of in-context examples grows. \citet{wakayama2025icl} provide a finite-sample counterpart for meta-learning over a mixture of tasks. They decompose the ICL risk into a Bayes gap and an irreducible posterior variance, and bound the former non-asymptotically for uniform-attention transformers trained by empirical risk minimization. These results establish Bayes-optimality under idealized generative assumptions, while whether trained models actually behave in a Bayesian manner remains an empirical question. \citet{falck2024bayesian} address it by observing that Bayesian predictive distributions over exchangeable data must satisfy a martingale property, and derive testable diagnostics from it. Applied to pretrained LLMs on synthetic tasks, these diagnostics reveal systematic violations. Following this diagnostic perspective, this work statistically tests whether ICRL behaves as a Bayesian posterior sampler.

\paragraph{Bayesian posterior sampling in decision-making.}
Bayesian posterior sampling, originating with \citep{thompson1933likelihood}, selects actions by drawing an environment from the current posterior and acting optimally for that draw. \citet{russo2016information} give an information-theoretic analysis of Thompson sampling, bounding Bayesian regret through the information ratio, which trades off instantaneous regret against information gained about the optimal action, and obtaining regret bounds for finite-armed and linear bandits. \citet{osband2013more} extend posterior sampling to episodic Markov decision processes (MDPs) with Bayesian regret guarantees, which samples an MDP from the posterior at the start of each episode and following its optimal policy. These works characterize the exact posterior sampling, whereas this paper explores whether ICRL schemes can be viewed as posterior samplers in practice.
\clearpage

\section{Theory and Proofs}

\subsection[Proofs of Propositions~\ref{prop:posterior_sufficient_statistics} and~\ref{prop:martingale}]{Proofs of Propositions~\ref{prop:posterior_sufficient_statistics} and~\ref{prop:martingale}}
\label{sec:proposition_proofs}

\begin{proof}[Proof of Proposition~\ref{prop:posterior_sufficient_statistics}]
\label{pf:posterior_sufficient_statistics}
Since we assumed the reward likelihood to be context independent, the joint context likelihood is proportional to the product of each reward likelihood, i.e.,
\begin{equation}
p(D_n\mid\theta)\propto\prod_{t=1}^n p(R_t\mid x_{A_t},\theta).
\label{eq:sequential_history_likelihood}
\end{equation}
Therefore, the posterior takes the form:
\[
p_{\pi_0}(\theta\mid D_n)\propto
\pi_0(\theta)\prod_{t=1}^n p(R_t\mid x_{A_t},\theta).
\]
For part $i)$, permuting the action--reward pairs leaves the product of reward likelihoods unchanged. Thus $D_n$ and $\sigma_1(D_n)$ induce the same posterior.

For part $ii)$, substituting~\eqref{eq:general_exponential_family} into~\eqref{eq:sequential_history_likelihood} gives
\[
p(D_n\mid\theta)
\propto\exp\left\{
\eta(\theta)^\top S_n^{(1)}-\lambda(\theta)^\top S_n^{(2)}
\right\}.
\]
It then follows from the Fisher--Neyman factorization theorem \citep{casella2002statistical} that $S(D_n)=(S_n^{(1)},S_n^{(2)})$ is sufficient for $\theta$. In particular, the posterior takes the form
\[
p_{\pi_0}(\theta\mid D_n)\propto
\pi_0(\theta)\exp\left\{
\eta(\theta)^\top S_n^{(1)}-\lambda(\theta)^\top S_n^{(2)}
\right\}.
\]
Hence $S(\sigma_2(D_n))=S(D_n)$ implies that $D_n$ and $\sigma_2(D_n)$ induce the same posteriors. For both parts, combining the same context-independent decision kernel $\mathcal{K}(a\mid\theta)$ and the same posteriors, as in Definition~\ref{def:transformer_bayesian_posterior}, leads to the equality of action distributions for every $a\in\mathcal{A}$.

\end{proof}

\begin{proof}[Proof of Proposition~\ref{prop:martingale}]
Fix an action $a\in\mathcal{A}$. By assumption, the context continuations and the posteriors in Definition~\ref{def:transformer_bayesian_posterior} arise from the same joint Bayesian model. Under this model,
\[
\pi_{\mathrm{TF}}(a\mid D_t)
=
\int_{\Theta}
\mathcal K(a\mid\theta)
p_{\pi_0}(\theta\mid D_t)
\,d\theta
=
\mathbb E
\left[
\mathcal K(a\mid\theta)
\mid
\mathcal F_t
\right].
\]
Since $0\leq\mathcal K(a\mid\theta)\leq1$, it is integrable. It then follows from the tower property that
\begin{align}
\mathbb E
\left[
\pi_{\mathrm{TF}}(a\mid D_{n+m})
\mid
\mathcal F_n
\right]
&=
\mathbb E
\left[
\mathbb E
\left[
\mathcal K(a\mid\theta)
\mid
\mathcal F_{n+m}
\right]
\middle|
\mathcal F_n
\right]
\\
&=
\mathbb E
\left[
\mathcal K(a\mid\theta)
\mid
\mathcal F_n
\right]
\\
&=
\pi_{\mathrm{TF}}(a\mid D_n).
\end{align}
Conditional on $D_n$, the law for the expectation on the left is precisely $P_{\mathrm{TF}}^{(m)}(\cdot\mid D_n)$, which proves~\eqref{eq:coherent_martingale}. Thus $\{\pi_{\mathrm{TF}}(a\mid D_n)\}_{n\geq0}$ is a martingale. This argument does not require an exponential-family likelihood or a sufficient statistic.

\end{proof}
\clearpage

\subsection{Distinct Context Distributions after sufficient statistics-preserving transformations}
\label{sec:context_transformations}

Let $D=((A_1, R_1),\ldots,(A_n, R_n))$ be an ordered context of length $n\geq2$ from a $K$-armed bandit. The results below are independent of the exact interaction policy. Conditional on the history and $A_t=a$, the reward law is $\mathcal{N}(\mu_a,\sigma^2)$ for Gaussian bandits, $\operatorname{Bernoulli}(p_a)$ for Bernoulli bandits, and $\mathcal{N}(x_a^\top\theta,\sigma^2)$ for linear bandits with fixed known features $x_a$. We assume the variance $\sigma^2>0$ is known, and $p_a\in(0,1)$.

Let $I_{\{A_t=a\}}$ be the indicator function of the event: the action taken at time $t$ is $a$, then let $$N_a=\sum_{t=1}^n I_{\{A_t=a\}}\text{ and }S_a=\sum_{t=1}^n I_{\{A_t=a\}} R_t.$$ 
The sufficient statistics are $\{(N_a,S_a)\}_{a=1}^K$ for Gaussian and Bernoulli bandits, and
\[
V=\sum_{a=1}^K N_ax_ax_a^\top,
\qquad b=\sum_{a=1}^K x_aS_a
\]
for linear bandits (see Appendix~\ref{sec:bandit_sufficient_statistics_proofs}). Thus preserving every arm's count and reward sum suffices is equivalent to preserving the sufficient statistics in all three settings.

\begin{proposition}[Context transformations]
\label{prop:context_transformations}
The following transformations produce a context $D'$ with the same sufficient statistics as $D$ but a different distribution:
\begin{enumerate}
    \item \textit{Sorting (all three bandits):} sort the action--reward pairs by ascending reward, retaining their original order within ties.
    \item \textit{Reward replacement (Gaussian and linear bandits only):} sample $U_t\sim^\text{i.i.d.}\operatorname{Uniform}[-c,c]$ independently of $D$ for some fixed $c>0$. For each observed arm $a$, set
    \[
    R'_t=\bar R_a+(U_t-\bar U_a)I_{\{A_t=a\}},
    \quad\text{where }
    \bar R_a=\frac{S_a}{N_a},\quad
    \bar U_a=\frac{1}{N_a}\sum_{s=1}^nU_sI_{\{A_s=a\}}.
    \]
    Actions remain unchanged. For this claim, suppose there exists an arm that is observed at least twice with positive probability\footnote{This can be easily established by choosing $n>K$}, i.e., $\exists a:\mathbb{P}(\max_a N_a\geq2)>0$.
\end{enumerate}
\end{proposition}

\begin{proof}
Both transformations preserve the arm counts. Sorting also preserves the arm reward sums, while replacement satisfies, for each observed arm $a$,
\[
\sum_{t=1}^nI_{\{A_t=a\}}R'_t
=N_a\bar R_a+\sum_{t=1}^nI_{\{A_t=a\}}(U_t-\bar U_a)
=S_a.
\]
Consequently, both transformations preserve the sufficient statistics.

For sorting, $\mathbb{P}(R'_1>R'_2)=0$. In Gaussian and linear bandits setting, conditional on the first observation and the second action, $R_2$ has a Gaussian distribution on $
\mathbb{R}$, so $\mathbb{P}(R_1>R_2)>0$. For Bernoulli bandits, $\mathbb{P}(R_1=1)>0$ and the conditional probability of $R_2=0$ is positive for every second action, so again $\mathbb{P}(R_1>R_2)>0$. This distinguishes the laws of $D$ and $D'$.

For replacement, whenever $A_s=A_t$,
\[
|R'_s-R'_t|=|U_s-U_t|\leq2c.
\]
The repetition assumption implies that some fixed $s<t$ satisfies $\mathbb{P}(A_s=A_t)>0$. Conditional on $D_{t-1}$ and $A_t$, the earlier reward $R_s$ is fixed and $R_t$ has a Gaussian distribution with full support. Therefore,
\[
\mathbb{P}\big(A_s=A_t,\ |R_s-R_t|>2c\big)>0,
\qquad
\mathbb{P}\big(A_s=A_t,\ |R'_s-R'_t|>2c\big)=0.
\]
These probabilities distinguish the laws in both Gaussian and linear bandits.

\end{proof}
\clearpage

\subsection{Sufficient Statistics for Bandits}
\label{sec:bandit_sufficient_statistics_proofs}
\label{sec:gaussian_sufficient_statistics_proof}

We apply Proposition~\ref{prop:posterior_sufficient_statistics} to the three bandit settings in Section~\ref{sec:case_studies}. We assume the transformer satisfies the assumptions of the proposition with the corresponding reward likelihood. For a context $D_n=\big((A_1,R_1),\ldots,(A_n,R_n)\big)$ and each arm $a\in\{1,\ldots,K\}$, let $I_{\{A_t=a\}}$ be the indicator variable for the event $\{A_t=a\}$ as in Appendix~\ref{sec:context_transformations}. Similarly, define the arm count and reward sum:
\[
N_a=\sum_{t=1}^n I_{\{A_t=a\}},
\qquad
S_a=\sum_{t=1}^n I_{\{A_t=a\}}R_t.
\]

\begin{corollary}[Gaussian bandits]
\label{cor:transformer_gaussian_bandit}
In the Gaussian bandits setting where we have $R_t\mid A_t=a\sim\mathcal{N}(\mu_a,\sigma^2)$ with unknown means $\mu=(\mu_1,\ldots,\mu_K)$ and known variance $\sigma^2>0$, the transformer's action distribution depends on $D_n$ only through $\{(N_a,S_a)\}_{a=1}^K$.
\end{corollary}
\begin{proof}
Up to a constant independent of $\mu$, the context likelihood is
\[
p(D_n\mid\mu)\propto
\exp\left\{\frac{1}{\sigma^2}\sum_{a=1}^K\mu_a S_a
-\frac{1}{2\sigma^2}\sum_{a=1}^K\mu_a^2N_a\right\}.
\]
This takes the exponential-family form in Proposition~\ref{prop:posterior_sufficient_statistics}, using the standard basis vector in $\mathbb{R}^K$ for $x_a$, $T(r,x)=rx$, and $\phi(x)=x$. The rest follows directly from the proposition.

\end{proof}

\begin{corollary}[Bernoulli bandits]
\label{cor:transformer_bernoulli_bandit}
Here, we have $R_t\mid A_t=a\sim\operatorname{Bernoulli}(p_a)$ with unknown success probabilities $p=(p_1,\ldots,p_K)\in(0,1)^K$, the transformer's action distribution depends on $D_n$ only through $\{(N_a,S_a)\}_{a=1}^K$.
\end{corollary}
\begin{proof}
Up to a constant independent of $p$, the context likelihood is
\[
p(D_n\mid p)\propto\prod_{a=1}^Kp_a^{S_a}(1-p_a)^{N_a-S_a}.
\]
With $x_a$ being the standard basis vector in $\mathbb{R}^K$, $T(r,x)=rx$, and $\phi(x)=x$, Proposition~\ref{prop:posterior_sufficient_statistics} directly implies the result.

\end{proof}

\begin{corollary}[Linear bandits]
\label{cor:transformer_linear_bandit}
Given known arm features $x_1,\ldots,x_K\in\mathbb{R}^d$ and rewards $R_t\mid A_t=a\sim\mathcal{N}(x_a^{\top}\theta,\sigma^2)$ with unknown parameter $\theta\in\mathbb{R}^d$ and known variance $\sigma^2>0$, the transformer's action distribution depends on $D_n$ only through
\begin{equation}\label{eq:linear_sufficient_statistics}
V=\sum_{t=1}^n x_{A_t}x_{A_t}^{\top}=\sum_{a=1}^K N_ax_ax_a^{\top},
\qquad
b=\sum_{t=1}^n x_{A_t}R_t=\sum_{a=1}^K x_aS_a.
\end{equation}
\end{corollary}
\begin{proof}
Up to a constant independent of $\theta$, the context likelihood is
\[
p(D_n\mid\theta)\propto
\exp\left\{\frac{\theta^{\top}b}{\sigma^2}
-\frac{\theta^{\top}V\theta}{2\sigma^2}\right\}.
\]
Apply Proposition~\ref{prop:posterior_sufficient_statistics} with $T(r,x)=rx$, $\eta(\theta)=\theta/\sigma^2$, $\phi(x)=\operatorname{vec}(xx^{\top})$, and $\lambda(\theta)=\operatorname{vec}(\theta\theta^{\top})/(2\sigma^2)$. The result follows directly from the proposition.

\end{proof}
\clearpage

\section{Experimental details}
\label{sec:experimental_details}

\subsection{Paired Hotelling Test}\label{sec:hotelling}
Suppose we observe $n$ i.i.d. pairs of random variables $(X_i,X_i')$, where each vector has $K\ge 2$ nonnegative components summing to $1$. Define the paired difference
$$D_i=X_i-X_i',\quad\delta=\E[D_i].$$
The hypotheses are 
$$H_0:\delta=0\quad v.s. \quad H_1:\delta\ne 0.$$
Since $\sum_{j=1}^K D_i^j=0$, the covariance matrix of the full $K$-dimensional differences is singular. Therefore, we drop any one coordinate and set $p=K-1$ and $Z_i=(D^1_i,...,D_i^{K-1})$. Now, $\E[Z_i]=0$ is equivalent to $H_0$ from above. Compute
\begin{align*}
    \bar{Z}&=\frac{1}{n}\sum_{i=1}^nZ_i,\\
    S&=\frac{1}{n-1}\sum_{i=1}^n(Z_i-\bar{Z})(Z_i-\bar{Z})^\top,\\
    T^2&=n\bar{Z}^\top S^{-1}\bar{Z}.
\end{align*}
Since $S$ is consistent, the central limit theorem gives
$$T^2\xrightarrow{\;\; d\;\;}\chi_p^2.$$
Thus an asymptotic level-$\alpha$ test rejects when
$$T^2>\chi_{p,1-\alpha}^2,\quad p=K-1.$$
In our setting, the observed pair $(X,X')$ is simply $\big(\piTF(\cdot\mid D),\piTF(\cdot\mid \sigma(D))\big)$.
\clearpage

\subsection{Experimental setup and hyperparameters}
\label{sec:experimental_hyperparameters}
Table~\ref{tab:experimental_hyperparameters} summarizes the configurations for DPT, DPT$'$, and GDPT. Unless a model-specific setting is indicated, the three models share the listed values.

\begin{table}[H]
\centering
\small
\setlength{\tabcolsep}{3pt}
\renewcommand{\arraystretch}{1.06}
\caption{Environment, model, and statistical test hyperparameters for DPT, DPT$'$, and GDPT.}
\label{tab:experimental_hyperparameters}
\begin{tabular*}{\linewidth}{@{\extracolsep{\fill}}lccc@{}}
\toprule
Parameter & Gaussian & Bernoulli & Linear \\
\midrule
\multicolumn{4}{@{}l}{\textbf{Environment and data}} \\
Number of arms $K$ & $5$ & $5$ & $10$ \\
Feature dimension $d$ & --- & --- & $2$ \\
Environment prior & $\mu_a\sim U[0,1]$ & $p_a\sim\mathrm{Beta}(1,1)$ & $\theta\sim\mathcal{N}(0,I_2/2)$ \\
Reward distribution & $\mathcal{N}(\mu_a,0.09)$ & $\mathrm{Bernoulli}(p_a)$ & $\mathcal{N}(x_a^\top\theta,0.09)$ \\
Context horizon $H$ & $500$ & $500$ & $200$ \\
Train / validation environments & $80{,}000/20{,}000$ & $80{,}000/20{,}000$ & $800{,}000/200{,}000$ \\
DPT, DPT$'$ context policy & \multicolumn{2}{c}{Dirichlet-mixture random policy} & Linear TS \\
GDPT context policy & Gaussian TS & Beta--Bernoulli TS & Linear TS \\
\midrule
\multicolumn{4}{@{}l}{\textbf{Model and optimization}} \\
Backbone & \multicolumn{3}{c}{GPT-2 with learned positions} \\
Layers / heads / embedding size & \multicolumn{3}{c}{$4\;/\;4\;/\;32$} \\
Dropout & \multicolumn{3}{c}{$0$ (embedding, attention, and residual)} \\
Optimizer & \multicolumn{3}{c}{AdamW} \\
Learning rate / weight decay & \multicolumn{3}{c}{$10^{-4}\;/\;10^{-4}$} \\
Training batch size & \multicolumn{3}{c}{$64$} \\
Reward-loss weight & \multicolumn{3}{c}{$1$ for DPT$'$ and GDPT \& no reward loss for DPT} \\
\midrule
\multicolumn{4}{@{}l}{\textbf{Evaluation and statistical tests}} \\
Performance environments & \multicolumn{3}{c}{$200$ fresh environments} \\
Sufficient-statistic test size $N$ & \multicolumn{3}{c}{$5000$ independent environments and contexts} \\
Sufficient-statistic lengths $n$ & \multicolumn{2}{c}{$\{50,100,200,300,400,500\}$} & $\{25,50,100,150,200\}$ \\
Context transformations & Sort \& Replace & Sort & Sort \& Replace \\
Replacement noise $U_t$ & $U[-1,1]$ & --- & $U[-1,1]$ \\
Martingale continuations $N'$ & \multicolumn{3}{c}{$5000$ per fixed context (DPT$'$, GDPT)} \\
Martingale starting lengths $n$ & \multicolumn{2}{c}{$\{50,100,200,300,400\}$} & $\{25,50,100,150\}$ \\
Test level / calibration & \multicolumn{3}{c}{$\alpha=0.05$ \& asymptotic paired Hotelling $T^2$} \\
\bottomrule
\end{tabular*}
\end{table}

\textbf{Architecture ablations.}
The DPT and DPT$'$ ablations keep the data and optimization settings above, remove positional encoding, replace causal attention with bidirectional attention, and at inference canonically order the context tokens and mean-pool their hidden states (excluding the query token), which makes predictions permutation invariant. Exact equality under sorting is reported as ``Equal'' without a $p$-value. GDPT is not ablated since its next-pair objective requires causal attention.

\clearpage

\section{Additional tables and plots}
\label{sec:additional_tables_and_plots}

\subsection{DPT figures}

\subsubsection{Gaussian evaluation figures}
\begin{figure}[H]
\centering
\begin{subfigure}{0.245\textwidth}
\includegraphics[width=\textwidth]{figures/gaussian_DPT/cumulative_regret_online.pdf}
\caption{Online regret}
\end{subfigure}
\hfill
\begin{subfigure}{0.245\textwidth}
\includegraphics[width=\textwidth]{figures/gaussian_DPT/tv_distance_online.pdf}
\caption{Online TV distance}
\end{subfigure}
\hfill
\begin{subfigure}{0.245\textwidth}
\includegraphics[width=\textwidth]{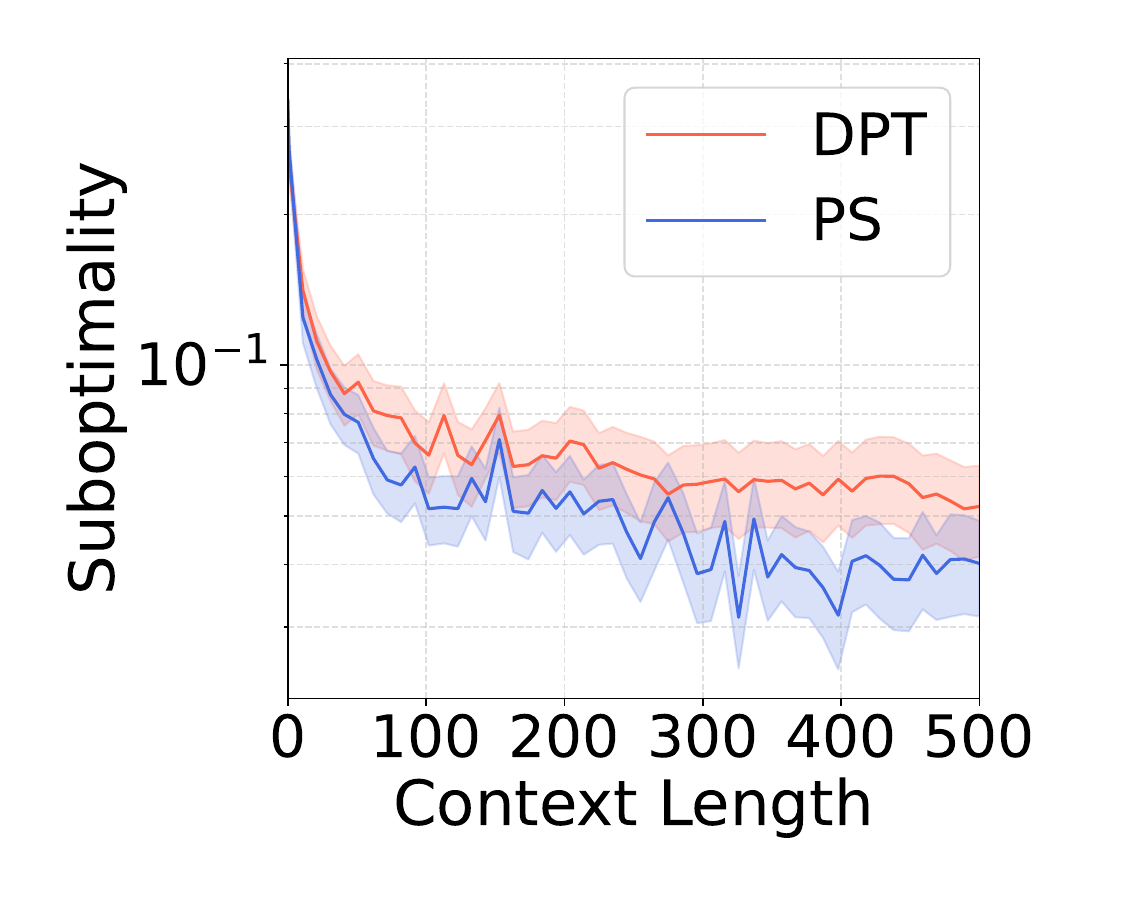}
\caption{Offline suboptimality}
\end{subfigure}
\hfill
\begin{subfigure}{0.245\textwidth}
\includegraphics[width=\textwidth]{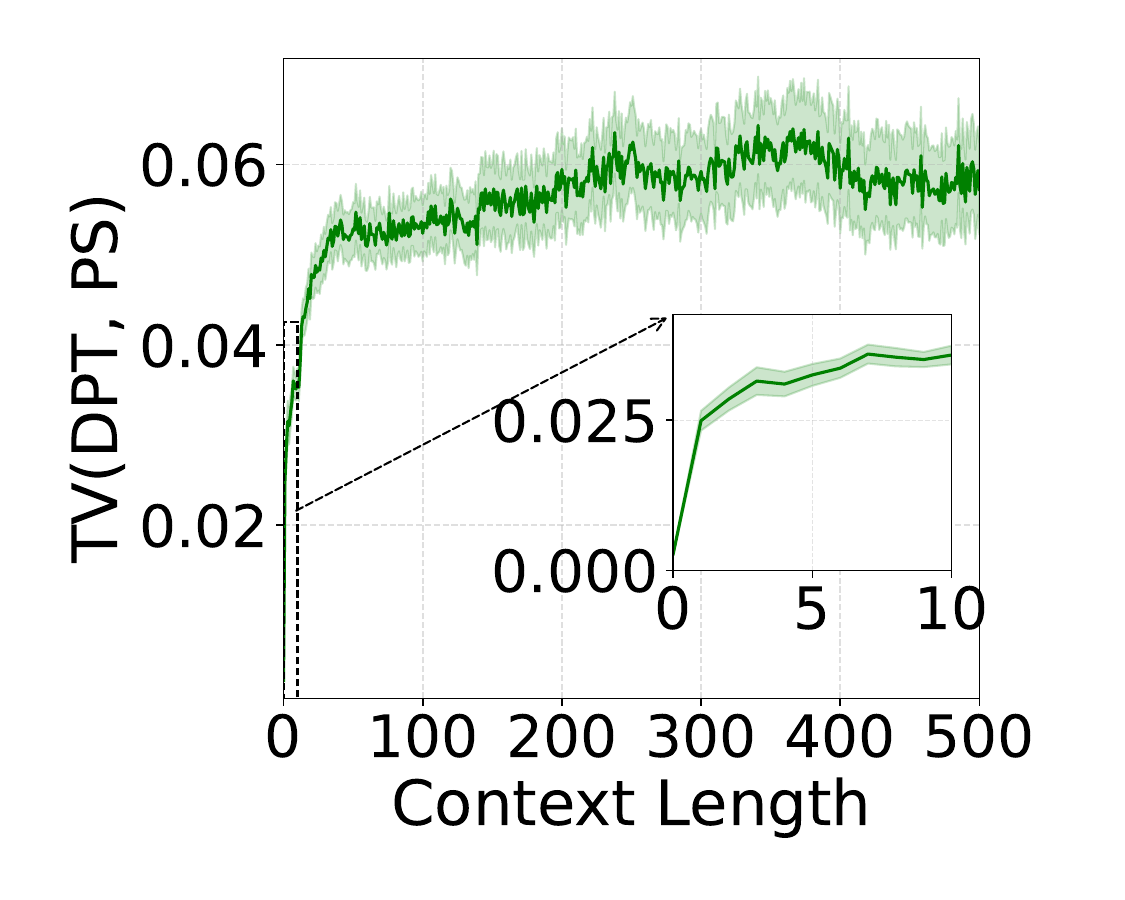}
\caption{Offline TV distance}
\end{subfigure}
\caption{Performance comparison of DPT and Thompson Sampling in Gaussian bandits under online and offline deployment.}
\label{fig_evaluation_gaussian_bandit}
\end{figure}

\subsubsection{Bernoulli evaluation figures}
\begin{figure}[H]
\centering
\begin{subfigure}{0.245\textwidth}
\includegraphics[width=\textwidth]{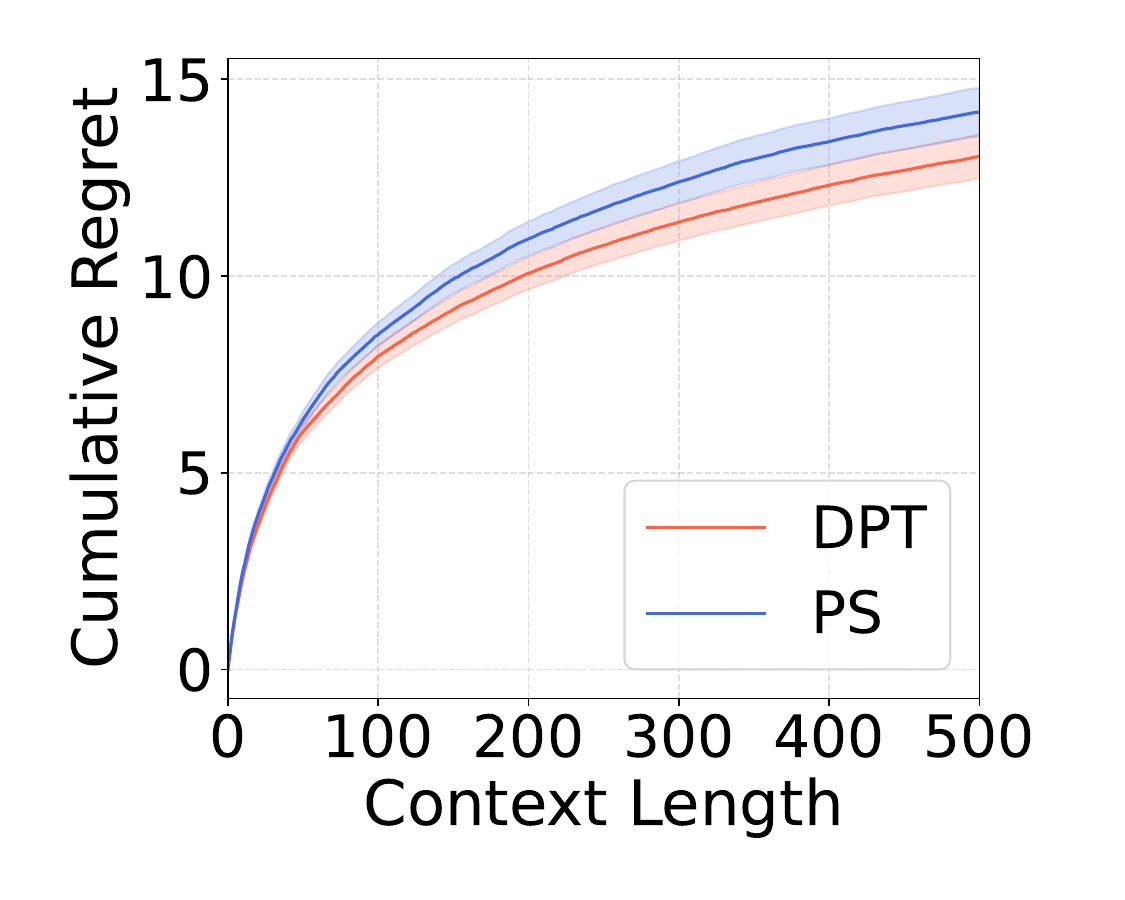}
\caption{Online regret}
\end{subfigure}
\hfill
\begin{subfigure}{0.245\textwidth}
\includegraphics[width=\textwidth]{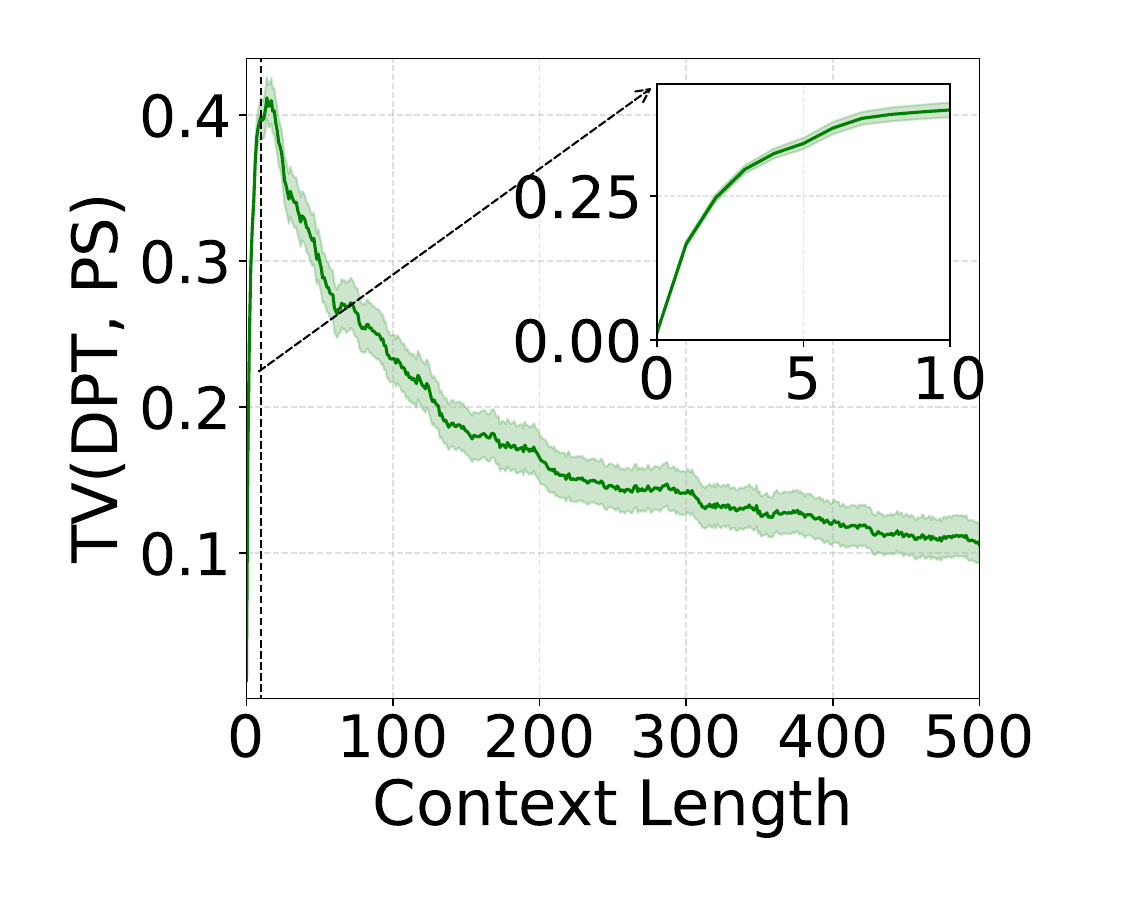}
\caption{Online TV distance}
\end{subfigure}
\hfill
\begin{subfigure}{0.245\textwidth}
\includegraphics[width=\textwidth]{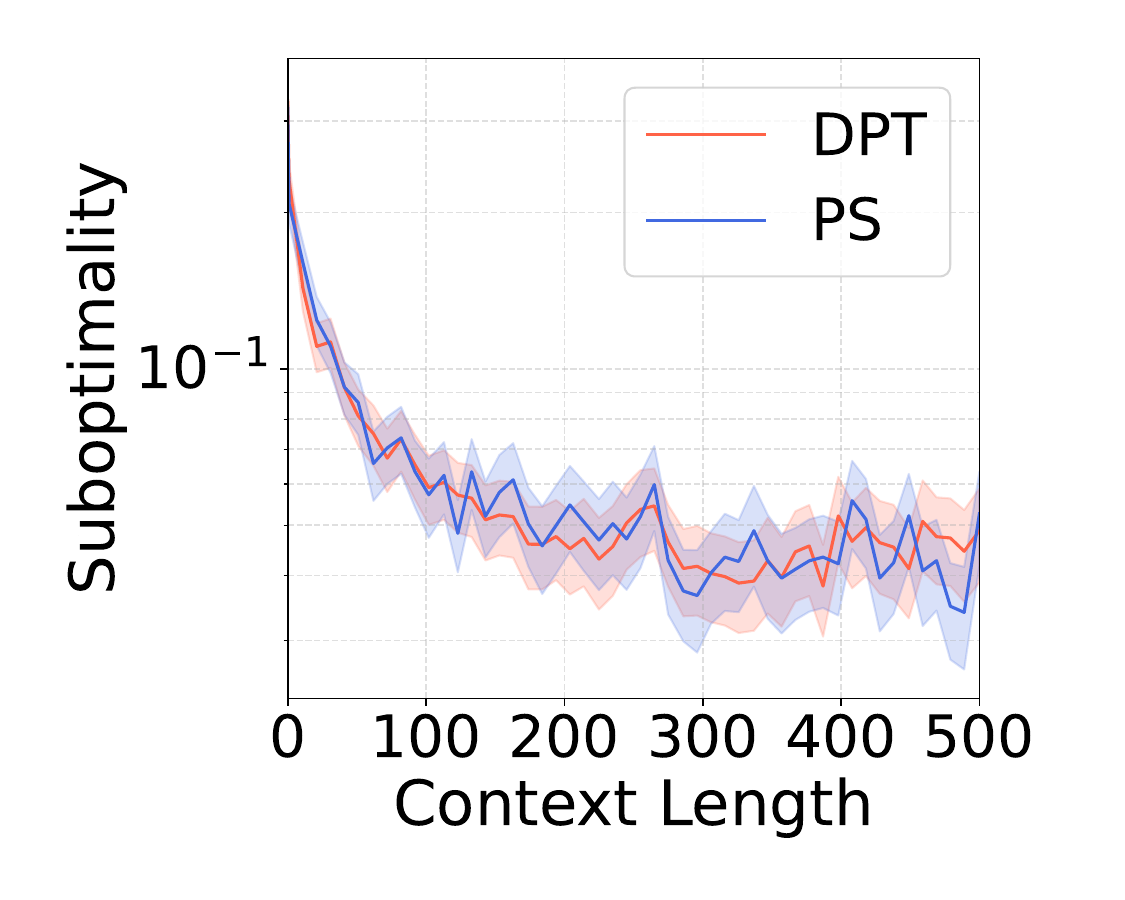}
\caption{Offline suboptimality}
\end{subfigure}
\hfill
\begin{subfigure}{0.245\textwidth}
\includegraphics[width=\textwidth]{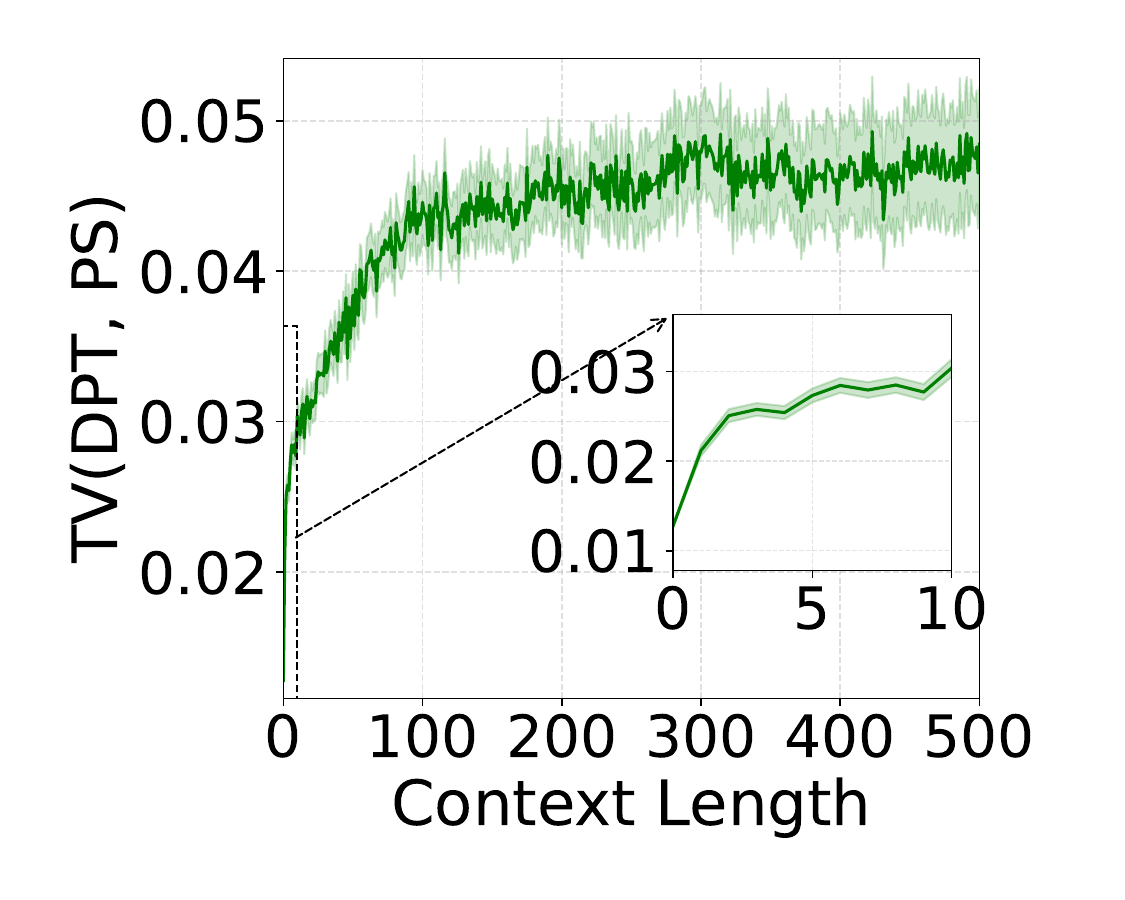}
\caption{Offline TV distance}
\end{subfigure}
\caption{Performance comparison of DPT and Thompson Sampling in Bernoulli bandits under online and offline deployment.}
\label{fig_evaluation_bernoulli_bandit}
\end{figure}

\subsubsection{Linear evaluation figures}
\begin{figure}[H]
\centering
\begin{subfigure}{0.245\textwidth}
\includegraphics[width=\textwidth]{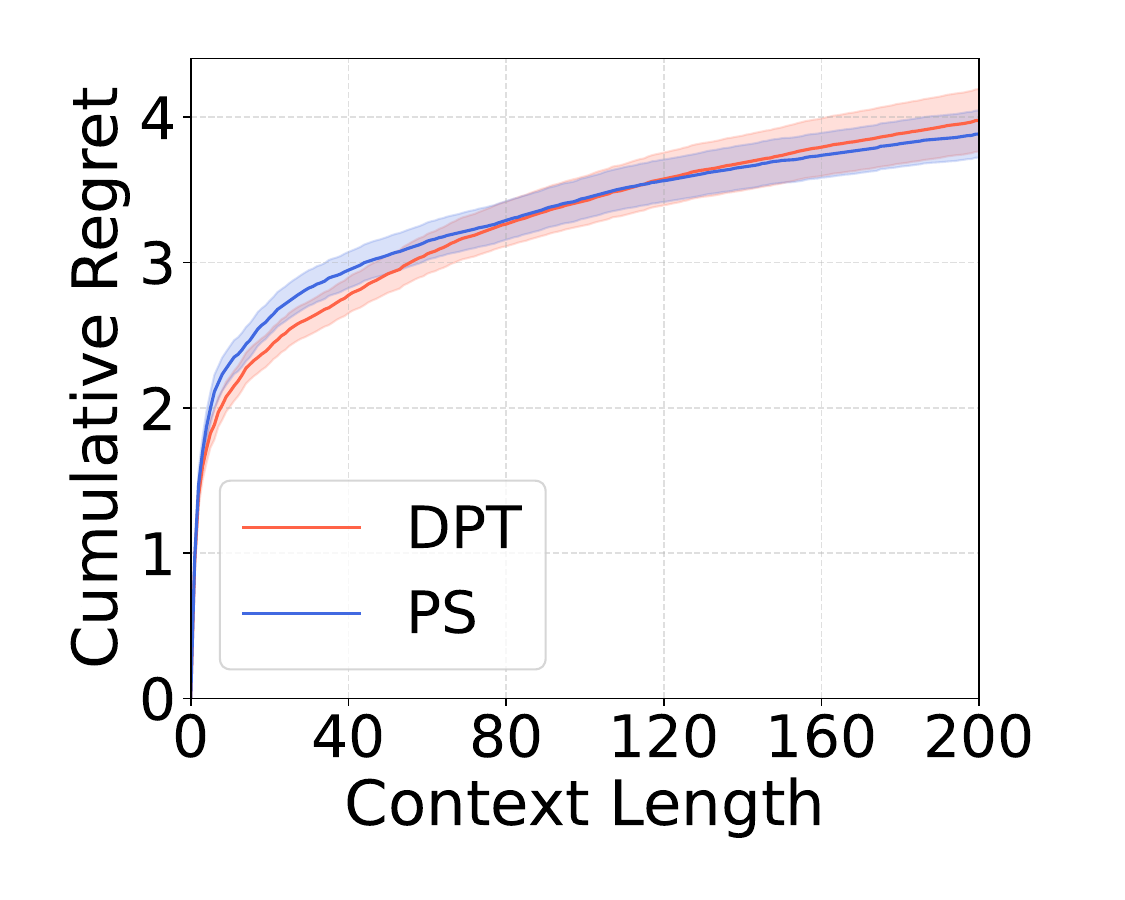}
\caption{Online regret}
\end{subfigure}
\hfill
\begin{subfigure}{0.245\textwidth}
\includegraphics[width=\textwidth]{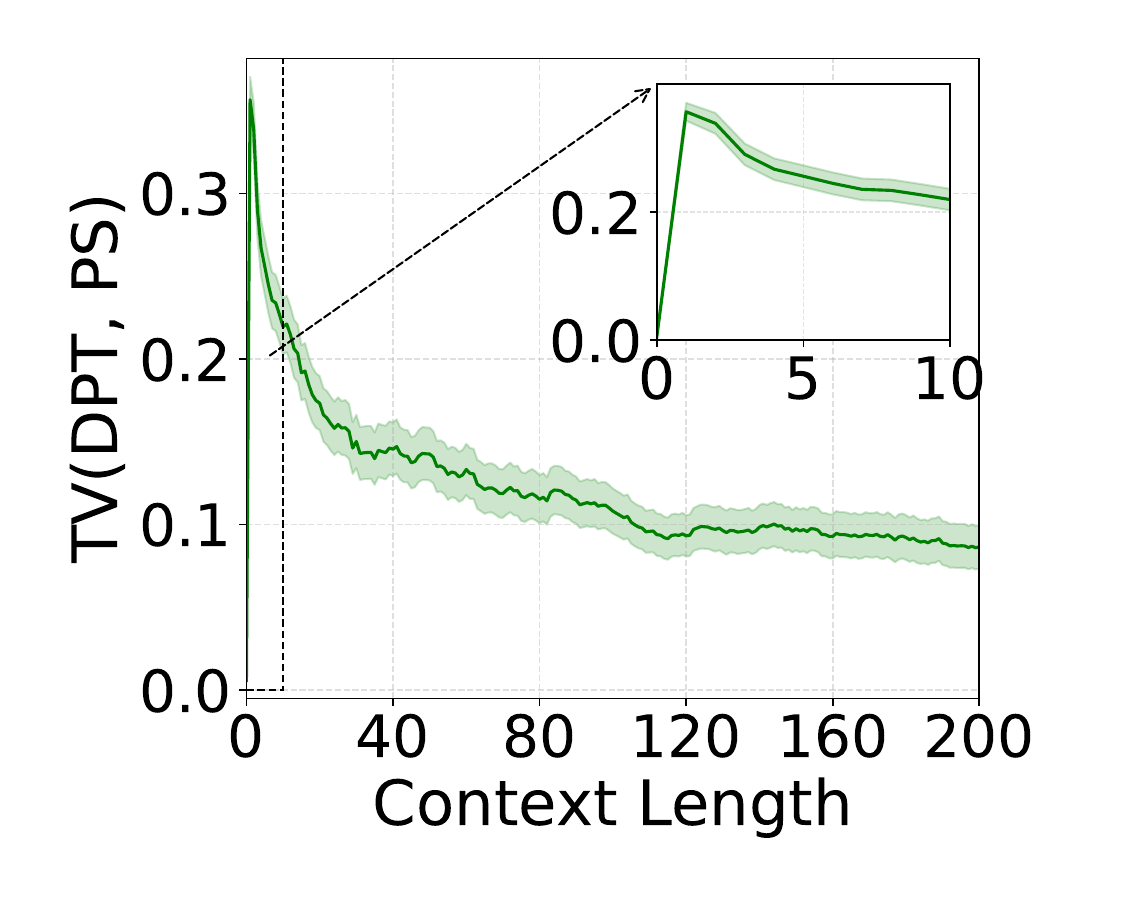}
\caption{Online TV distance}
\end{subfigure}
\hfill
\begin{subfigure}{0.245\textwidth}
\includegraphics[width=\textwidth]{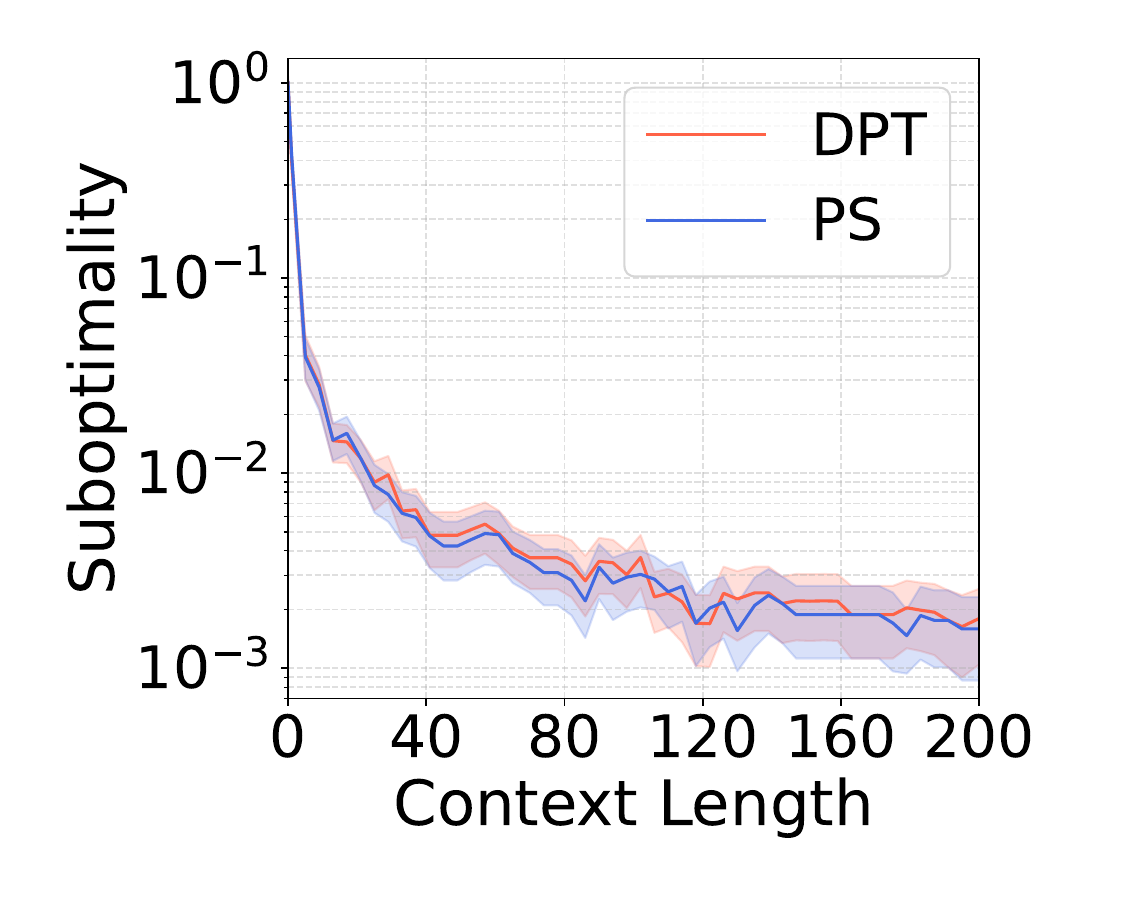}
\caption{Offline suboptimality}
\end{subfigure}
\hfill
\begin{subfigure}{0.245\textwidth}
\includegraphics[width=\textwidth]{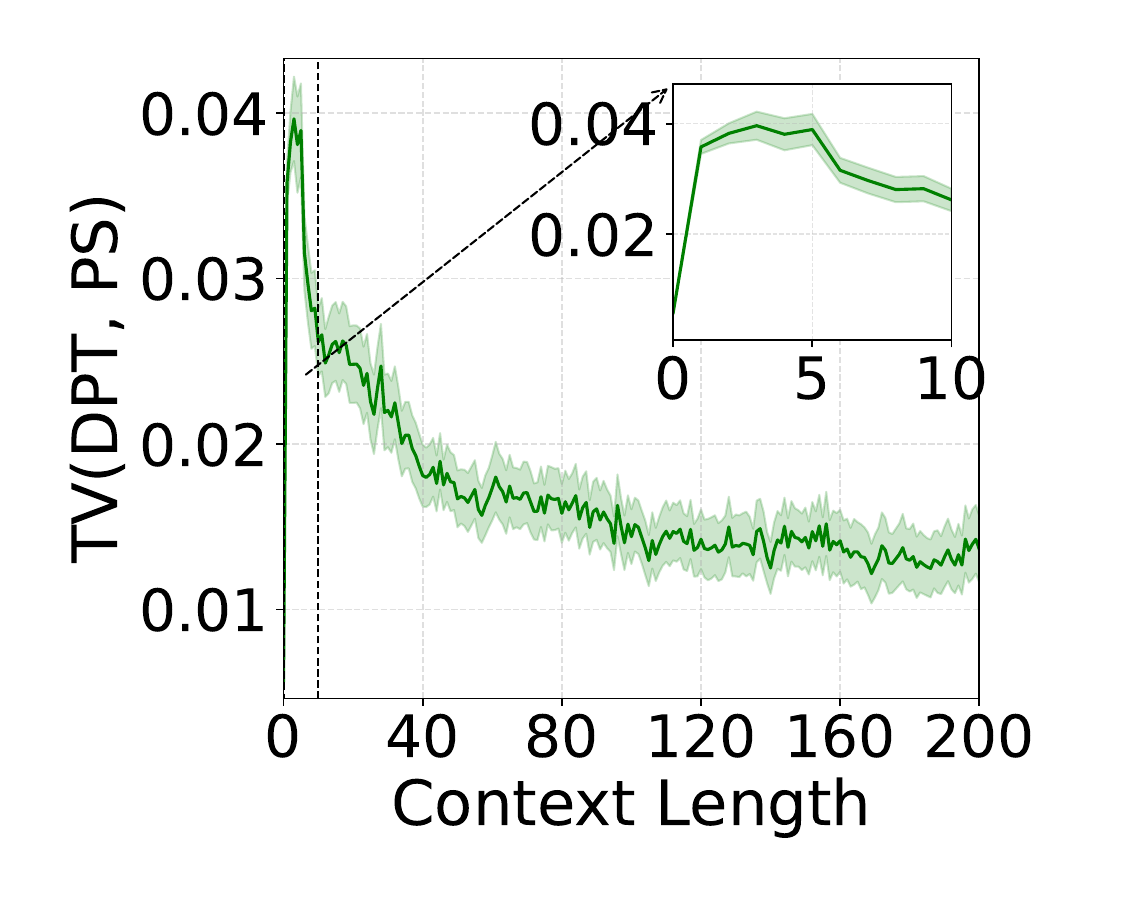}
\caption{Offline TV distance}
\end{subfigure}
\caption{Performance comparison of DPT and Thompson Sampling in linear bandits under online and offline deployment.}
\label{fig_evaluation_linear_bandit}
\end{figure}

\subsection{DPT\texorpdfstring{$'$}{'} sufficient-statistics table and evaluation figures}
\label{sec:dpt_prime_tables_and_plots}

\subsubsection{Sufficient-statistics table}
\begin{table}[H]
\centering
\small
\setlength{\tabcolsep}{6pt}
\renewcommand{\arraystretch}{1.08}
\caption{$p$-values for sufficient-statistic invariance tests on DPT$'$ in Gaussian, Bernoulli, and linear bandits ($N=5000$, $\alpha=0.05$). Here, $n$ denotes the context length.}
\label{tab:ss_summary_dpt_prime}

\sbox{\banditupperbox}{%
\begin{tabular}{@{}r rrrr rr@{}}
\toprule
& \multicolumn{4}{c}{Gaussian}
& \multicolumn{2}{c}{Bernoulli} \\
\cmidrule(lr){2-5}\cmidrule(lr){6-7}
& \multicolumn{2}{c}{Offline}
& \multicolumn{2}{c}{Online}
& \multicolumn{1}{c}{Offline}
& \multicolumn{1}{c}{Online} \\
$n$
& \multicolumn{1}{c}{Sort}
& \multicolumn{1}{c}{Replace}
& \multicolumn{1}{c}{Sort}
& \multicolumn{1}{c}{Replace}
& \multicolumn{1}{c}{Sort}
& \multicolumn{1}{c}{Sort} \\
\midrule
$50$  & $2.31\times 10^{-105}$ & $9.28\times 10^{-19}$ & $<10^{-300}$ & $6.36\times 10^{-205}$ & $1.67\times 10^{-196}$ & $2.17\times 10^{-215}$ \\
$100$ & $3.09\times 10^{-194}$ & $6.26\times 10^{-60}$ & $<10^{-300}$ & $<10^{-300}$ & $5.66\times 10^{-160}$ & $4.75\times 10^{-237}$ \\
$200$ & $7.78\times 10^{-157}$ & $1.04\times 10^{-126}$ & $3.59\times 10^{-265}$ & $1.35\times 10^{-285}$ & $1.08\times 10^{-125}$ & $8.26\times 10^{-250}$ \\
$300$ & $6.67\times 10^{-137}$ & $1.42\times 10^{-156}$ & $2.58\times 10^{-226}$ & $3.05\times 10^{-239}$ & $2.00\times 10^{-120}$ & $5.46\times 10^{-229}$ \\
$400$ & $8.63\times 10^{-133}$ & $1.37\times 10^{-188}$ & $1.88\times 10^{-205}$ & $1.02\times 10^{-193}$ & $3.39\times 10^{-112}$ & $2.36\times 10^{-226}$ \\
$500$ & $7.34\times 10^{-139}$ & $9.23\times 10^{-216}$ & $2.56\times 10^{-222}$ & $3.69\times 10^{-189}$ & $2.08\times 10^{-111}$ & $3.16\times 10^{-216}$ \\
\midrule
\end{tabular}}

\resizebox{\linewidth}{!}{%
\begin{minipage}{\wd\banditupperbox}
\noindent\usebox{\banditupperbox}\par\nointerlineskip
\noindent
\begin{tabular*}{\linewidth}{@{\extracolsep{\fill}}r rrrr@{}}
& \multicolumn{4}{c}{Linear} \\
\cmidrule(lr){2-5}
& \multicolumn{2}{c}{Offline}
& \multicolumn{2}{c}{Online} \\
$n$
& \multicolumn{1}{c}{Sort}
& \multicolumn{1}{c}{Replace}
& \multicolumn{1}{c}{Sort}
& \multicolumn{1}{c}{Replace} \\
\midrule
$25$  & $1.96\times 10^{-102}$ & $5.71\times 10^{-28}$ & $1.15\times 10^{-86}$ & $4.19\times 10^{-29}$ \\
$50$  & $6.30\times 10^{-96}$  & $2.33\times 10^{-27}$ & $7.27\times 10^{-93}$ & $2.22\times 10^{-28}$ \\
$100$ & $9.44\times 10^{-108}$ & $9.99\times 10^{-18}$ & $3.07\times 10^{-99}$ & $1.10\times 10^{-21}$ \\
$150$ & $1.78\times 10^{-93}$  & $3.22\times 10^{-16}$ & $8.76\times 10^{-84}$ & $3.62\times 10^{-14}$ \\
$200$ & $2.68\times 10^{-83}$  & $1.17\times 10^{-11}$ & $3.13\times 10^{-79}$ & $2.73\times 10^{-11}$ \\
\bottomrule
\end{tabular*}
\end{minipage}%
}
\end{table}

\subsubsection{Gaussian evaluation figures and martingale plots}
\begin{figure}[H]
\centering
\begin{subfigure}{0.245\textwidth}
\includegraphics[width=\textwidth]{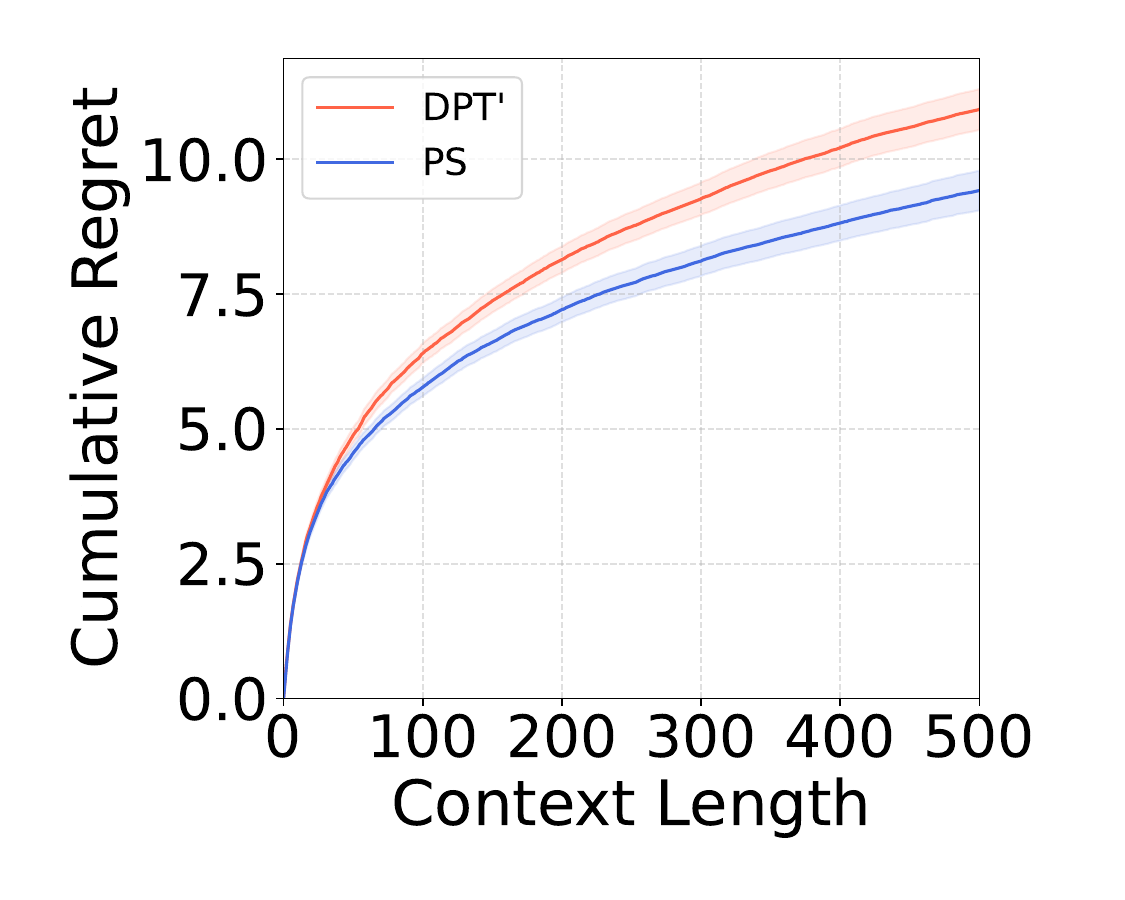}
\caption{Online regret}
\end{subfigure}
\hfill
\begin{subfigure}{0.245\textwidth}
\includegraphics[width=\textwidth]{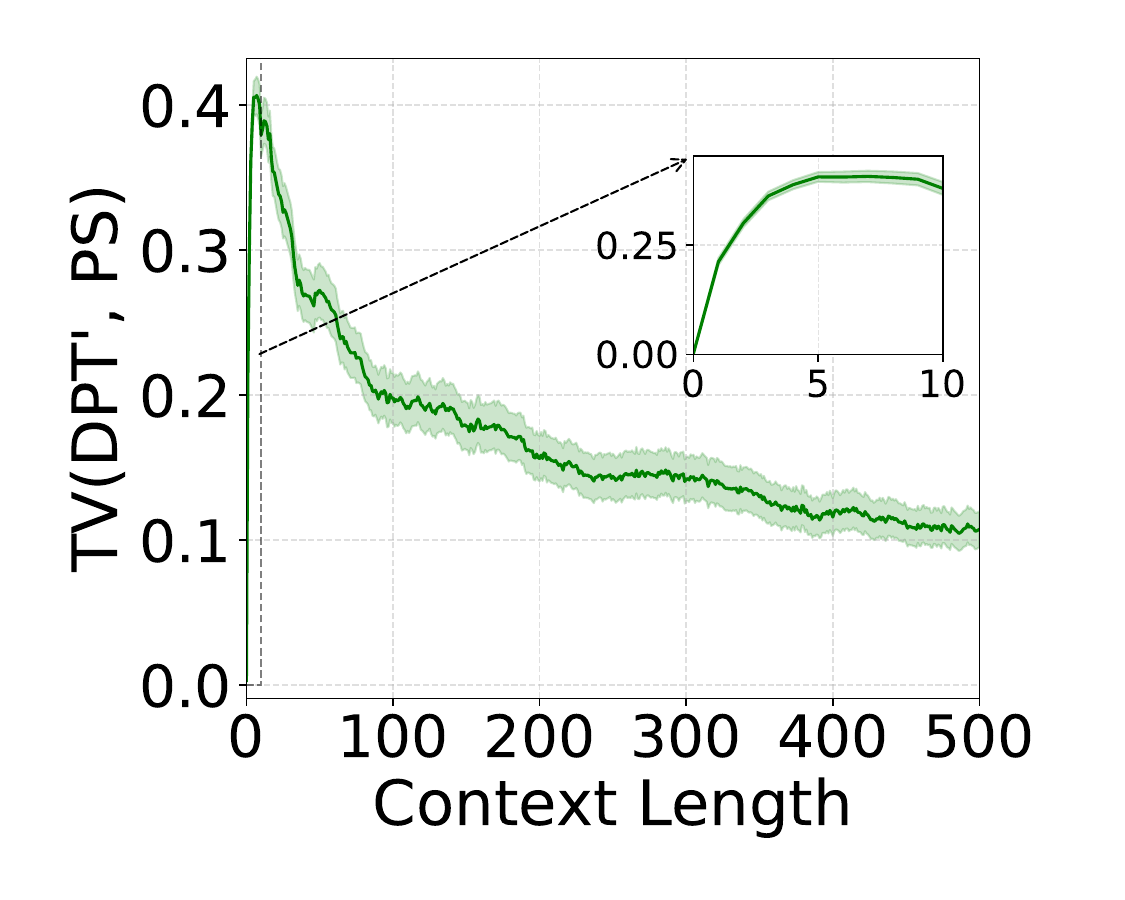}
\caption{Online TV distance}
\end{subfigure}
\hfill
\begin{subfigure}{0.245\textwidth}
\includegraphics[width=\textwidth]{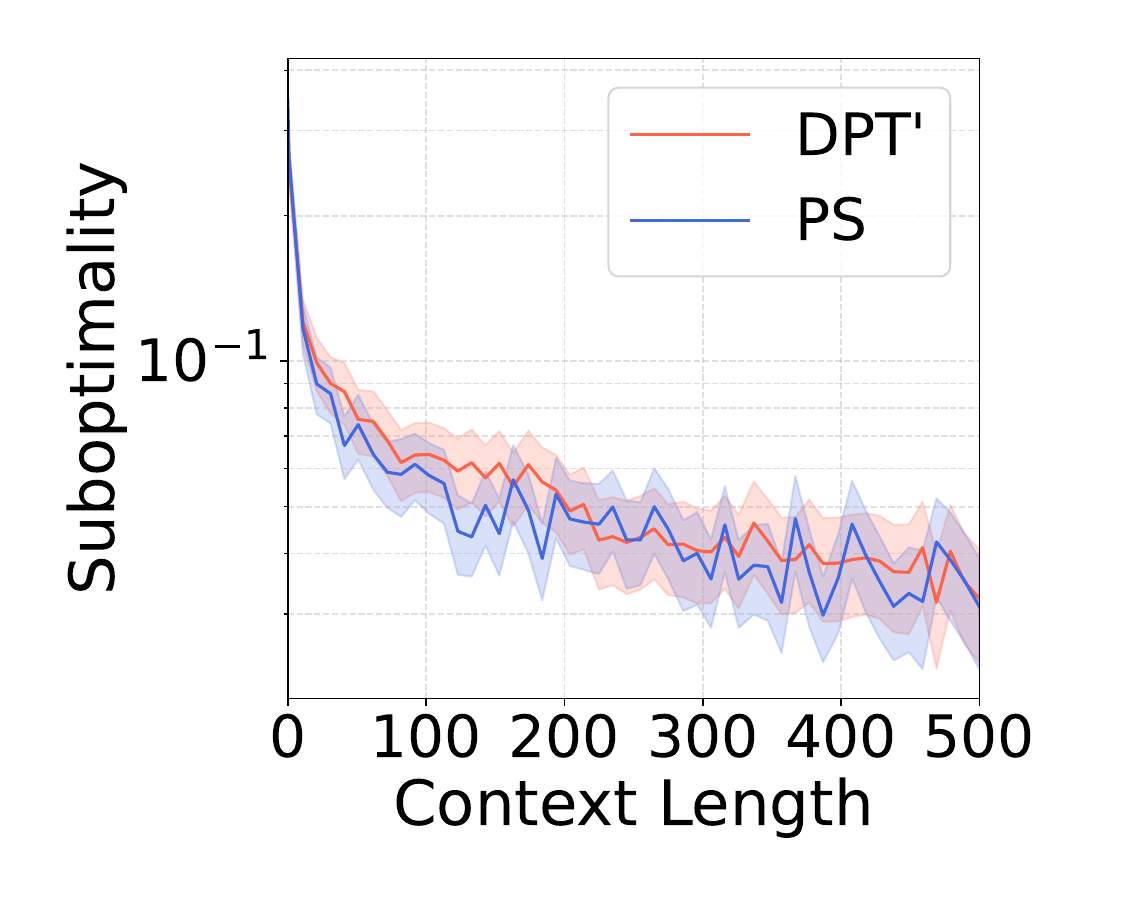}
\caption{Offline suboptimality}
\end{subfigure}
\hfill
\begin{subfigure}{0.245\textwidth}
\includegraphics[width=\textwidth]{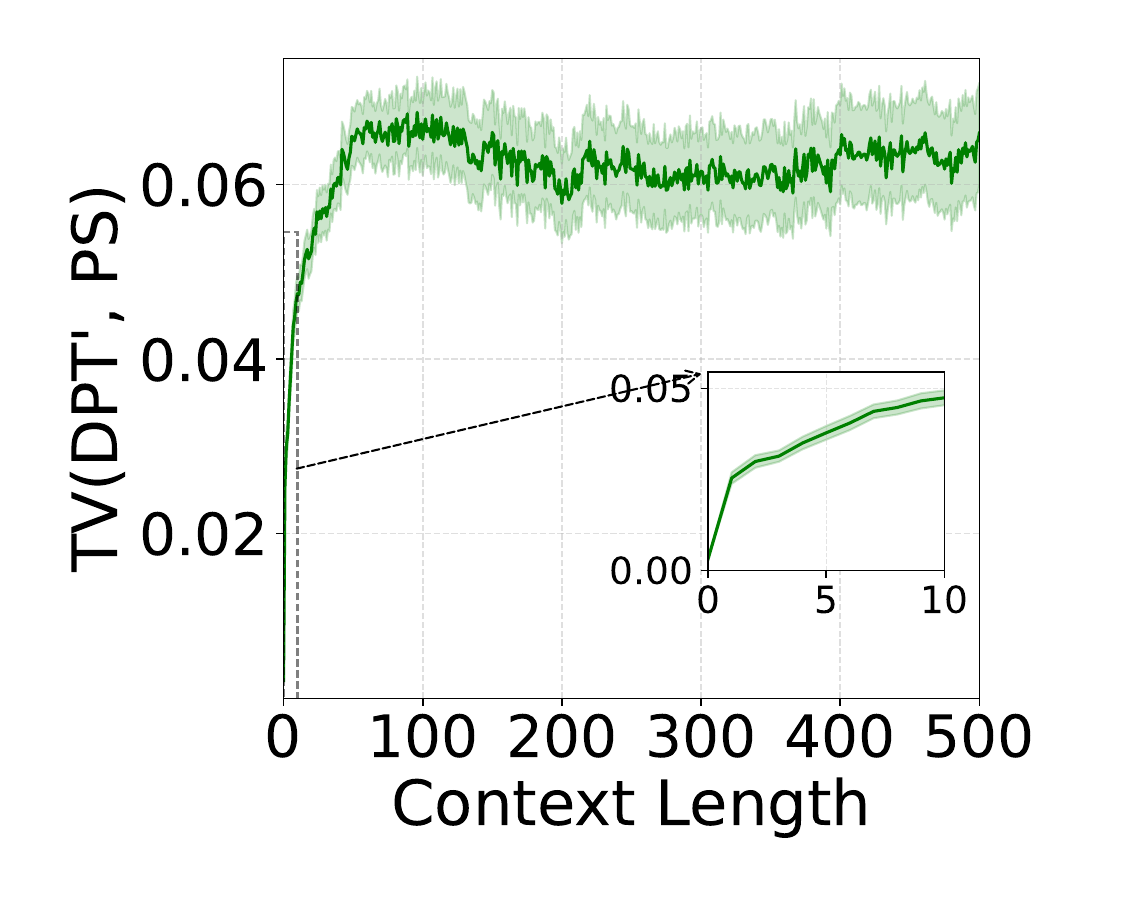}
\caption{Offline TV distance}
\end{subfigure}
\caption{Performance comparison of DPT$'$ and Thompson Sampling in Gaussian bandits under online and offline deployment.}
\label{fig_evaluation_gaussian_bandit_DPT_prime}
\end{figure}

\begin{figure}[H]
    \centering
    \begin{subfigure}[b]{0.32\textwidth}
        \centering
        \includegraphics[width=\textwidth]{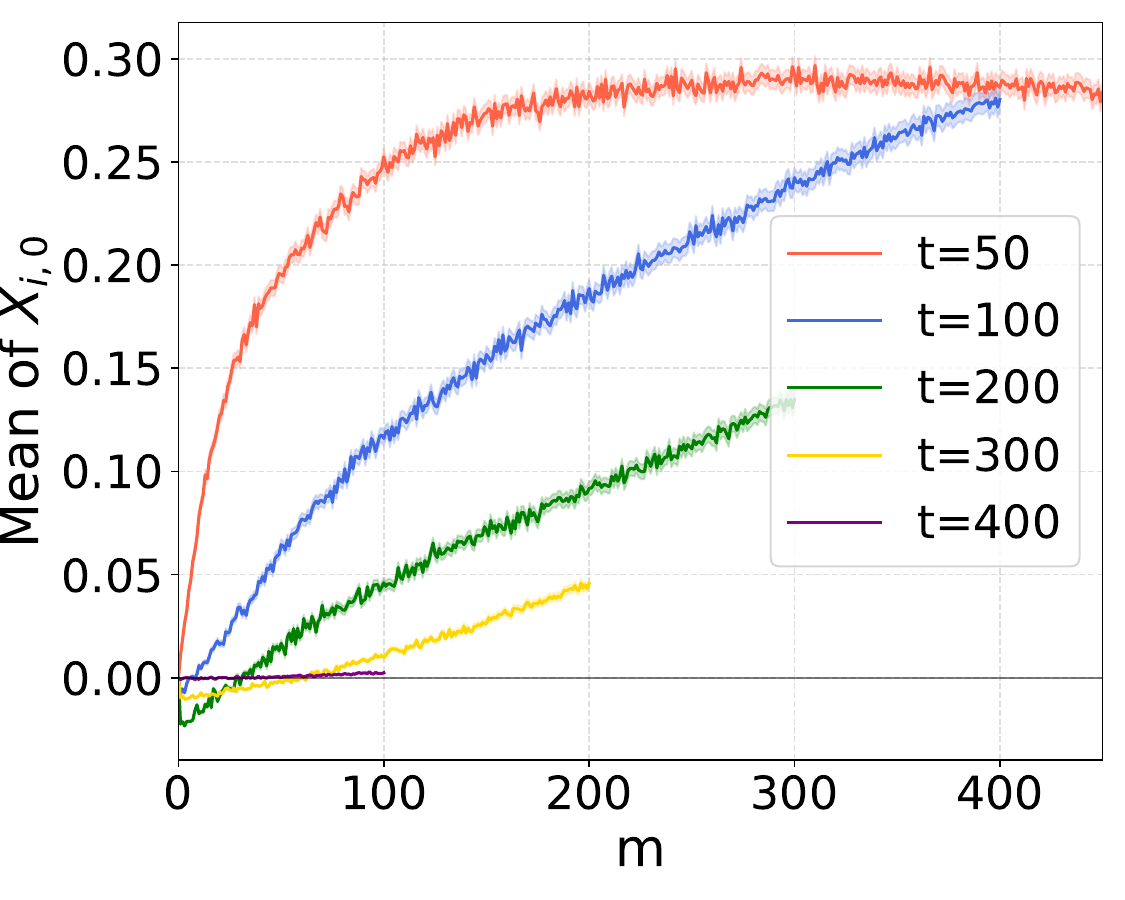}
        \caption{Arm 0}
    \end{subfigure}
    \begin{subfigure}[b]{0.32\textwidth}
        \centering
        \includegraphics[width=\textwidth]{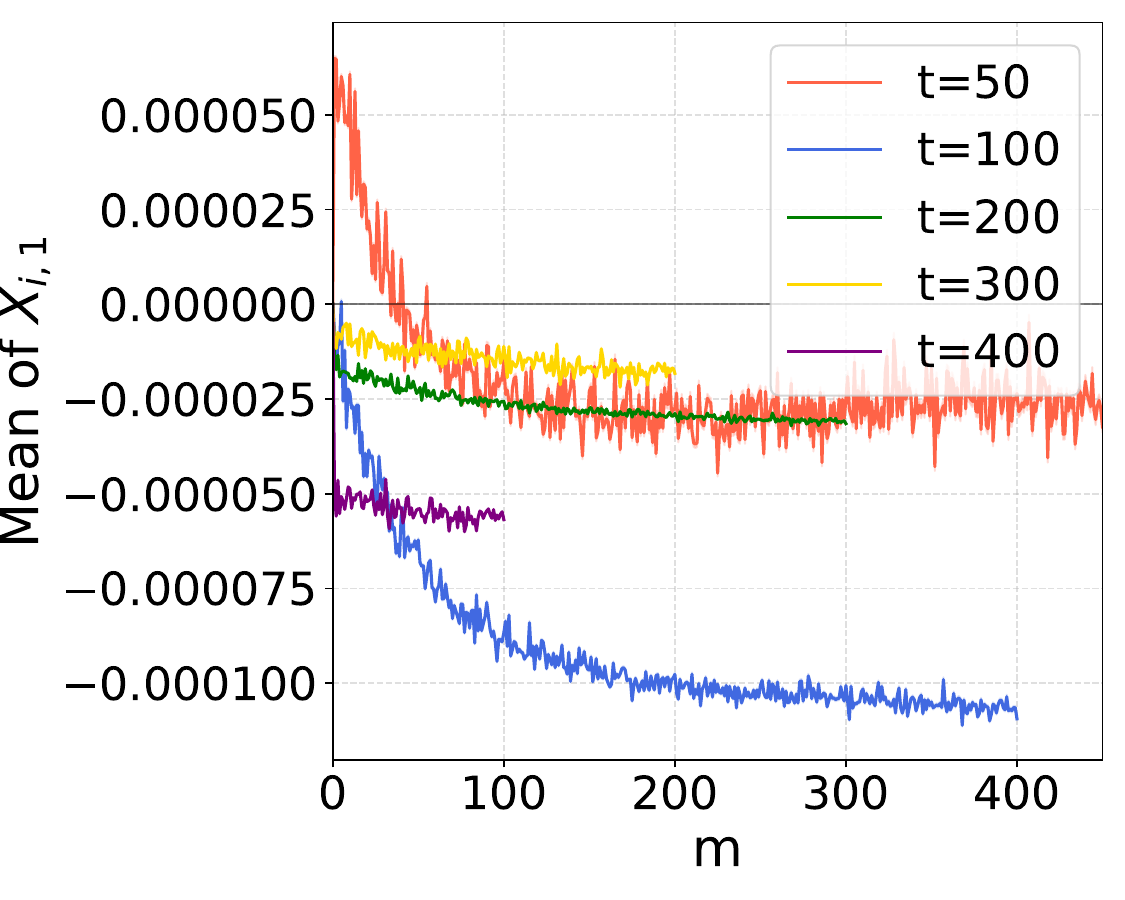}
        \caption{Arm 1}
    \end{subfigure}
    \begin{subfigure}[b]{0.32\textwidth}
        \centering
        \includegraphics[width=\textwidth]{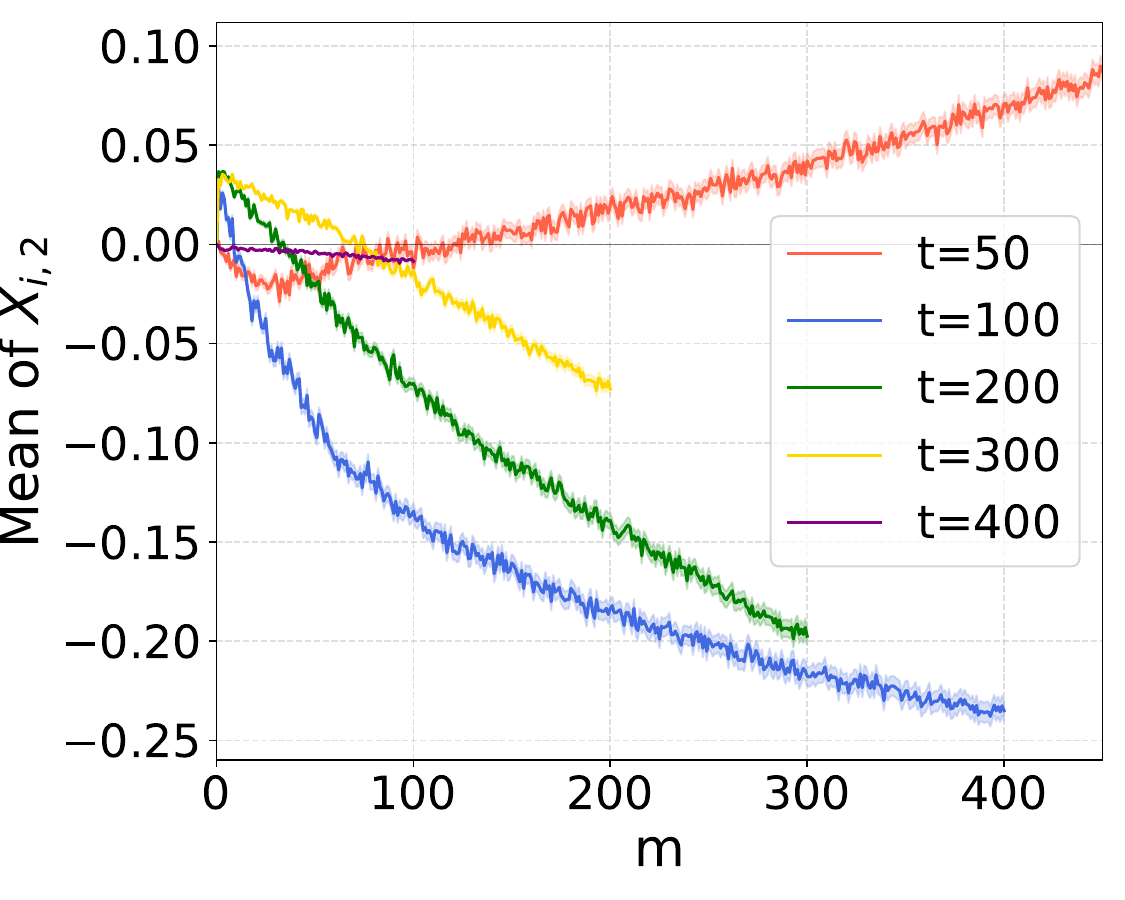}
        \caption{Arm 2}
    \end{subfigure}
    \\
    \begin{subfigure}[b]{0.32\textwidth}
        \centering
        \includegraphics[width=\textwidth]{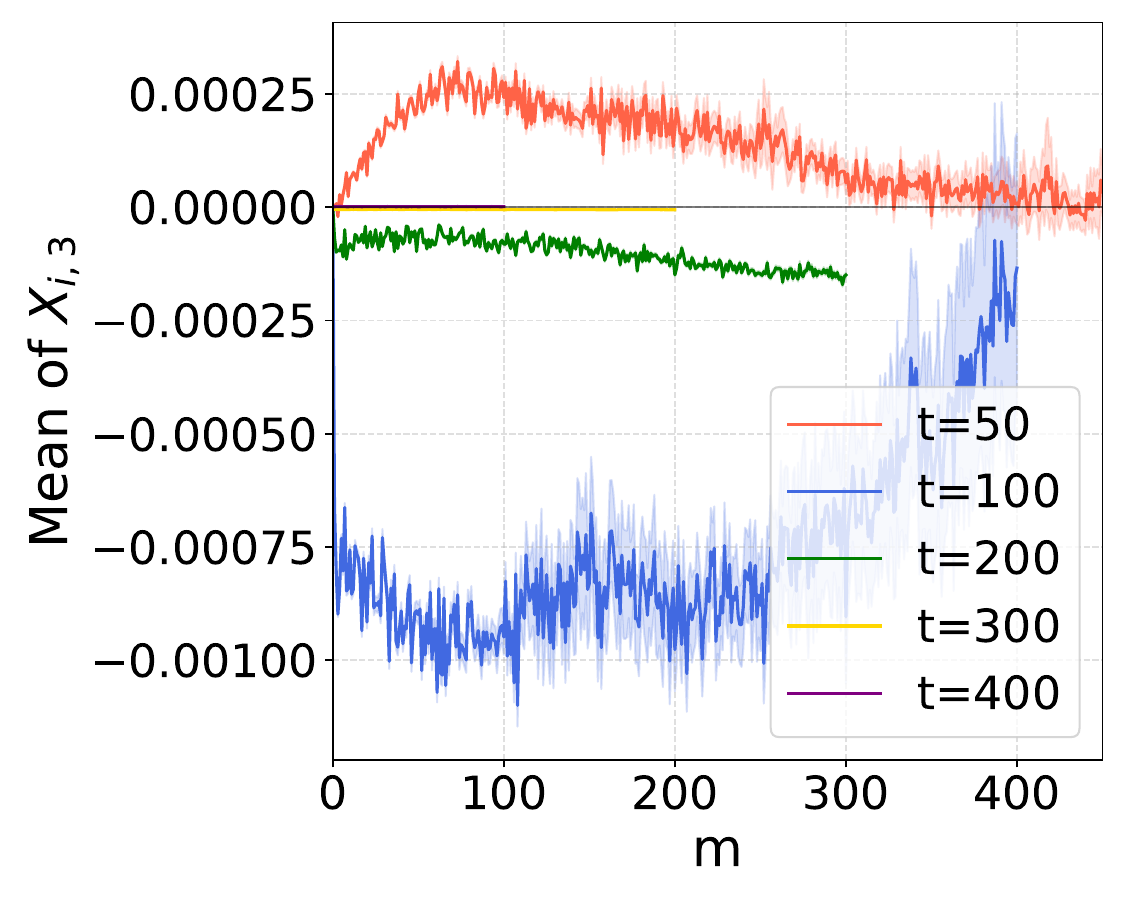}
        \caption{Arm 3}
    \end{subfigure}
    \begin{subfigure}[b]{0.32\textwidth}
        \centering
        \includegraphics[width=\textwidth]{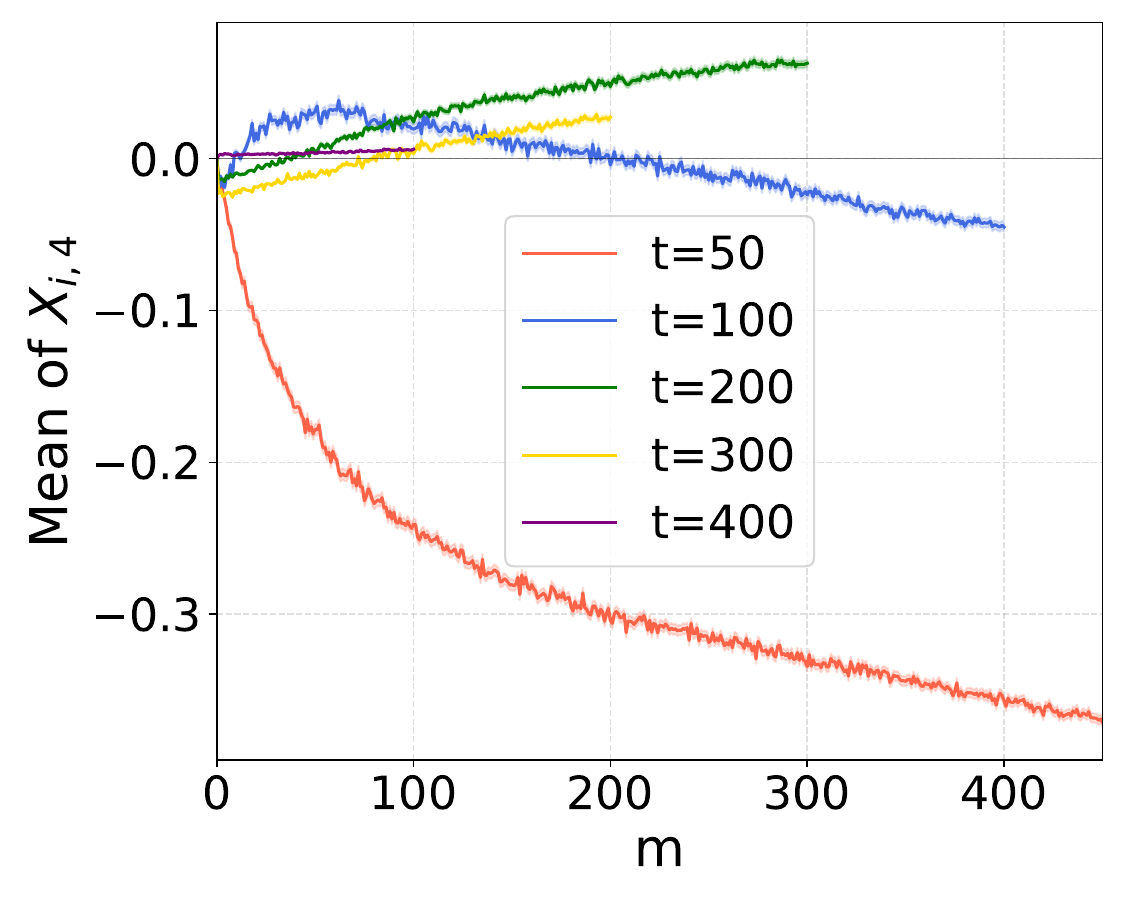}
        \caption{Arm 4}
    \end{subfigure}
    \caption{Per-arm martingale increments for DPT$'$ in Gaussian bandits with offline starting contexts, where $m$ and $t$ denote the lengths of the martingale continuation and the starting contexts, respectively.}
    \label{fig_martingale_increment_offline_gaussian_DPT_prime}
\end{figure}

\begin{figure}[H]
    \centering
    \begin{subfigure}[b]{0.32\textwidth}
        \centering
        \includegraphics[width=\textwidth]{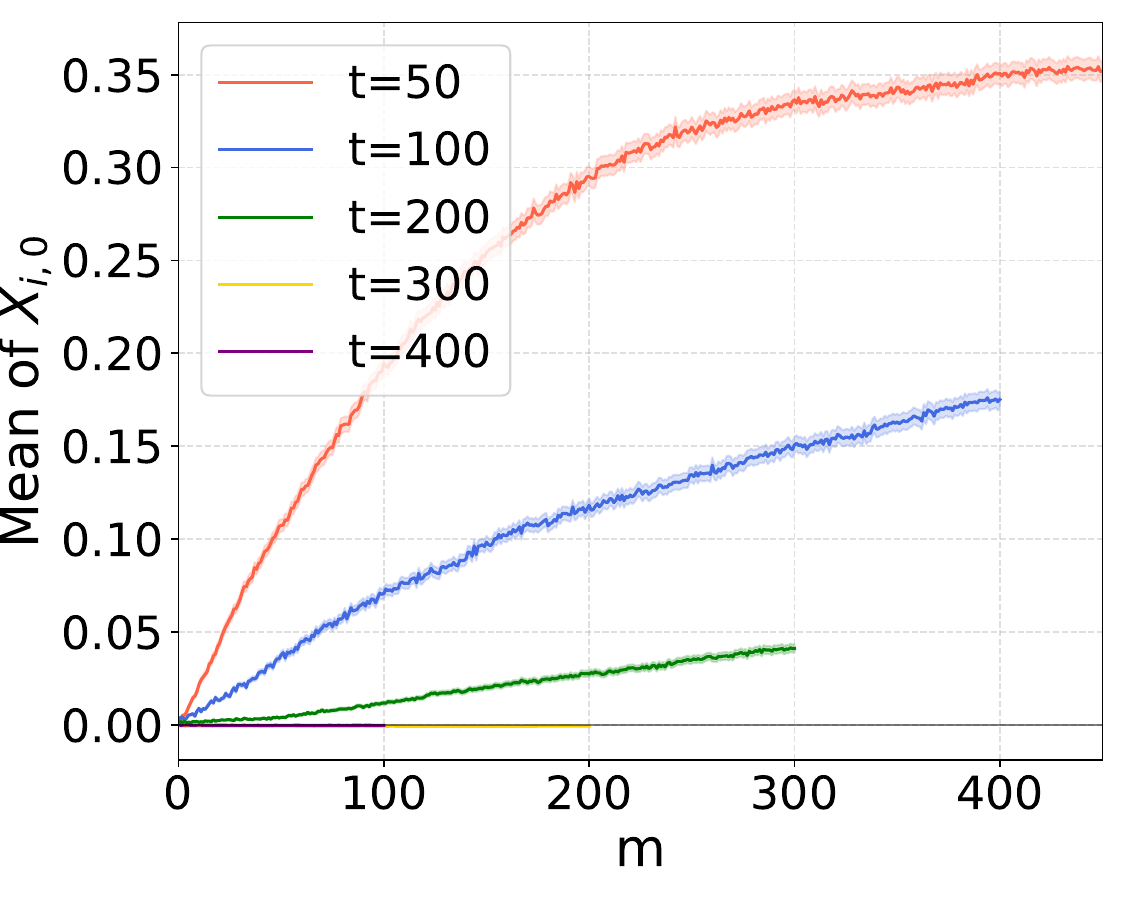}
        \caption{Arm 0}
    \end{subfigure}
    \begin{subfigure}[b]{0.32\textwidth}
        \centering
        \includegraphics[width=\textwidth]{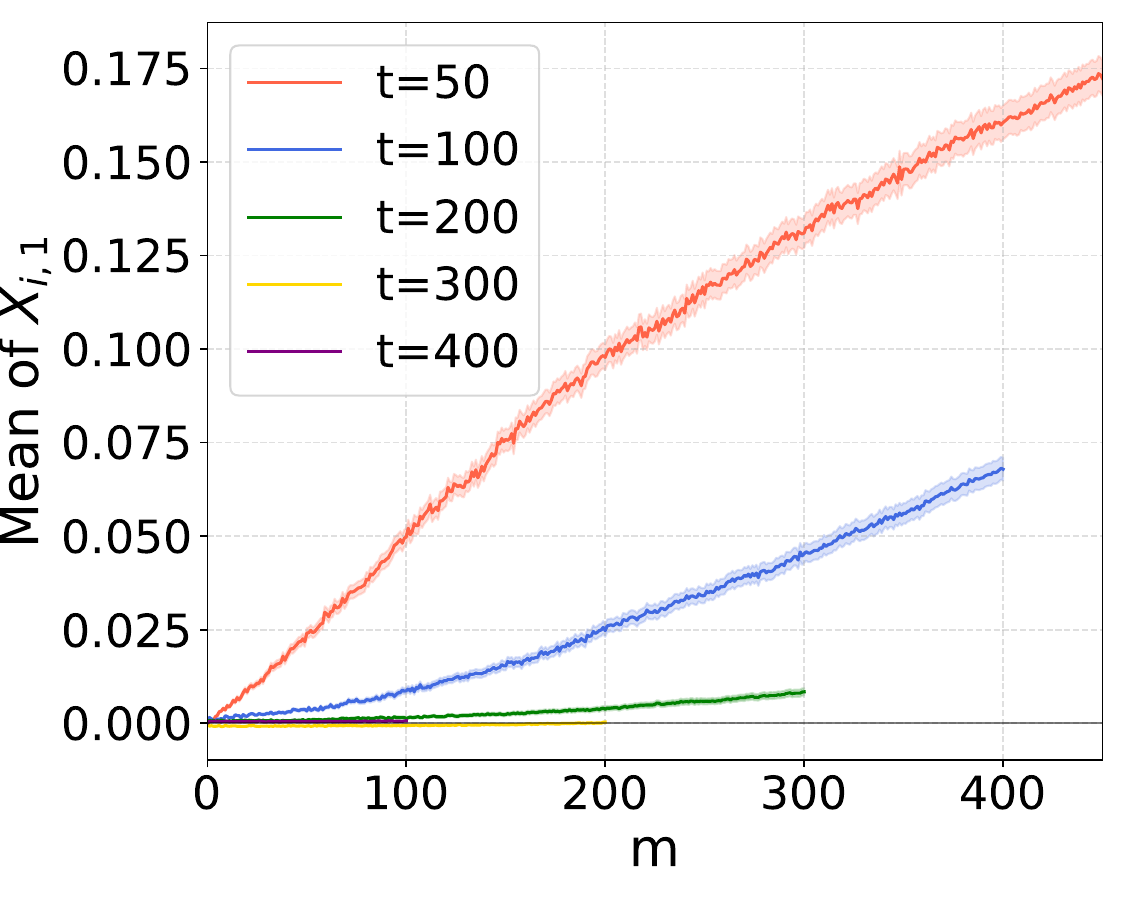}
        \caption{Arm 1}
    \end{subfigure}
    \begin{subfigure}[b]{0.32\textwidth}
        \centering
        \includegraphics[width=\textwidth]{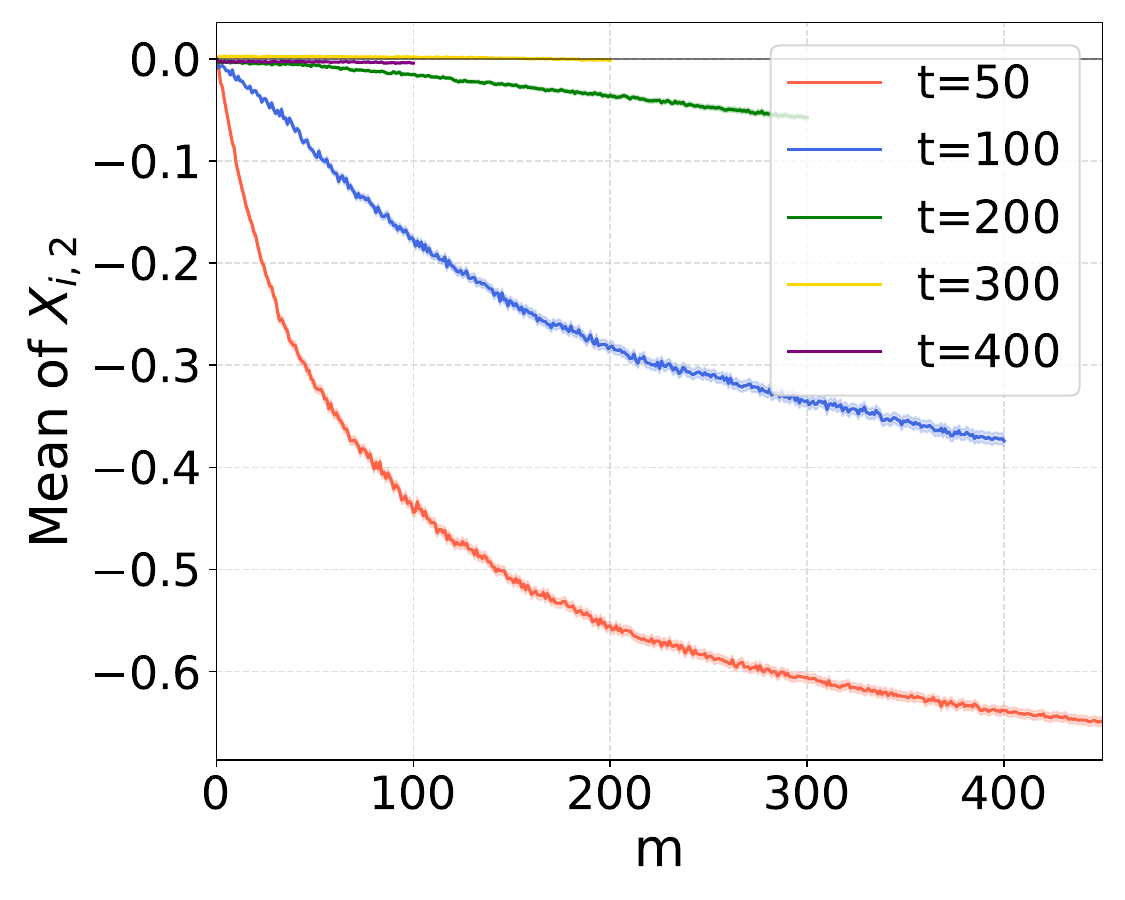}
        \caption{Arm 2}
    \end{subfigure}
    \\
    \begin{subfigure}[b]{0.32\textwidth}
        \centering
        \includegraphics[width=\textwidth]{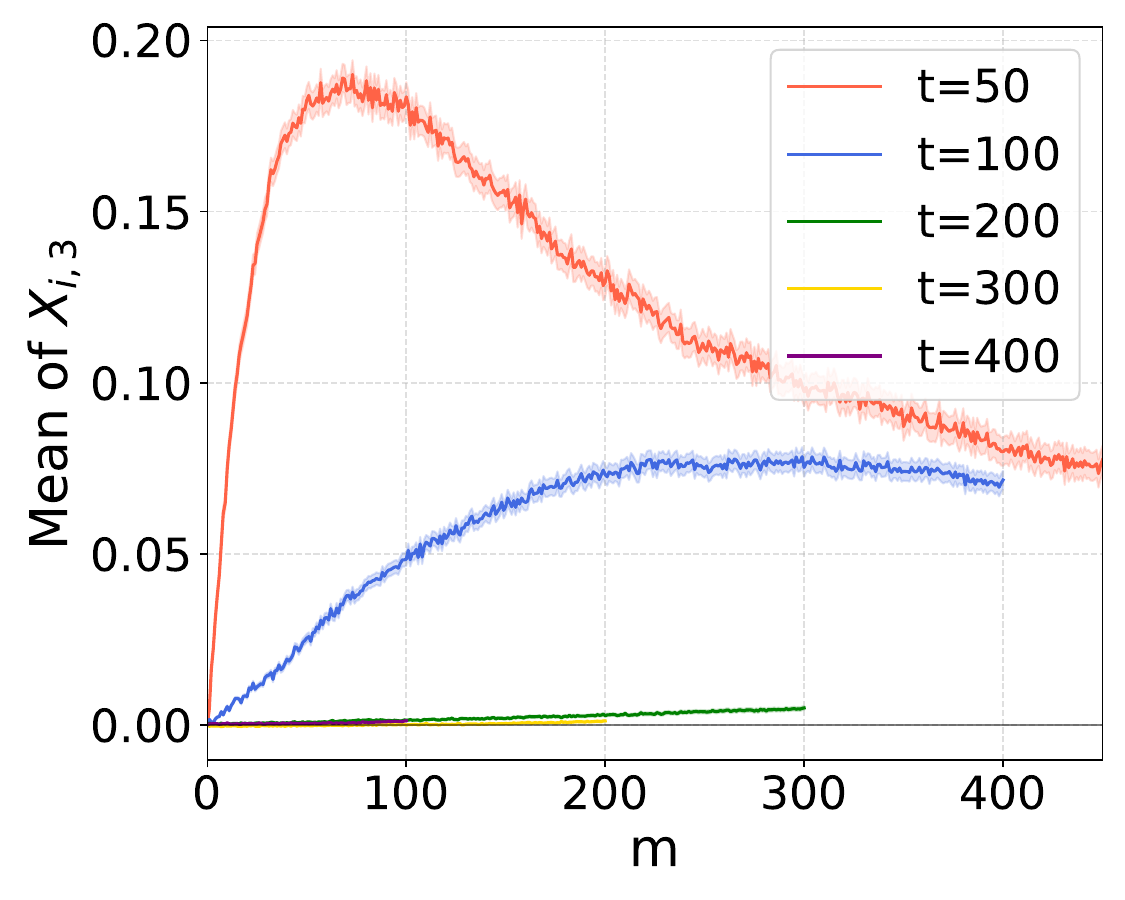}
        \caption{Arm 3}
    \end{subfigure}
    \begin{subfigure}[b]{0.32\textwidth}
        \centering
        \includegraphics[width=\textwidth]{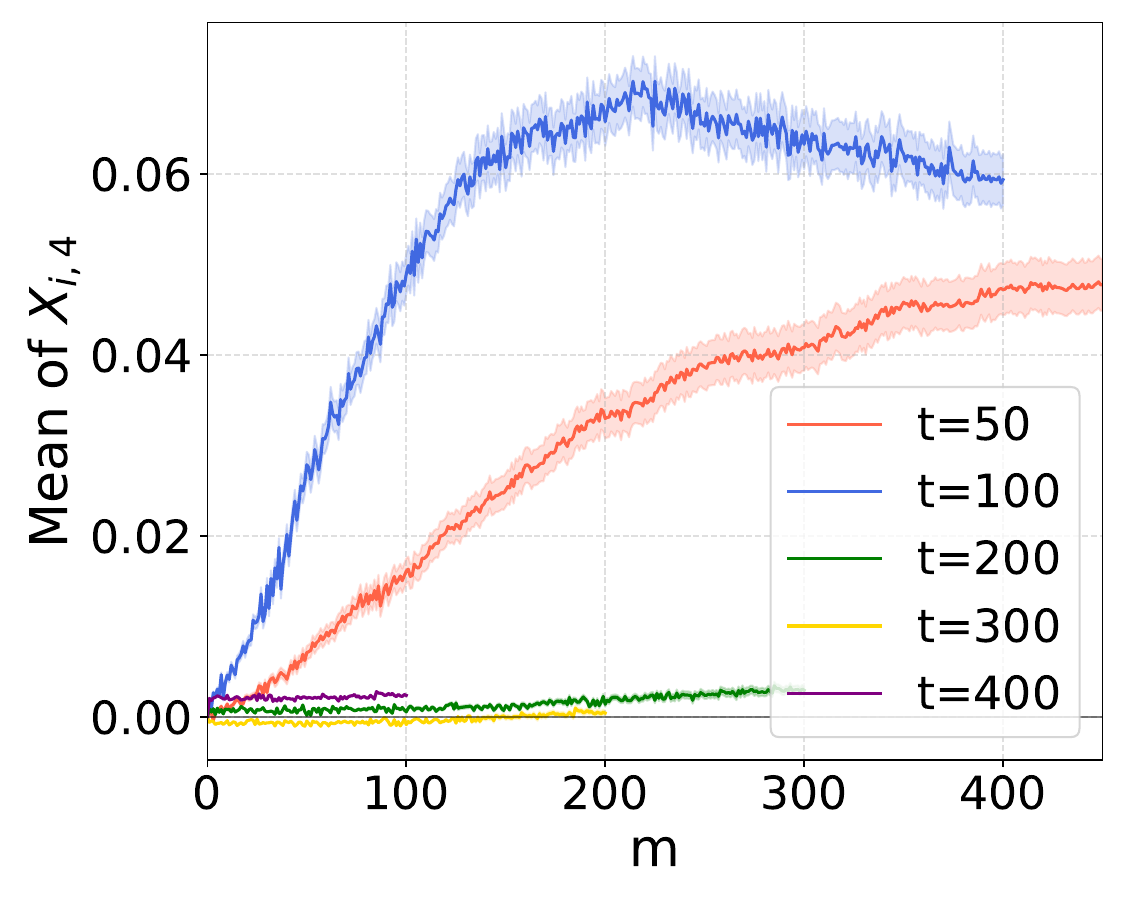}
        \caption{Arm 4}
    \end{subfigure}
    \caption{Per-arm martingale increments for DPT$'$ in Gaussian bandits with online starting contexts, where $m$ and $t$ denote the lengths of the martingale continuation and the starting contexts, respectively.}
    \label{fig_martingale_increment_online_gaussian_DPT_prime}
\end{figure}

\subsubsection{Bernoulli evaluation figures and martingale plots}
\begin{figure}[H]
\centering
\begin{subfigure}{0.245\textwidth}
\includegraphics[width=\textwidth]{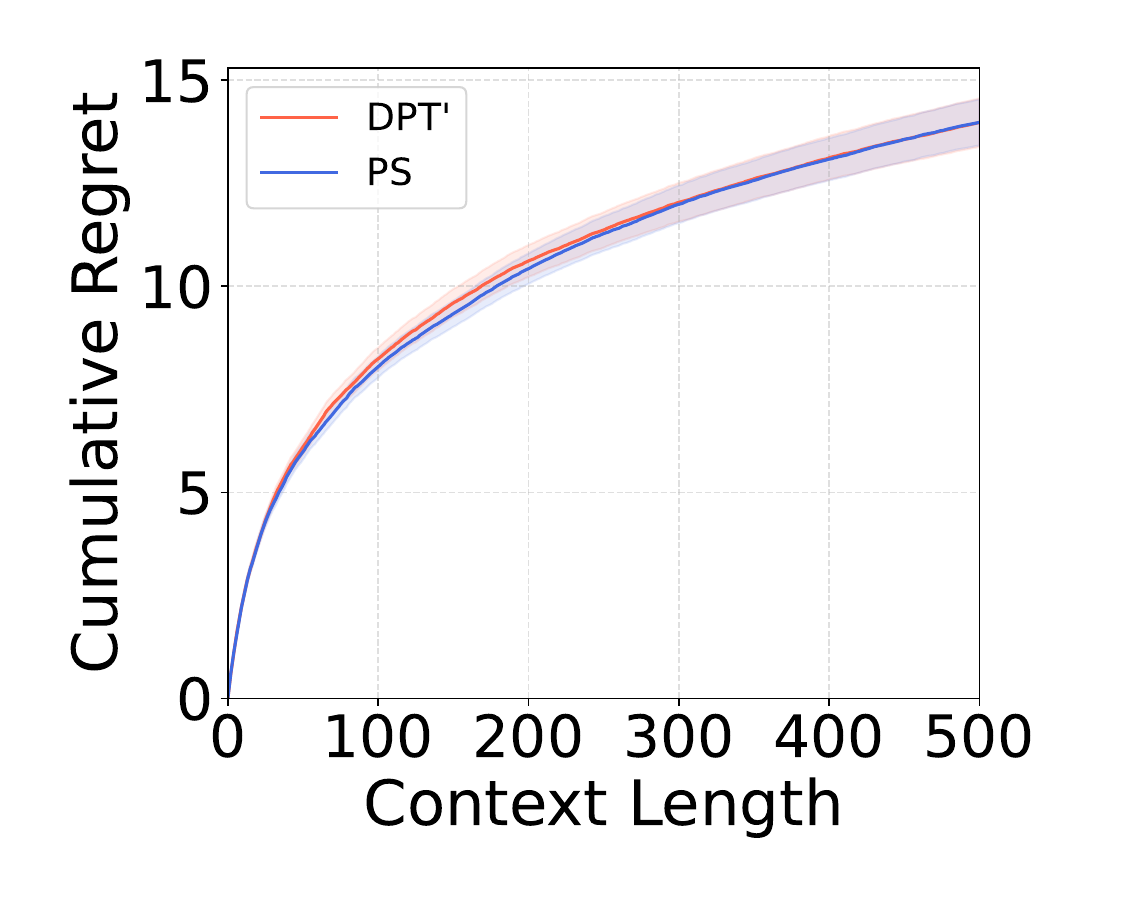}
\caption{Online regret}
\end{subfigure}
\hfill
\begin{subfigure}{0.245\textwidth}
\includegraphics[width=\textwidth]{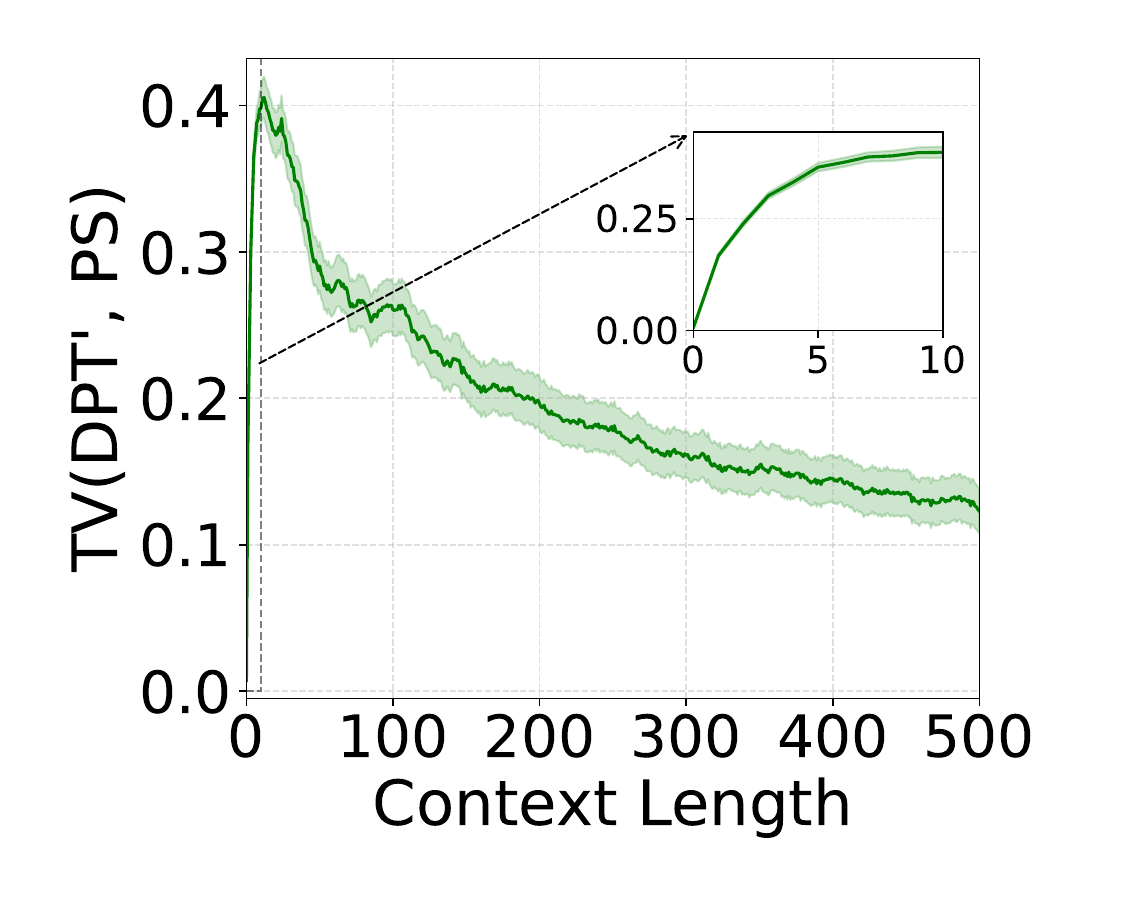}
\caption{Online TV distance}
\end{subfigure}
\hfill
\begin{subfigure}{0.245\textwidth}
\includegraphics[width=\textwidth]{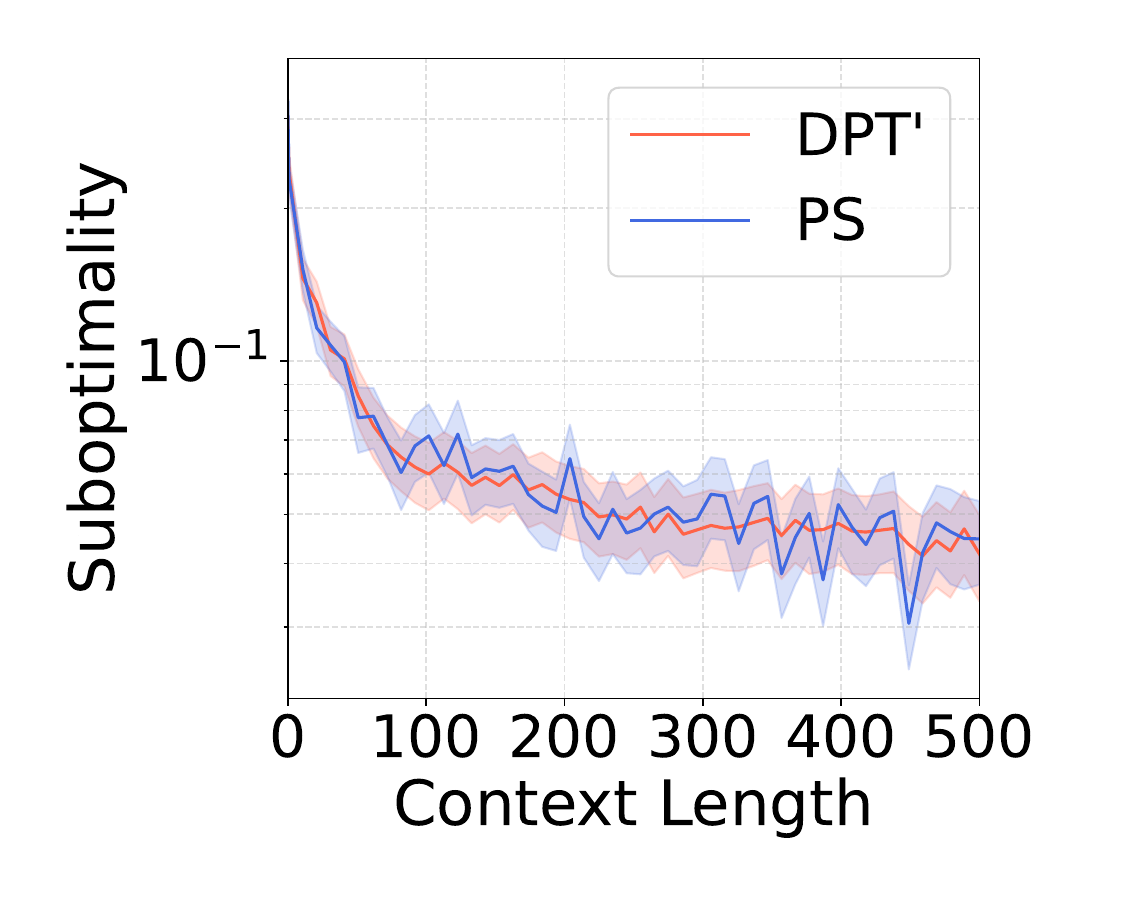}
\caption{Offline suboptimality}
\end{subfigure}
\hfill
\begin{subfigure}{0.245\textwidth}
\includegraphics[width=\textwidth]{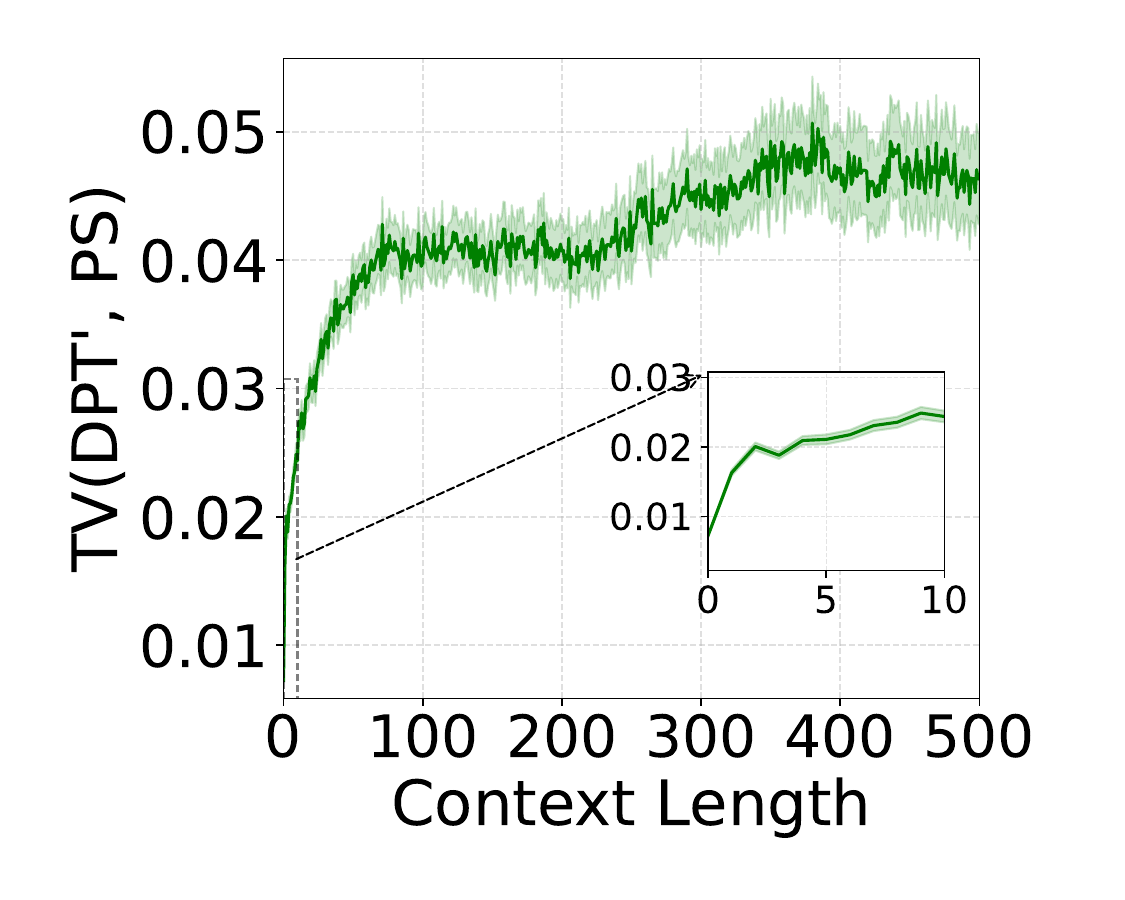}
\caption{Offline TV distance}
\end{subfigure}
\caption{Performance comparison of DPT$'$ and Thompson Sampling in Bernoulli bandits under online and offline deployment.}
\label{fig_evaluation_bernoulli_bandit_DPT_prime}
\end{figure}

\begin{figure}[H]
    \centering
    \begin{subfigure}[b]{0.32\textwidth}
        \centering
        \includegraphics[width=\textwidth]{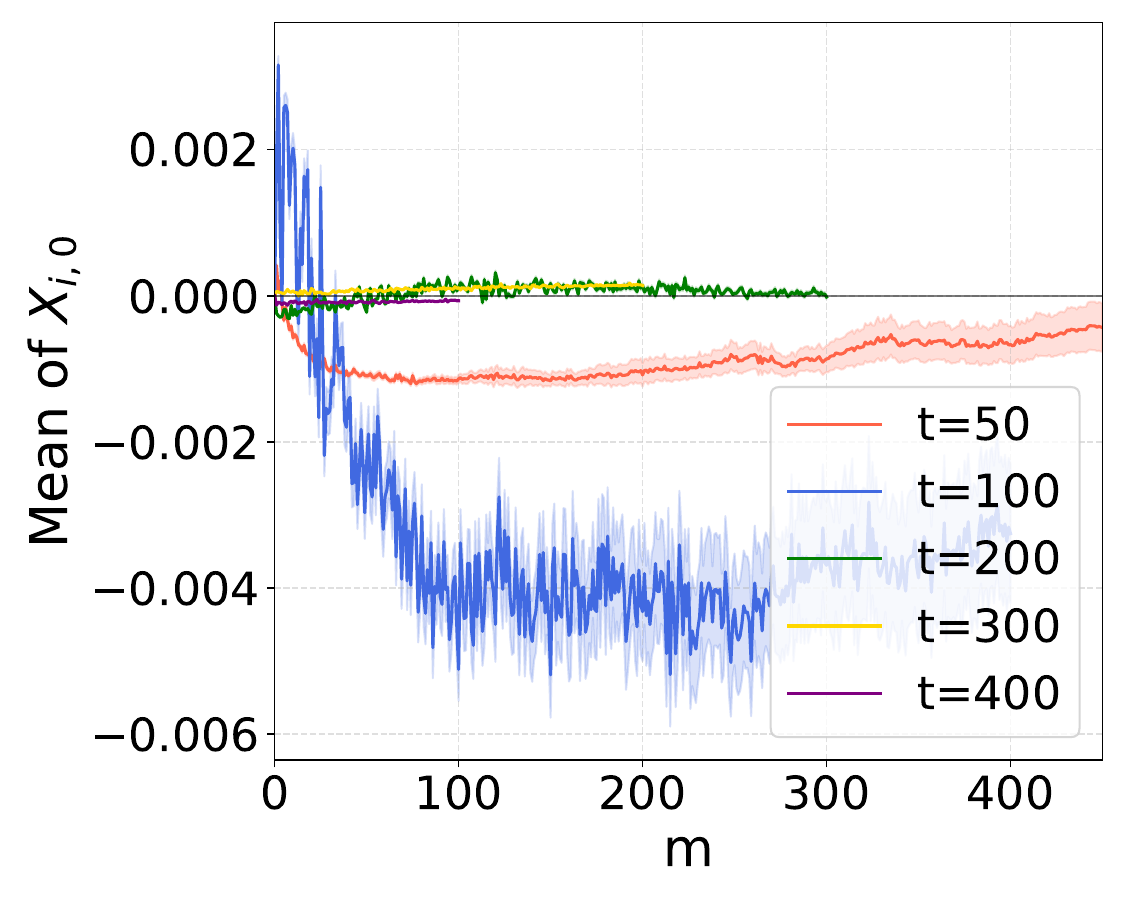}
        \caption{Arm 0}
    \end{subfigure}
    \begin{subfigure}[b]{0.32\textwidth}
        \centering
        \includegraphics[width=\textwidth]{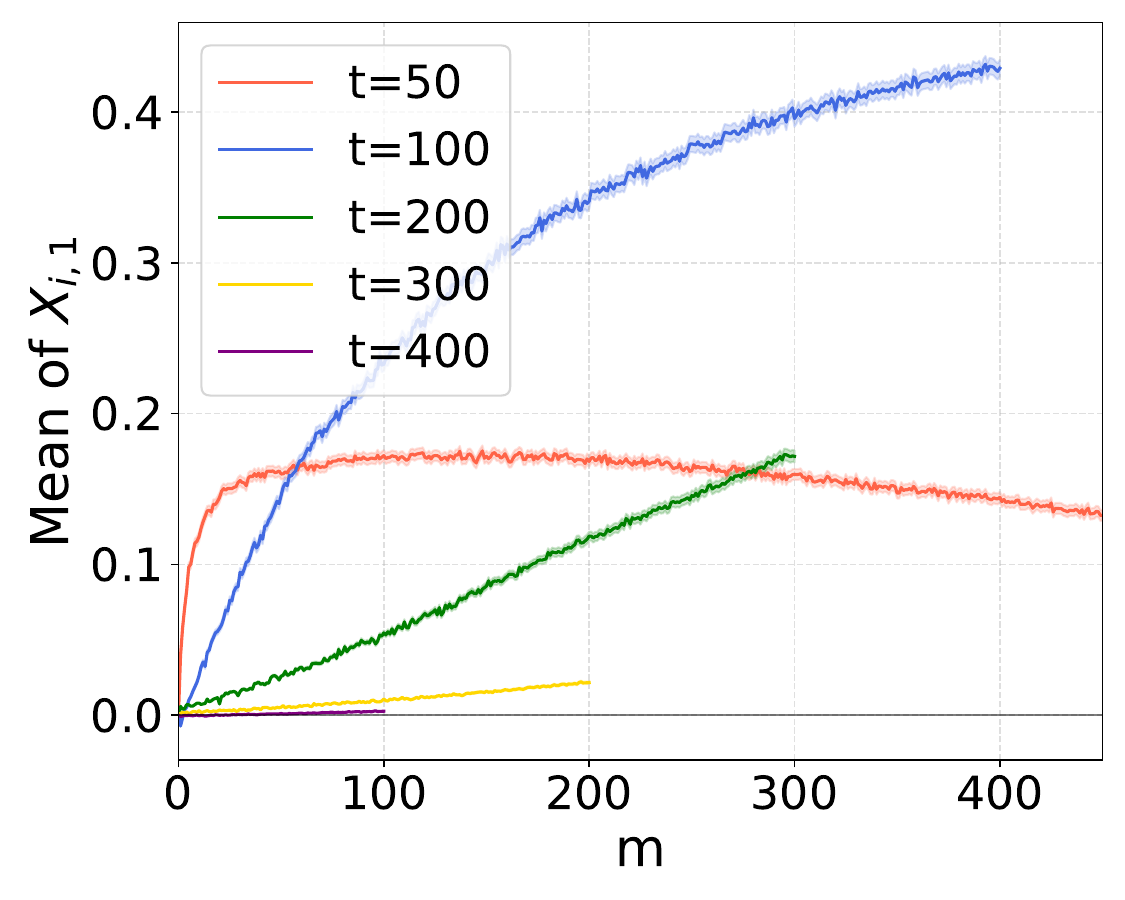}
        \caption{Arm 1}
    \end{subfigure}
    \begin{subfigure}[b]{0.32\textwidth}
        \centering
        \includegraphics[width=\textwidth]{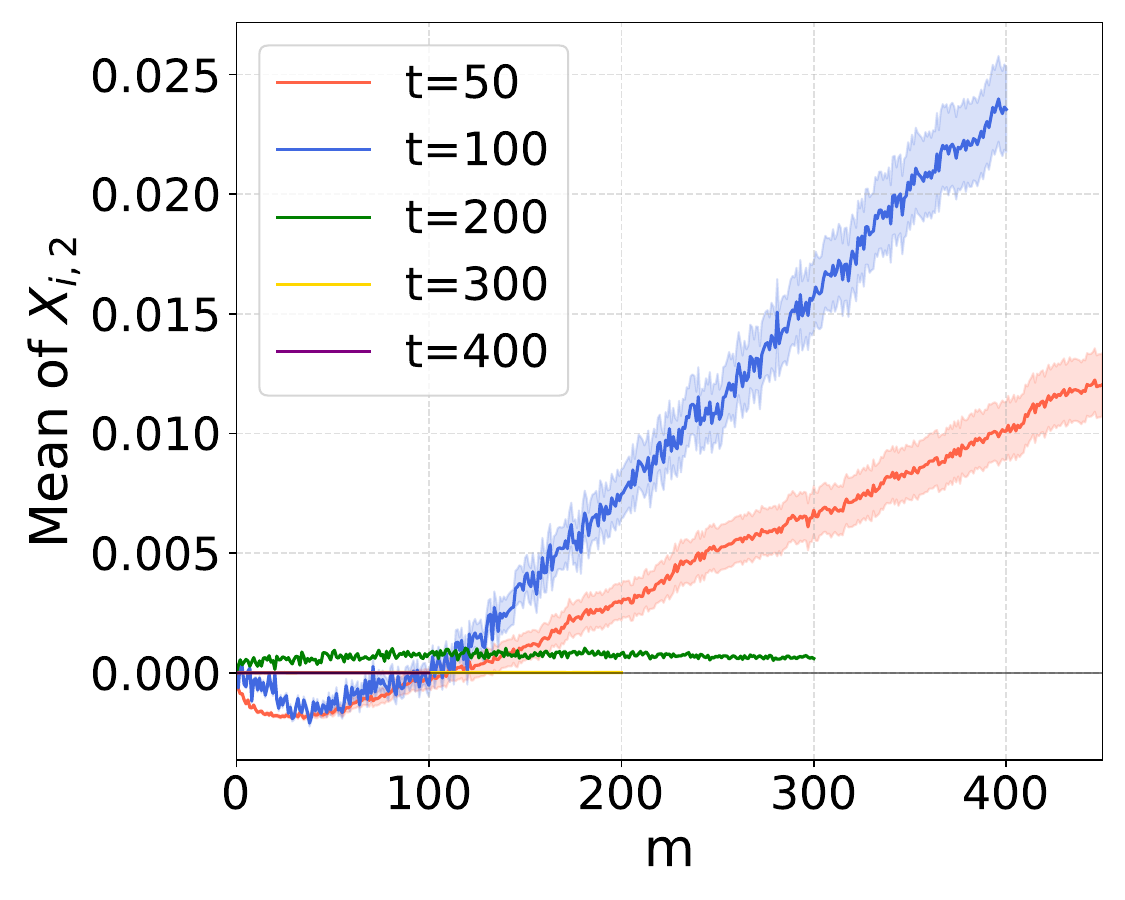}
        \caption{Arm 2}
    \end{subfigure}
    \\
    \begin{subfigure}[b]{0.32\textwidth}
        \centering
        \includegraphics[width=\textwidth]{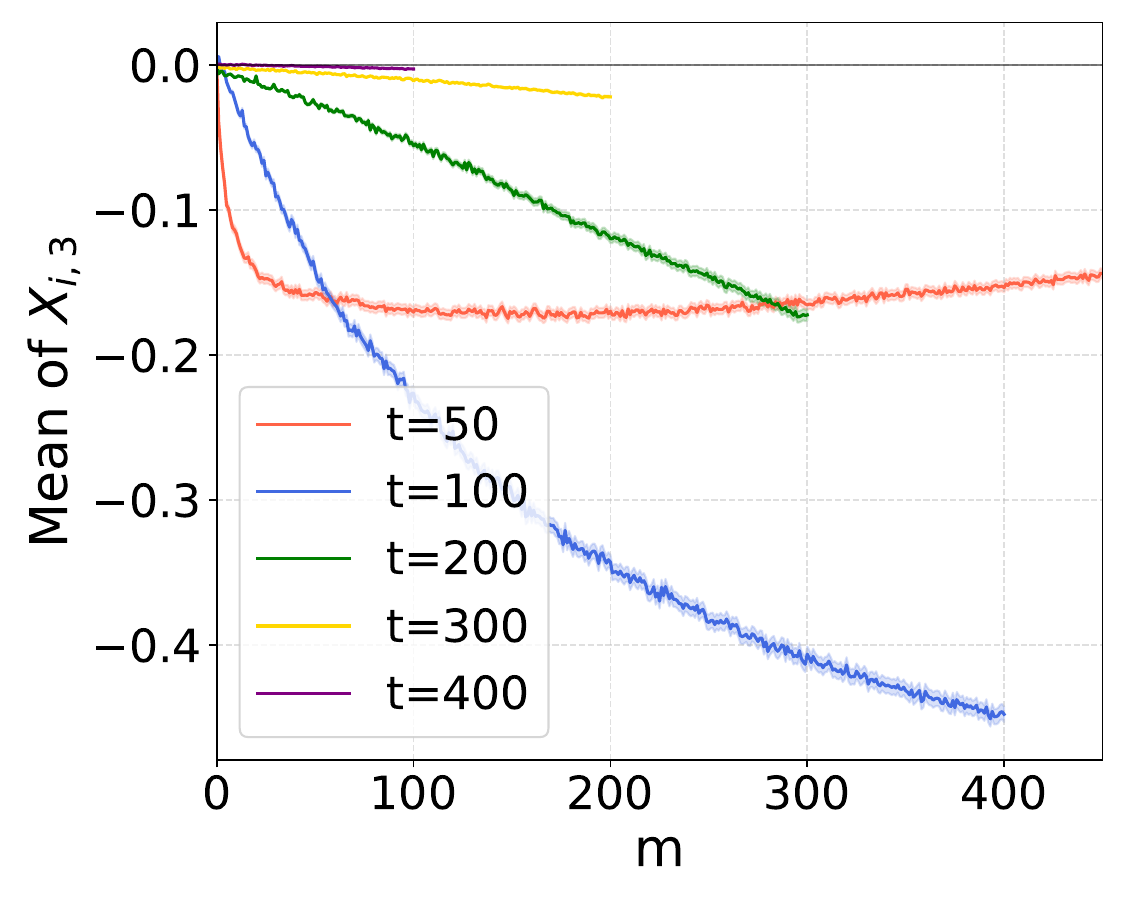}
        \caption{Arm 3}
    \end{subfigure}
    \begin{subfigure}[b]{0.32\textwidth}
        \centering
        \includegraphics[width=\textwidth]{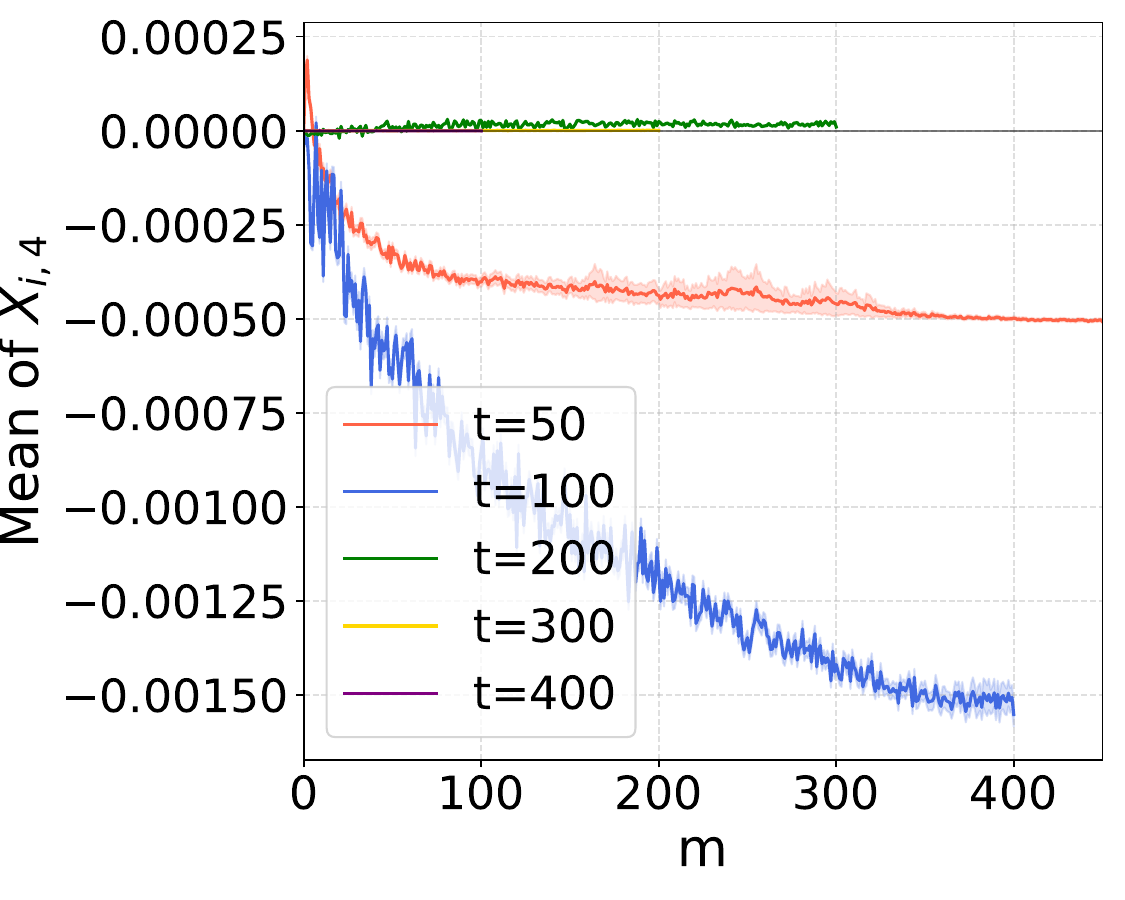}
        \caption{Arm 4}
    \end{subfigure}
    \caption{Per-arm martingale increments for DPT$'$ in Bernoulli bandits with offline starting contexts, where $m$ and $t$ denote the lengths of the martingale continuation and the starting contexts, respectively.}
    \label{fig_martingale_increment_offline_bernoulli_DPT_prime}
\end{figure}

\begin{figure}[H]
    \centering
    \begin{subfigure}[b]{0.32\textwidth}
        \centering
        \includegraphics[width=\textwidth]{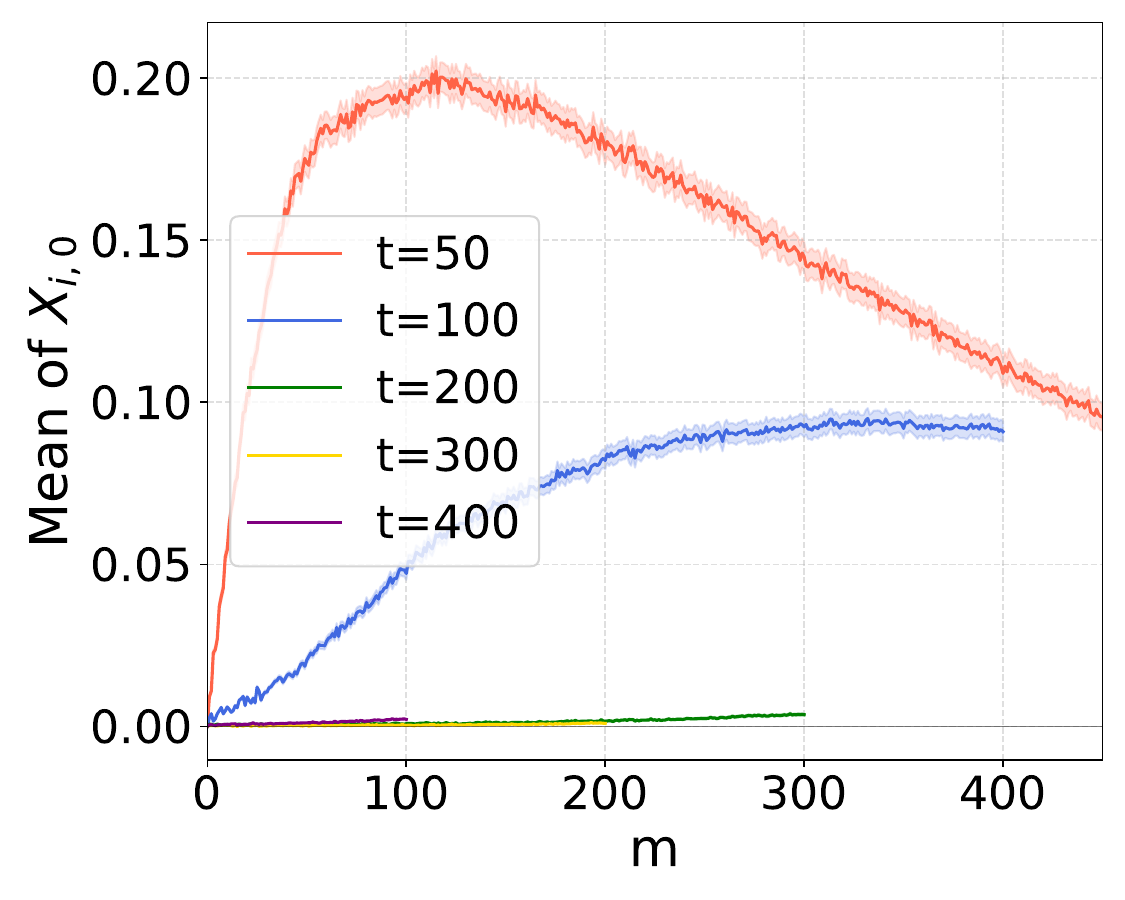}
        \caption{Arm 0}
    \end{subfigure}
    \begin{subfigure}[b]{0.32\textwidth}
        \centering
        \includegraphics[width=\textwidth]{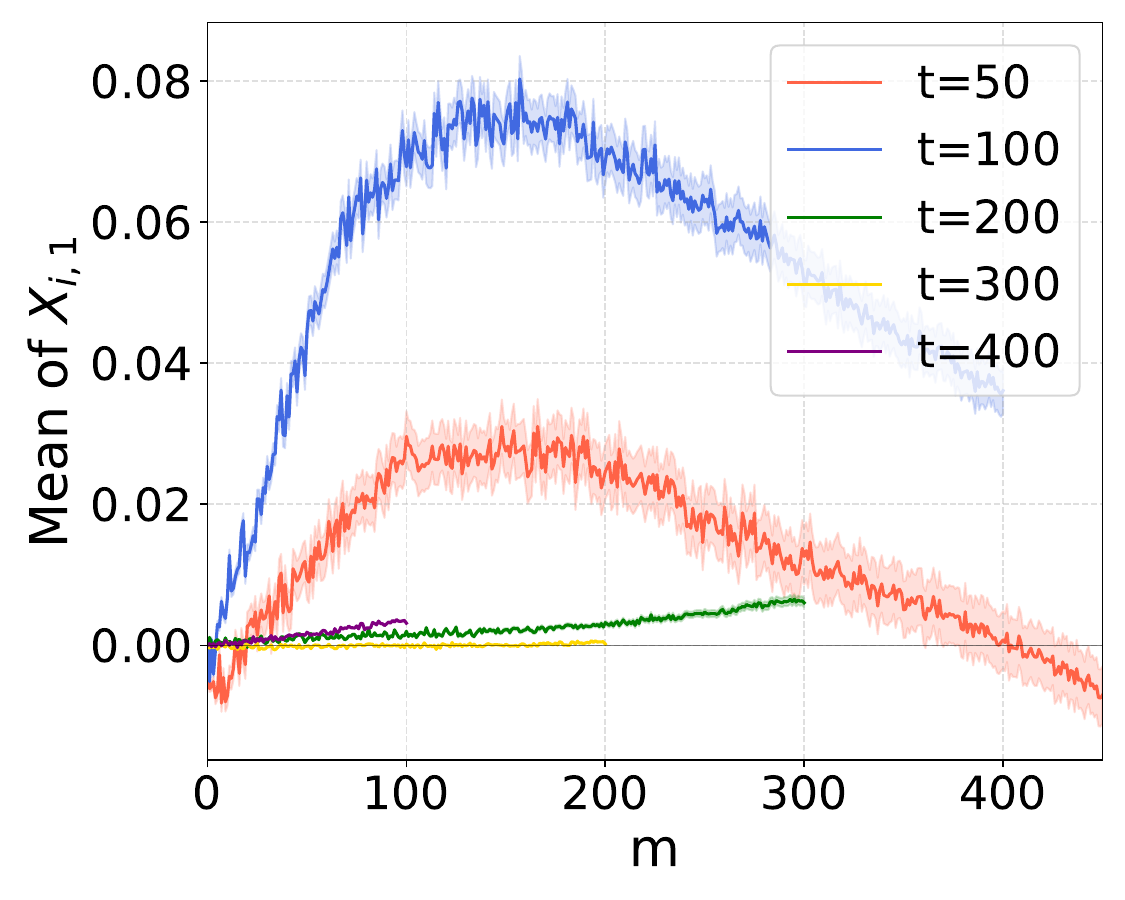}
        \caption{Arm 1}
    \end{subfigure}
    \begin{subfigure}[b]{0.32\textwidth}
        \centering
        \includegraphics[width=\textwidth]{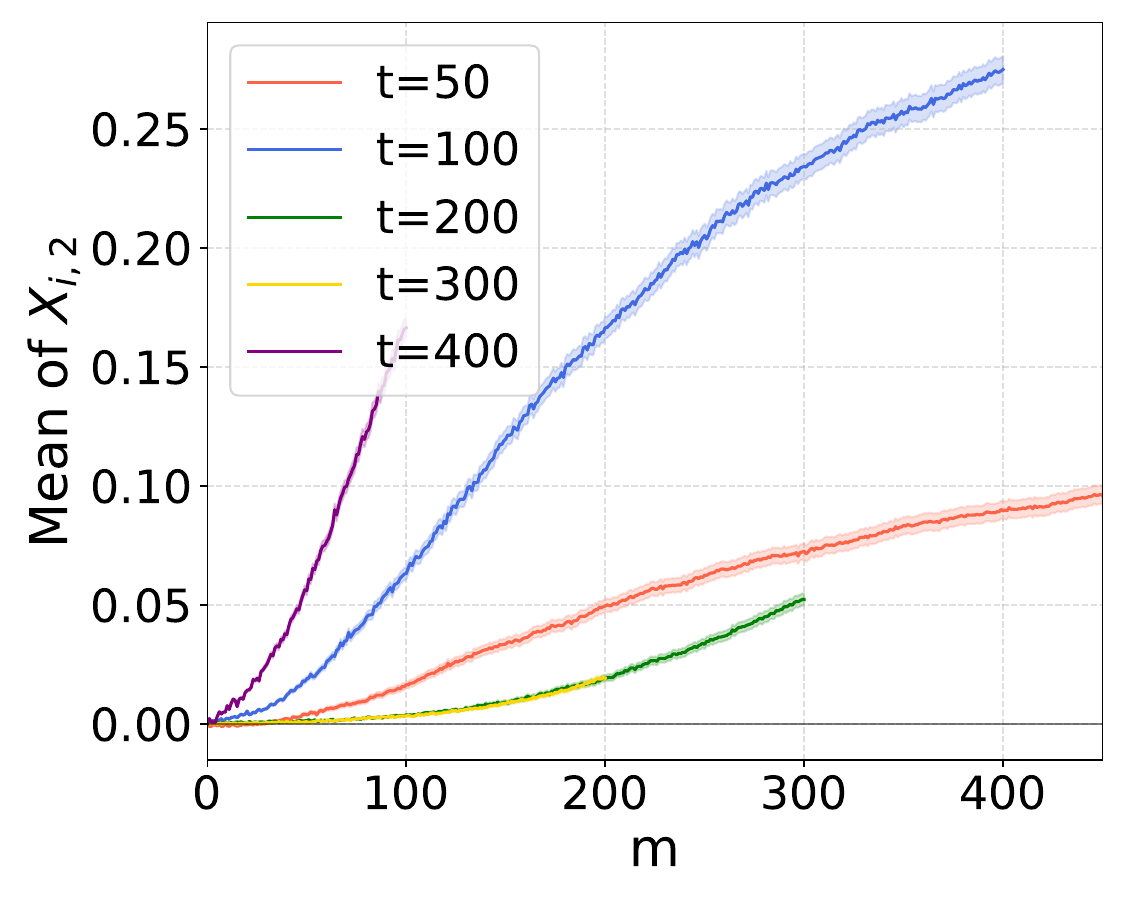}
        \caption{Arm 2}
    \end{subfigure}
    \\
    \begin{subfigure}[b]{0.32\textwidth}
        \centering
        \includegraphics[width=\textwidth]{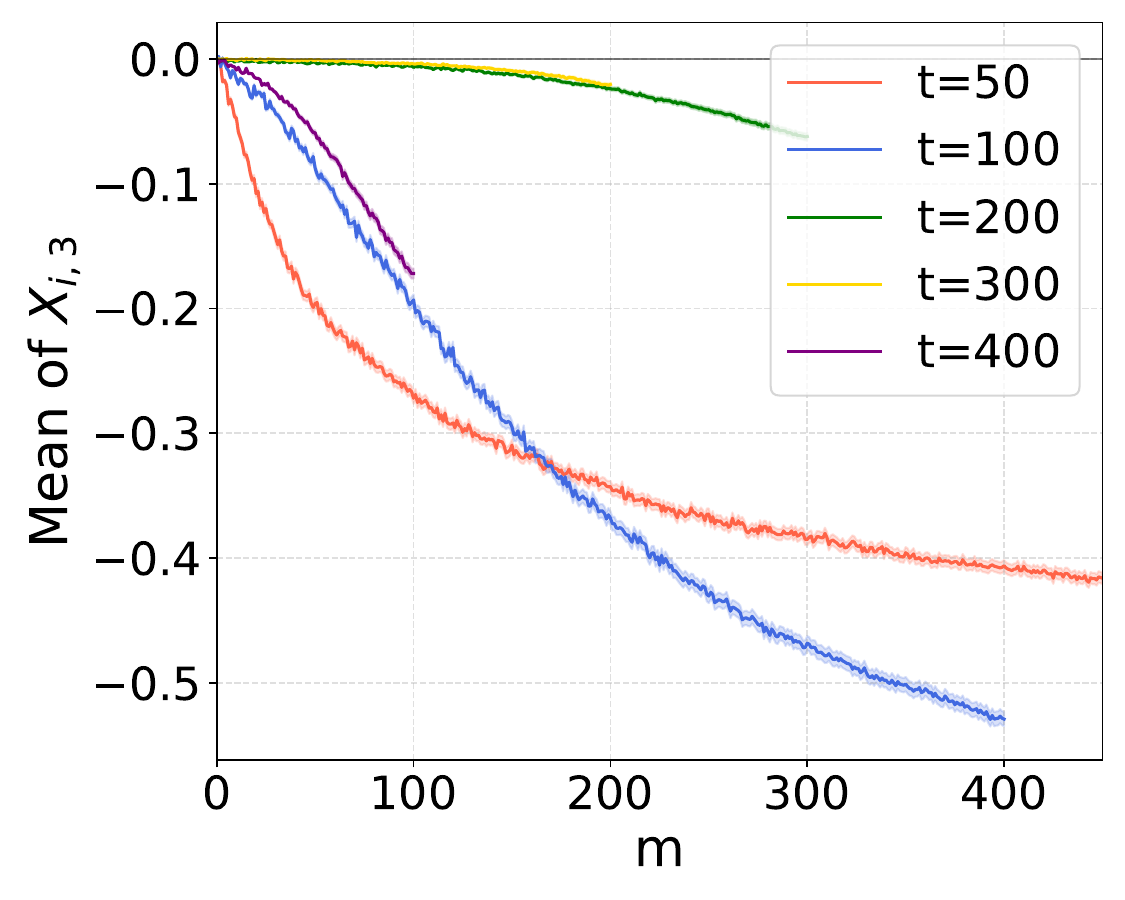}
        \caption{Arm 3}
    \end{subfigure}
    \begin{subfigure}[b]{0.32\textwidth}
        \centering
        \includegraphics[width=\textwidth]{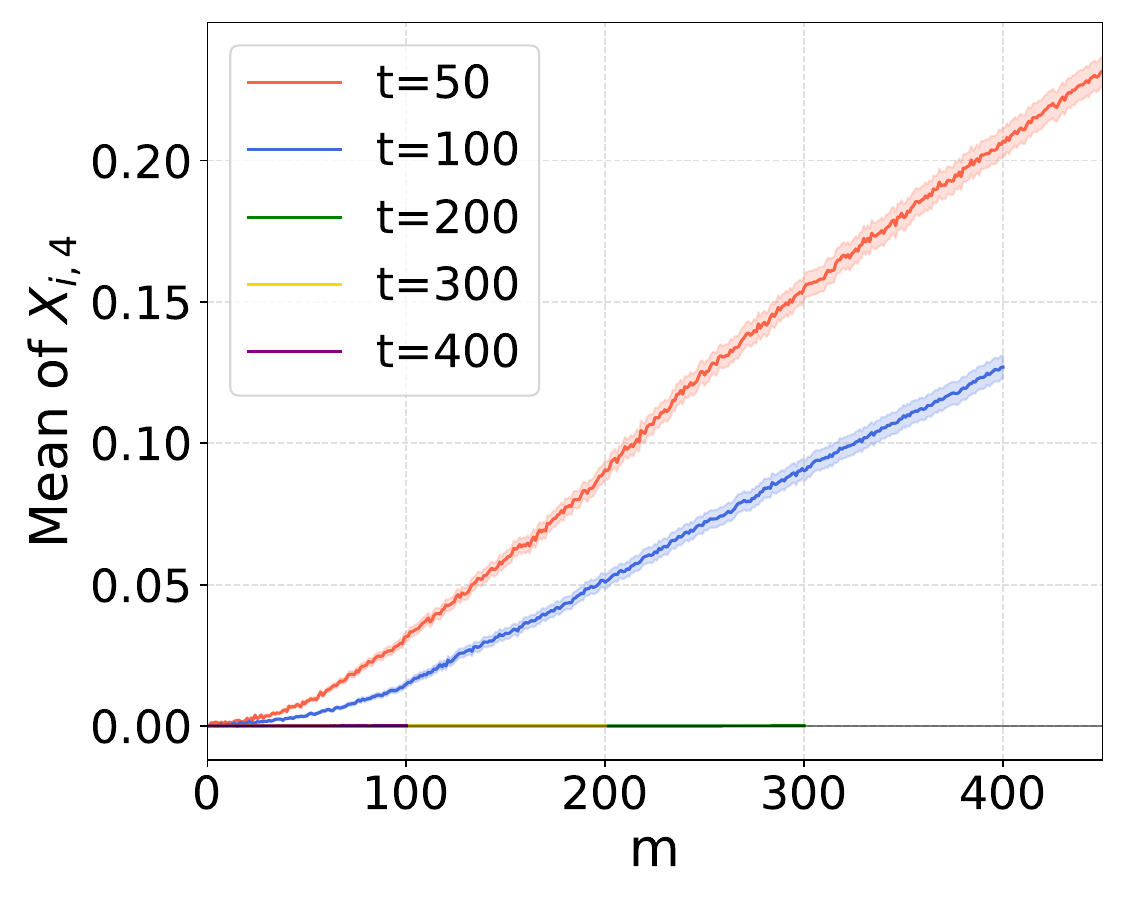}
        \caption{Arm 4}
    \end{subfigure}
    \caption{Per-arm martingale increments for DPT$'$ in Bernoulli bandits with online starting contexts, where $m$ and $t$ denote the lengths of the martingale continuation and the starting contexts, respectively.}
    \label{fig_martingale_increment_online_bernoulli_DPT_prime}
\end{figure}

\subsubsection{Linear evaluation figures and martingale plots}
\begin{figure}[H]
\centering
\begin{subfigure}{0.245\textwidth}
\includegraphics[width=\textwidth]{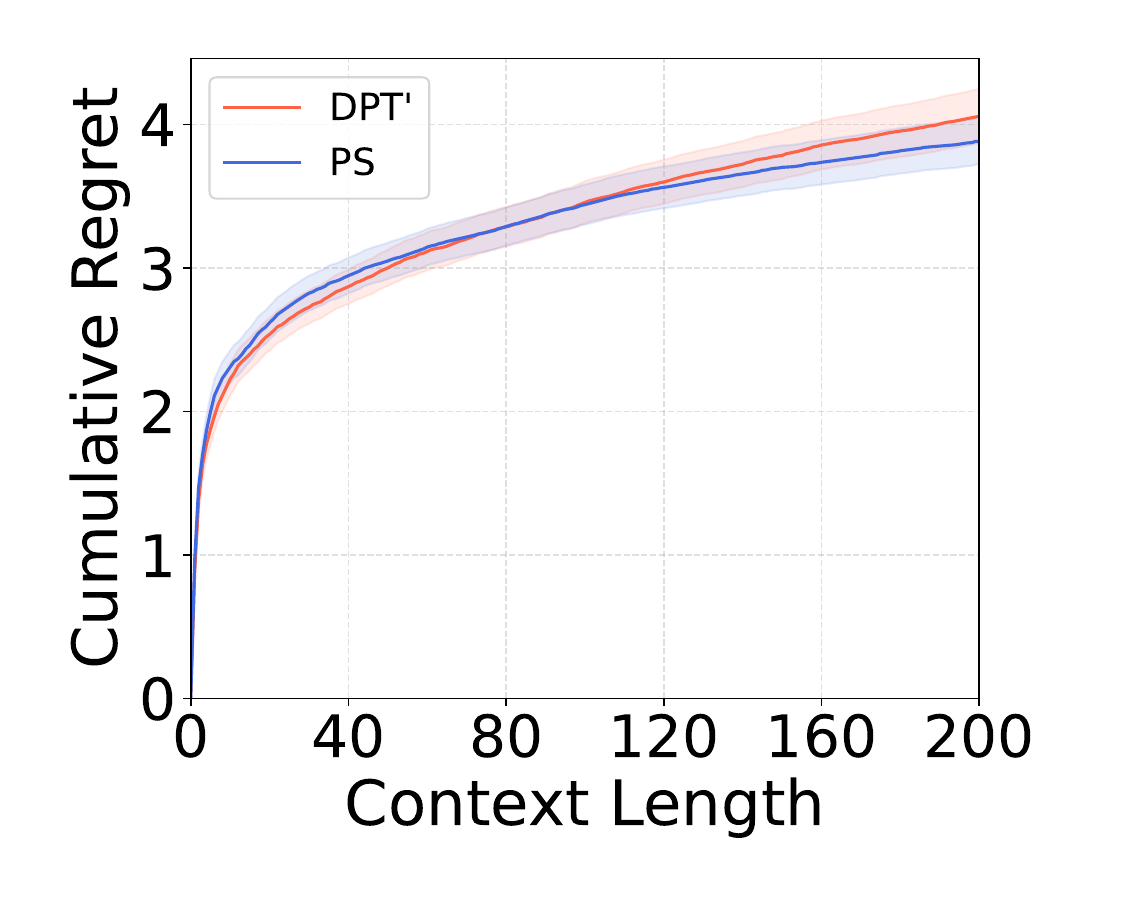}
\caption{Online regret}
\end{subfigure}
\hfill
\begin{subfigure}{0.245\textwidth}
\includegraphics[width=\textwidth]{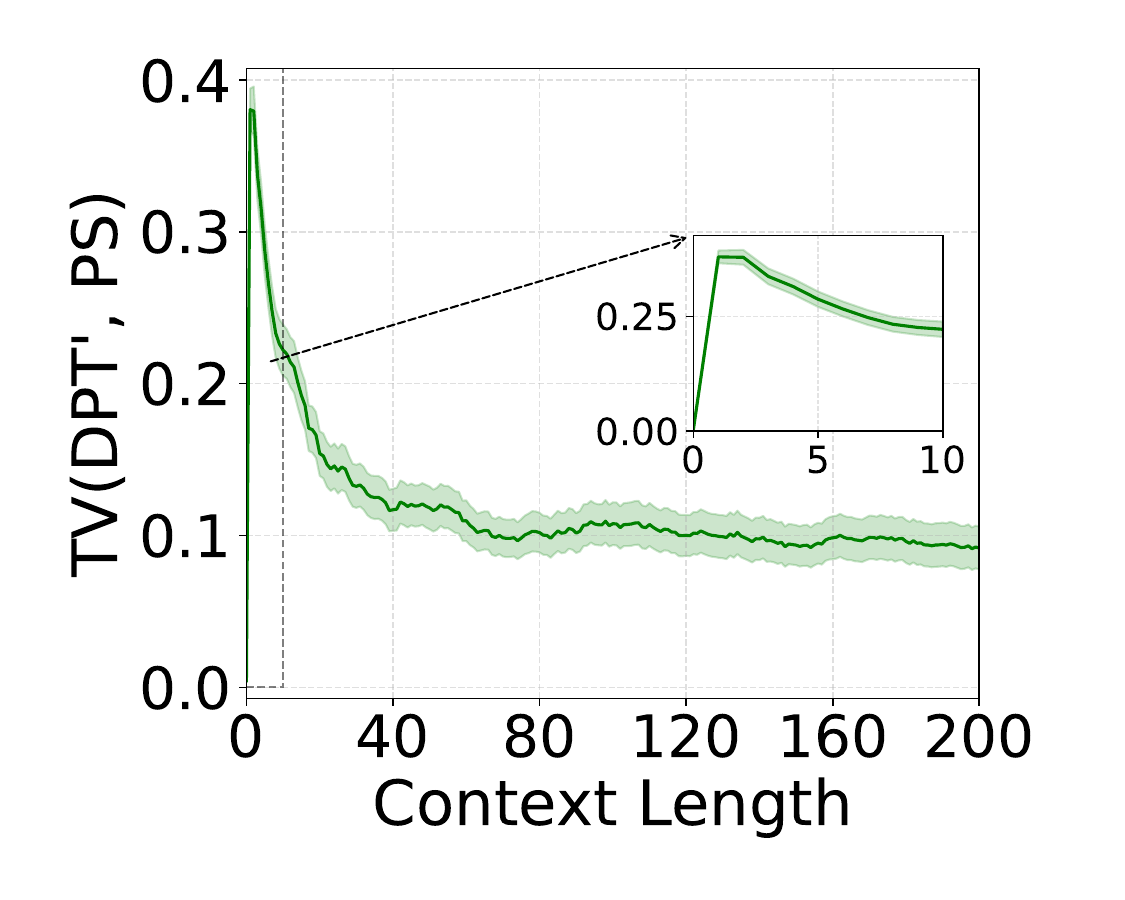}
\caption{Online TV distance}
\end{subfigure}
\hfill
\begin{subfigure}{0.245\textwidth}
\includegraphics[width=\textwidth]{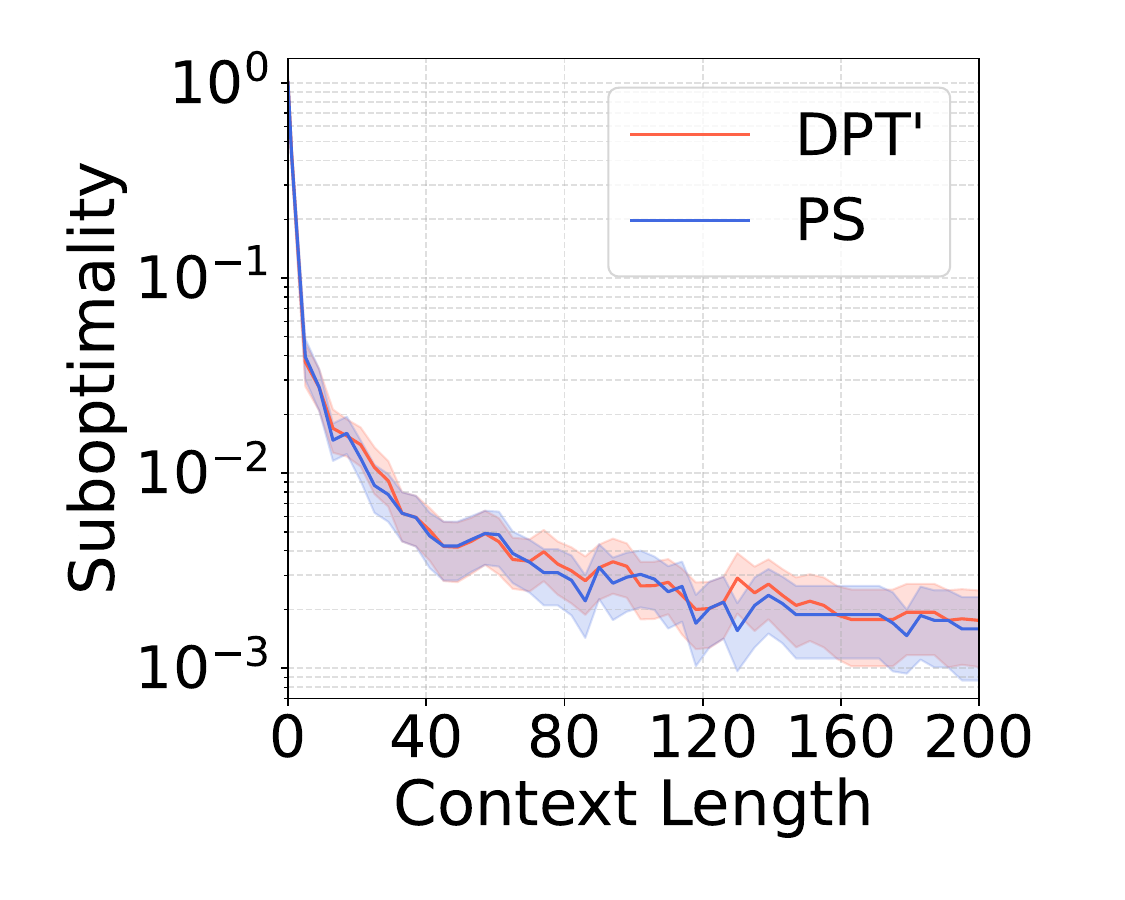}
\caption{Offline suboptimality}
\end{subfigure}
\hfill
\begin{subfigure}{0.245\textwidth}
\includegraphics[width=\textwidth]{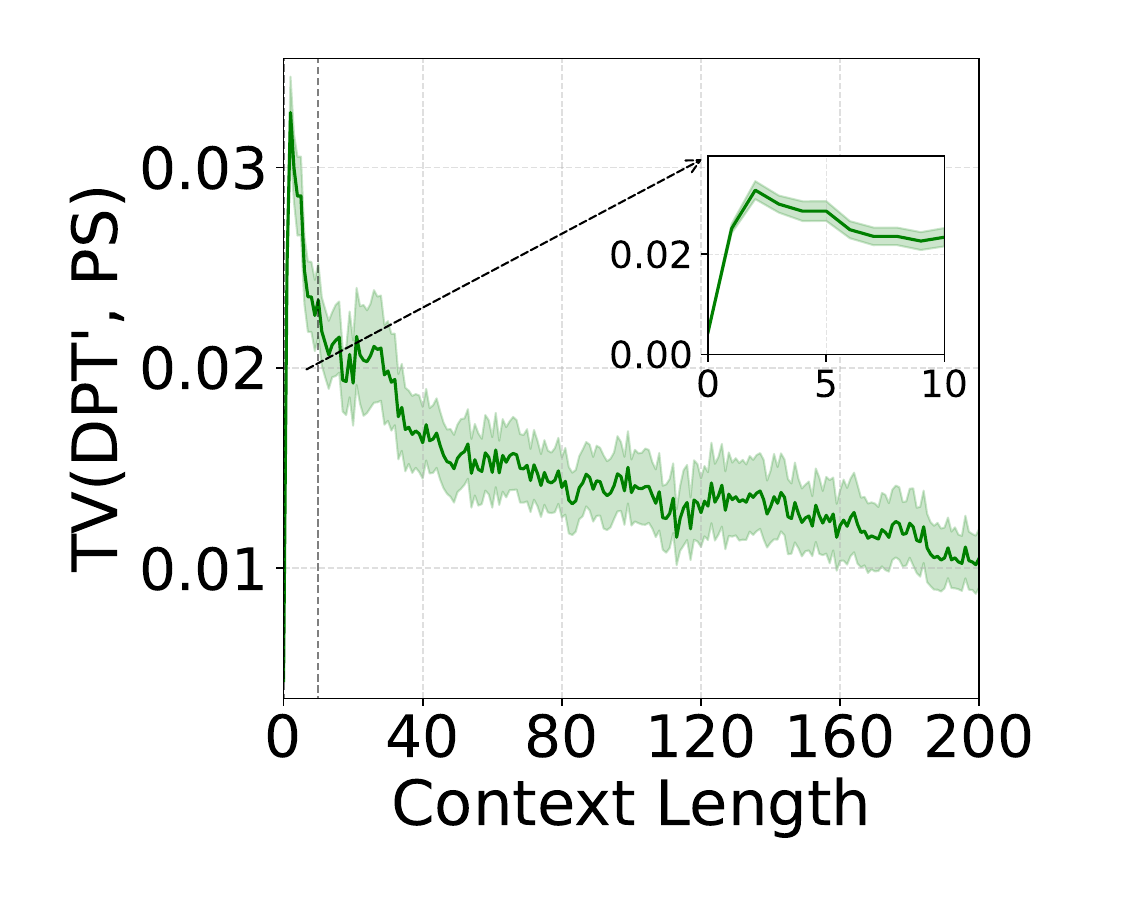}
\caption{Offline TV distance}
\end{subfigure}
\caption{Performance comparison of DPT$'$ and Thompson Sampling in linear bandits under online and offline deployment.}
\label{fig_evaluation_linear_bandit_DPT_prime}
\end{figure}

\begin{figure}[H]
    \centering
    \begin{subfigure}[b]{0.32\textwidth}
        \centering
        \includegraphics[width=\textwidth]{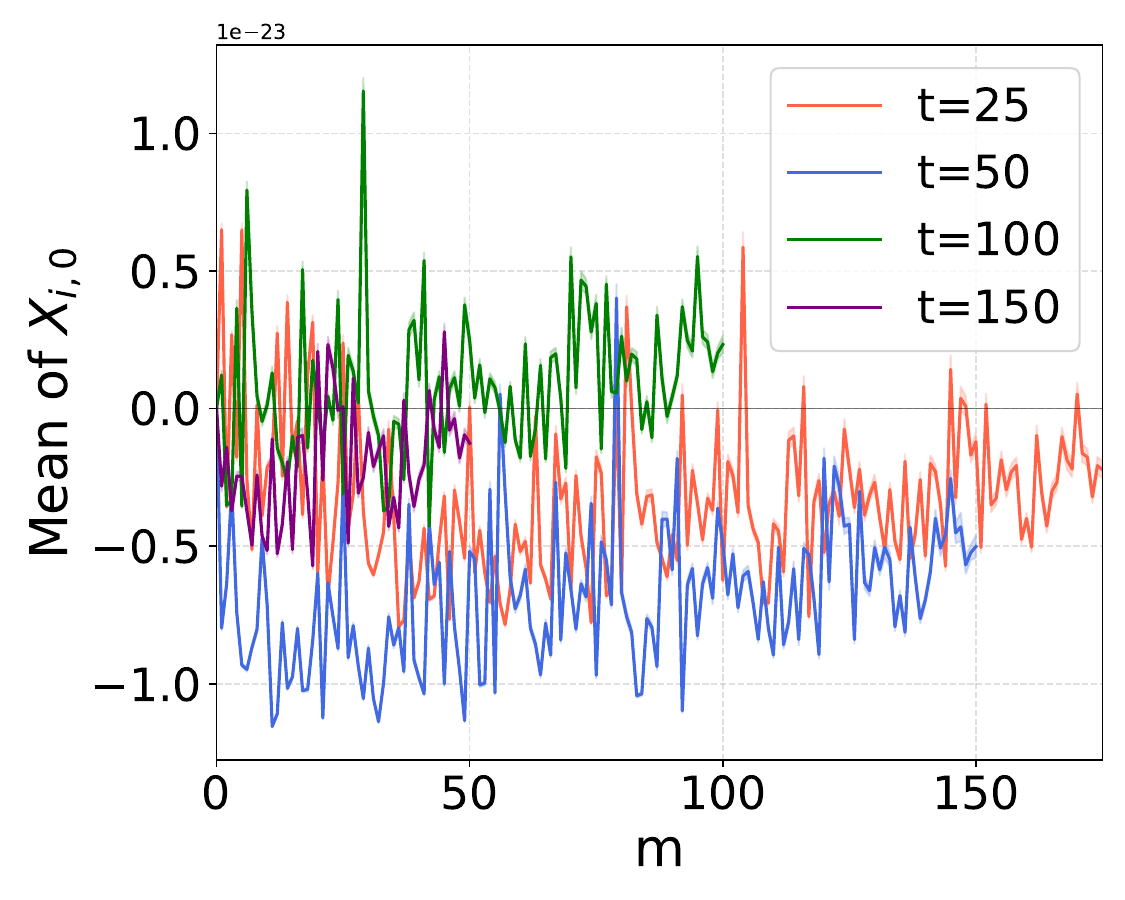}
        \caption{Arm 0}
    \end{subfigure}
    \begin{subfigure}[b]{0.32\textwidth}
        \centering
        \includegraphics[width=\textwidth]{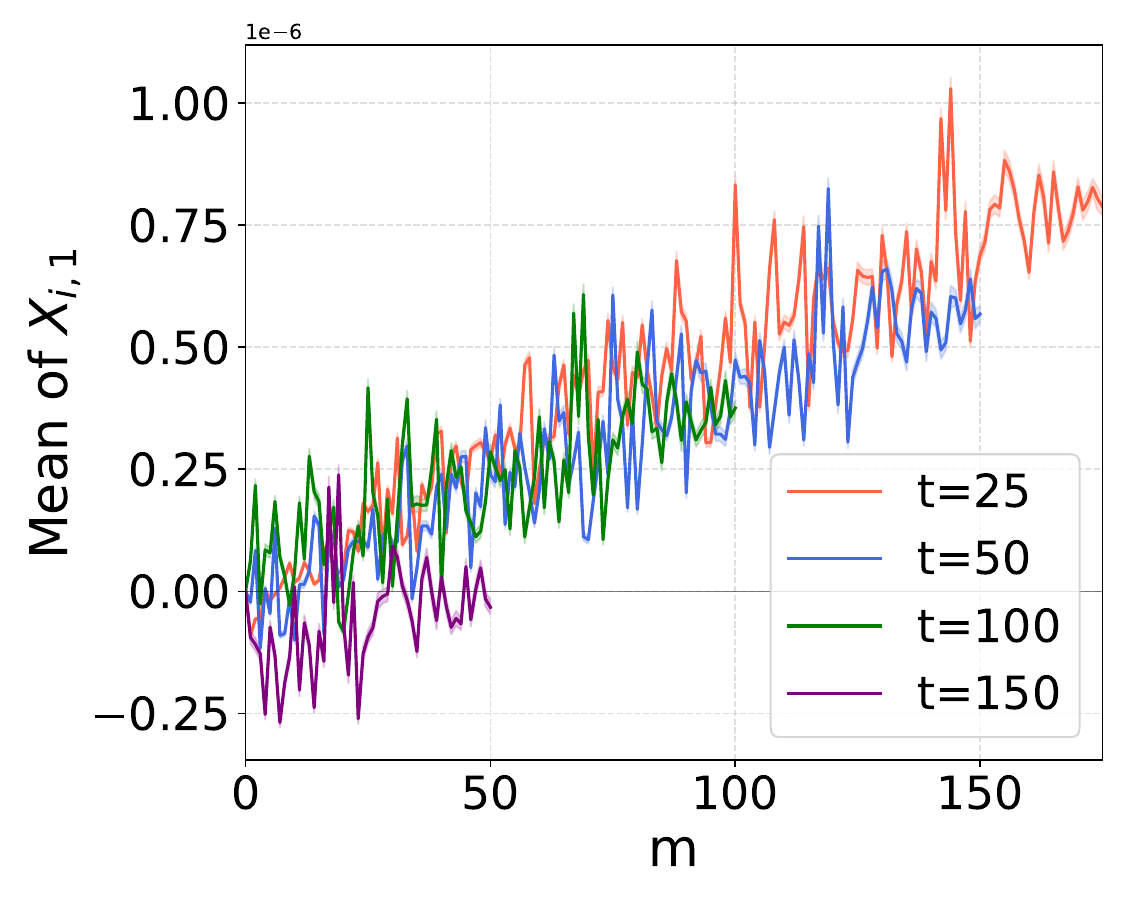}
        \caption{Arm 1}
    \end{subfigure}
    \begin{subfigure}[b]{0.32\textwidth}
        \centering
        \includegraphics[width=\textwidth]{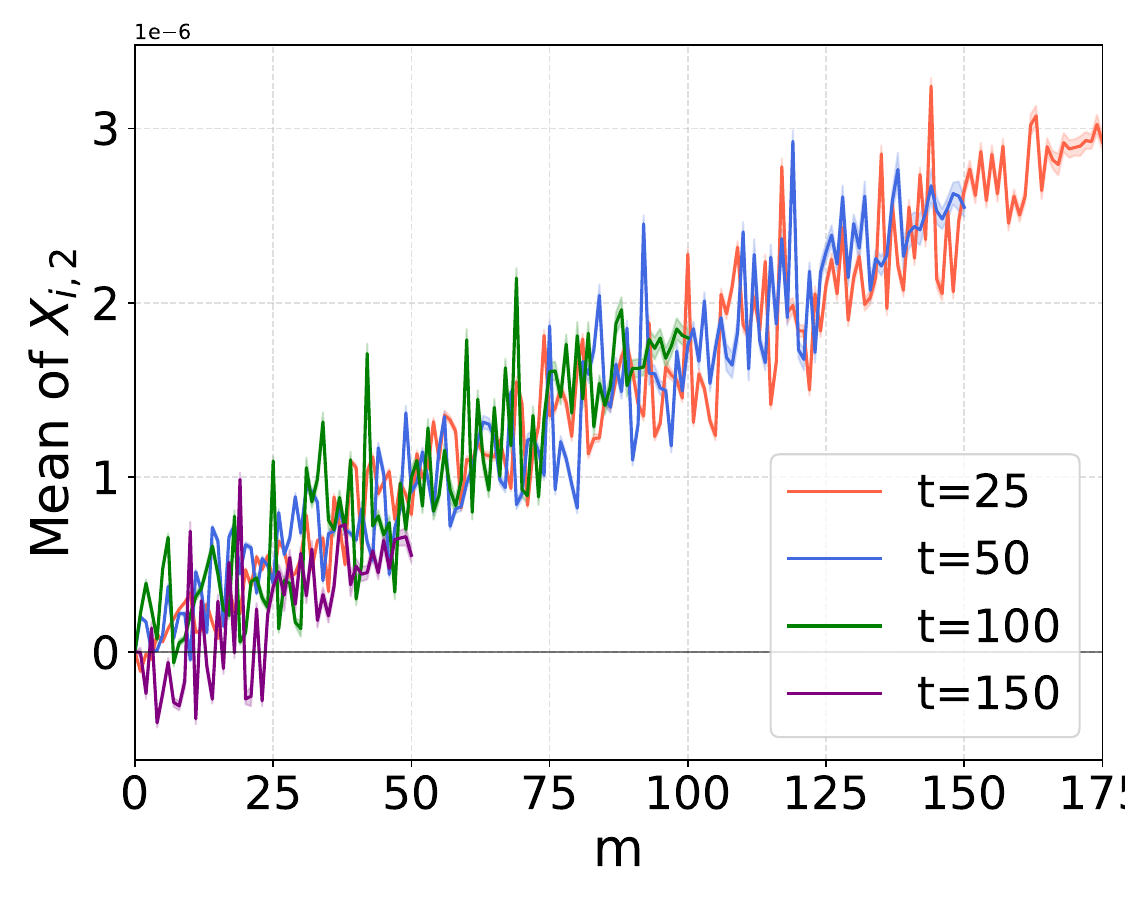}
        \caption{Arm 2}
    \end{subfigure}
    \\
    \begin{subfigure}[b]{0.32\textwidth}
        \centering
        \includegraphics[width=\textwidth]{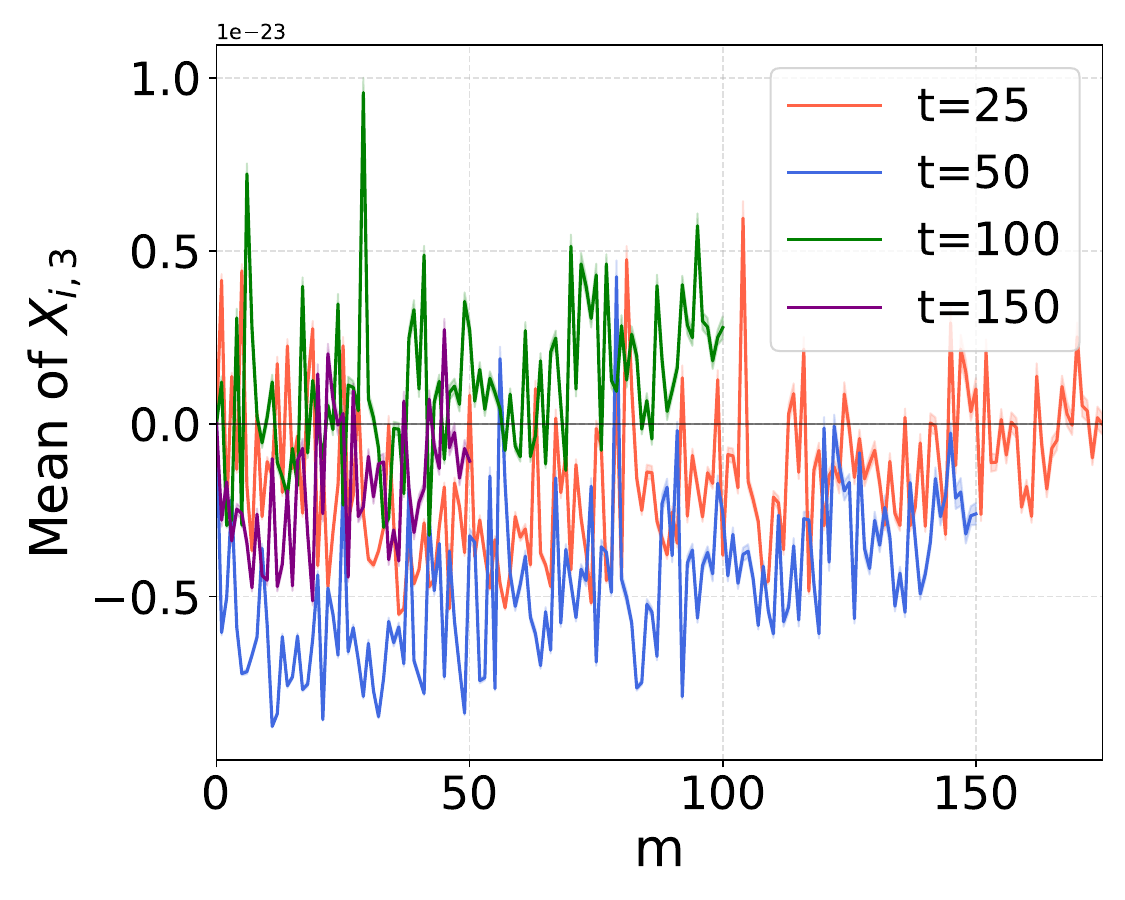}
        \caption{Arm 3}
    \end{subfigure}
    \begin{subfigure}[b]{0.32\textwidth}
        \centering
        \includegraphics[width=\textwidth]{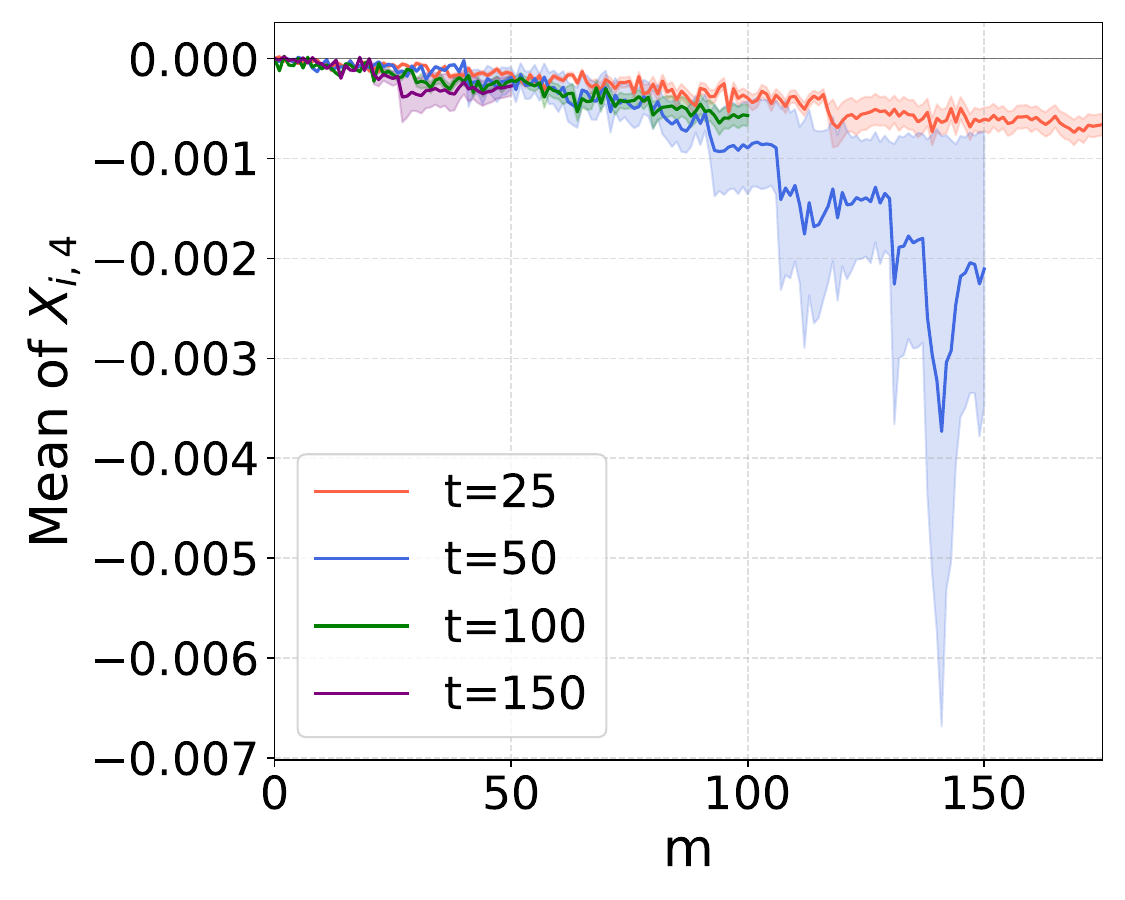}
        \caption{Arm 4}
    \end{subfigure}
\\
    \begin{subfigure}[b]{0.32\textwidth}
        \centering
        \includegraphics[width=\textwidth]{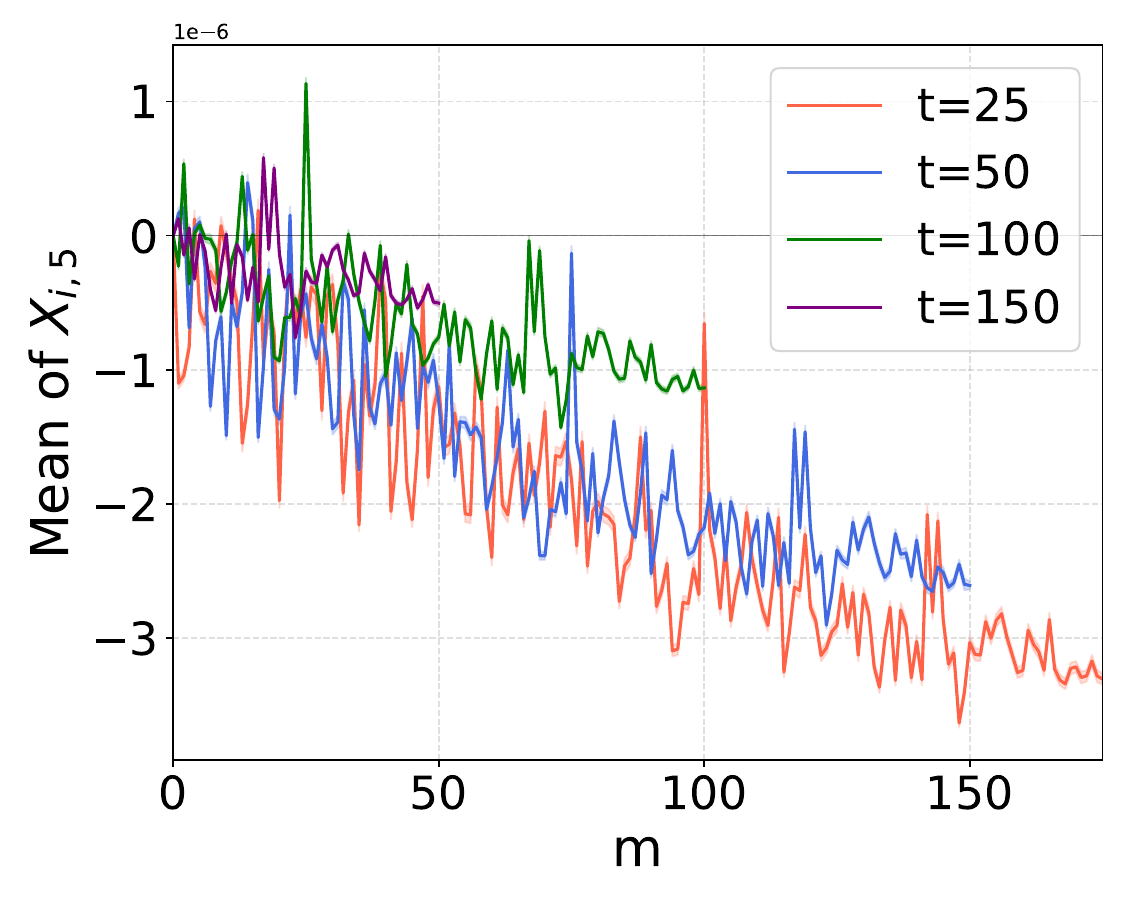}
        \caption{Arm 5}
    \end{subfigure}
    \begin{subfigure}[b]{0.32\textwidth}
        \centering
        \includegraphics[width=\textwidth]{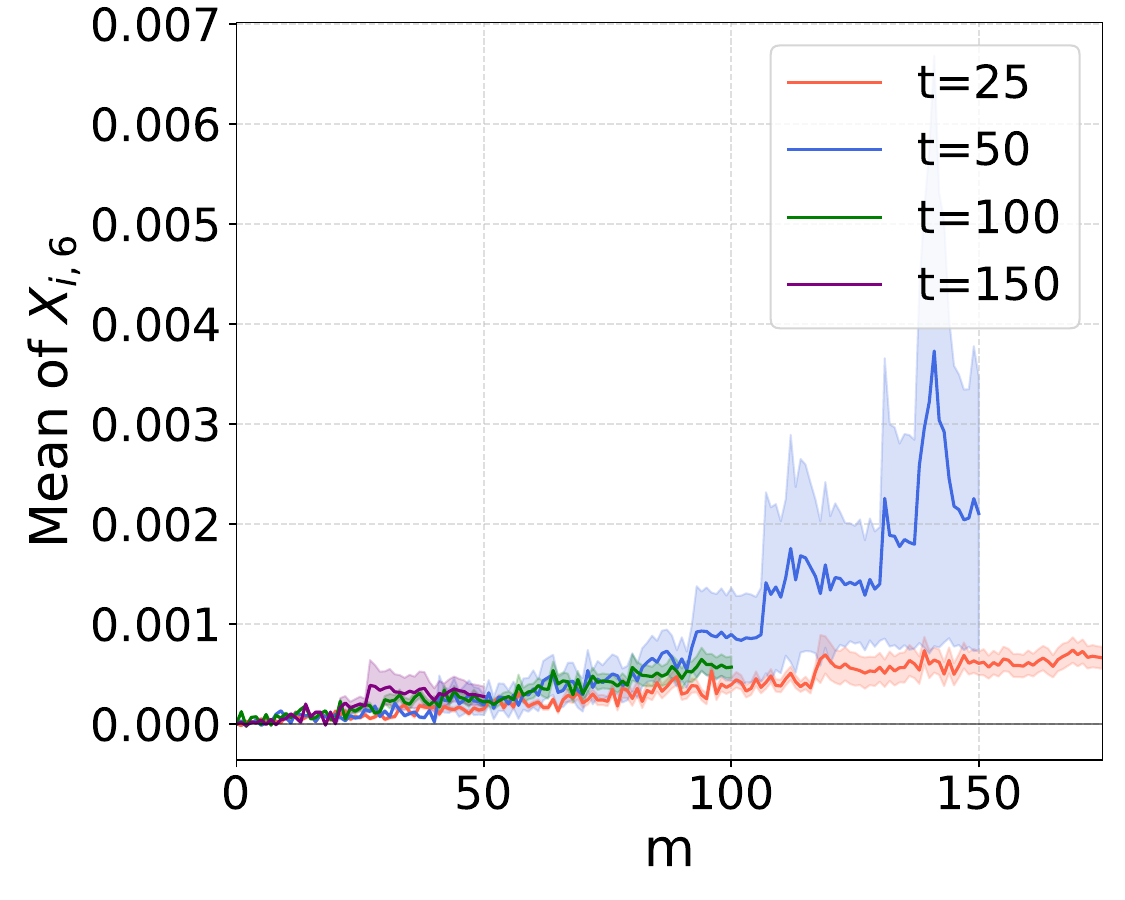}
        \caption{Arm 6}
    \end{subfigure}
    \begin{subfigure}[b]{0.32\textwidth}
        \centering
        \includegraphics[width=\textwidth]{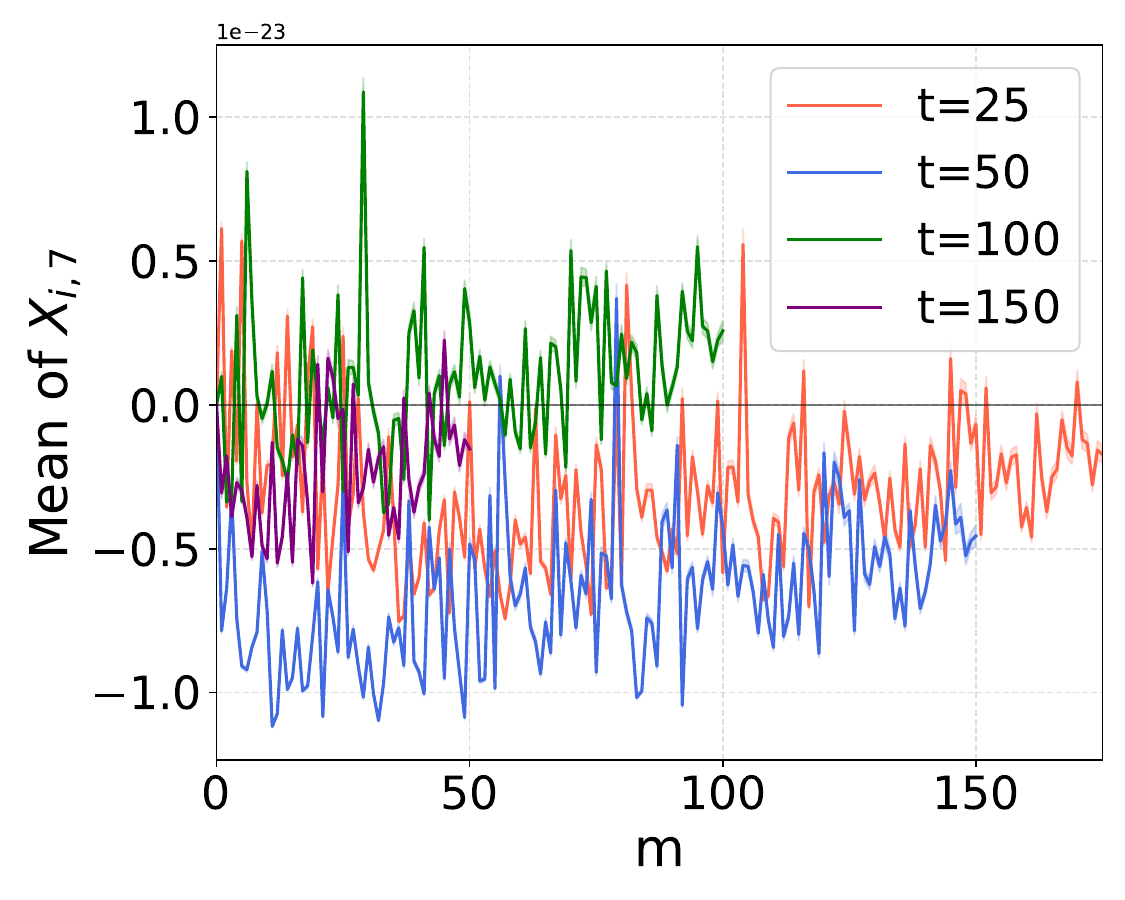}
        \caption{Arm 7}
    \end{subfigure}
    \\
    \begin{subfigure}[b]{0.32\textwidth}
        \centering
        \includegraphics[width=\textwidth]{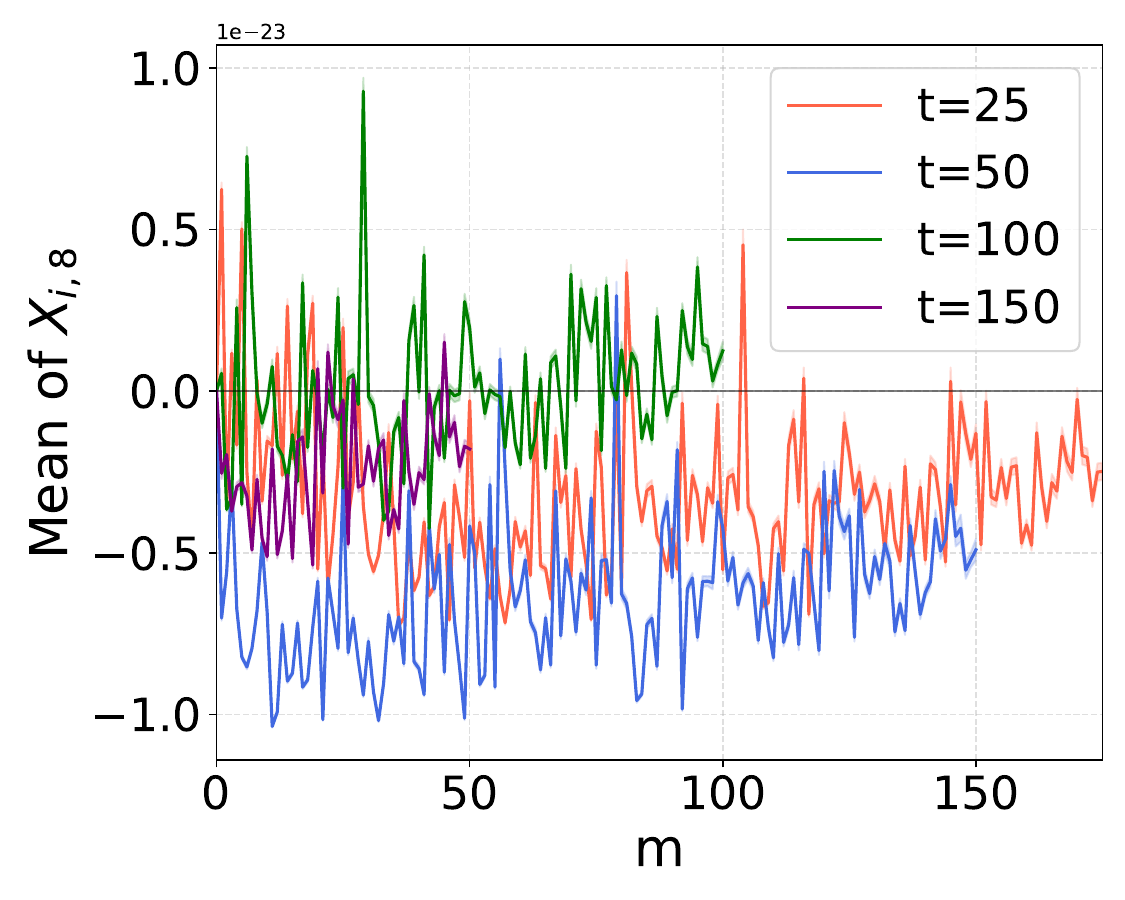}
        \caption{Arm 8}
    \end{subfigure}
    \begin{subfigure}[b]{0.32\textwidth}
        \centering
        \includegraphics[width=\textwidth]{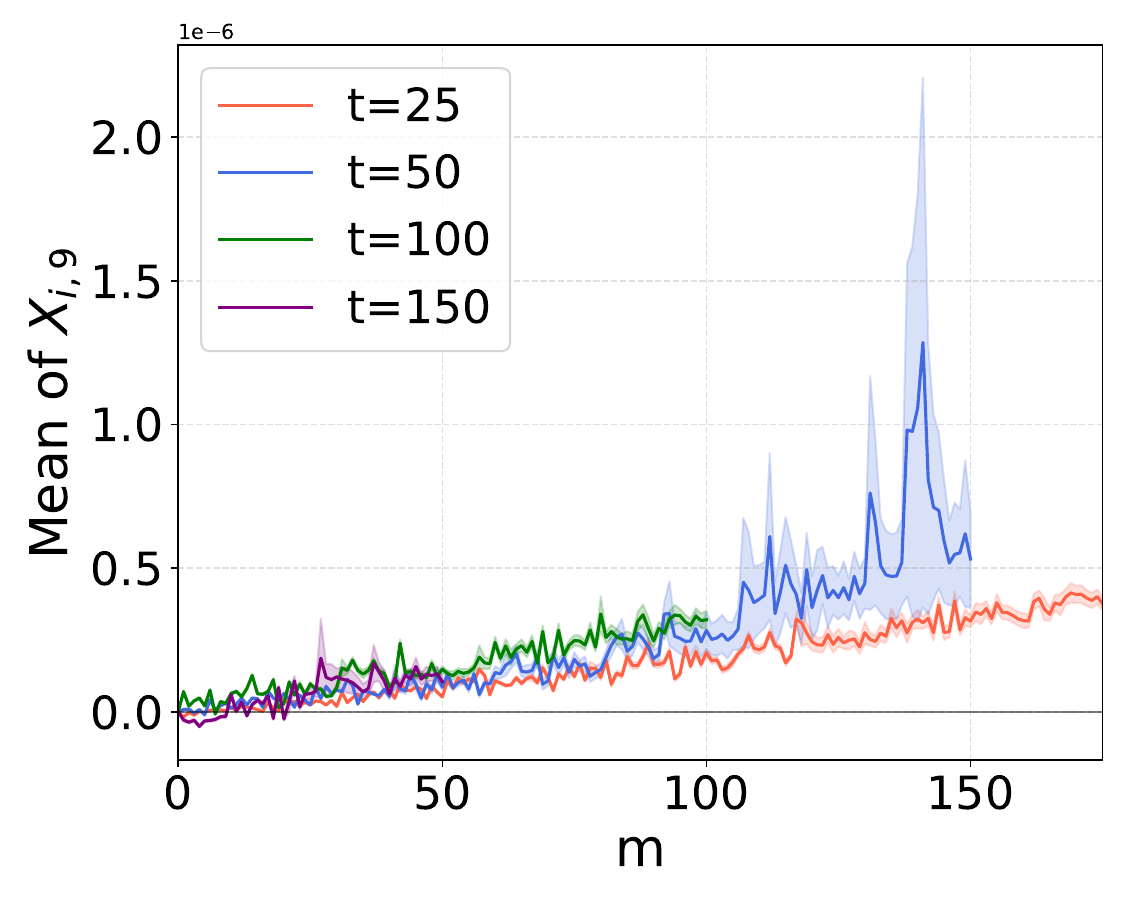}
        \caption{Arm 9}
    \end{subfigure}
    \caption{Per-arm martingale increments for DPT$'$ in linear bandits with offline starting contexts, where $m$ and $t$ denote the lengths of the martingale continuation and the starting contexts, respectively.}
    \label{fig_martingale_increment_offline_linear_DPT_prime}
\end{figure}

\begin{figure}[H]
    \centering
    \begin{subfigure}[b]{0.32\textwidth}
        \centering
        \includegraphics[width=\textwidth]{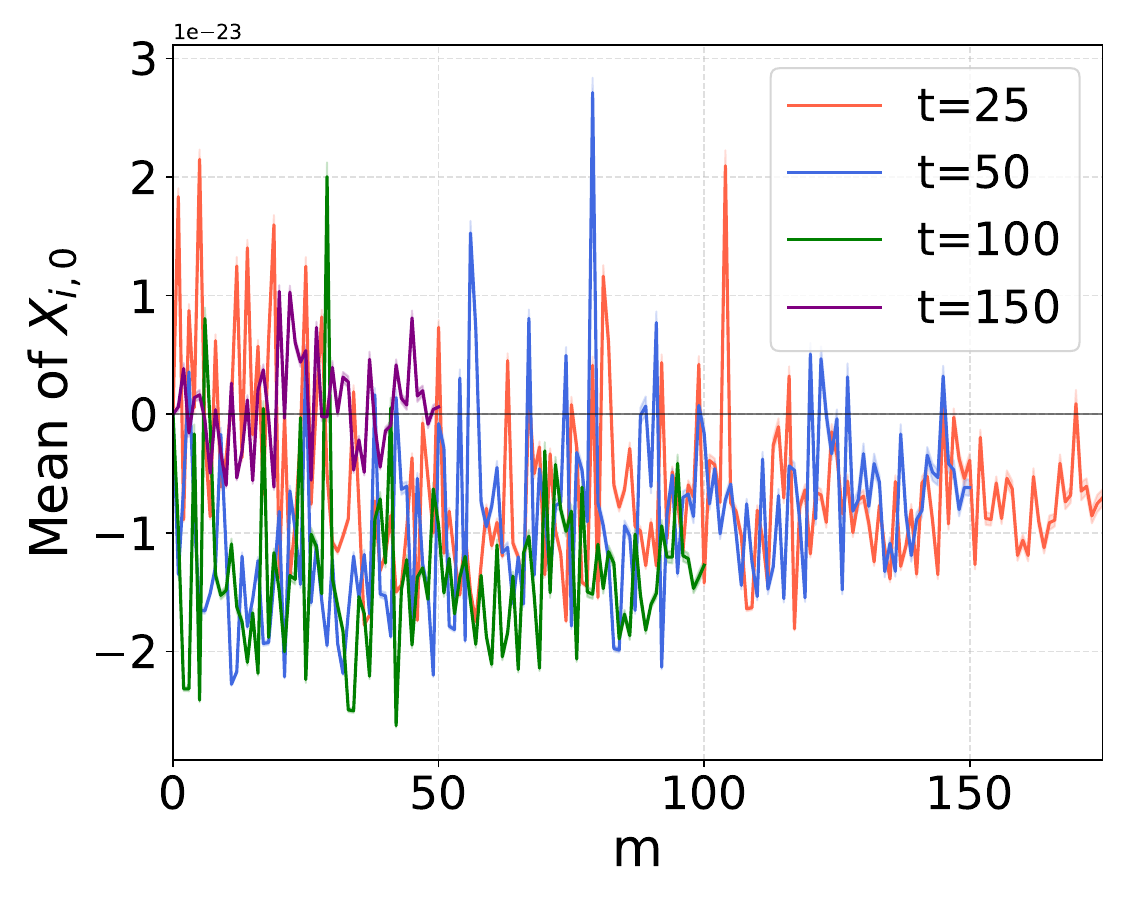}
        \caption{Arm 0}
    \end{subfigure}
    \begin{subfigure}[b]{0.32\textwidth}
        \centering
        \includegraphics[width=\textwidth]{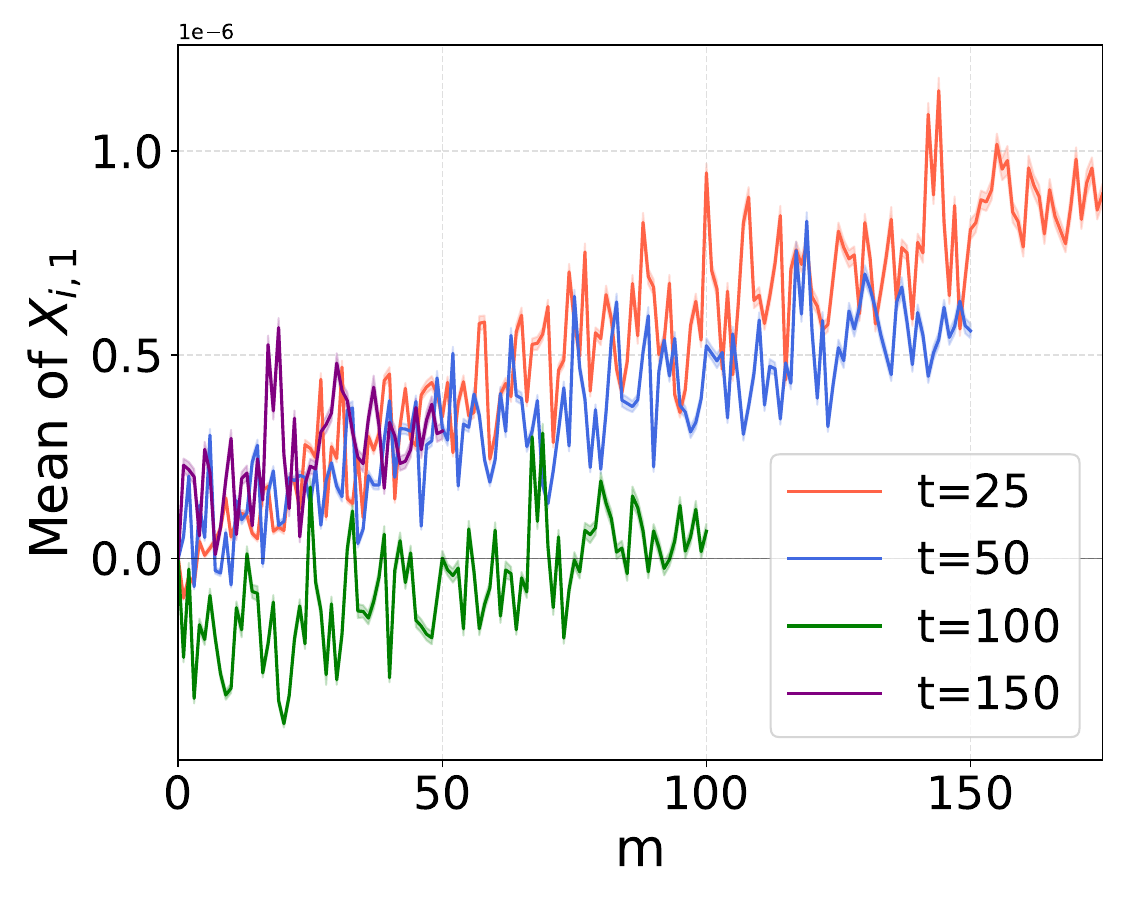}
        \caption{Arm 1}
    \end{subfigure}
    \begin{subfigure}[b]{0.32\textwidth}
        \centering
        \includegraphics[width=\textwidth]{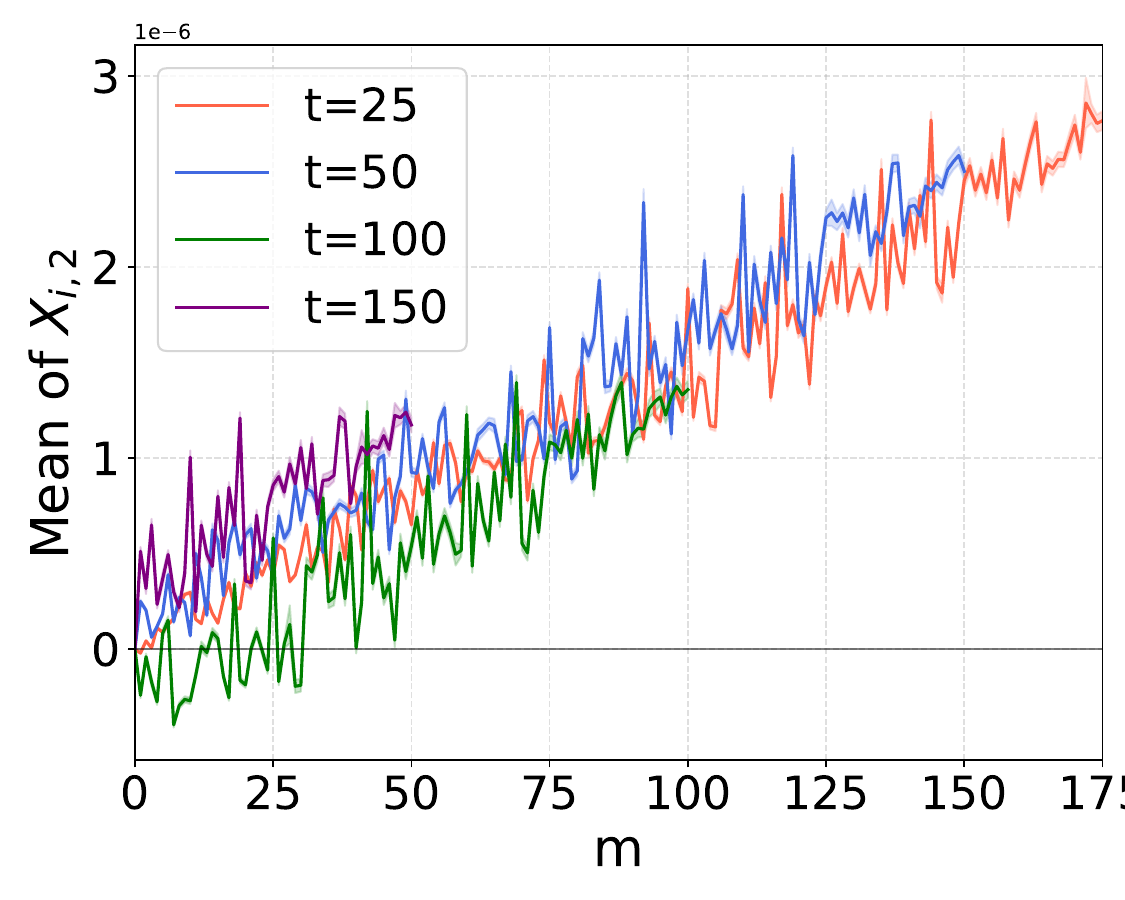}
        \caption{Arm 2}
    \end{subfigure}
    \\
    \begin{subfigure}[b]{0.32\textwidth}
        \centering
        \includegraphics[width=\textwidth]{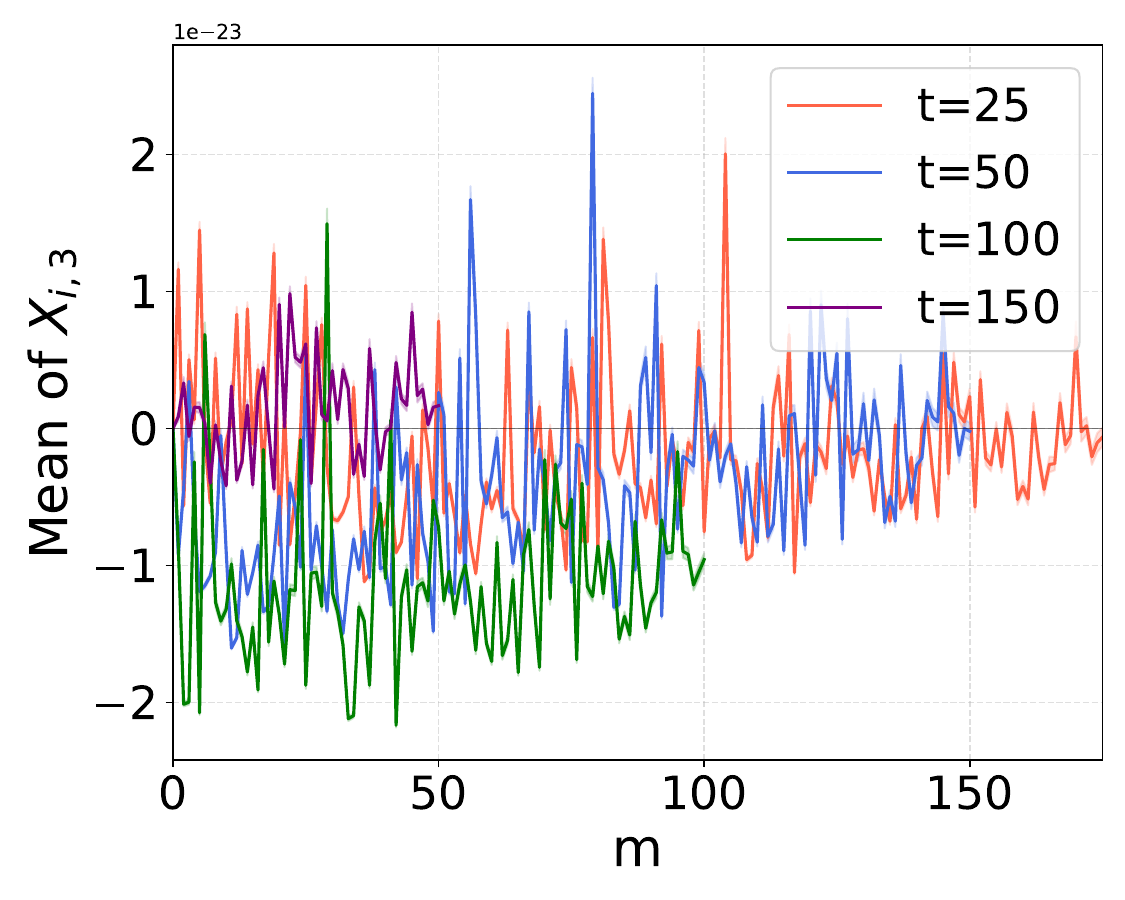}
        \caption{Arm 3}
    \end{subfigure}
    \begin{subfigure}[b]{0.32\textwidth}
        \centering
        \includegraphics[width=\textwidth]{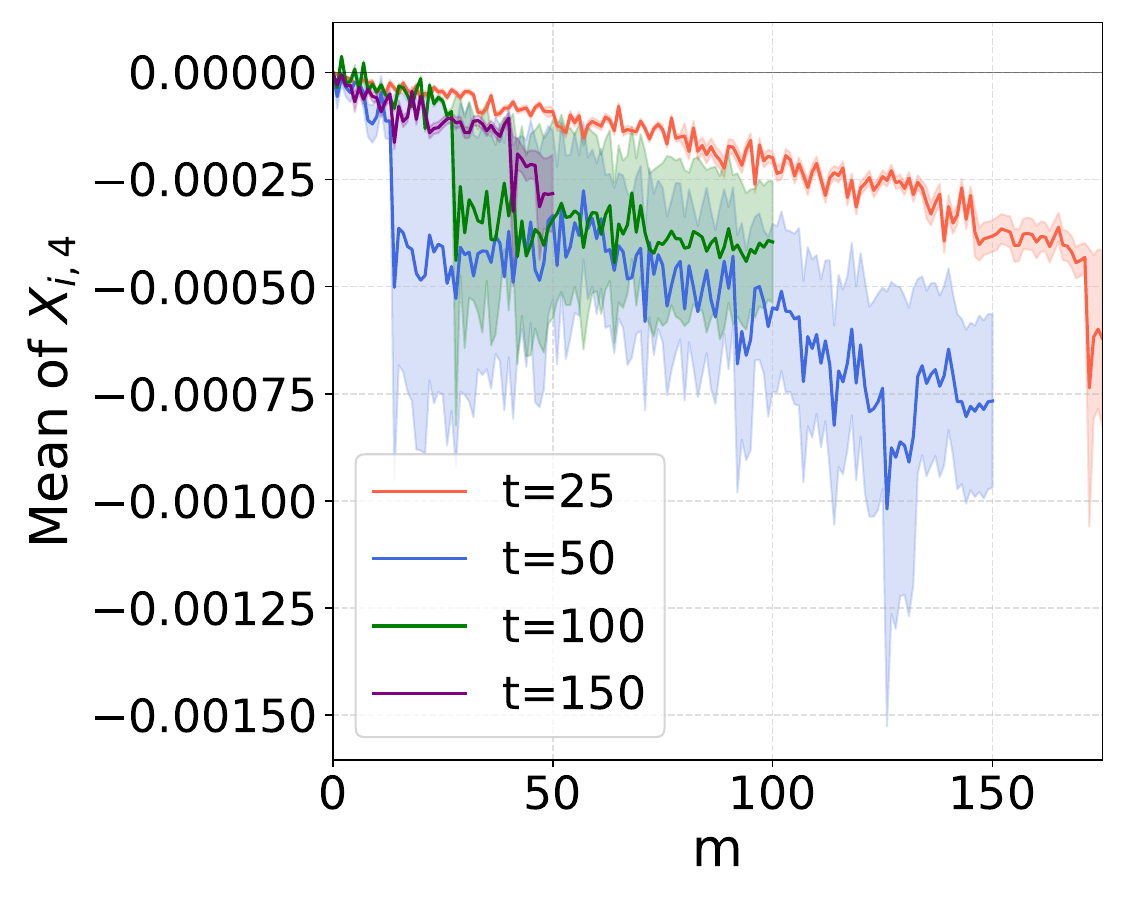}
        \caption{Arm 4}
    \end{subfigure}
\\
    \begin{subfigure}[b]{0.32\textwidth}
        \centering
        \includegraphics[width=\textwidth]{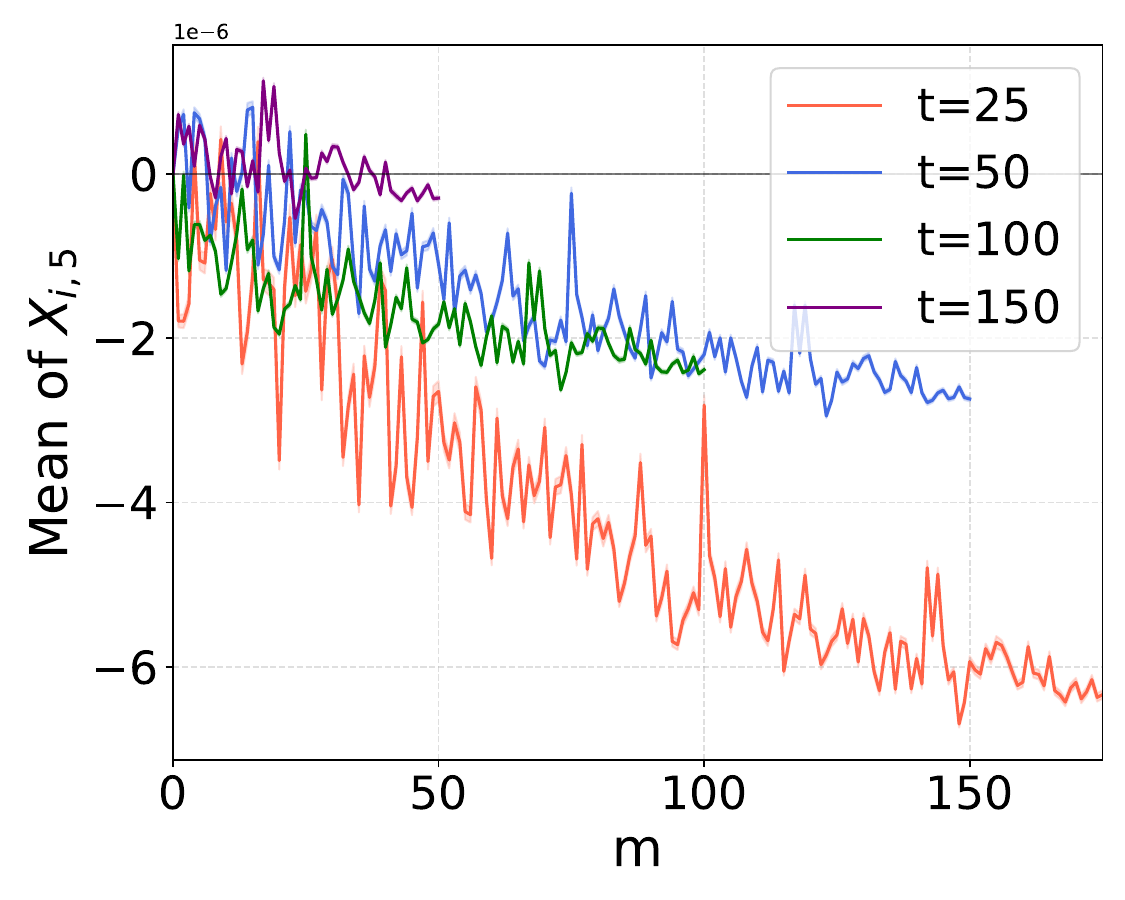}
        \caption{Arm 5}
    \end{subfigure}
    \begin{subfigure}[b]{0.32\textwidth}
        \centering
        \includegraphics[width=\textwidth]{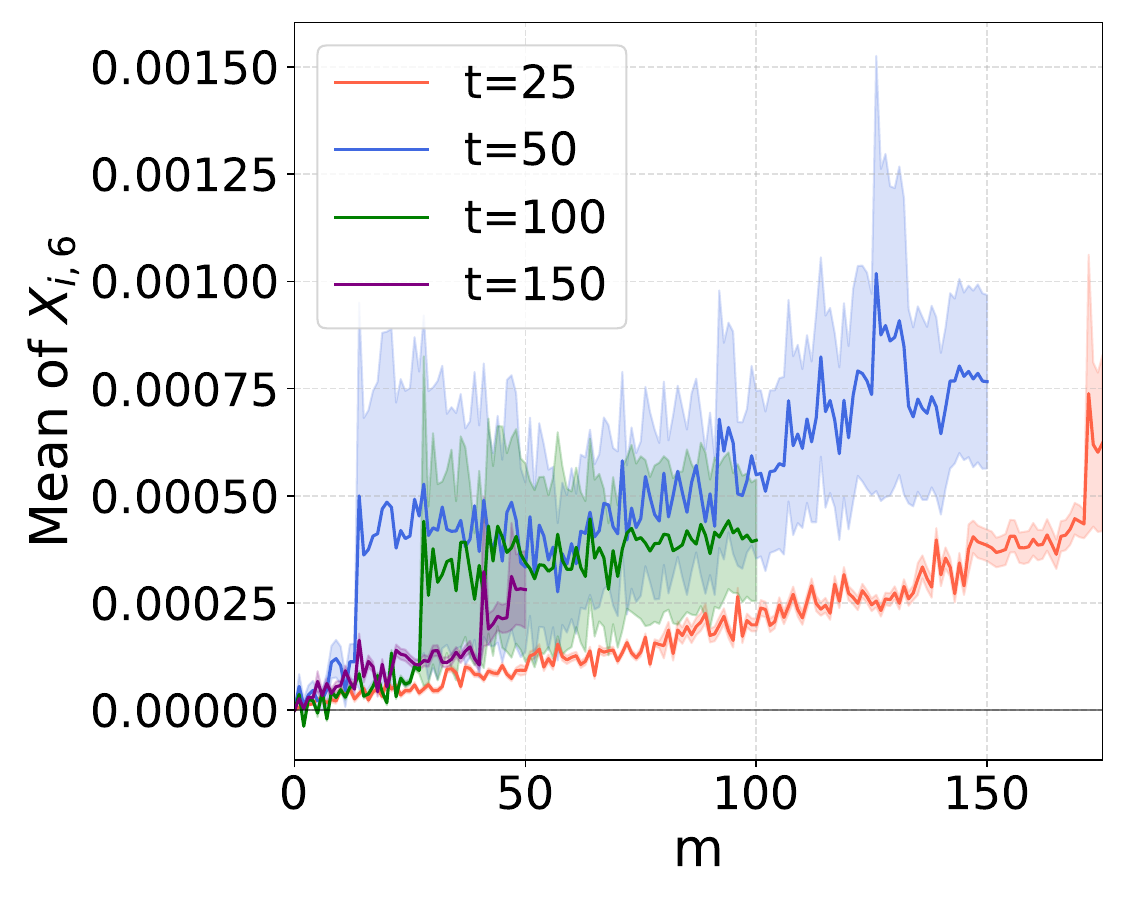}
        \caption{Arm 6}
    \end{subfigure}
    \begin{subfigure}[b]{0.32\textwidth}
        \centering
        \includegraphics[width=\textwidth]{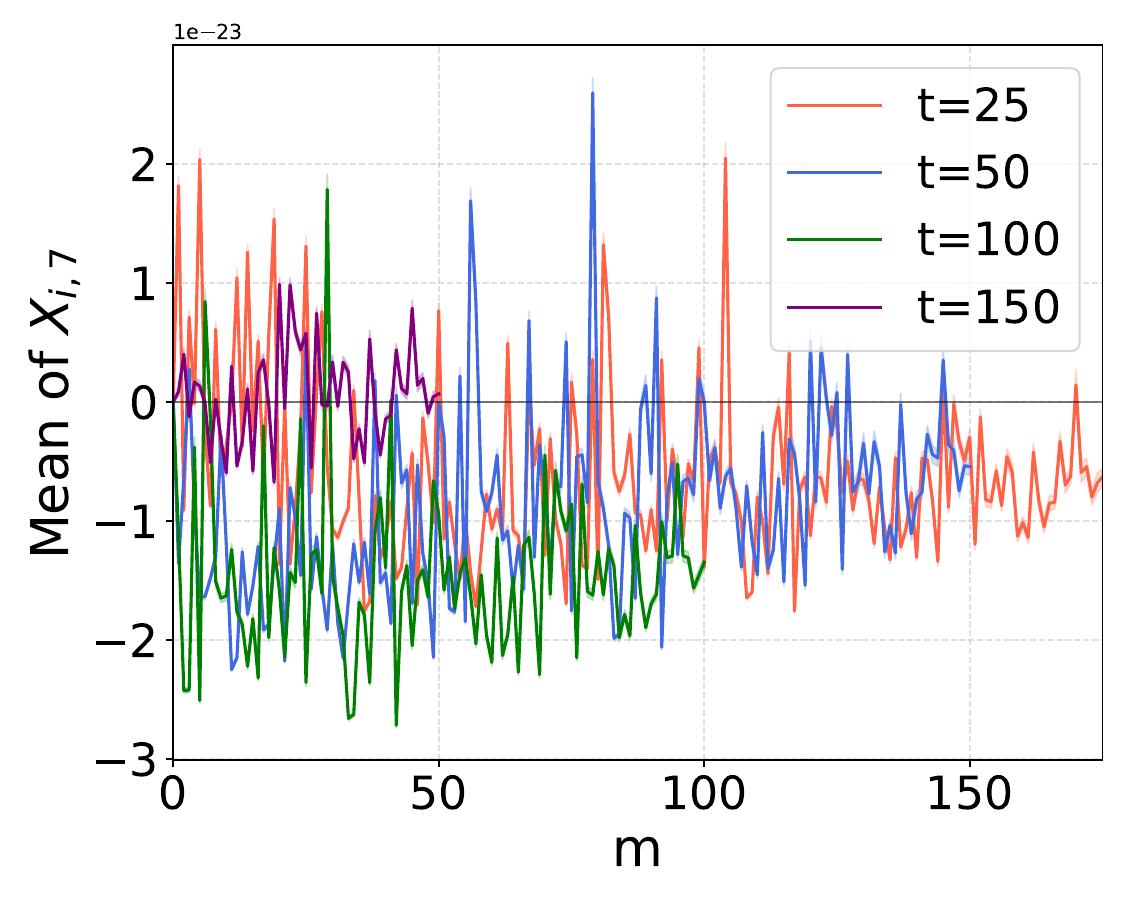}
        \caption{Arm 7}
    \end{subfigure}
    \\
    \begin{subfigure}[b]{0.32\textwidth}
        \centering
        \includegraphics[width=\textwidth]{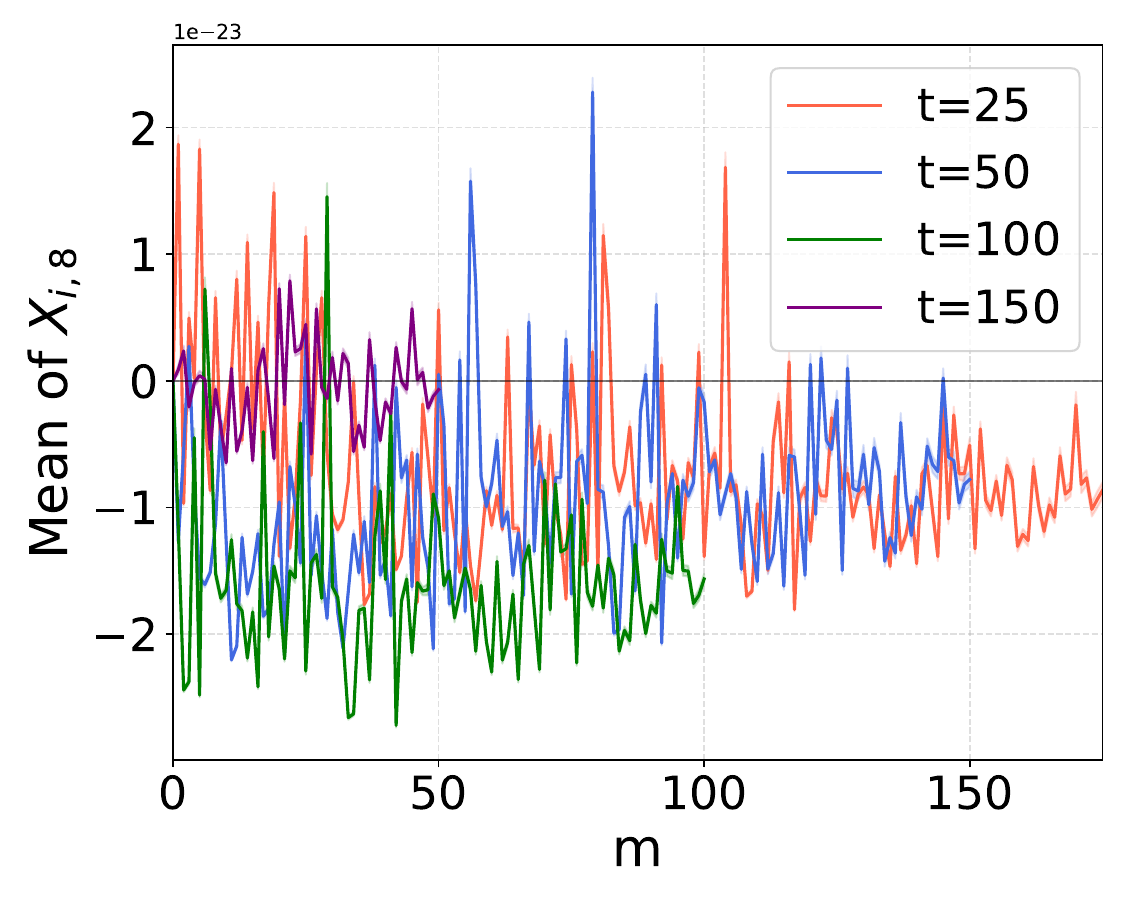}
        \caption{Arm 8}
    \end{subfigure}
    \begin{subfigure}[b]{0.32\textwidth}
        \centering
        \includegraphics[width=\textwidth]{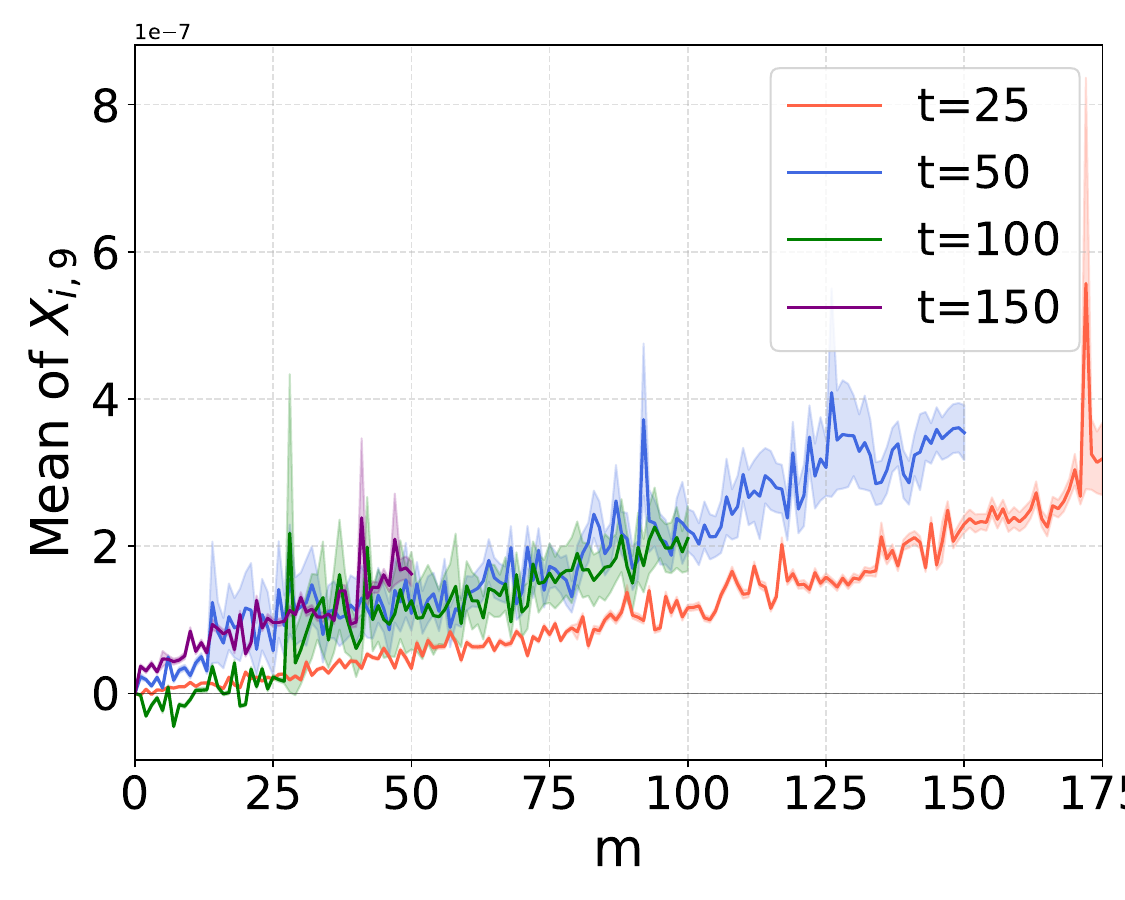}
        \caption{Arm 9}
    \end{subfigure}
    \caption{Per-arm martingale increments for DPT$'$ in linear bandits with online starting contexts, where $m$ and $t$ denote the lengths of the martingale continuation and the starting contexts, respectively.}
    \label{fig_martingale_increment_online_linear_DPT_prime}
\end{figure}

\subsection{GDPT sufficient-statistics table and evaluation figures}
\label{sec:gdpt_tables_and_plots}

\subsubsection{Sufficient-statistics table}
\begin{table}[H]
\centering
\small
\setlength{\tabcolsep}{6pt}
\renewcommand{\arraystretch}{1.08}
\caption{$p$-values for sufficient-statistic invariance tests on GDPT in Gaussian, Bernoulli, and linear bandits ($N=5000$, $\alpha=0.05$). Here, $n$ denotes the context length.}
\label{tab:ss_summary_gdpt}

\sbox{\banditupperbox}{%
\begin{tabular}{@{}r rrrr rr@{}}
\toprule
& \multicolumn{4}{c}{Gaussian}
& \multicolumn{2}{c}{Bernoulli} \\
\cmidrule(lr){2-5}\cmidrule(lr){6-7}
& \multicolumn{2}{c}{Offline}
& \multicolumn{2}{c}{Online}
& \multicolumn{1}{c}{Offline}
& \multicolumn{1}{c}{Online} \\
$n$
& \multicolumn{1}{c}{Sort}
& \multicolumn{1}{c}{Replace}
& \multicolumn{1}{c}{Sort}
& \multicolumn{1}{c}{Replace}
& \multicolumn{1}{c}{Sort}
& \multicolumn{1}{c}{Sort} \\
\midrule
$50$ & $<10^{-300}$ & $1.53\times 10^{-106}$ & $<10^{-300}$ & $2.50\times 10^{-80}$ & $1.36\times 10^{-127}$ & $4.33\times 10^{-125}$ \\
$100$ & $3.40\times 10^{-233}$ & $6.19\times 10^{-73}$ & $7.93\times 10^{-208}$ & $7.01\times 10^{-60}$ & $3.05\times 10^{-81}$ & $1.94\times 10^{-81}$ \\
$200$ & $1.55\times 10^{-134}$ & $1.79\times 10^{-59}$ & $4.62\times 10^{-135}$ & $7.43\times 10^{-43}$ & $2.07\times 10^{-26}$ & $1.06\times 10^{-20}$ \\
$300$ & $6.21\times 10^{-128}$ & $8.78\times 10^{-102}$ & $4.39\times 10^{-132}$ & $1.79\times 10^{-106}$ & $1.78\times 10^{-25}$ & $1.50\times 10^{-18}$ \\
$400$ & $8.54\times 10^{-114}$ & $7.85\times 10^{-109}$ & $1.72\times 10^{-116}$ & $2.18\times 10^{-113}$ & $1.26\times 10^{-34}$ & $1.53\times 10^{-31}$ \\
$500$ & $3.38\times 10^{-92}$ & $5.25\times 10^{-128}$ & $1.45\times 10^{-102}$ & $3.10\times 10^{-140}$ & $5.24\times 10^{-50}$ & $2.27\times 10^{-47}$ \\
\midrule
\end{tabular}}

\resizebox{\linewidth}{!}{%
\begin{minipage}{\wd\banditupperbox}
\noindent\usebox{\banditupperbox}\par\nointerlineskip
\noindent
\begin{tabular*}{\linewidth}{@{\extracolsep{\fill}}r rrrr@{}}
& \multicolumn{4}{c}{Linear} \\
\cmidrule(lr){2-5}
& \multicolumn{2}{c}{Offline}
& \multicolumn{2}{c}{Online} \\
$n$
& \multicolumn{1}{c}{Sort}
& \multicolumn{1}{c}{Replace}
& \multicolumn{1}{c}{Sort}
& \multicolumn{1}{c}{Replace} \\
\midrule
$25$  & $1.37\times 10^{-107}$ & $3.64\times 10^{-36}$ & $6.02\times 10^{-91}$ & $3.81\times 10^{-36}$ \\
$50$  & $5.32\times 10^{-91}$  & $1.83\times 10^{-41}$ & $1.22\times 10^{-81}$ & $1.34\times 10^{-32}$ \\
$100$ & $1.99\times 10^{-49}$  & $2.37\times 10^{-51}$ & $2.73\times 10^{-42}$ & $1.91\times 10^{-39}$ \\
$150$ & $2.60\times 10^{-31}$  & $2.88\times 10^{-49}$ & $3.15\times 10^{-29}$ & $9.43\times 10^{-45}$ \\
$200$ & $2.73\times 10^{-30}$  & $8.05\times 10^{-46}$ & $2.60\times 10^{-32}$ & $4.09\times 10^{-47}$ \\
\bottomrule
\end{tabular*}
\end{minipage}%
}
\end{table}

\subsubsection{Gaussian evaluation figures and martingale plots}
\begin{figure}[H]
\centering
\begin{subfigure}{0.245\textwidth}
\includegraphics[width=\textwidth]{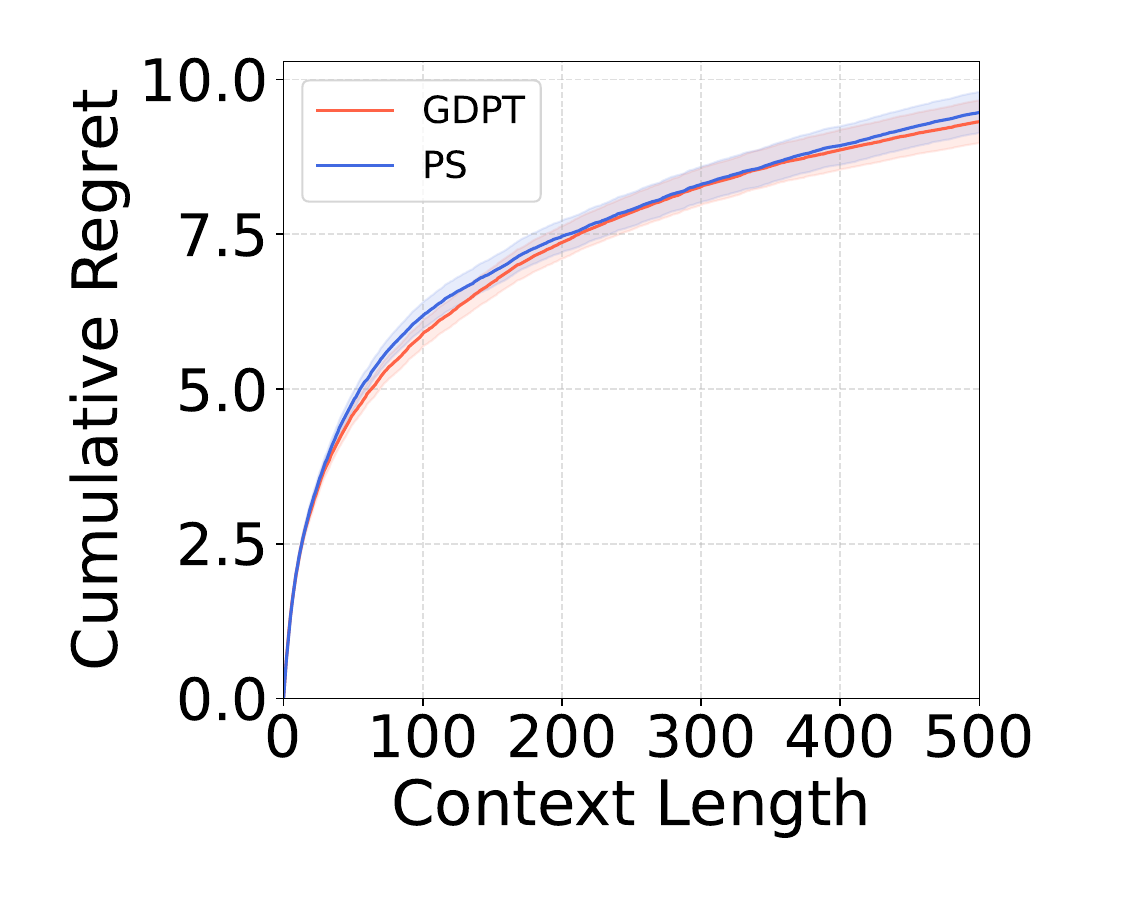}
\caption{Online regret}
\end{subfigure}
\hfill
\begin{subfigure}{0.245\textwidth}
\includegraphics[width=\textwidth]{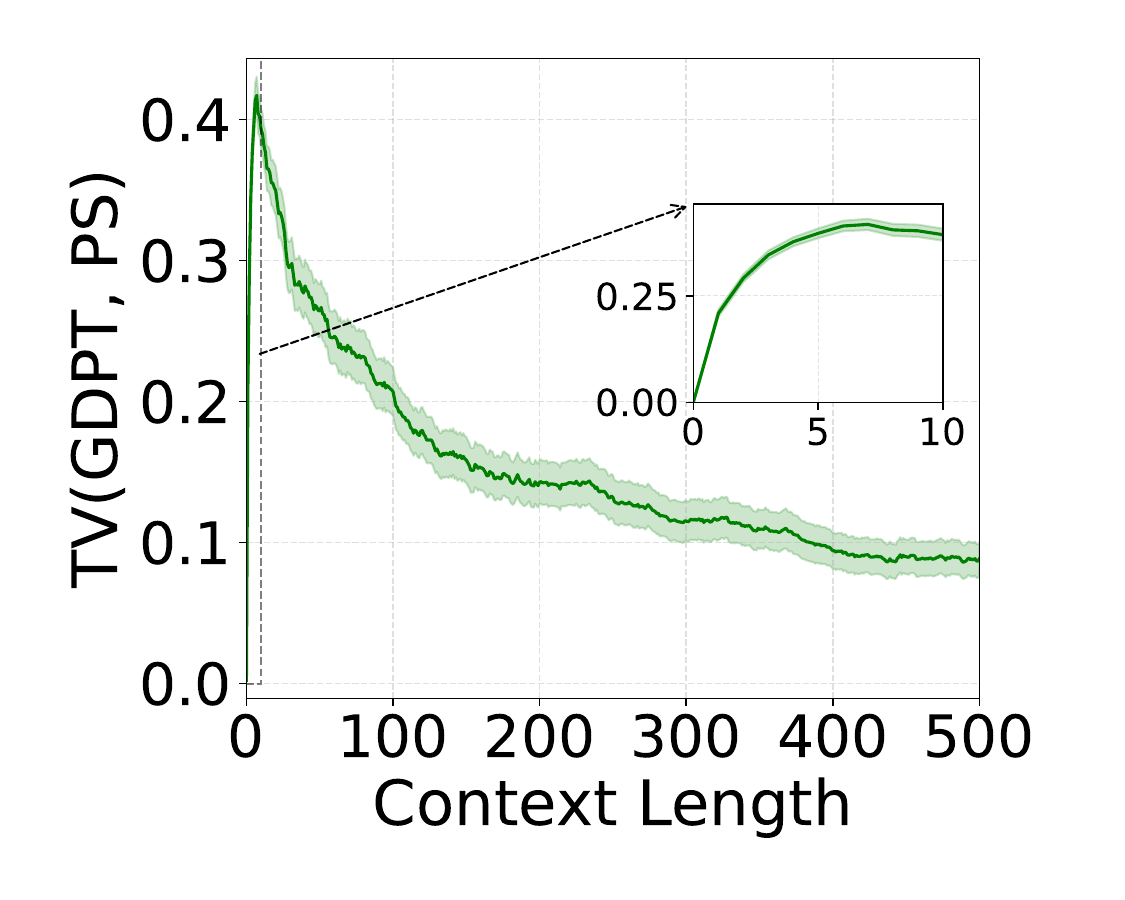}
\caption{Online TV distance}
\end{subfigure}
\hfill
\begin{subfigure}{0.245\textwidth}
\includegraphics[width=\textwidth]{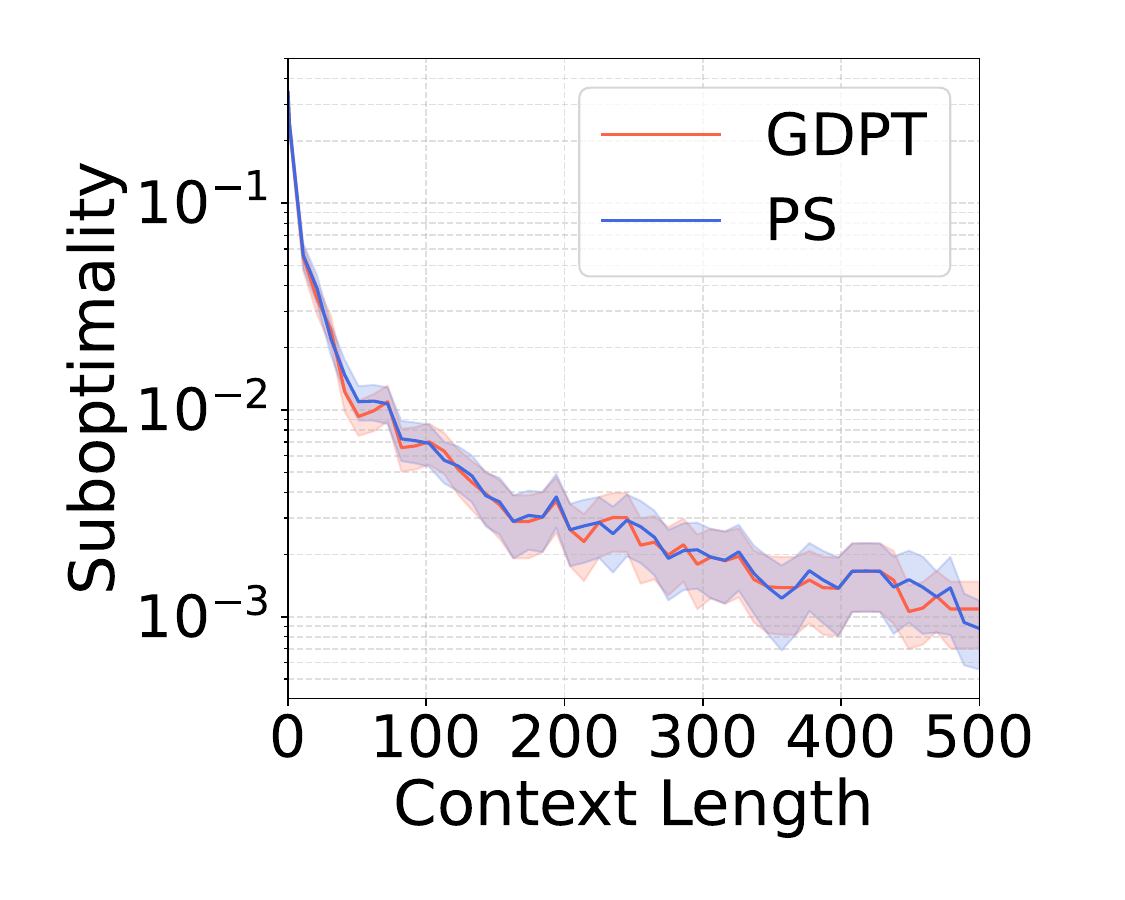}
\caption{Offline suboptimality}
\end{subfigure}
\hfill
\begin{subfigure}{0.245\textwidth}
\includegraphics[width=\textwidth]{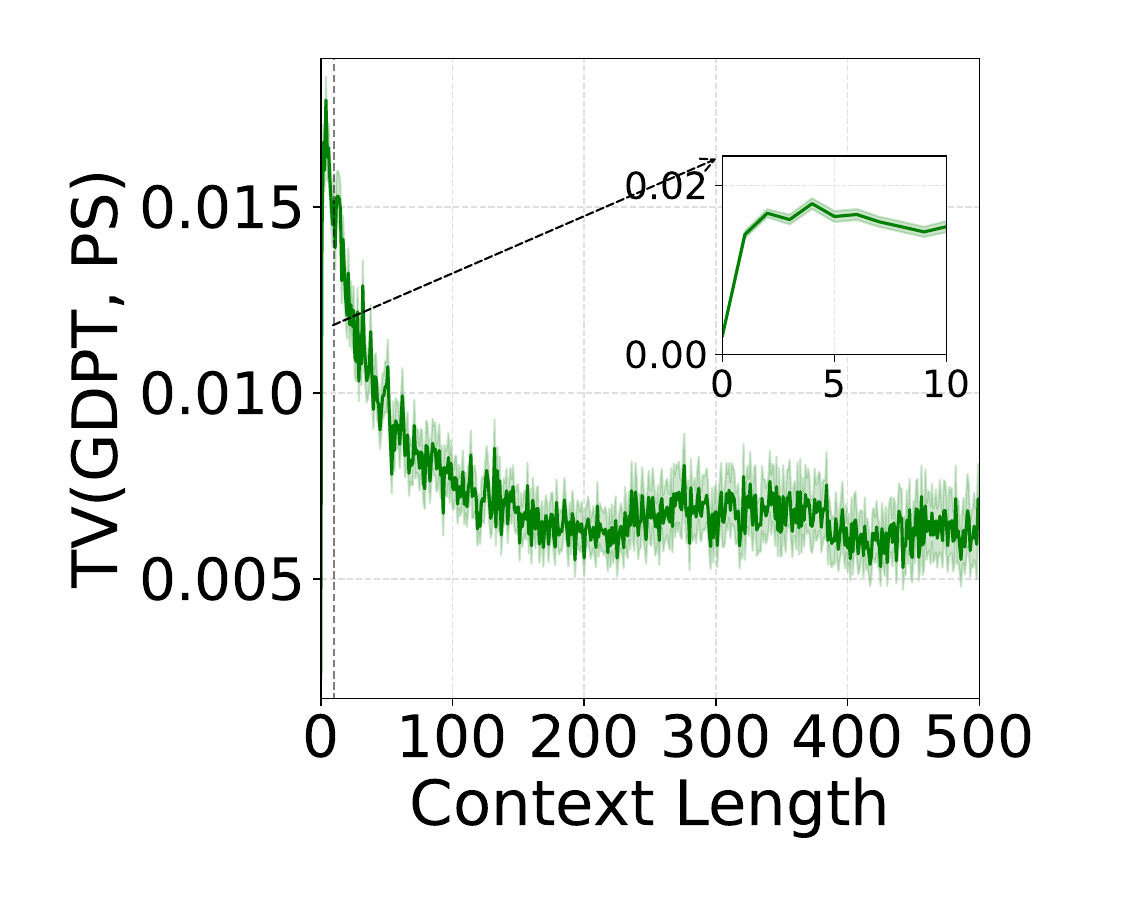}
\caption{Offline TV distance}
\end{subfigure}
\caption{Performance comparison of GDPT and Thompson Sampling in Gaussian bandits under online and offline deployment.}
\label{fig_evaluation_gaussian_bandit_GDPT}
\end{figure}

\begin{figure}[H]
    \centering
    \begin{subfigure}[b]{0.32\textwidth}
        \centering
        \includegraphics[width=\textwidth]{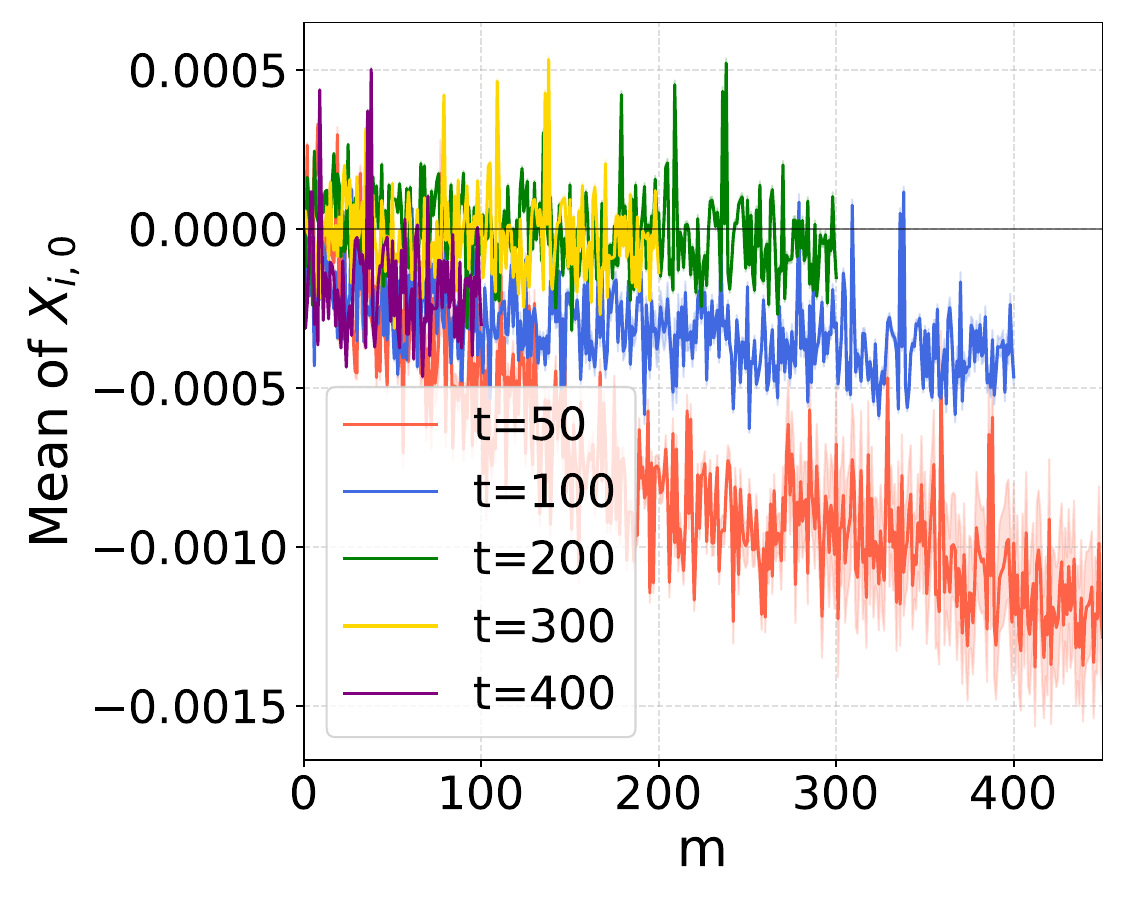}
        \caption{Arm 0}
    \end{subfigure}
    \begin{subfigure}[b]{0.32\textwidth}
        \centering
        \includegraphics[width=\textwidth]{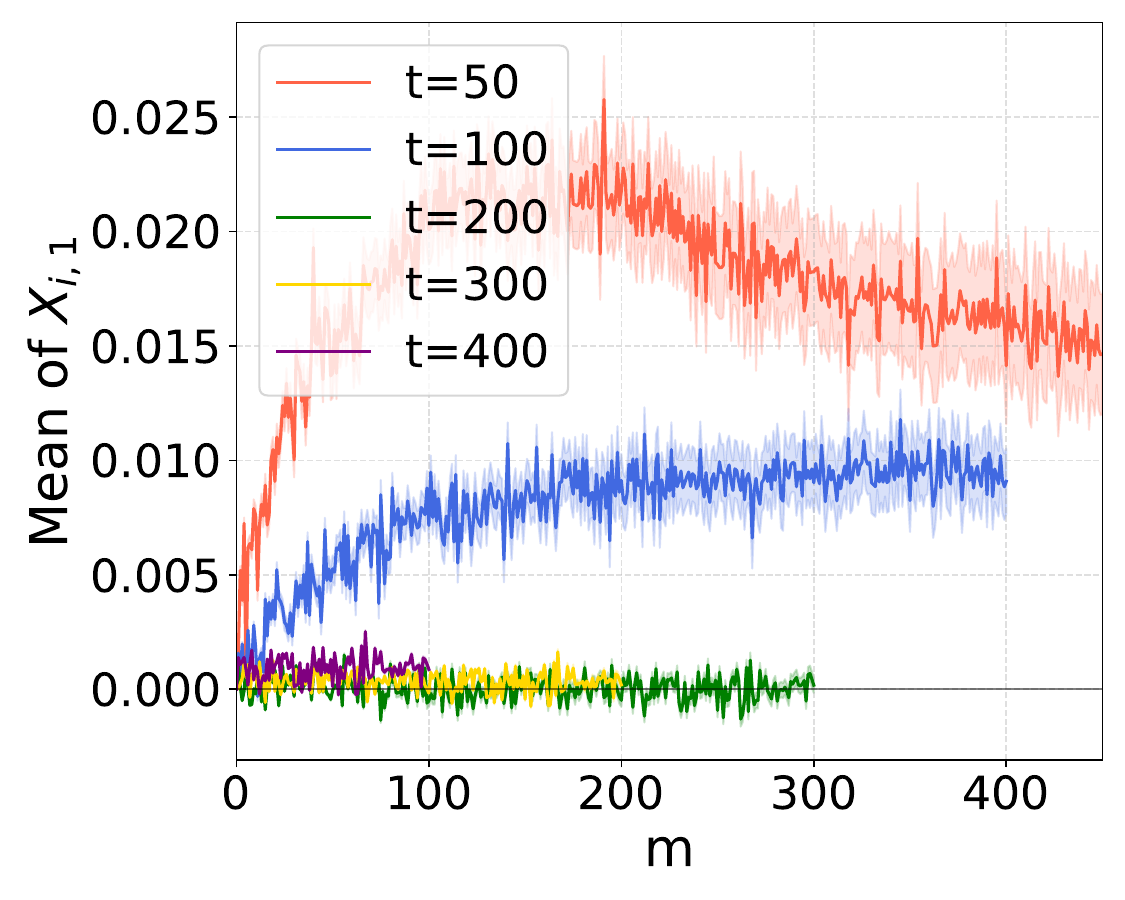}
        \caption{Arm 1}
    \end{subfigure}
    \begin{subfigure}[b]{0.32\textwidth}
        \centering
        \includegraphics[width=\textwidth]{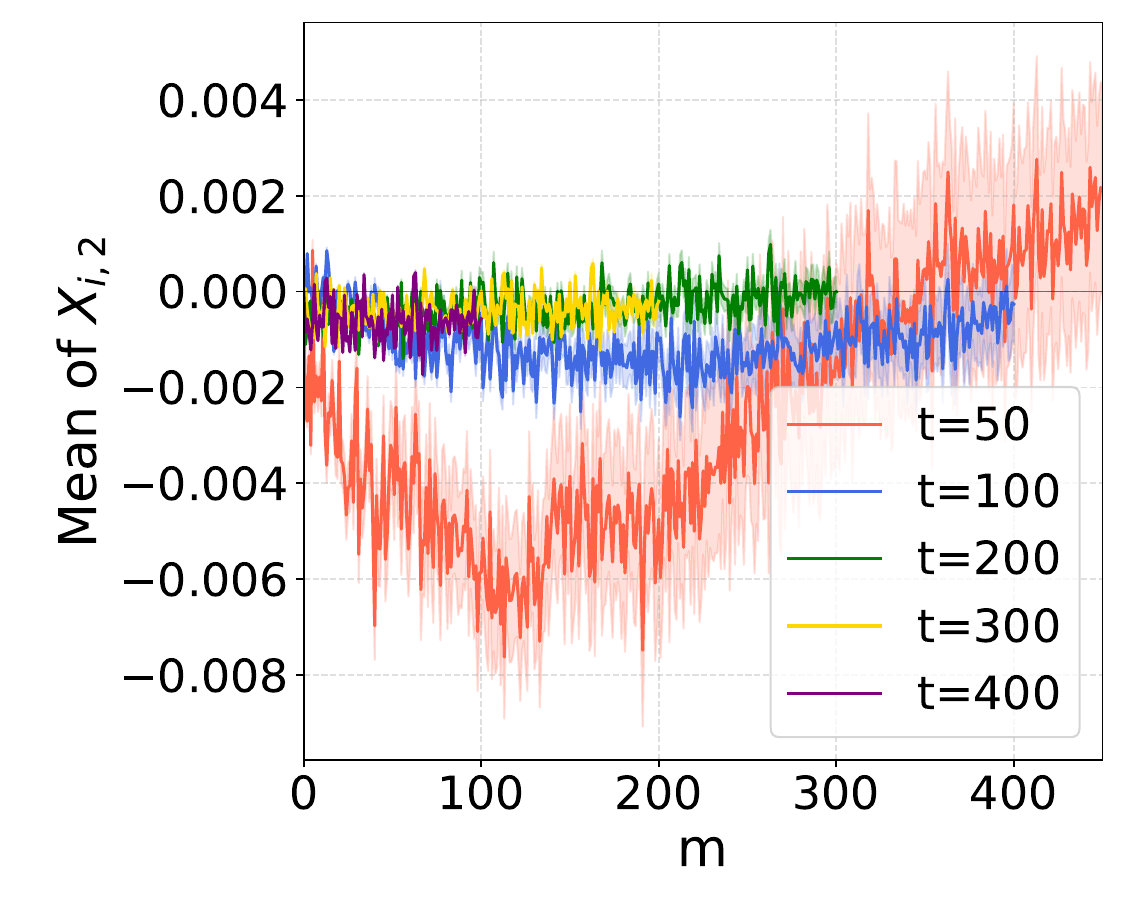}
        \caption{Arm 2}
    \end{subfigure}
    \\
    \begin{subfigure}[b]{0.32\textwidth}
        \centering
        \includegraphics[width=\textwidth]{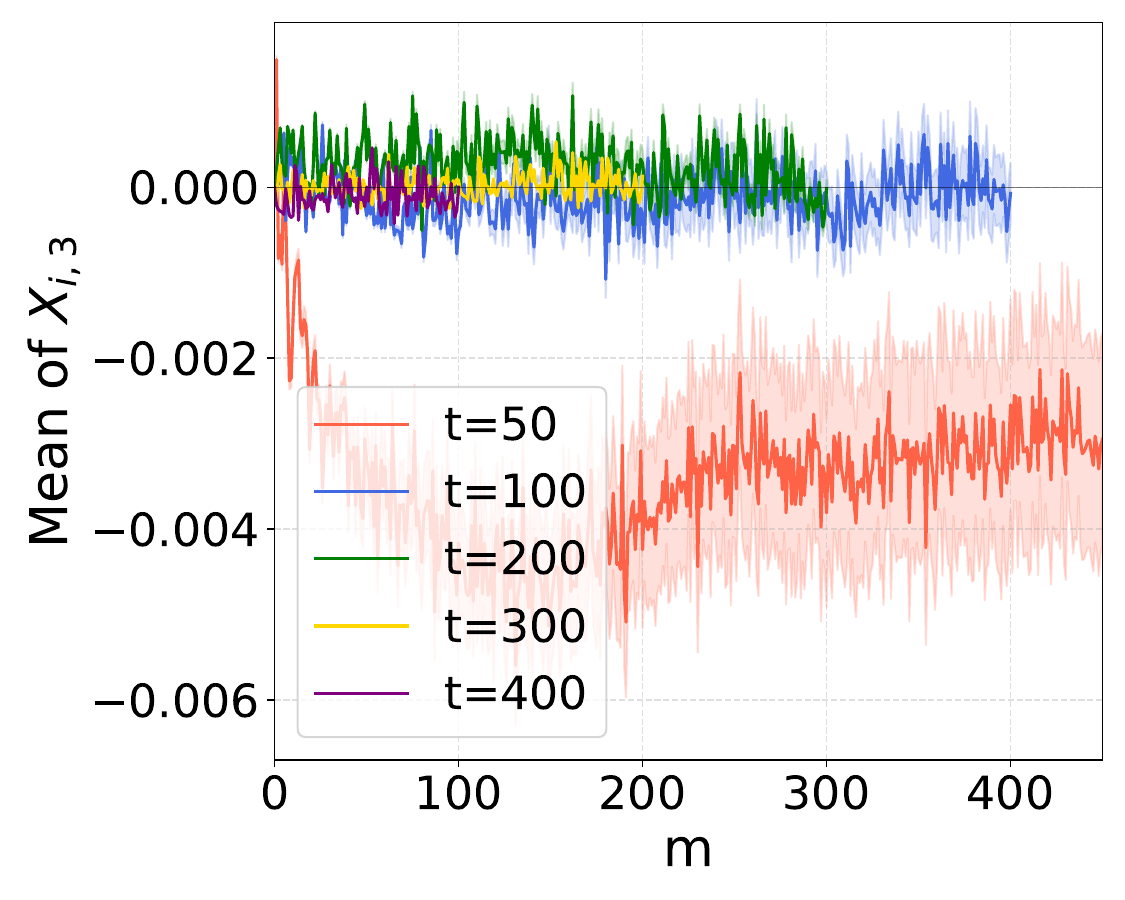}
        \caption{Arm 3}
    \end{subfigure}
    \begin{subfigure}[b]{0.32\textwidth}
        \centering
        \includegraphics[width=\textwidth]{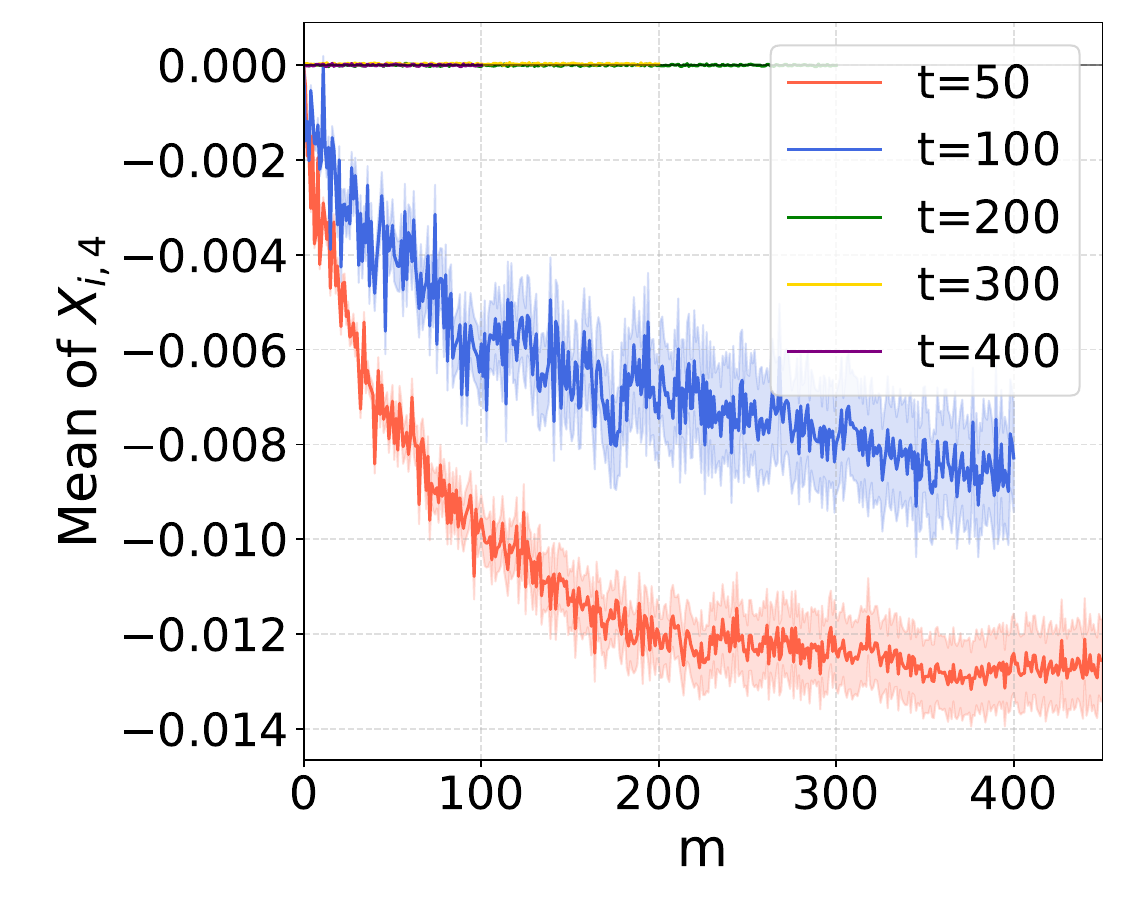}
        \caption{Arm 4}
    \end{subfigure}
    \caption{Per-arm martingale increments for GDPT in Gaussian bandits with offline starting contexts, where $m$ and $t$ denote the lengths of the martingale continuation and the starting contexts, respectively.}
    \label{fig_martingale_increment_offline_gaussian_GDPT}
\end{figure}

\begin{figure}[H]
    \centering
    \begin{subfigure}[b]{0.32\textwidth}
        \centering
        \includegraphics[width=\textwidth]{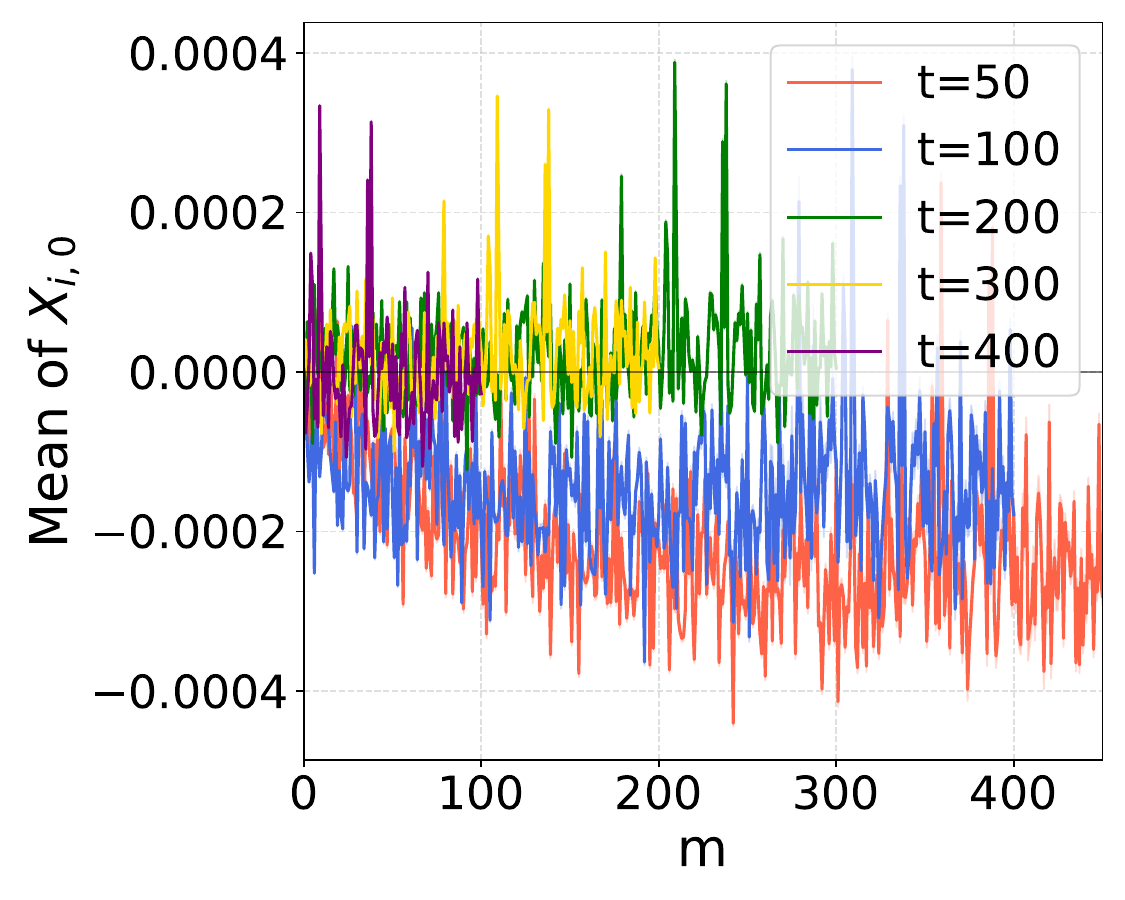}
        \caption{Arm 0}
    \end{subfigure}
    \begin{subfigure}[b]{0.32\textwidth}
        \centering
        \includegraphics[width=\textwidth]{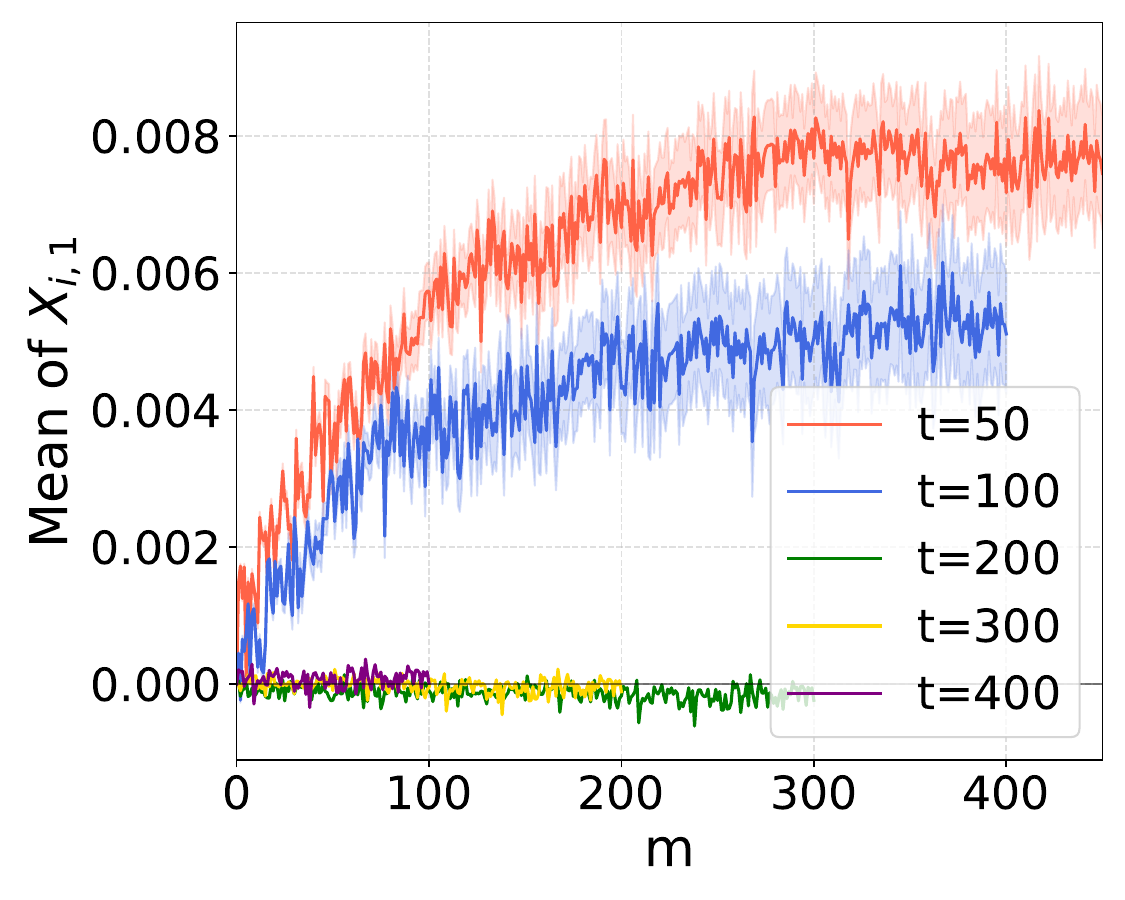}
        \caption{Arm 1}
    \end{subfigure}
    \begin{subfigure}[b]{0.32\textwidth}
        \centering
        \includegraphics[width=\textwidth]{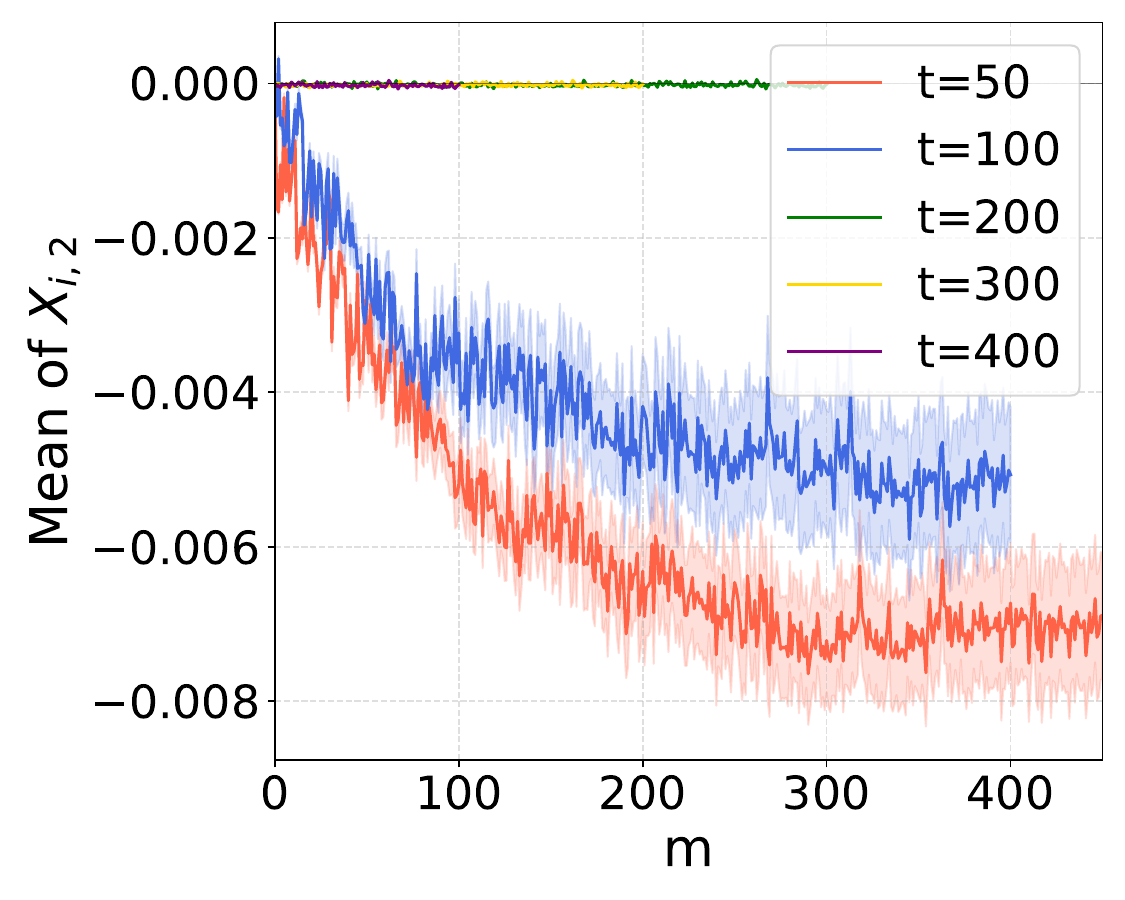}
        \caption{Arm 2}
    \end{subfigure}
    \\
    \begin{subfigure}[b]{0.32\textwidth}
        \centering
        \includegraphics[width=\textwidth]{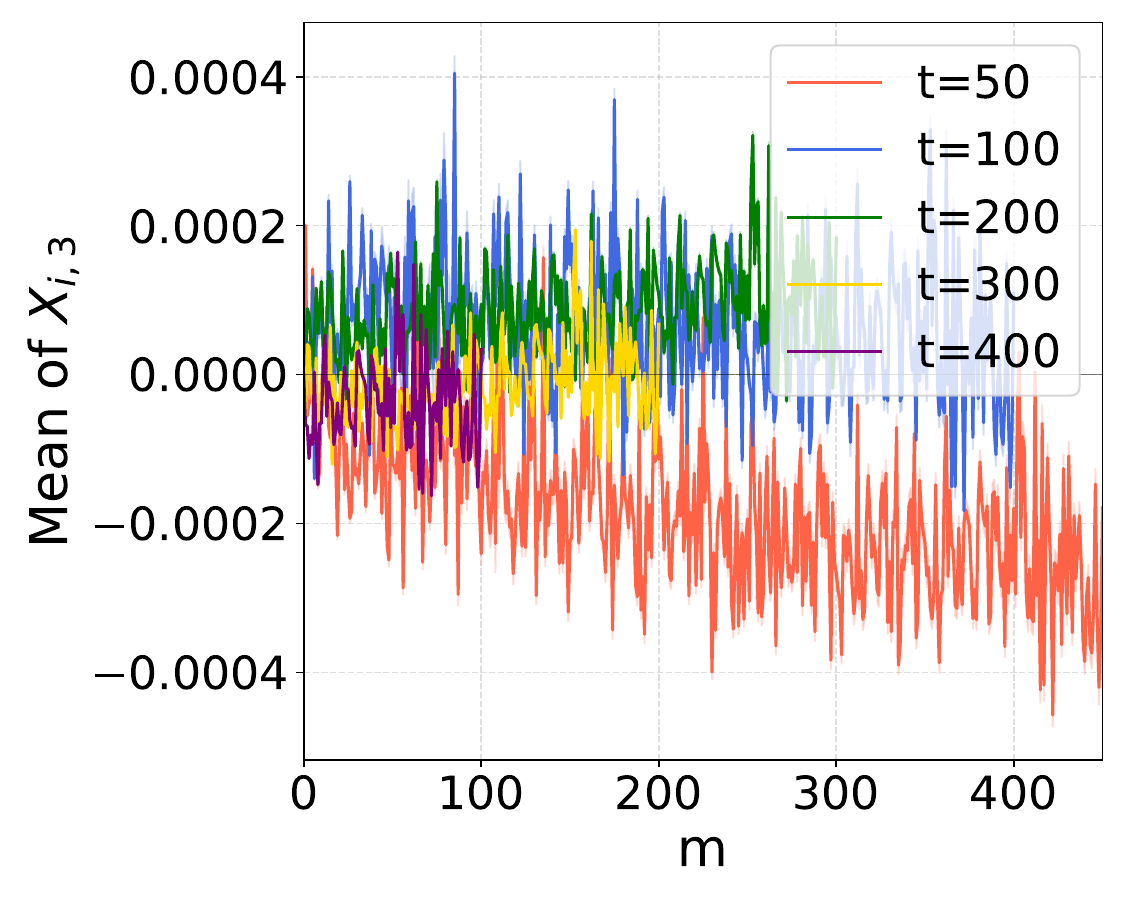}
        \caption{Arm 3}
    \end{subfigure}
    \begin{subfigure}[b]{0.32\textwidth}
        \centering
        \includegraphics[width=\textwidth]{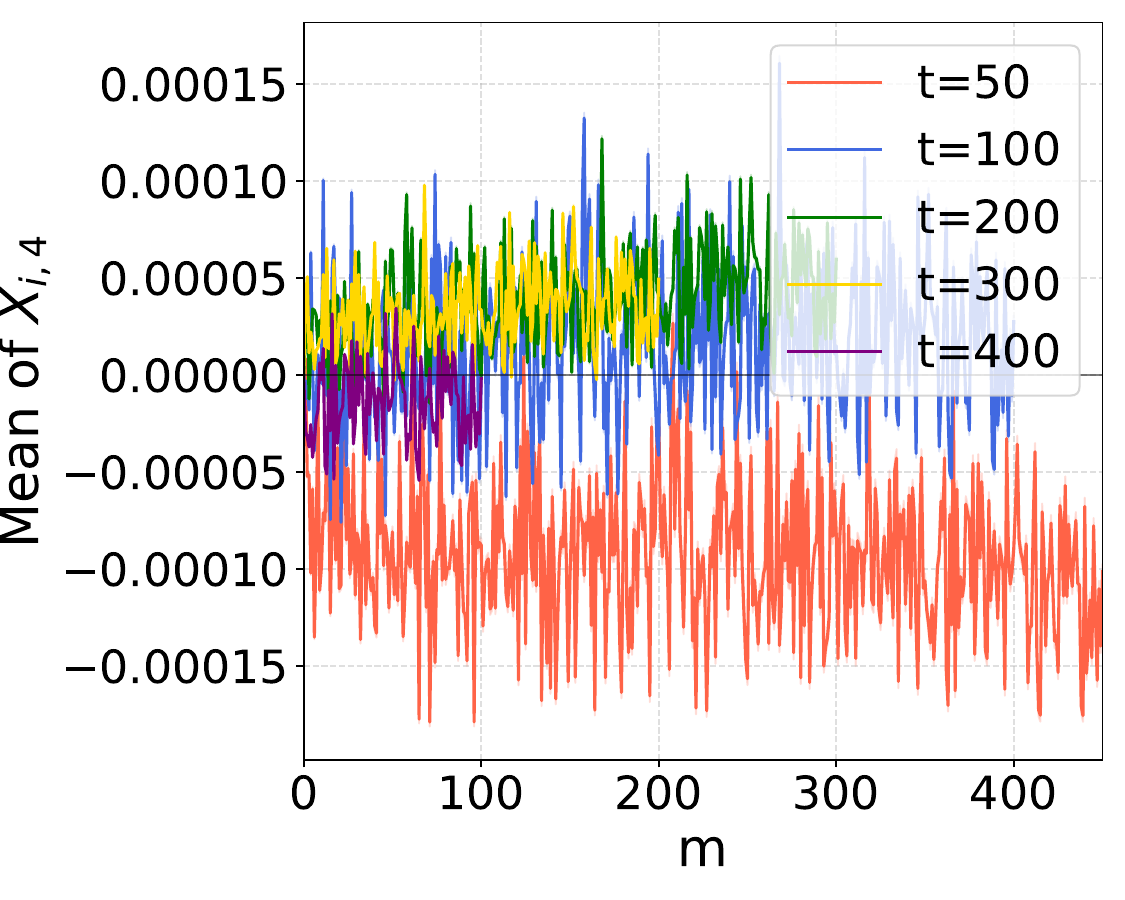}
        \caption{Arm 4}
    \end{subfigure}
    \caption{Per-arm martingale increments for GDPT in Gaussian bandits with online starting contexts, where $m$ and $t$ denote the lengths of the martingale continuation and the starting contexts, respectively.}
    \label{fig_martingale_increment_online_gaussian_GDPT}
\end{figure}

\subsubsection{Bernoulli evaluation figures and martingale plots}
\begin{figure}[H]
\centering
\begin{subfigure}{0.245\textwidth}
\includegraphics[width=\textwidth]{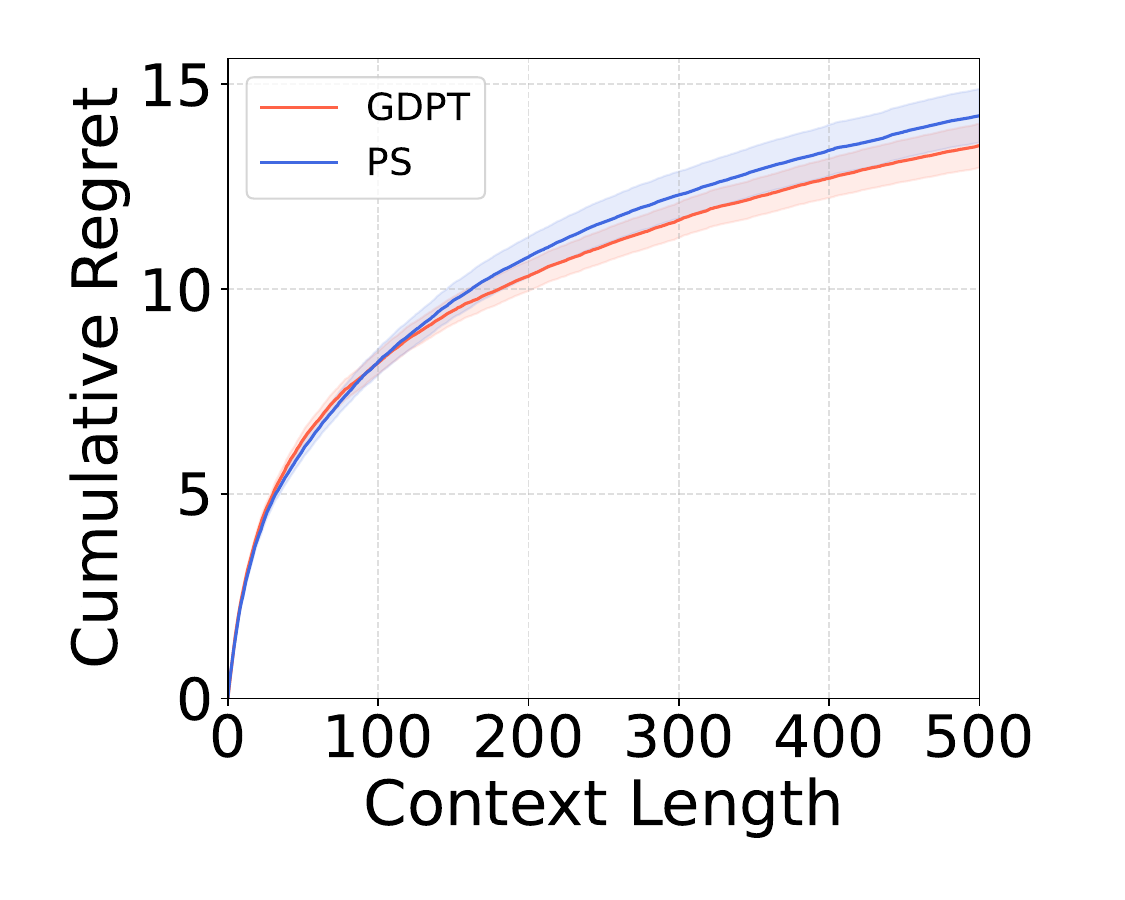}
\caption{Online regret}
\end{subfigure}
\hfill
\begin{subfigure}{0.245\textwidth}
\includegraphics[width=\textwidth]{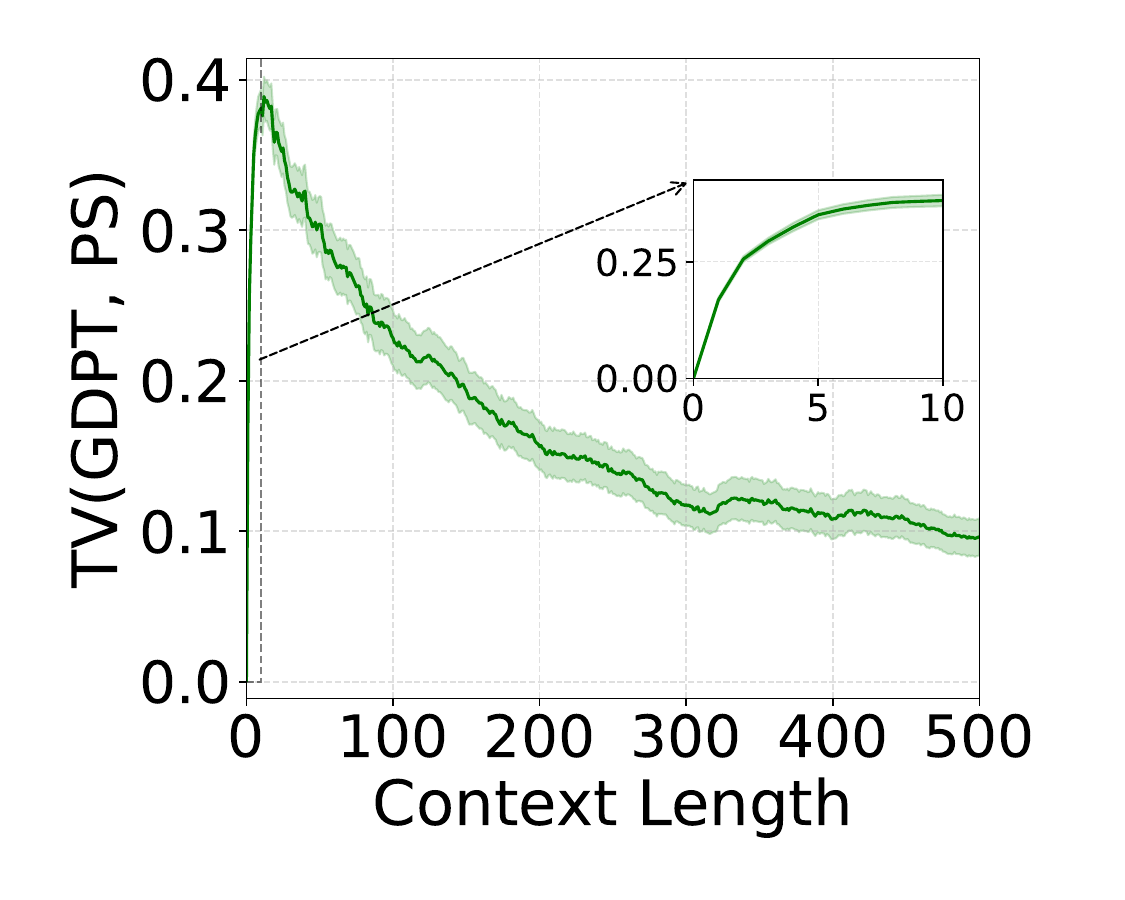}
\caption{Online TV distance}
\end{subfigure}
\hfill
\begin{subfigure}{0.245\textwidth}
\includegraphics[width=\textwidth]{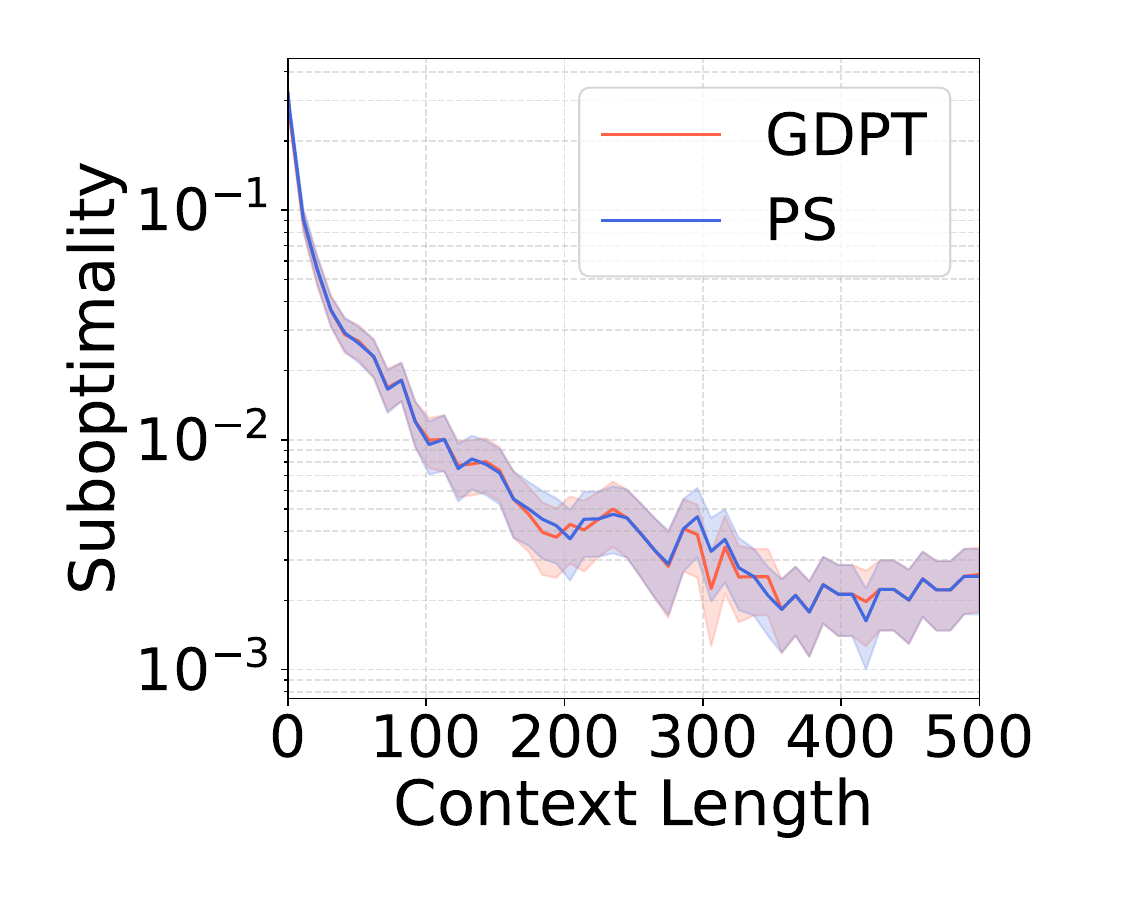}
\caption{Offline suboptimality}
\end{subfigure}
\hfill
\begin{subfigure}{0.245\textwidth}
\includegraphics[width=\textwidth]{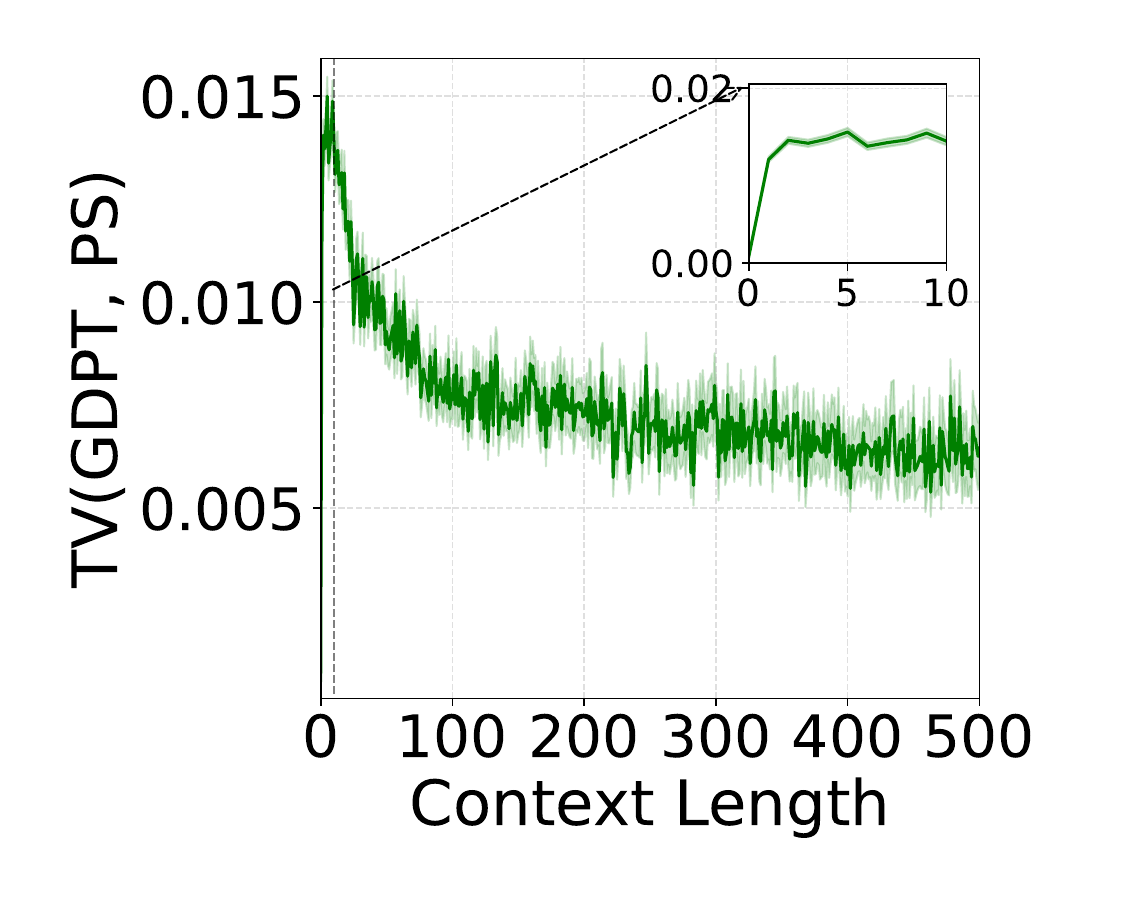}
\caption{Offline TV distance}
\end{subfigure}
\caption{Performance comparison of GDPT and Thompson Sampling in Bernoulli bandits under online and offline deployment.}
\label{fig_evaluation_bernoulli_bandit_GDPT}
\end{figure}

\begin{figure}[H]
    \centering
    \begin{subfigure}[b]{0.32\textwidth}
        \centering
        \includegraphics[width=\textwidth]{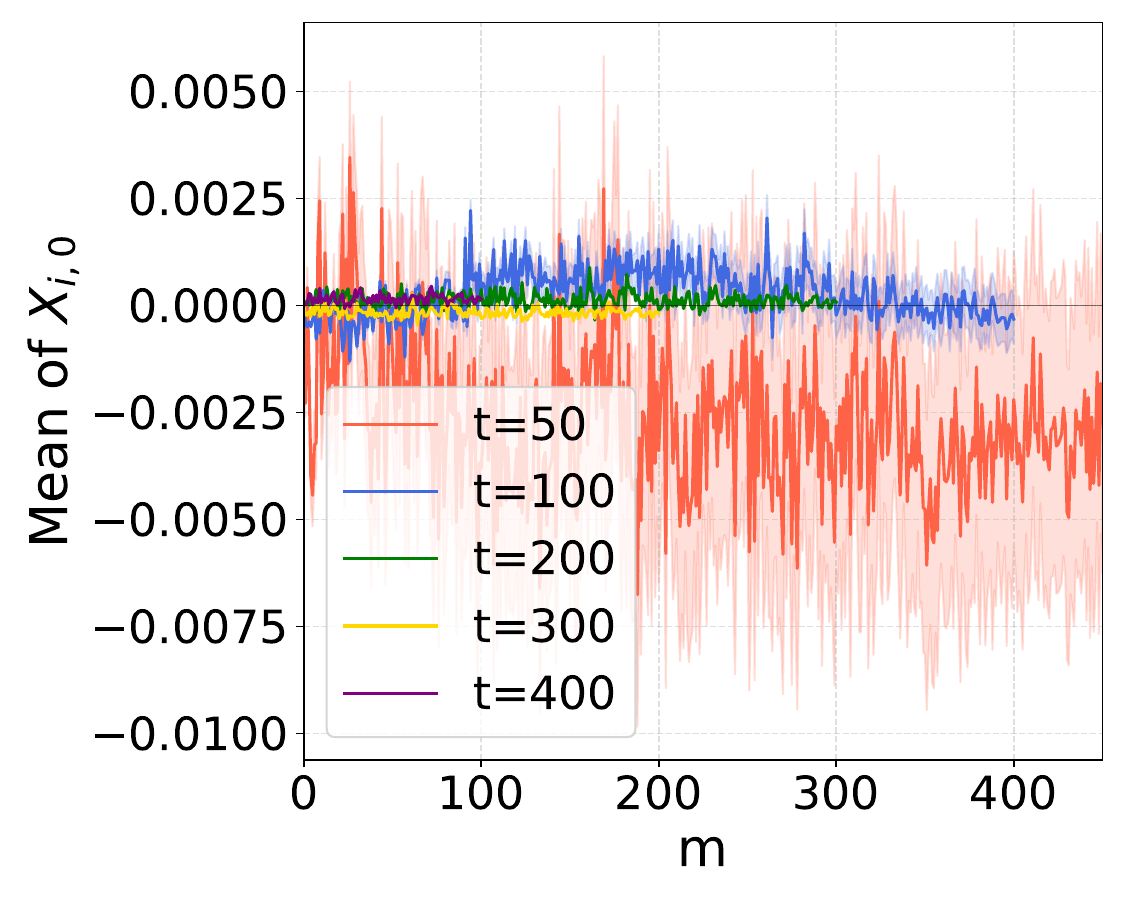}
        \caption{Arm 0}
    \end{subfigure}
    \begin{subfigure}[b]{0.32\textwidth}
        \centering
        \includegraphics[width=\textwidth]{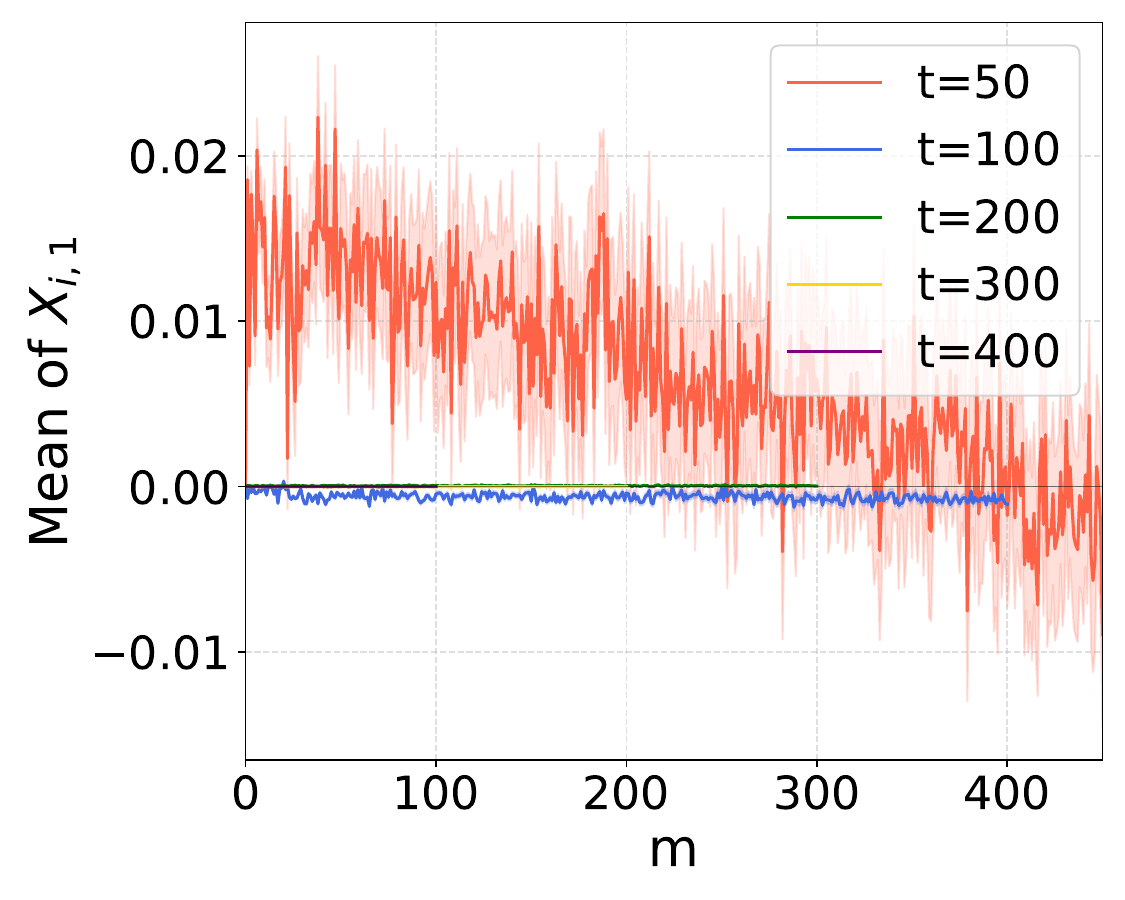}
        \caption{Arm 1}
    \end{subfigure}
    \begin{subfigure}[b]{0.32\textwidth}
        \centering
        \includegraphics[width=\textwidth]{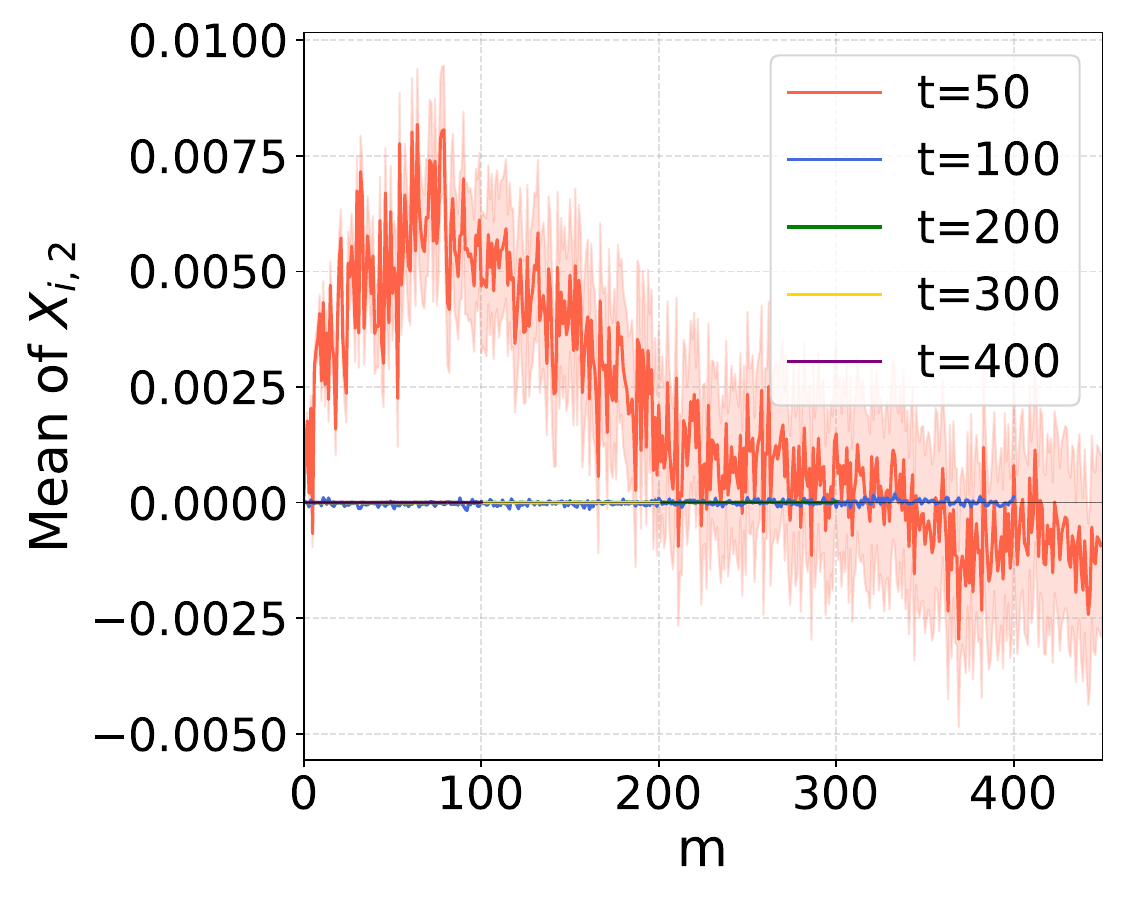}
        \caption{Arm 2}
    \end{subfigure}
    \\
    \begin{subfigure}[b]{0.32\textwidth}
        \centering
        \includegraphics[width=\textwidth]{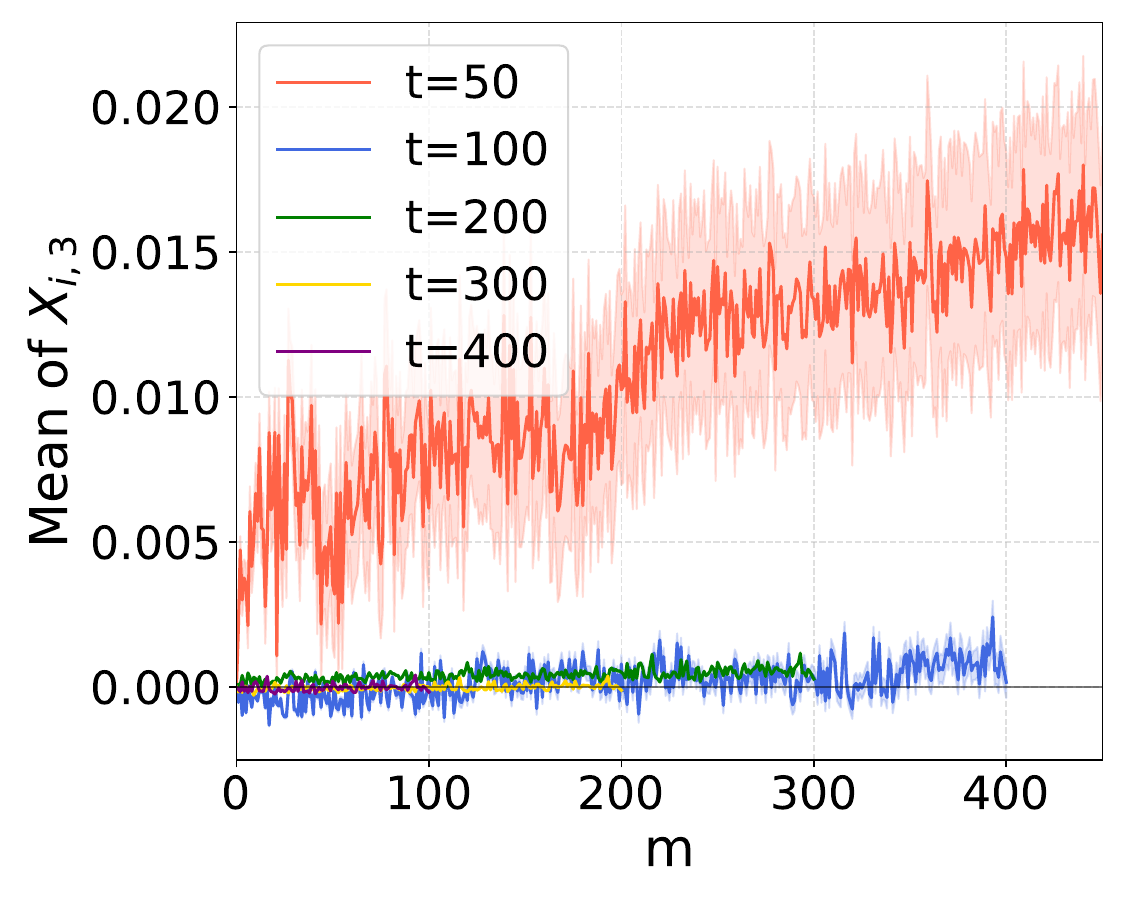}
        \caption{Arm 3}
    \end{subfigure}
    \begin{subfigure}[b]{0.32\textwidth}
        \centering
        \includegraphics[width=\textwidth]{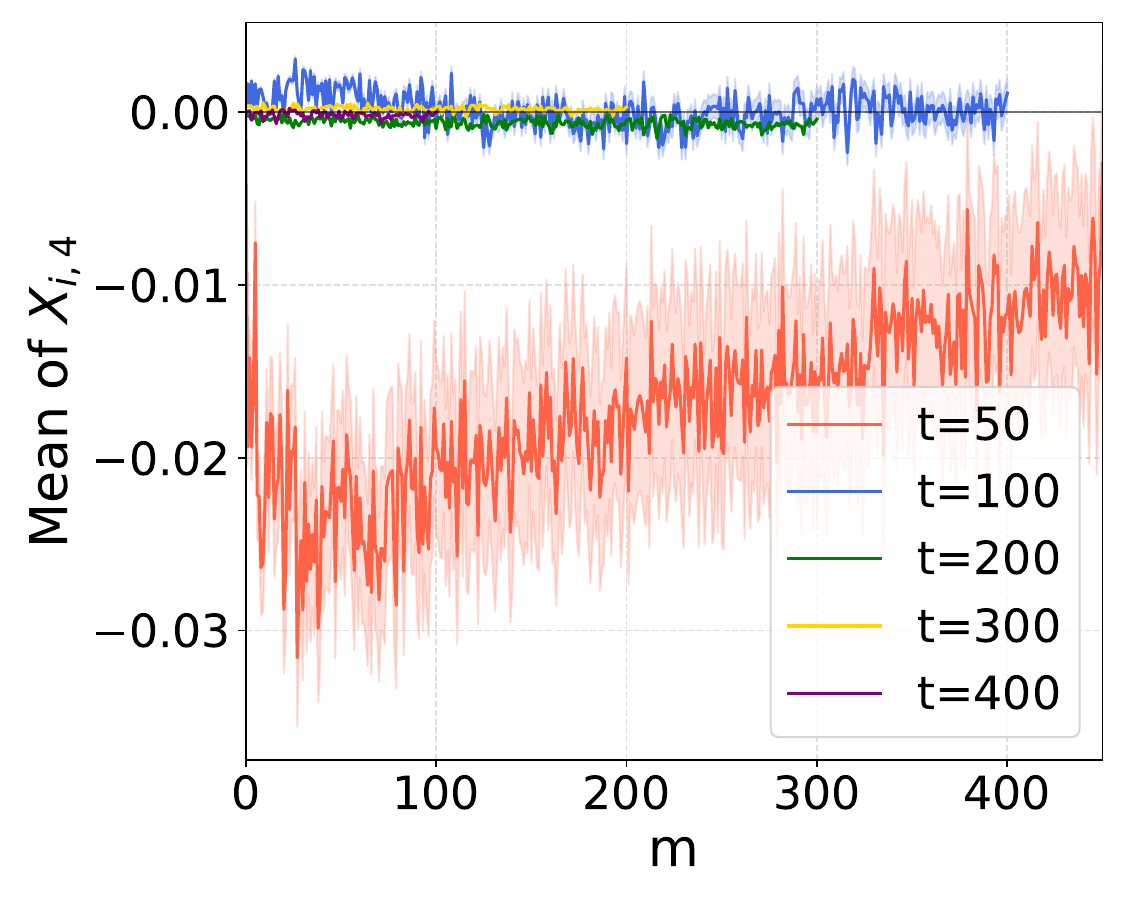}
        \caption{Arm 4}
    \end{subfigure}
    \caption{Per-arm martingale increments for GDPT in Bernoulli bandits with offline starting contexts, where $m$ and $t$ denote the lengths of the martingale continuation and the starting contexts, respectively.}
    \label{fig_martingale_increment_offline_bernoulli_GDPT}
\end{figure}

\begin{figure}[H]
    \centering
    \begin{subfigure}[b]{0.32\textwidth}
        \centering
        \includegraphics[width=\textwidth]{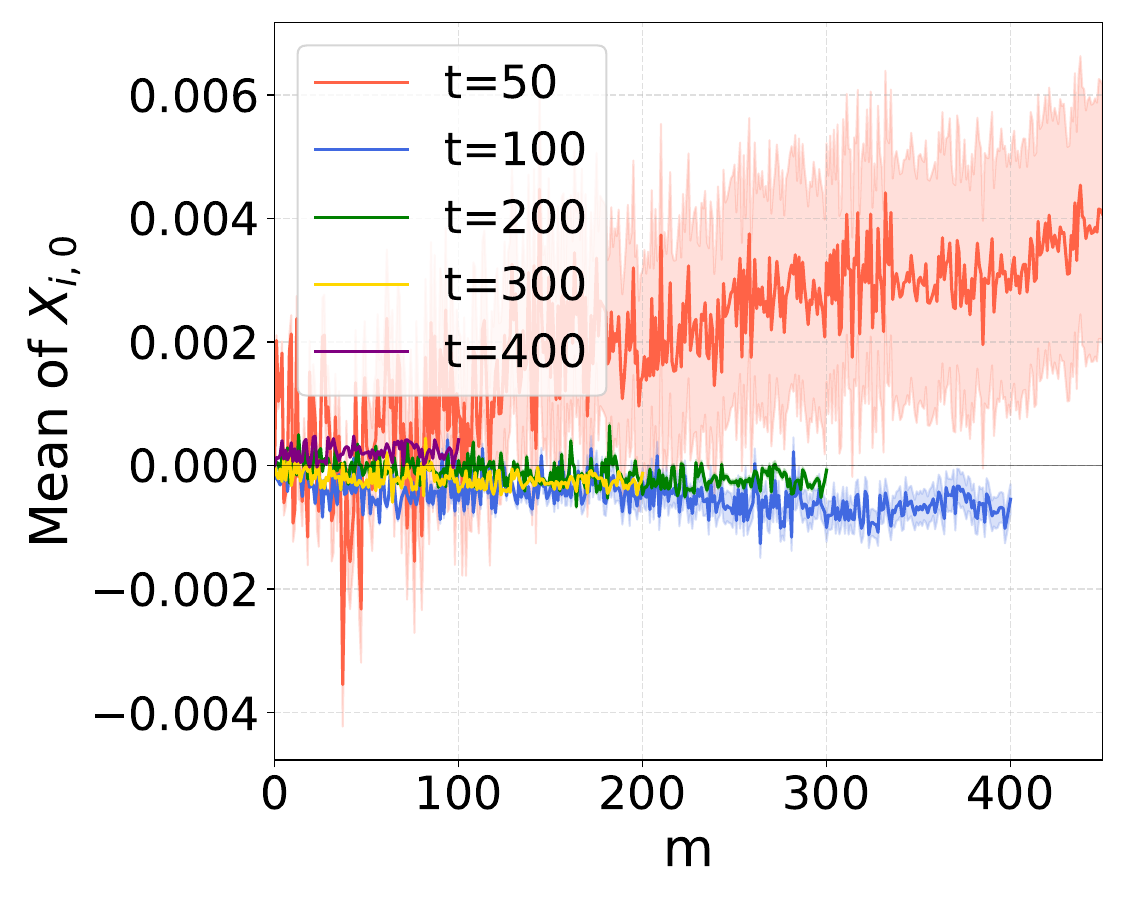}
        \caption{Arm 0}
    \end{subfigure}
    \begin{subfigure}[b]{0.32\textwidth}
        \centering
        \includegraphics[width=\textwidth]{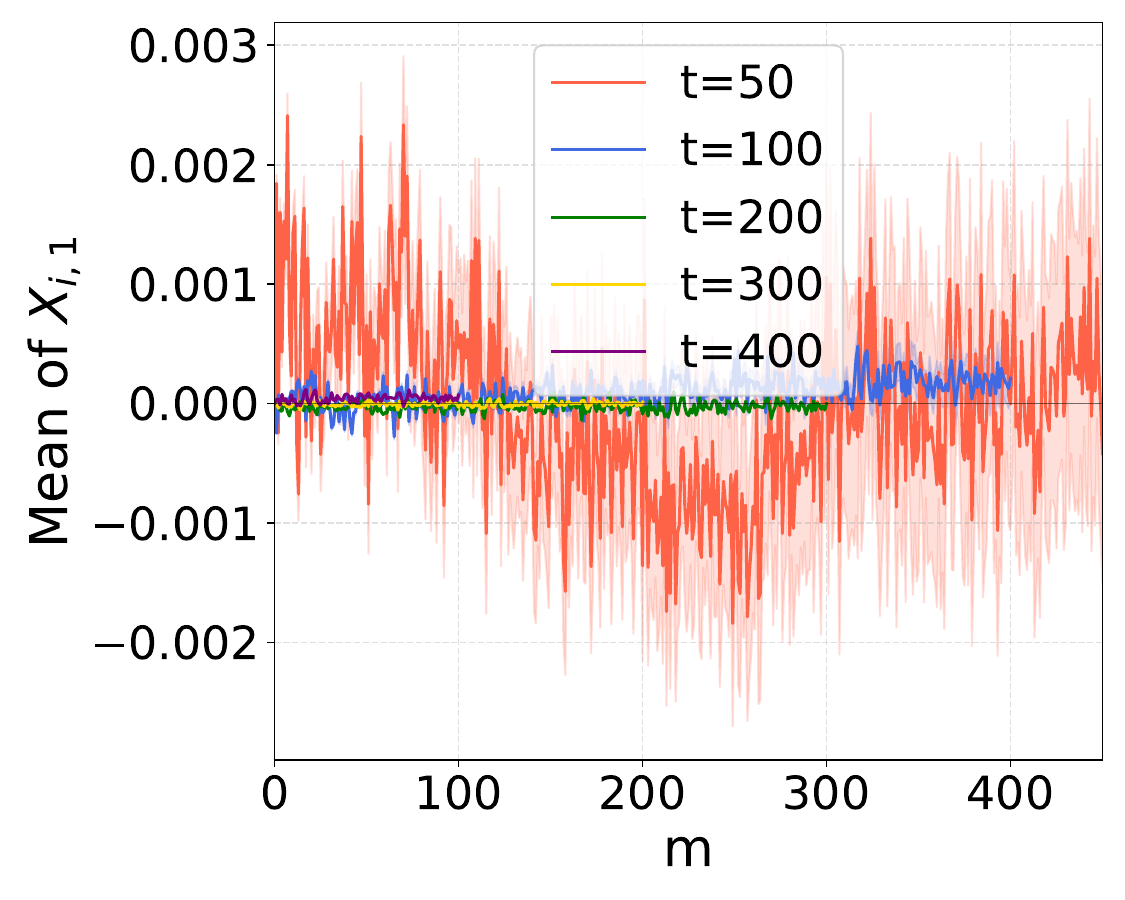}
        \caption{Arm 1}
    \end{subfigure}
    \begin{subfigure}[b]{0.32\textwidth}
        \centering
        \includegraphics[width=\textwidth]{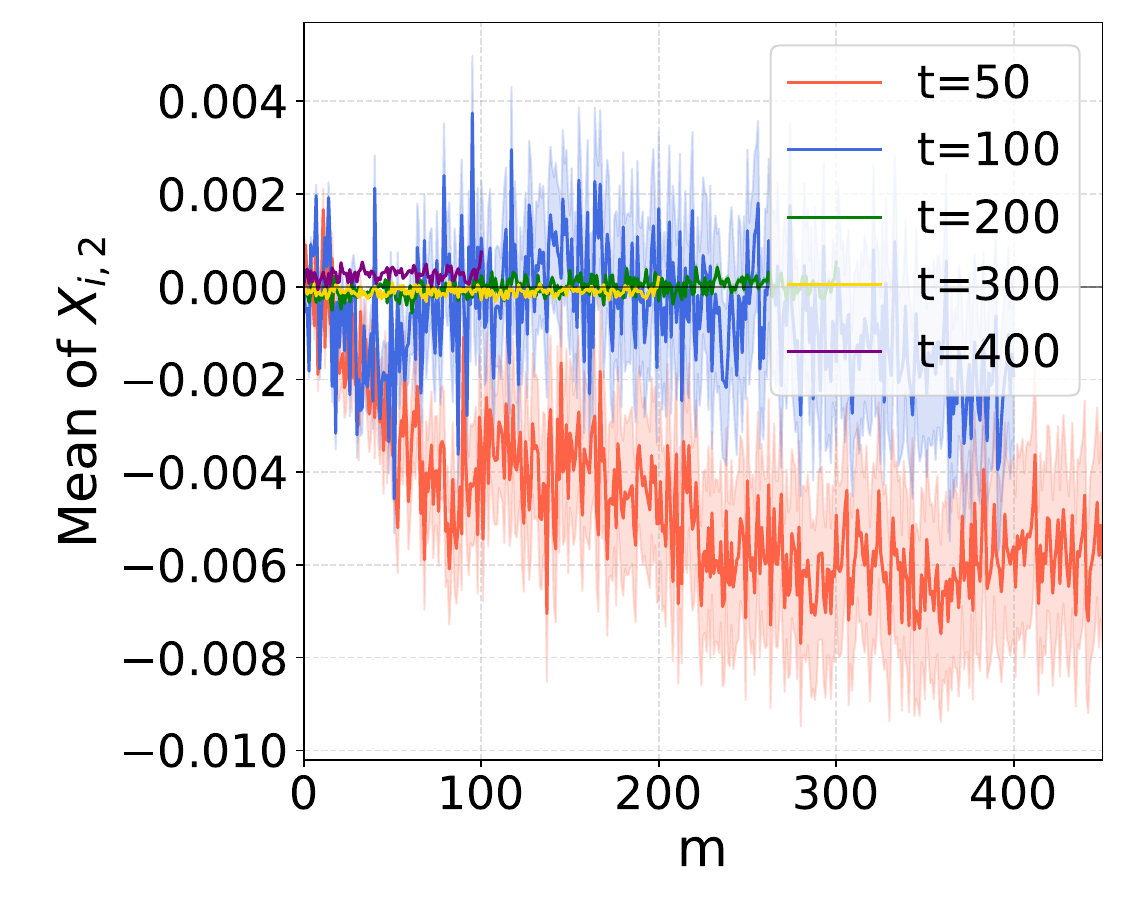}
        \caption{Arm 2}
    \end{subfigure}
    \\
    \begin{subfigure}[b]{0.32\textwidth}
        \centering
        \includegraphics[width=\textwidth]{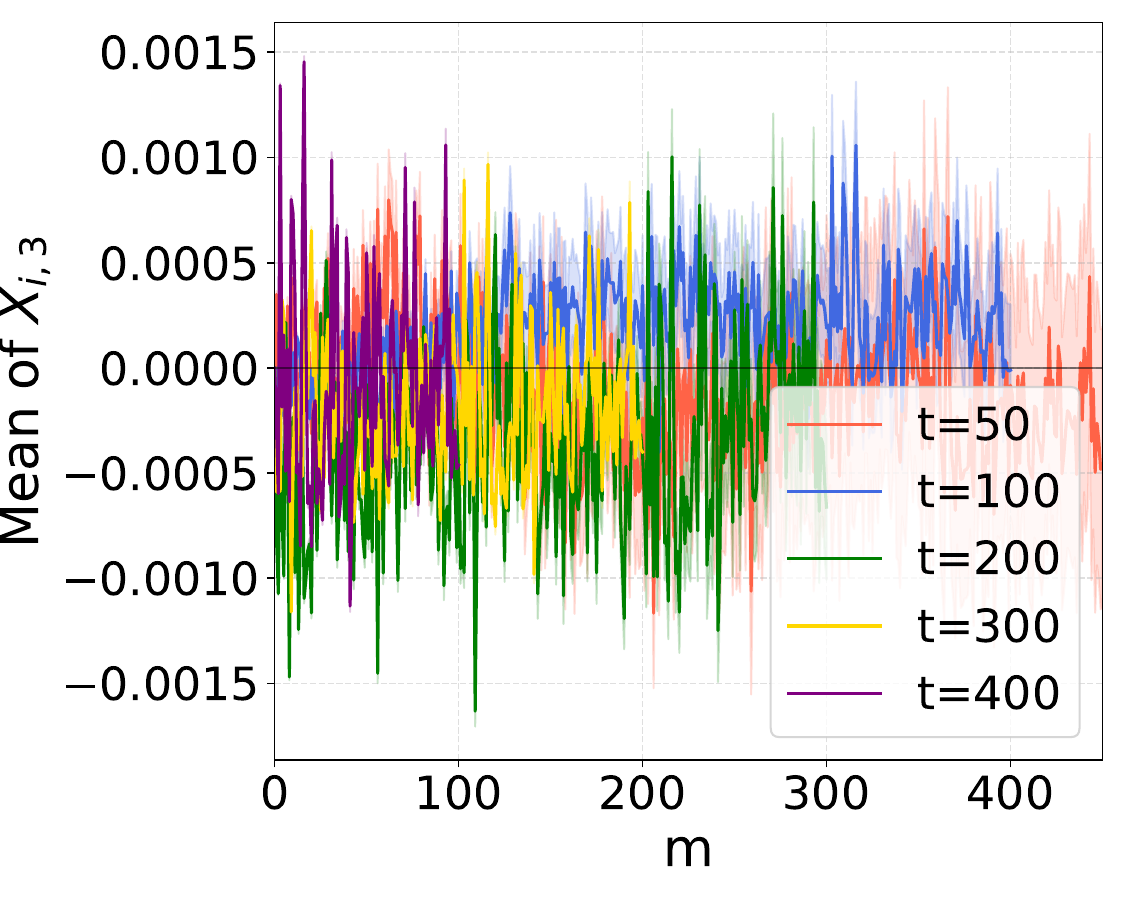}
        \caption{Arm 3}
    \end{subfigure}
    \begin{subfigure}[b]{0.32\textwidth}
        \centering
        \includegraphics[width=\textwidth]{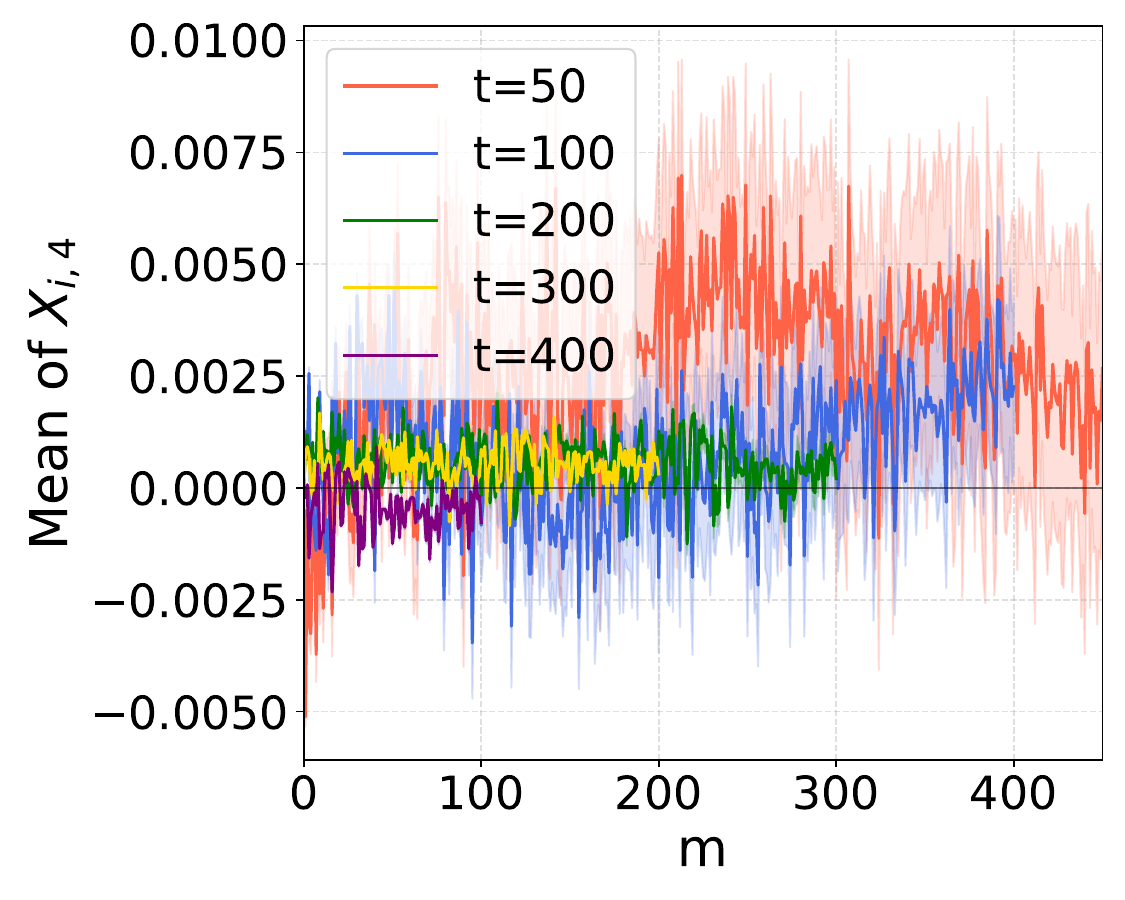}
        \caption{Arm 4}
    \end{subfigure}
    \caption{Per-arm martingale increments for GDPT in Bernoulli bandits with online starting contexts, where $m$ and $t$ denote the lengths of the martingale continuation and the starting contexts, respectively.}
    \label{fig_martingale_increment_online_bernoulli_GDPT}
\end{figure}

\subsubsection{Linear evaluation figures and martingale plots}
\begin{figure}[H]
\centering
\begin{subfigure}{0.245\textwidth}
\includegraphics[width=\textwidth]{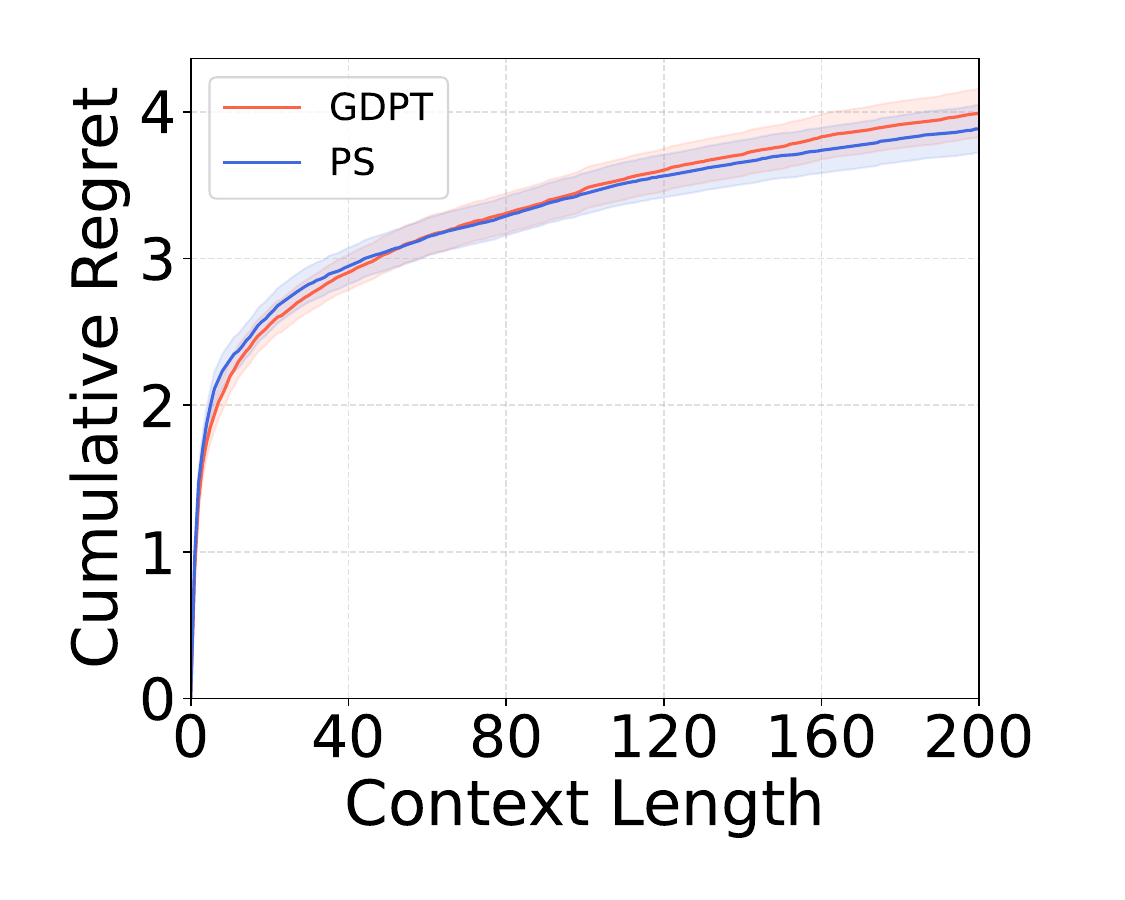}
\caption{Online regret}
\end{subfigure}
\hfill
\begin{subfigure}{0.245\textwidth}
\includegraphics[width=\textwidth]{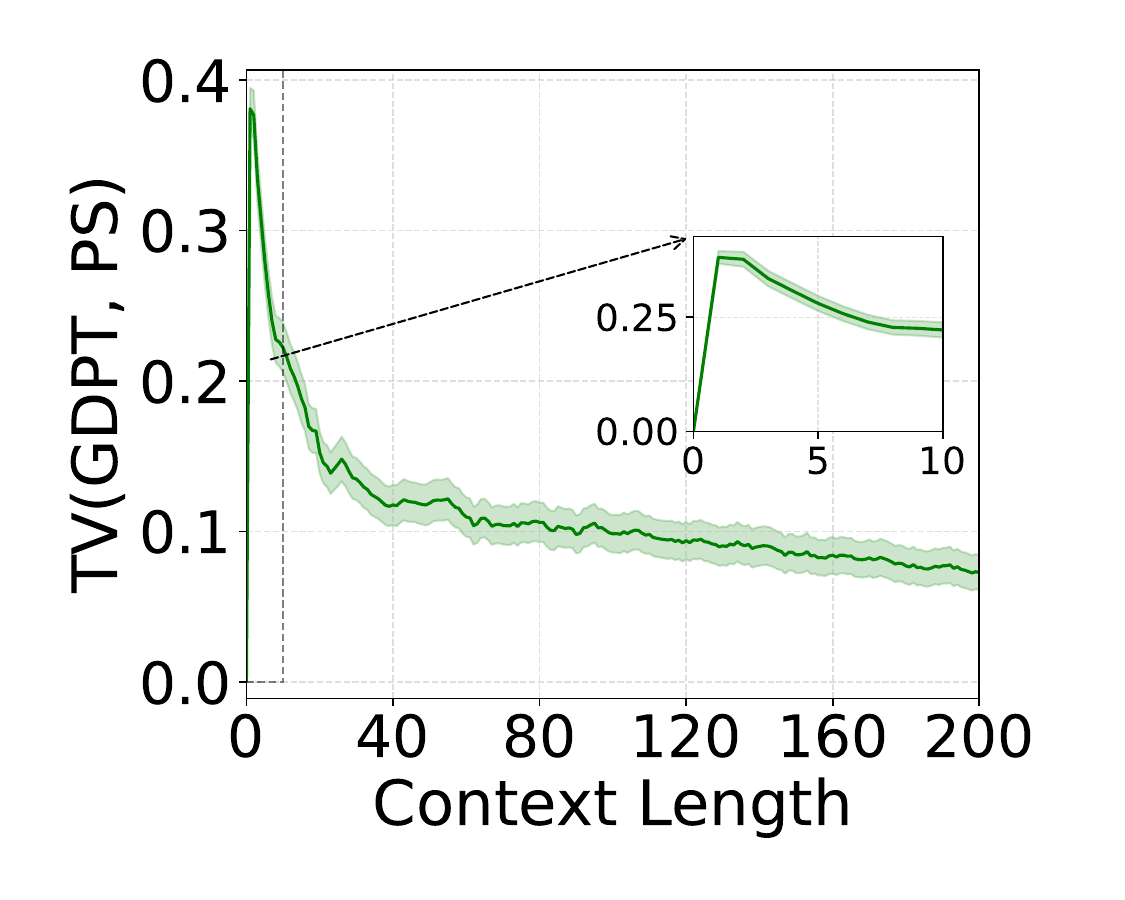}
\caption{Online TV distance}
\end{subfigure}
\hfill
\begin{subfigure}{0.245\textwidth}
\includegraphics[width=\textwidth]{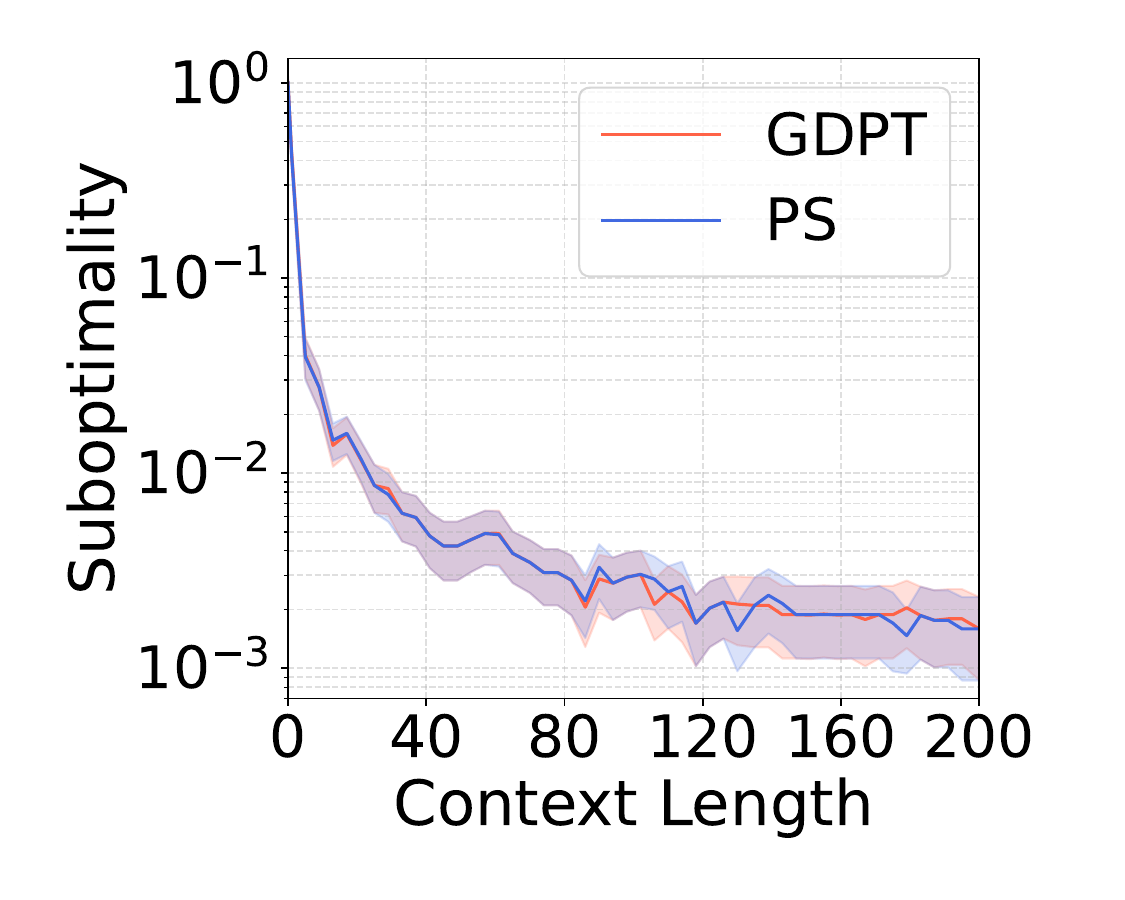}
\caption{Offline suboptimality}
\end{subfigure}
\hfill
\begin{subfigure}{0.245\textwidth}
\includegraphics[width=\textwidth]{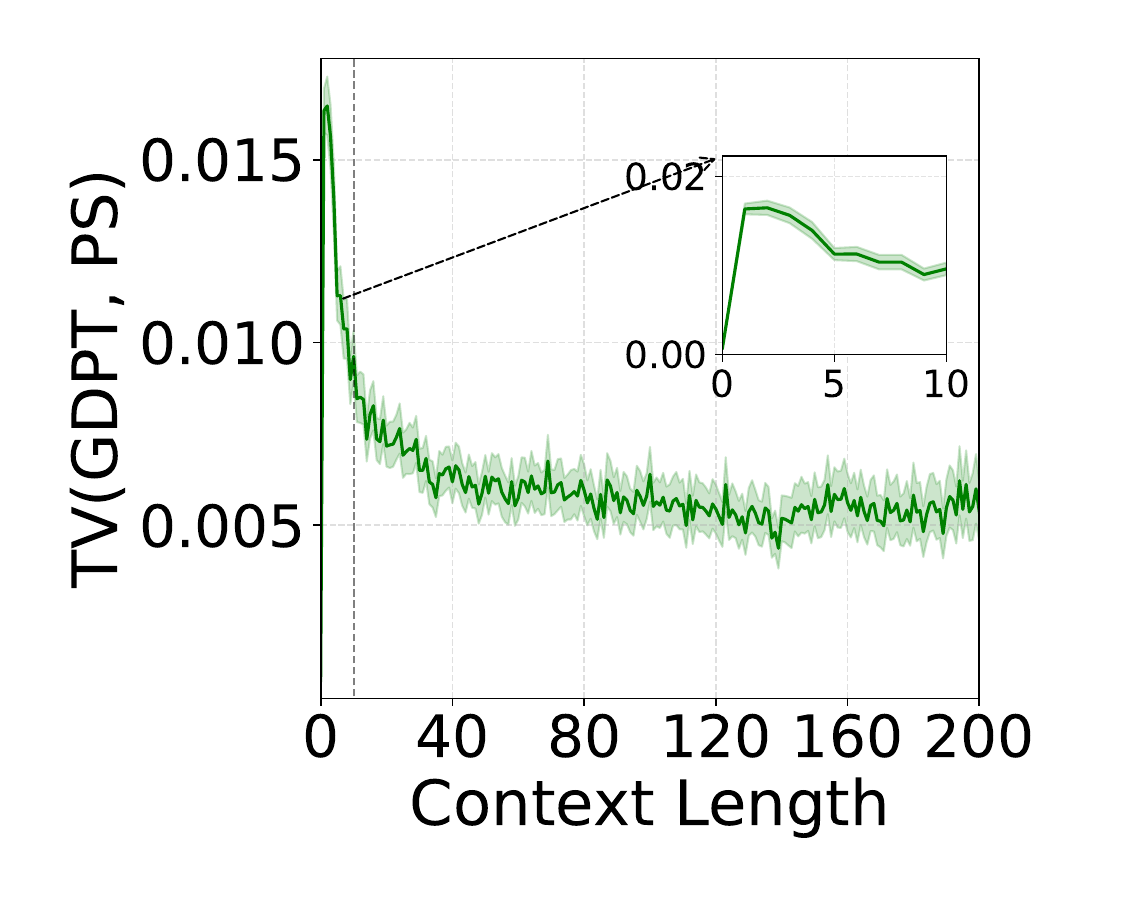}
\caption{Offline TV distance}
\end{subfigure}
\caption{Performance comparison of GDPT and Thompson Sampling in linear bandits under online and offline deployment.}
\label{fig_evaluation_linear_bandit_GDPT}
\end{figure}

\begin{figure}[H]
    \centering
    \begin{subfigure}[b]{0.32\textwidth}
        \centering
        \includegraphics[width=\textwidth]{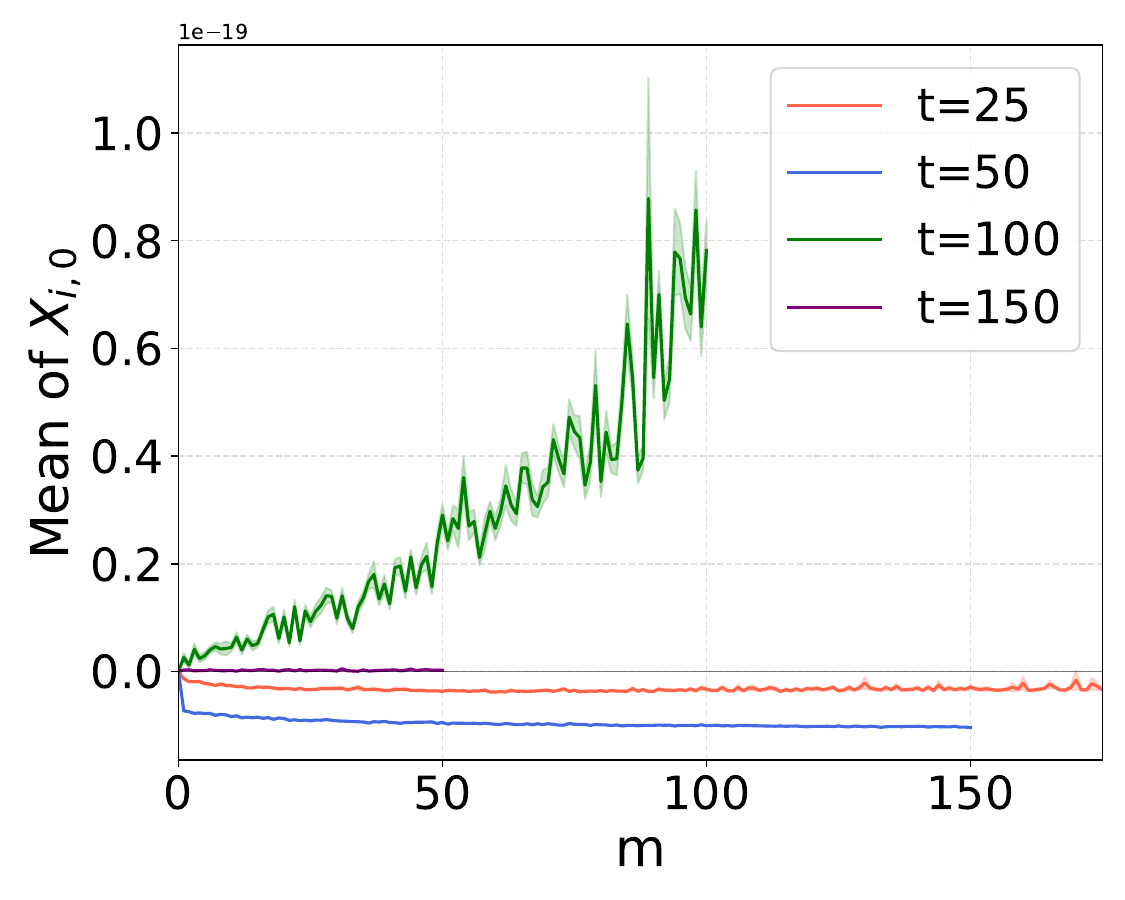}
        \caption{Arm 0}
    \end{subfigure}
    \begin{subfigure}[b]{0.32\textwidth}
        \centering
        \includegraphics[width=\textwidth]{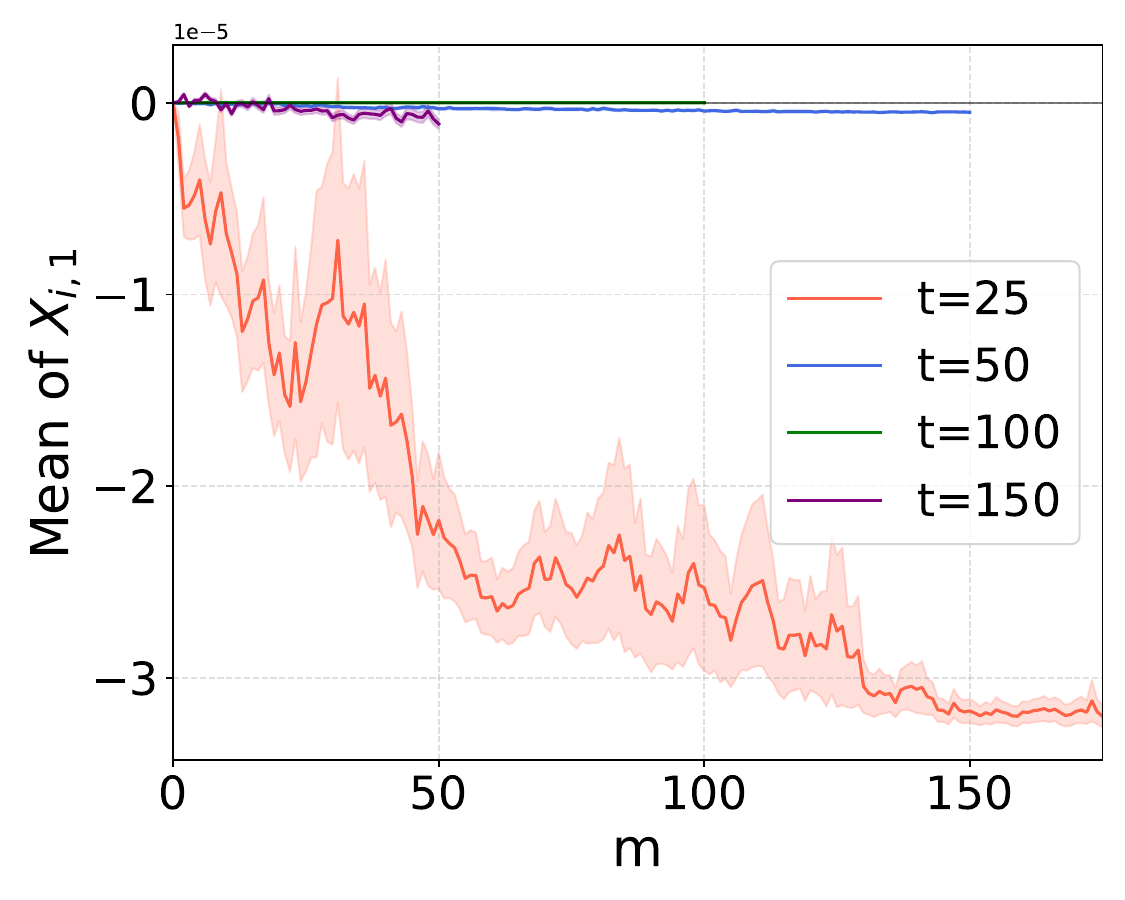}
        \caption{Arm 1}
    \end{subfigure}
    \begin{subfigure}[b]{0.32\textwidth}
        \centering
        \includegraphics[width=\textwidth]{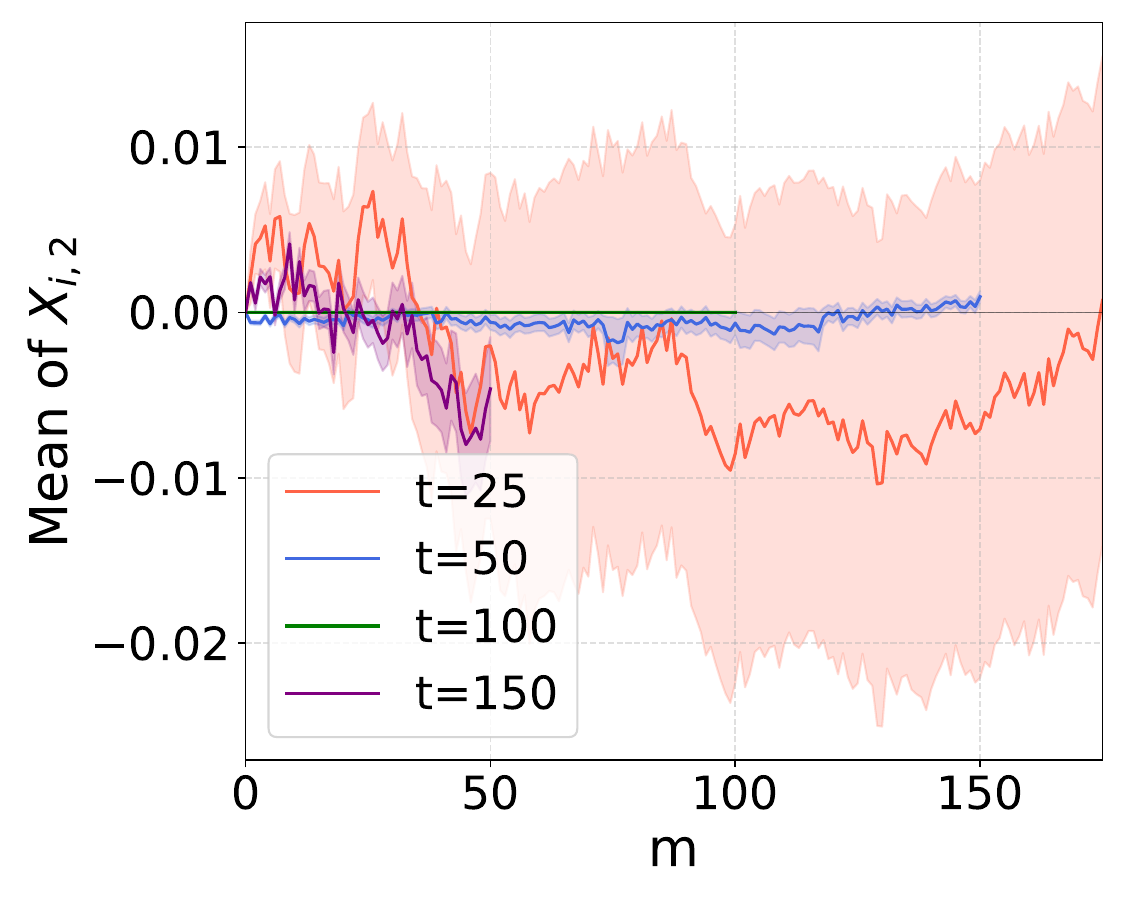}
        \caption{Arm 2}
    \end{subfigure}
    \\
    \begin{subfigure}[b]{0.32\textwidth}
        \centering
        \includegraphics[width=\textwidth]{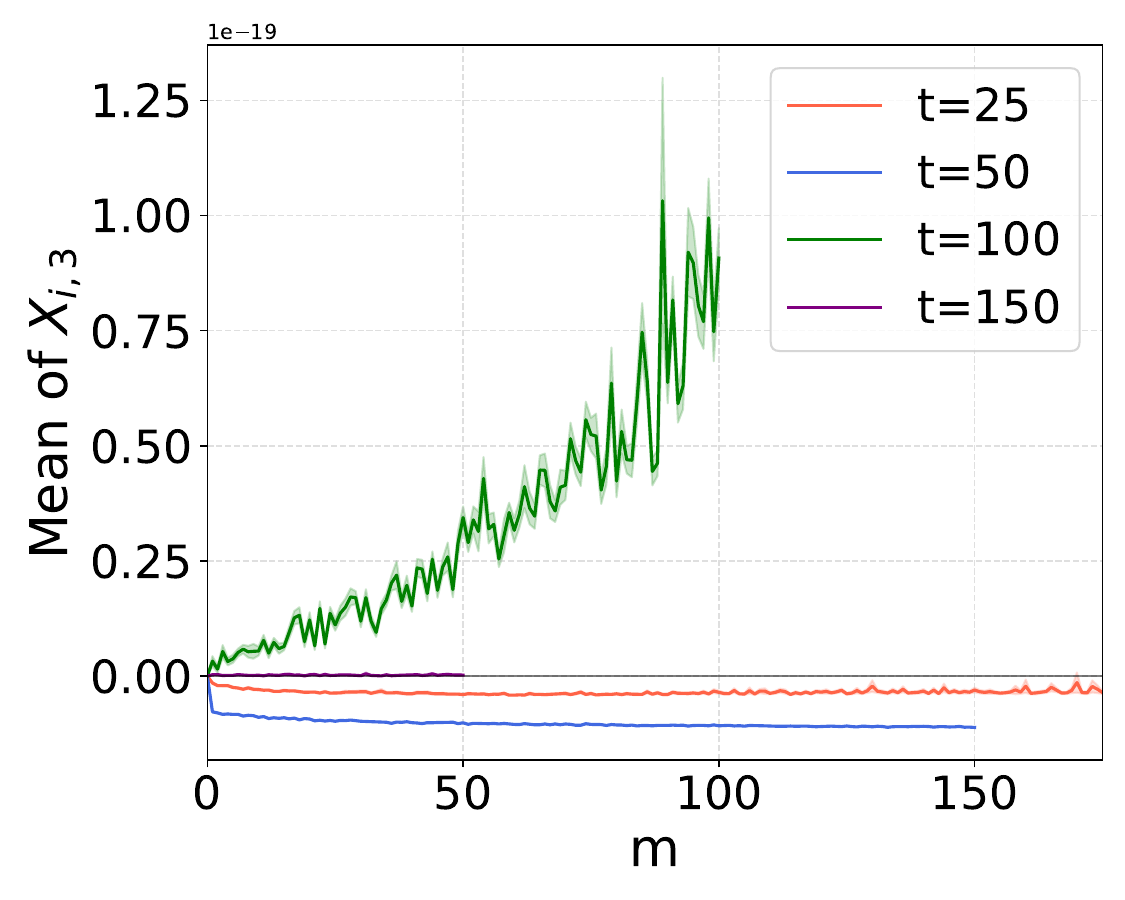}
        \caption{Arm 3}
    \end{subfigure}
    \begin{subfigure}[b]{0.32\textwidth}
        \centering
        \includegraphics[width=\textwidth]{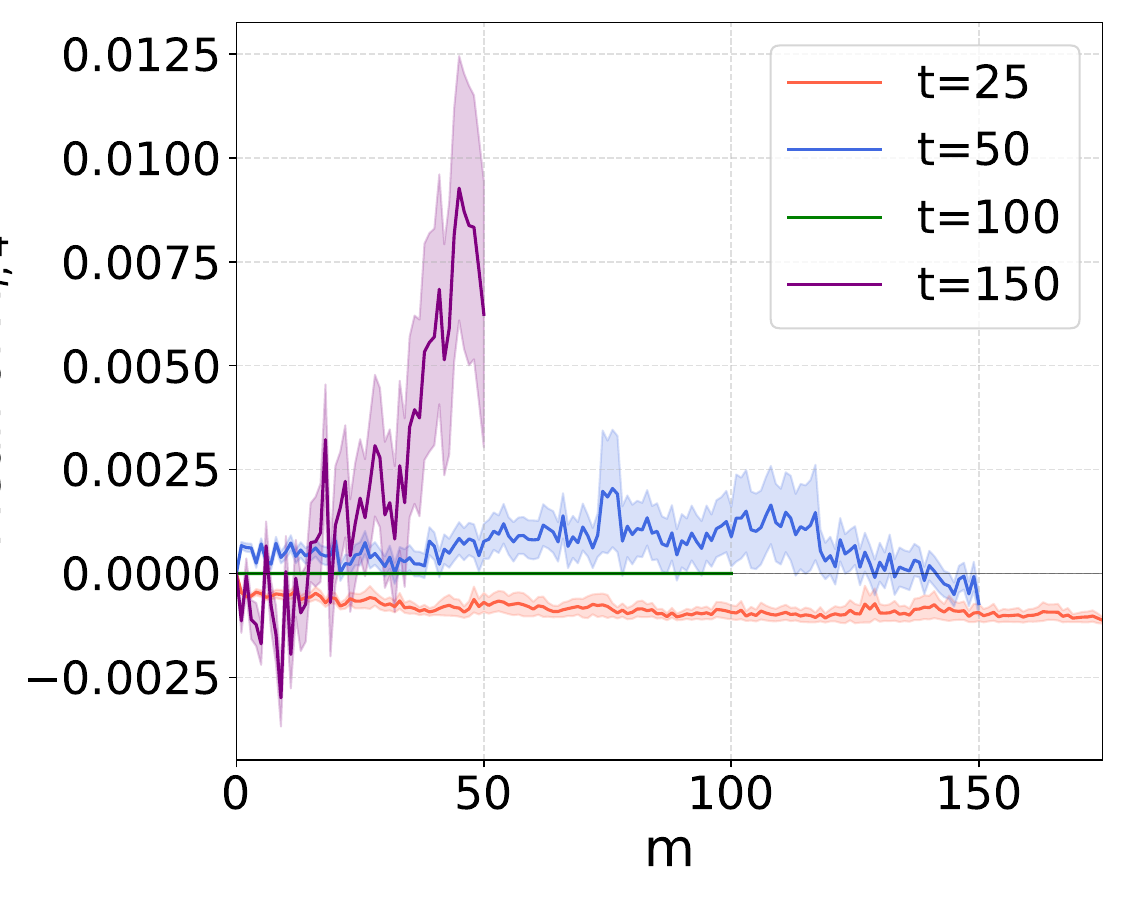}
        \caption{Arm 4}
    \end{subfigure}
    \\
    \begin{subfigure}[b]{0.32\textwidth}
        \centering
        \includegraphics[width=\textwidth]{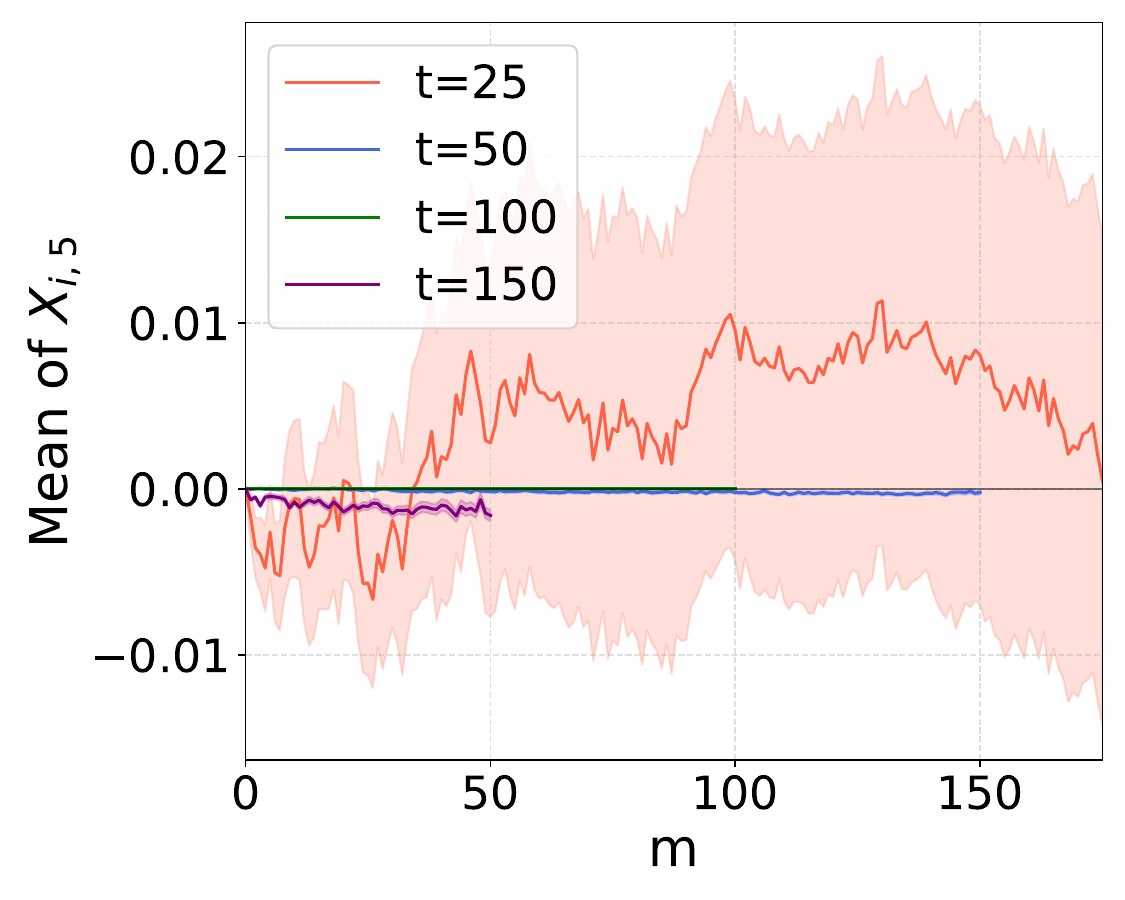}
        \caption{Arm 5}
    \end{subfigure}
    \begin{subfigure}[b]{0.32\textwidth}
        \centering
        \includegraphics[width=\textwidth]{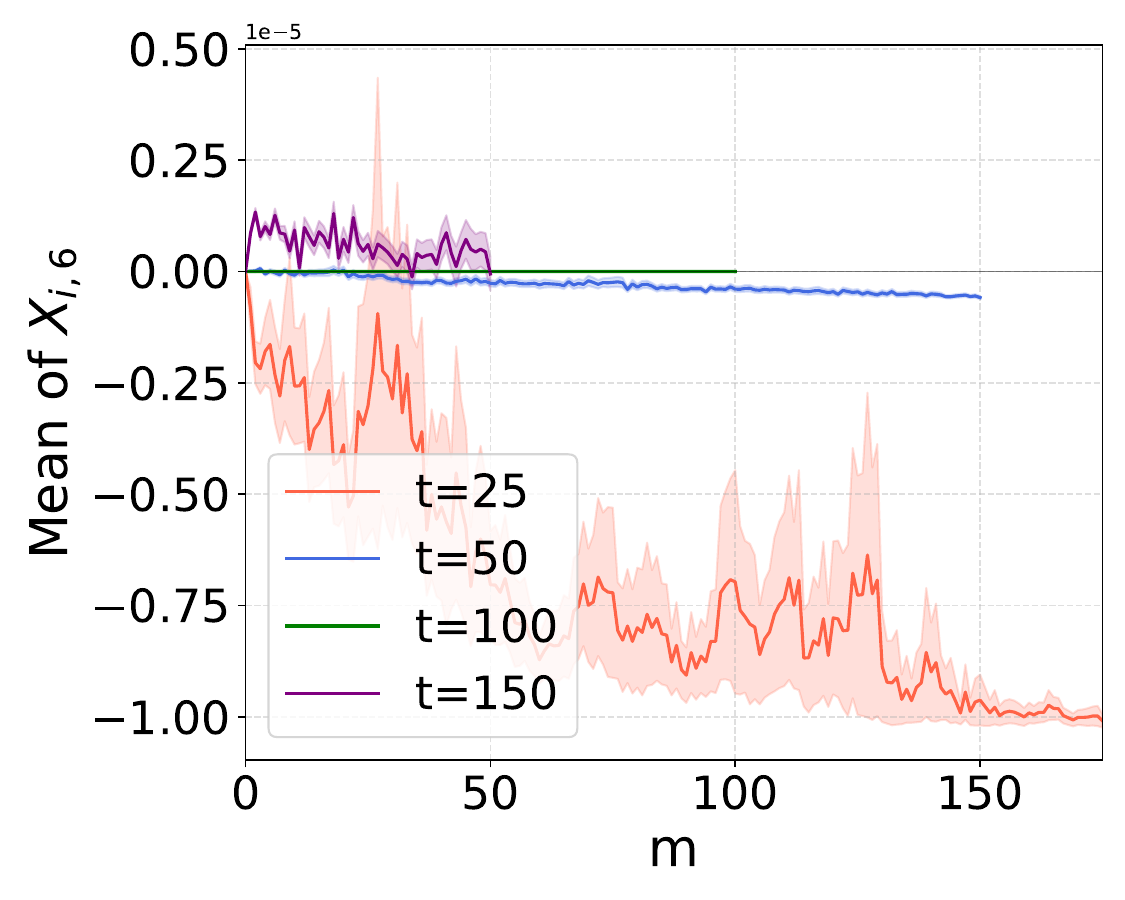}
        \caption{Arm 6}
    \end{subfigure}
    \begin{subfigure}[b]{0.32\textwidth}
        \centering
        \includegraphics[width=\textwidth]{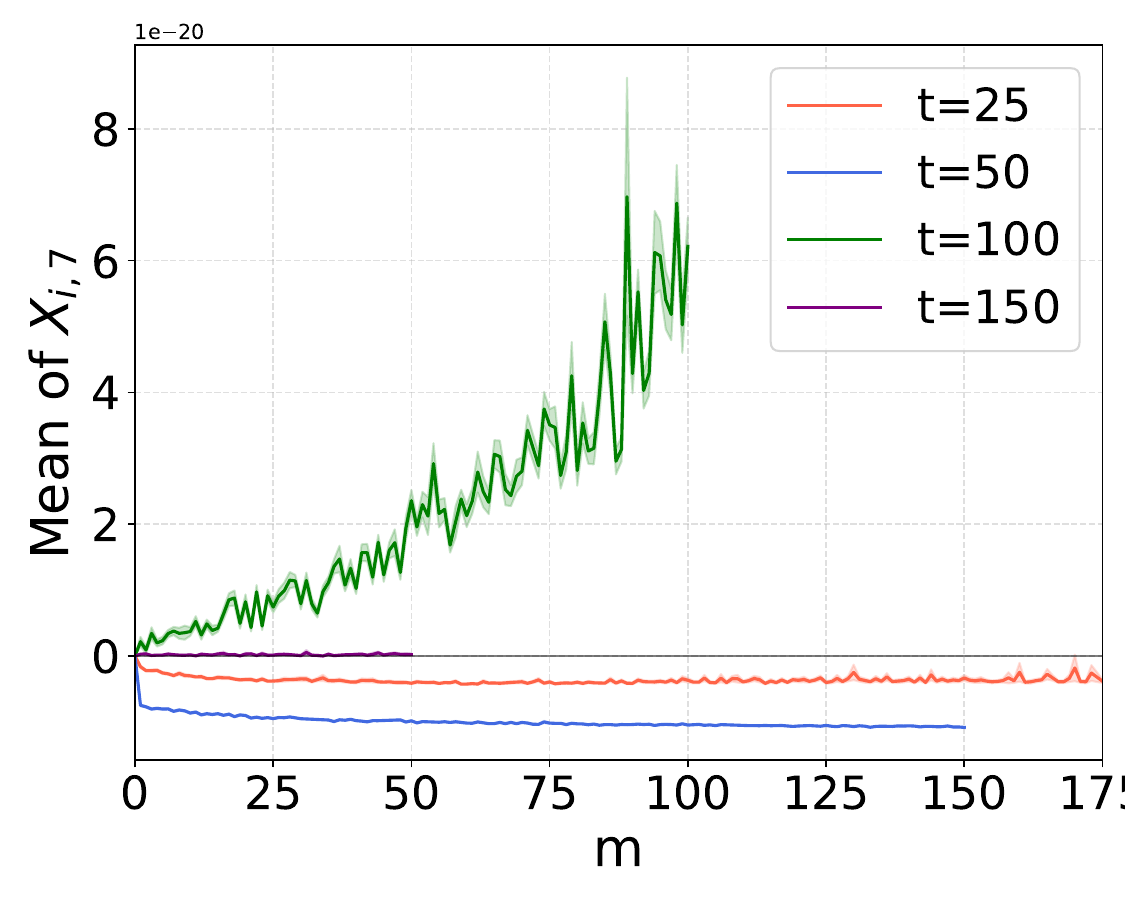}
        \caption{Arm 7}
    \end{subfigure}
    \\
    \begin{subfigure}[b]{0.32\textwidth}
        \centering
        \includegraphics[width=\textwidth]{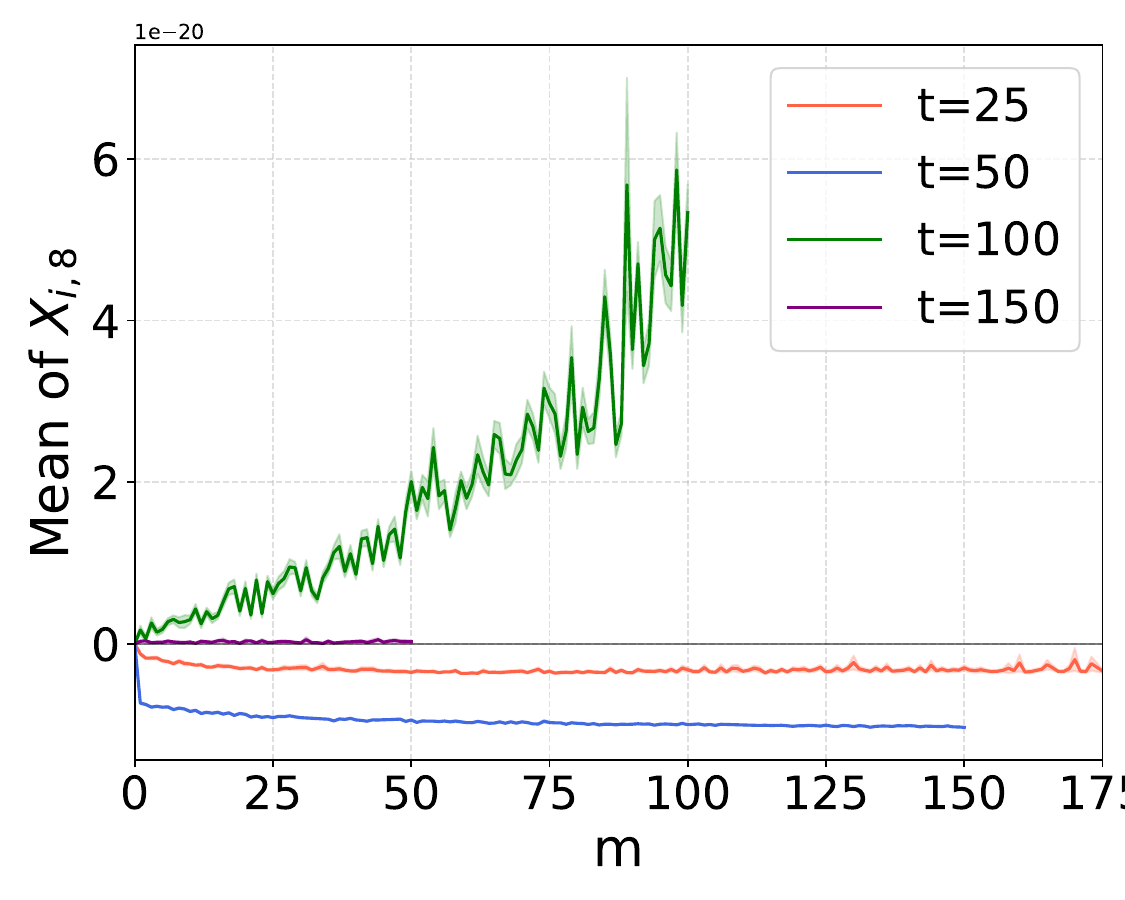}
        \caption{Arm 8}
    \end{subfigure}
    \begin{subfigure}[b]{0.32\textwidth}
        \centering
        \includegraphics[width=\textwidth]{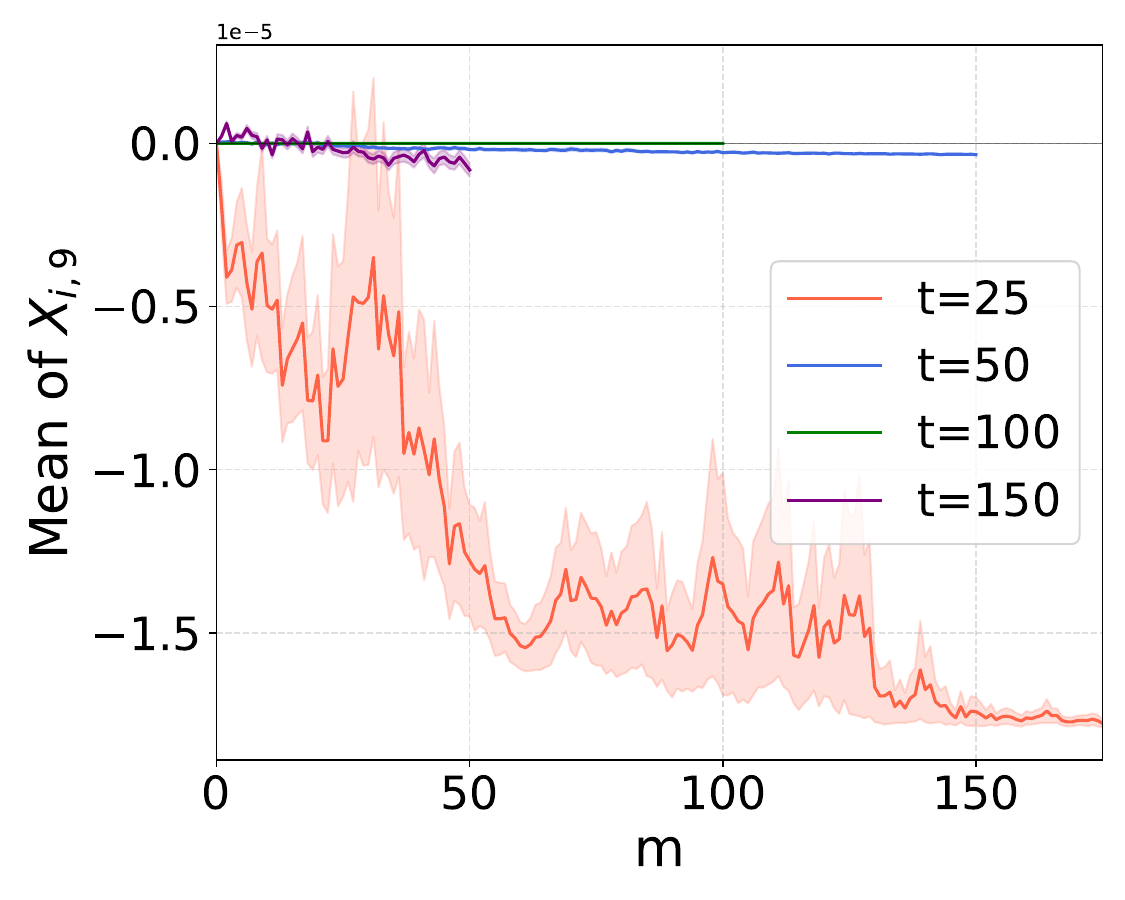}
        \caption{Arm 9}
    \end{subfigure}
    \caption{Per-arm martingale increments for GDPT in linear bandits with offline starting contexts, where $m$ and $t$ denote the lengths of the martingale continuation and the starting contexts, respectively.}
    \label{fig_martingale_increment_offline_linear_GDPT}
\end{figure}

\begin{figure}[H]
    \centering
    \begin{subfigure}[b]{0.32\textwidth}
        \centering
        \includegraphics[width=\textwidth]{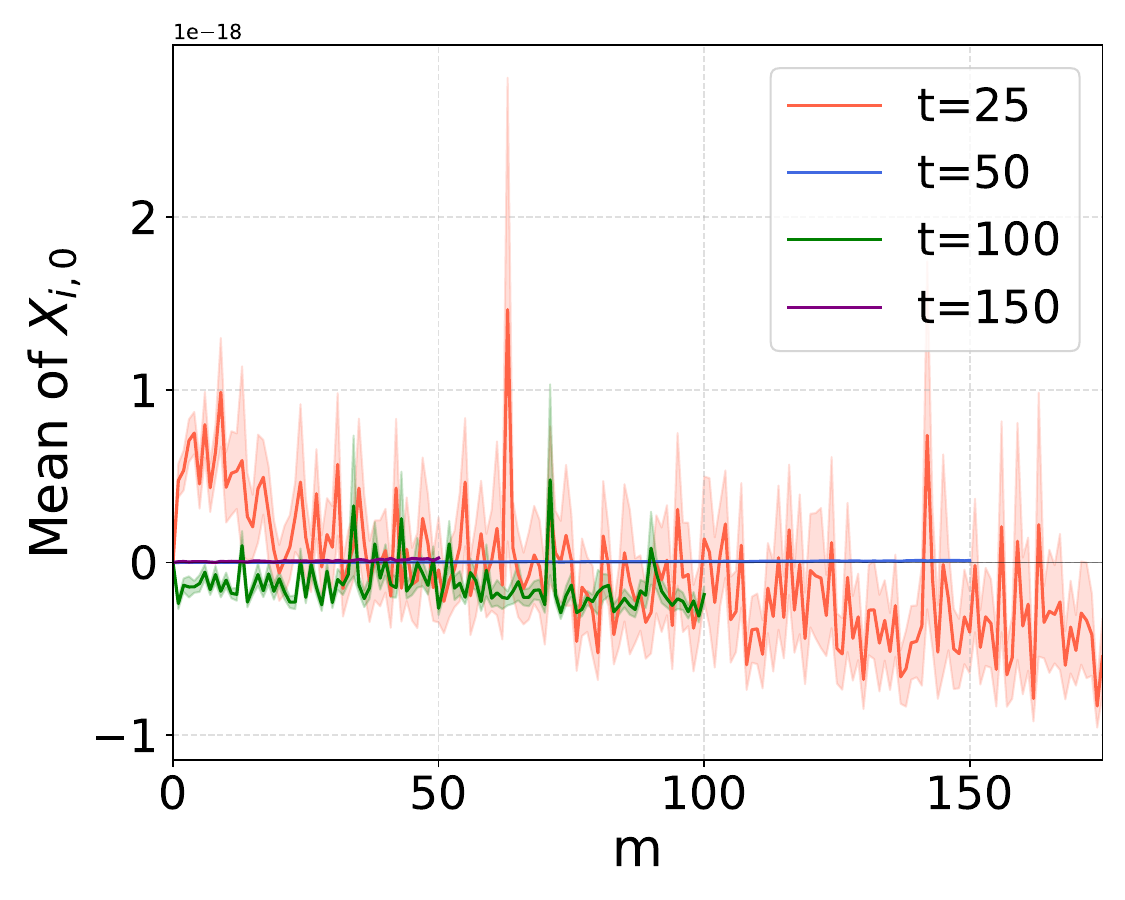}
        \caption{Arm 0}
    \end{subfigure}
    \begin{subfigure}[b]{0.32\textwidth}
        \centering
        \includegraphics[width=\textwidth]{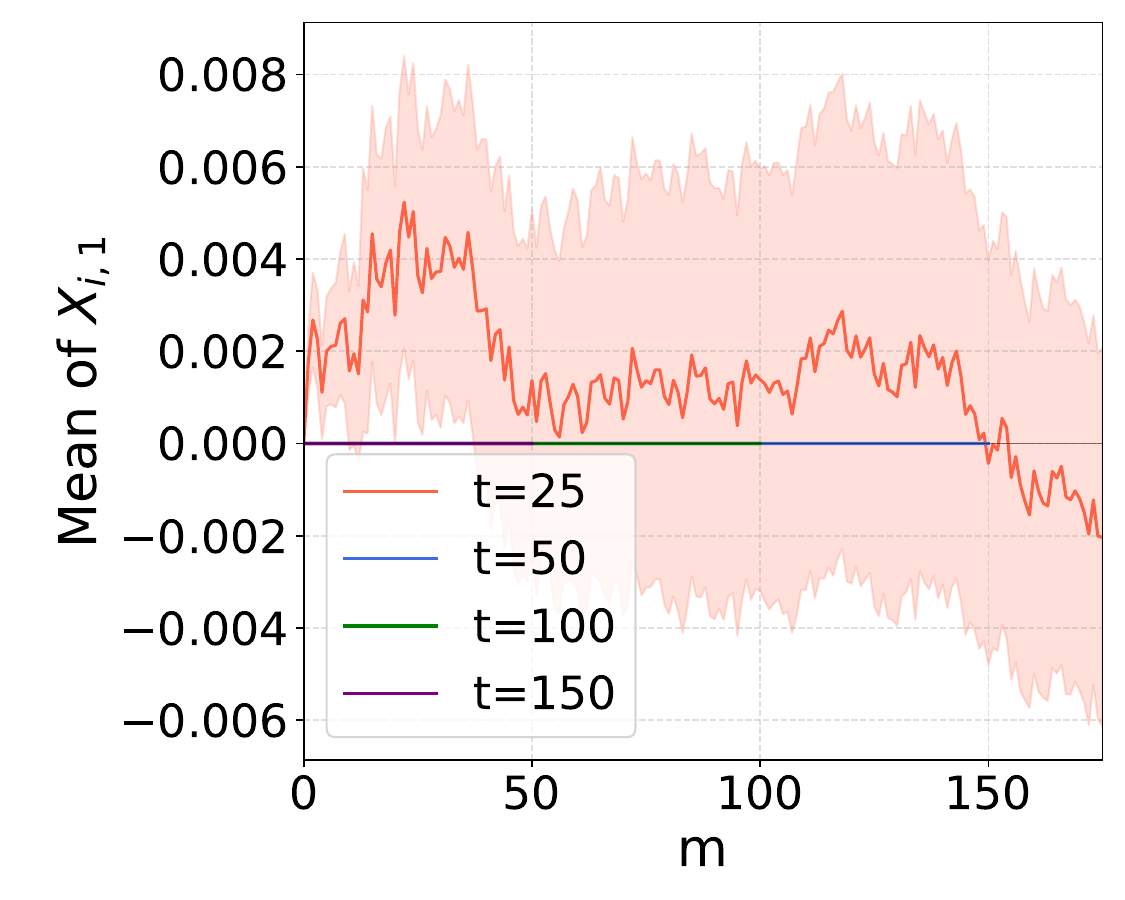}
        \caption{Arm 1}
    \end{subfigure}
    \begin{subfigure}[b]{0.32\textwidth}
        \centering
        \includegraphics[width=\textwidth]{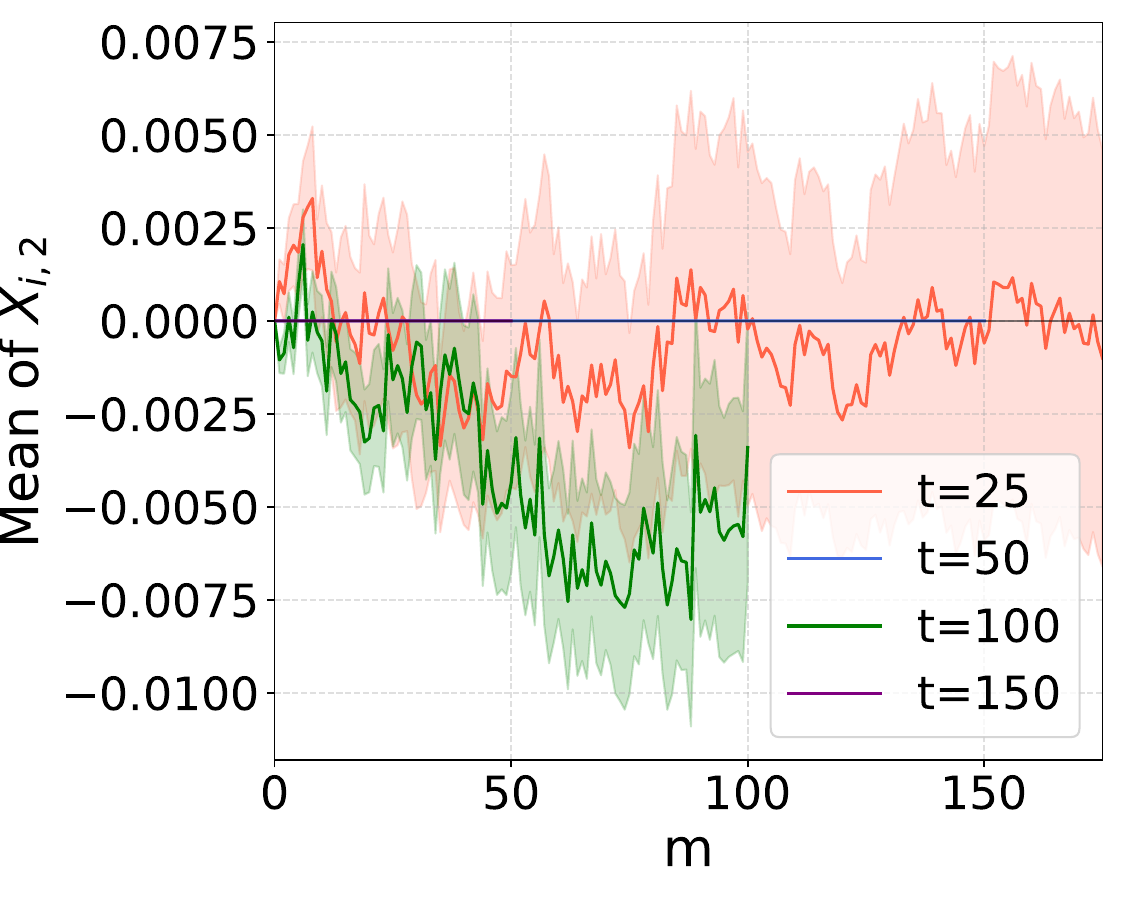}
        \caption{Arm 2}
    \end{subfigure}
    \\
    \begin{subfigure}[b]{0.32\textwidth}
        \centering
        \includegraphics[width=\textwidth]{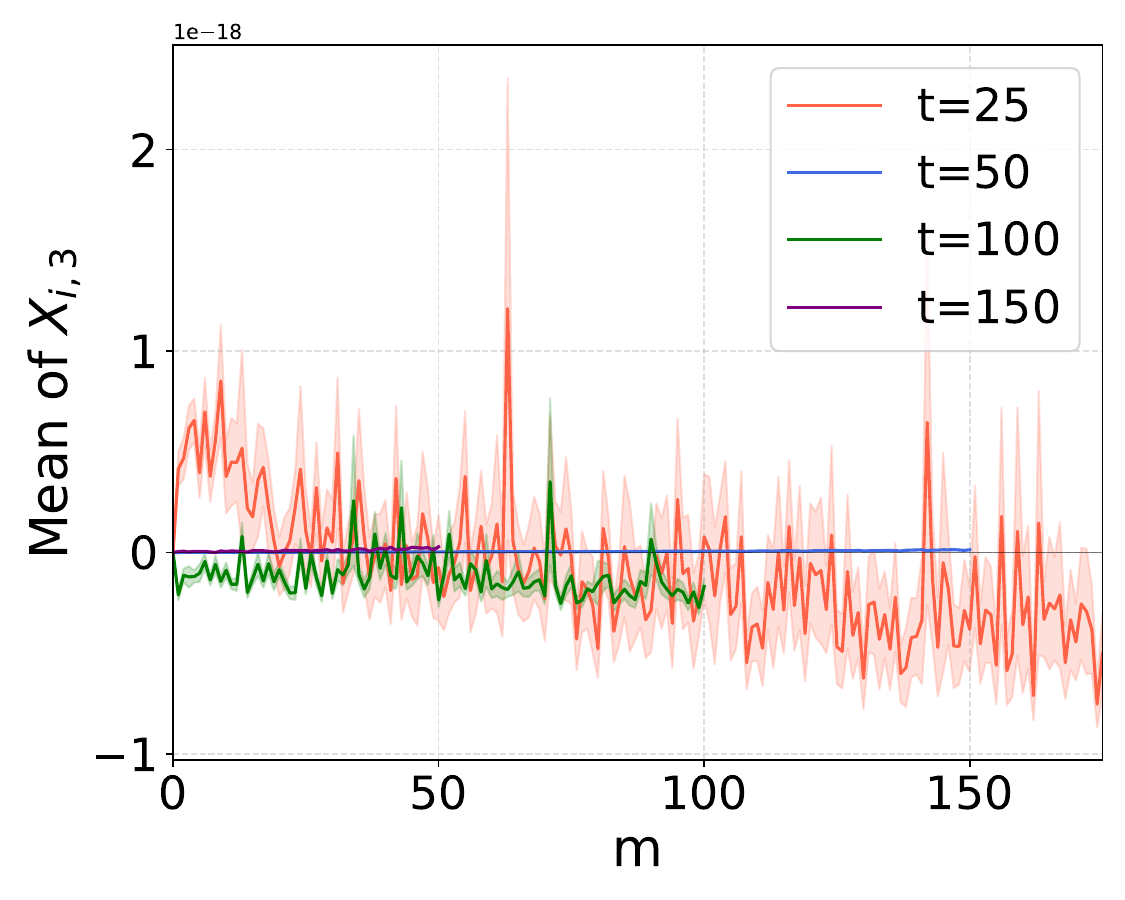}
        \caption{Arm 3}
    \end{subfigure}
    \begin{subfigure}[b]{0.32\textwidth}
        \centering
        \includegraphics[width=\textwidth]{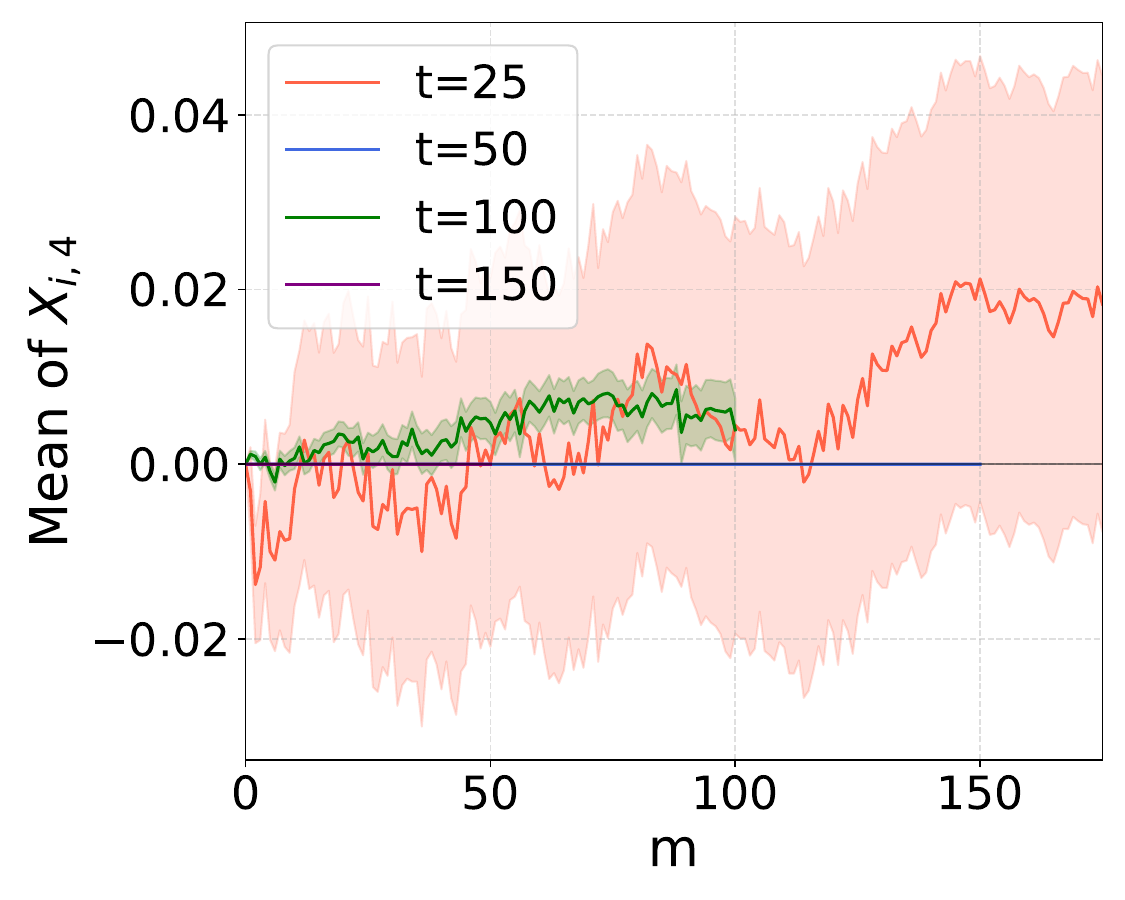}
        \caption{Arm 4}
    \end{subfigure}
    \\
    \begin{subfigure}[b]{0.32\textwidth}
        \centering
        \includegraphics[width=\textwidth]{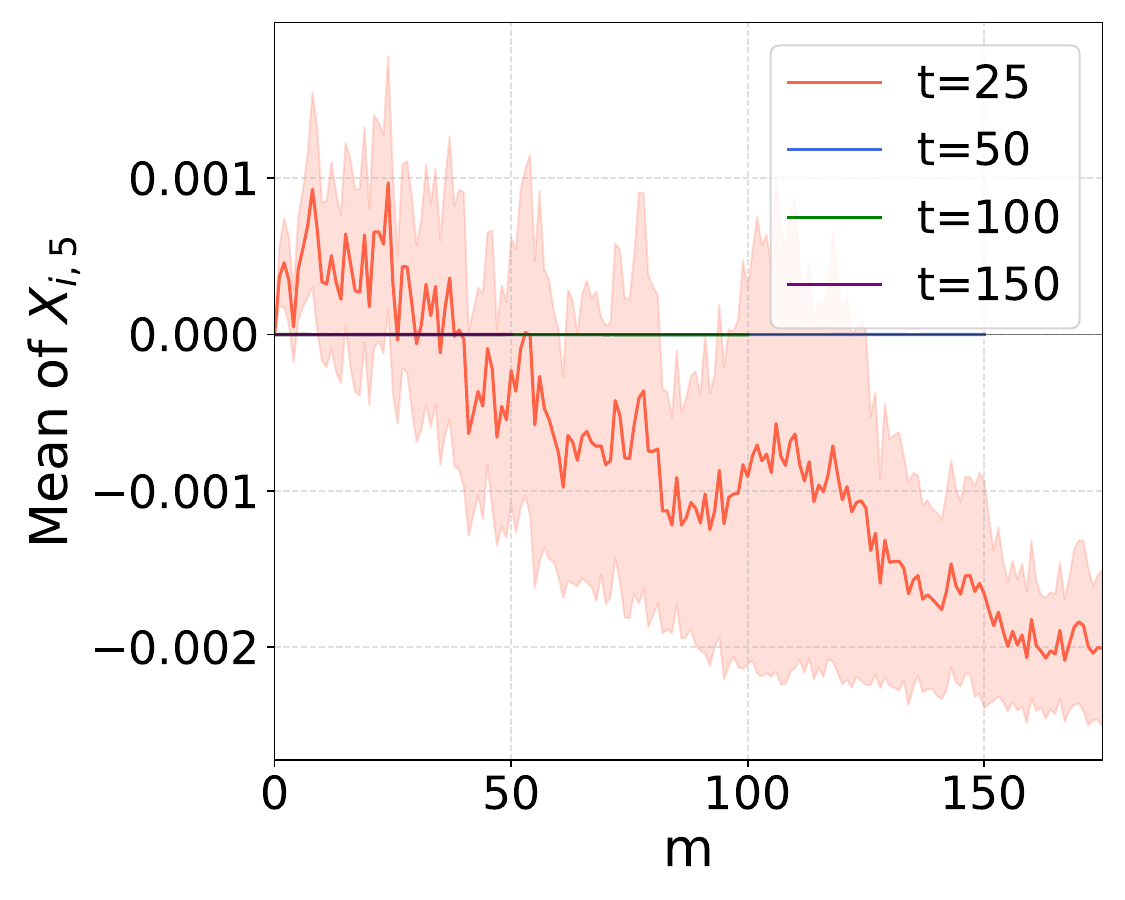}
        \caption{Arm 5}
    \end{subfigure}
    \begin{subfigure}[b]{0.32\textwidth}
        \centering
        \includegraphics[width=\textwidth]{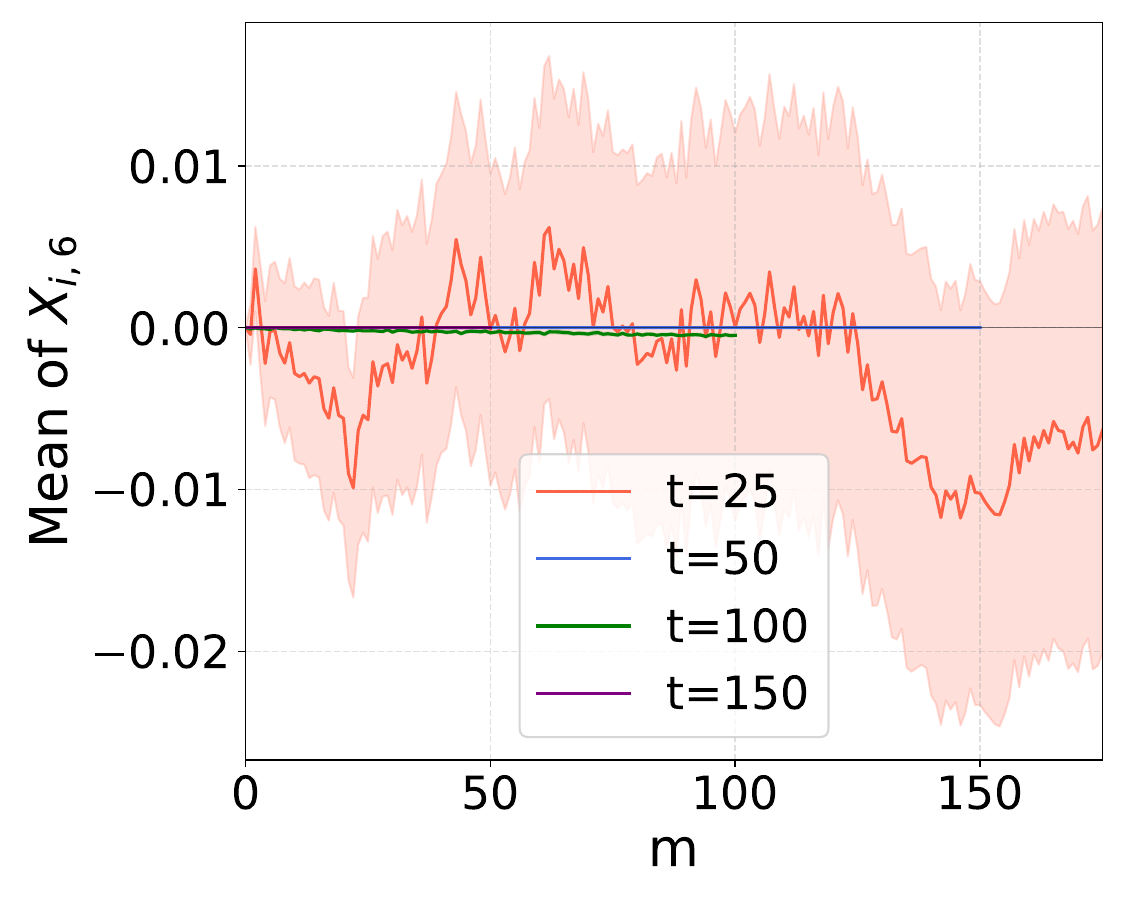}
        \caption{Arm 6}
    \end{subfigure}
    \begin{subfigure}[b]{0.32\textwidth}
        \centering
        \includegraphics[width=\textwidth]{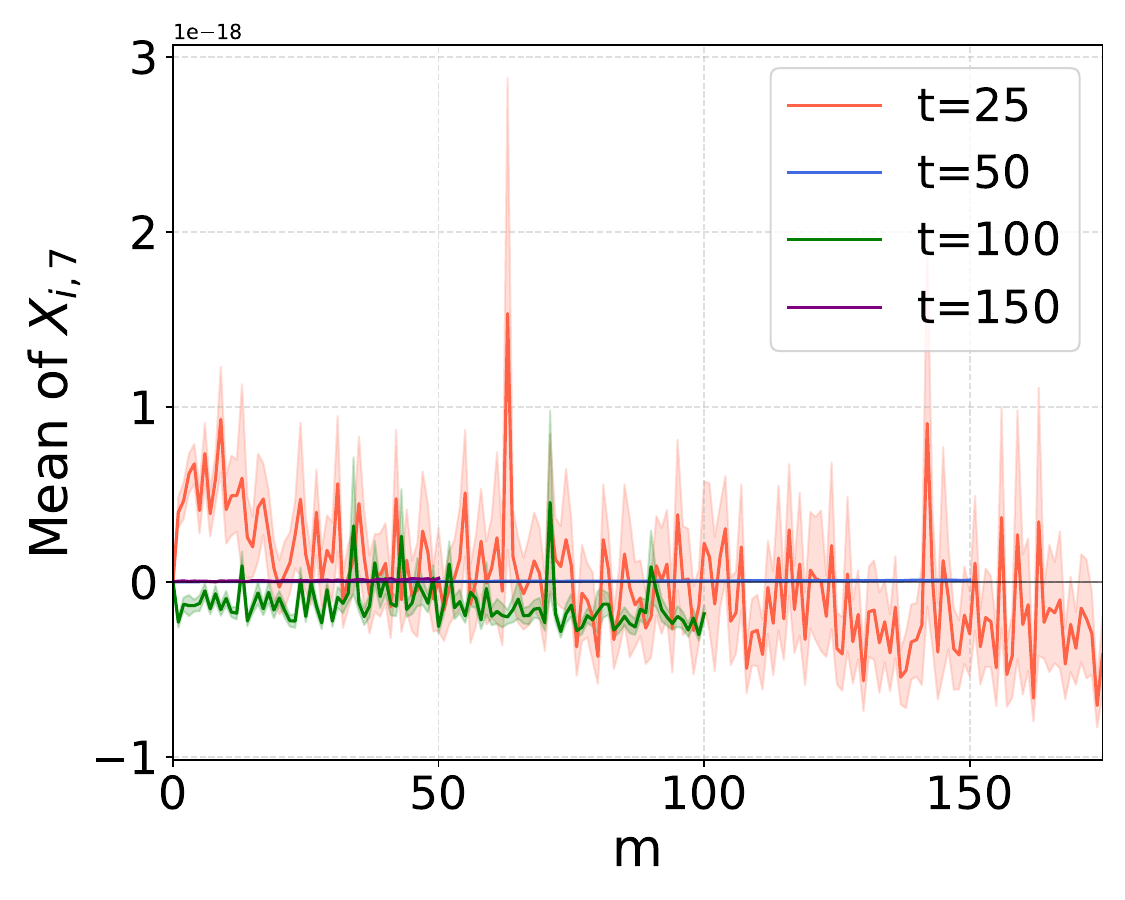}
        \caption{Arm 7}
    \end{subfigure}
    \\
    \begin{subfigure}[b]{0.32\textwidth}
        \centering
        \includegraphics[width=\textwidth]{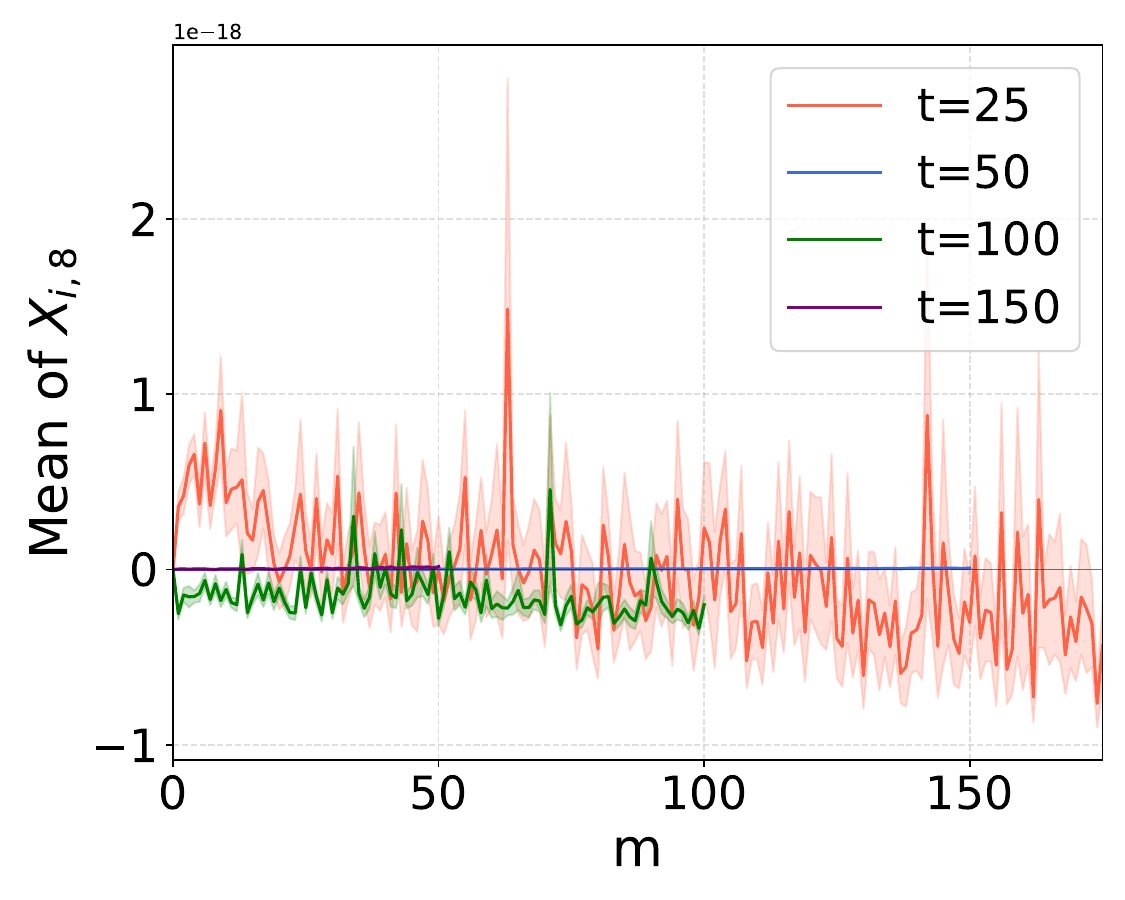}
        \caption{Arm 8}
    \end{subfigure}
    \begin{subfigure}[b]{0.32\textwidth}
        \centering
        \includegraphics[width=\textwidth]{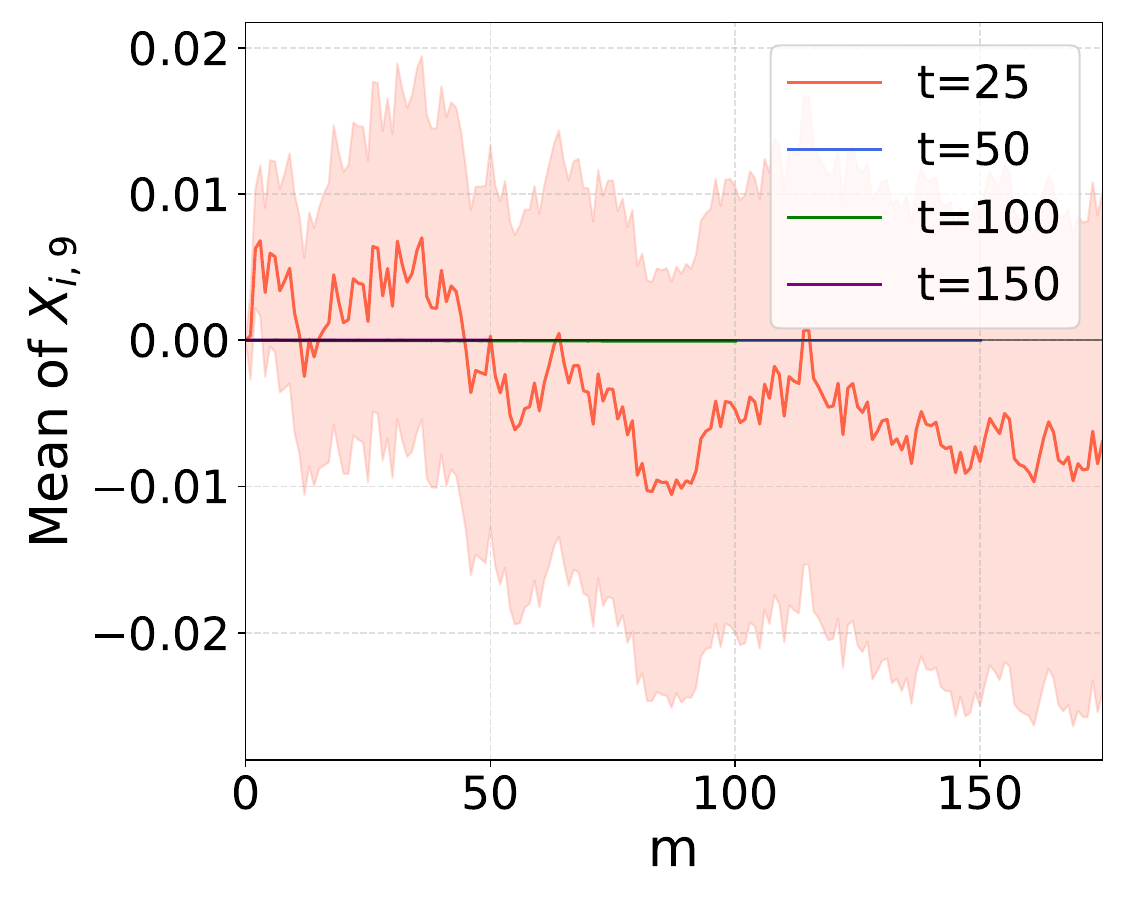}
        \caption{Arm 9}
    \end{subfigure}
    \caption{Per-arm martingale increments for GDPT in linear bandits with online starting contexts, where $m$ and $t$ denote the lengths of the martingale continuation and the starting contexts, respectively.}
    \label{fig_martingale_increment_online_linear_GDPT}
\end{figure}

\subsection{Ablation studies}
\label{sec:ablation_tables_and_plots}
\subsubsection{DPT sufficient-statistics table and evaluation figures}
\begin{table}[H]
\centering
\small
\setlength{\tabcolsep}{8pt}
\renewcommand{\arraystretch}{1.15}
\caption{$p$-values for sufficient-statistic invariance tests on permutation-invariant DPT in Gaussian bandits ($N=5000$, $\alpha=0.05$). ``Equal'' denotes identical action probabilities and no $p$-value is reported.}
\label{tab:ss_ablation_dpt}
\begin{tabular}{@{}r rrrr@{}}
\toprule
& \multicolumn{4}{c}{Gaussian} \\
\cmidrule(lr){2-5}
& \multicolumn{2}{c}{Offline}
& \multicolumn{2}{c}{Online} \\
$n$
& \multicolumn{1}{c}{Sort}
& \multicolumn{1}{c}{Replace}
& \multicolumn{1}{c}{Sort}
& \multicolumn{1}{c}{Replace} \\
\midrule
$50$  & Equal & $6.65\times 10^{-36}$
    & Equal & $6.60\times 10^{-90}$ \\
$100$ & Equal & $2.79\times 10^{-94}$
    & Equal & $2.54\times 10^{-62}$ \\
$200$ & Equal & $5.38\times 10^{-217}$
    & Equal & $1.46\times 10^{-46}$ \\
$300$ & Equal & $8.01\times 10^{-305}$
    & Equal & $2.11\times 10^{-44}$ \\
$400$ & Equal & $<10^{-300}$
    & Equal & $6.13\times 10^{-56}$ \\
$500$ & Equal & $<10^{-300}$
    & Equal & $2.55\times 10^{-67}$ \\
\bottomrule
\end{tabular}
\end{table}

\begin{figure}[H]
\centering
\begin{subfigure}{0.245\textwidth}
\includegraphics[width=\textwidth]{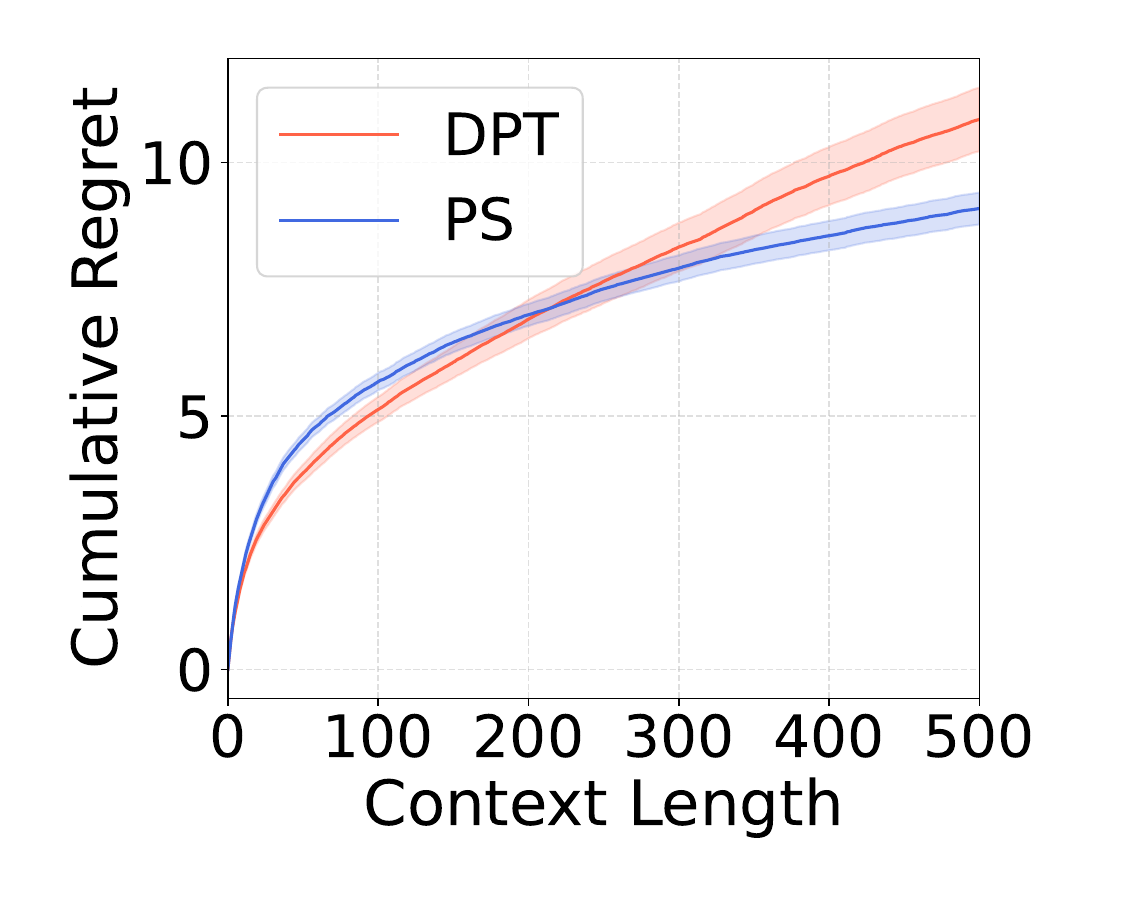}
\caption{Online regret}
\end{subfigure}
\hfill
\begin{subfigure}{0.245\textwidth}
\includegraphics[width=\textwidth]{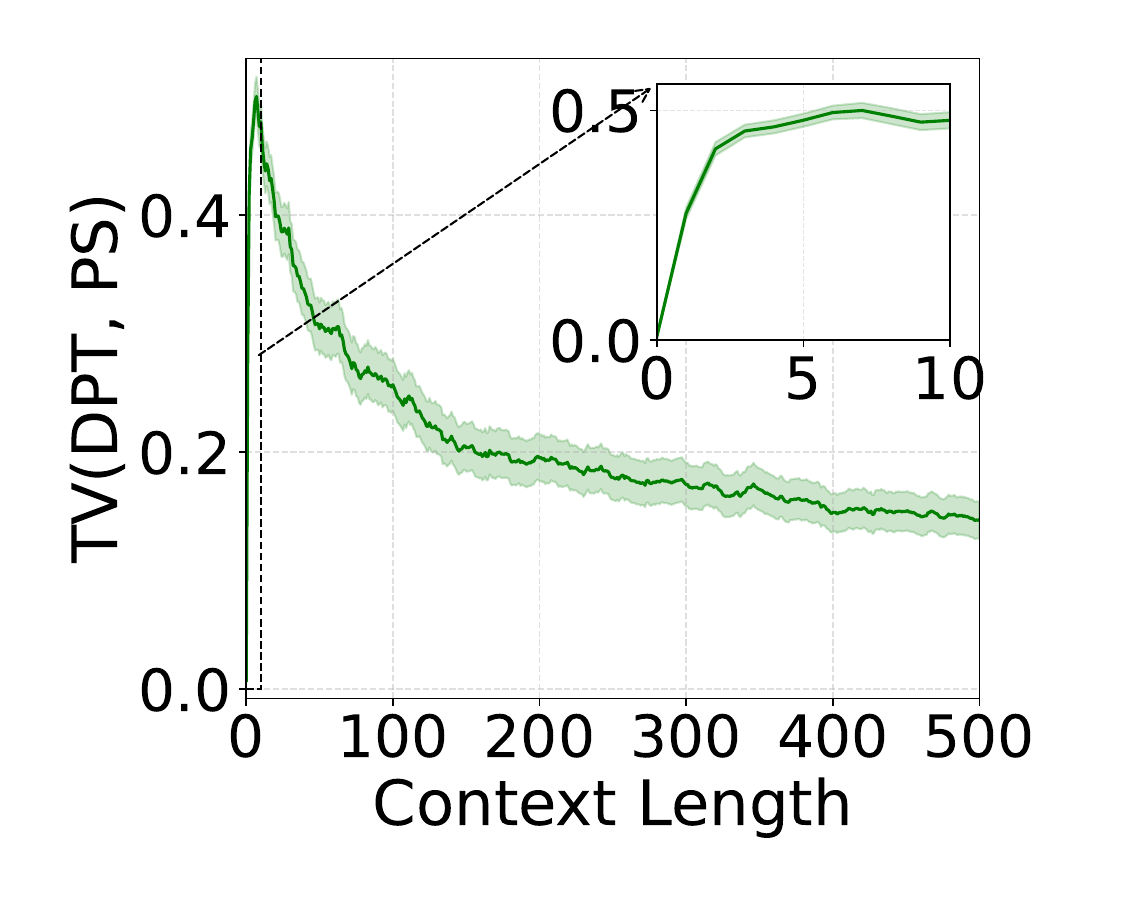}
\caption{Online TV distance}
\end{subfigure}
\hfill
\begin{subfigure}{0.245\textwidth}
\includegraphics[width=\textwidth]{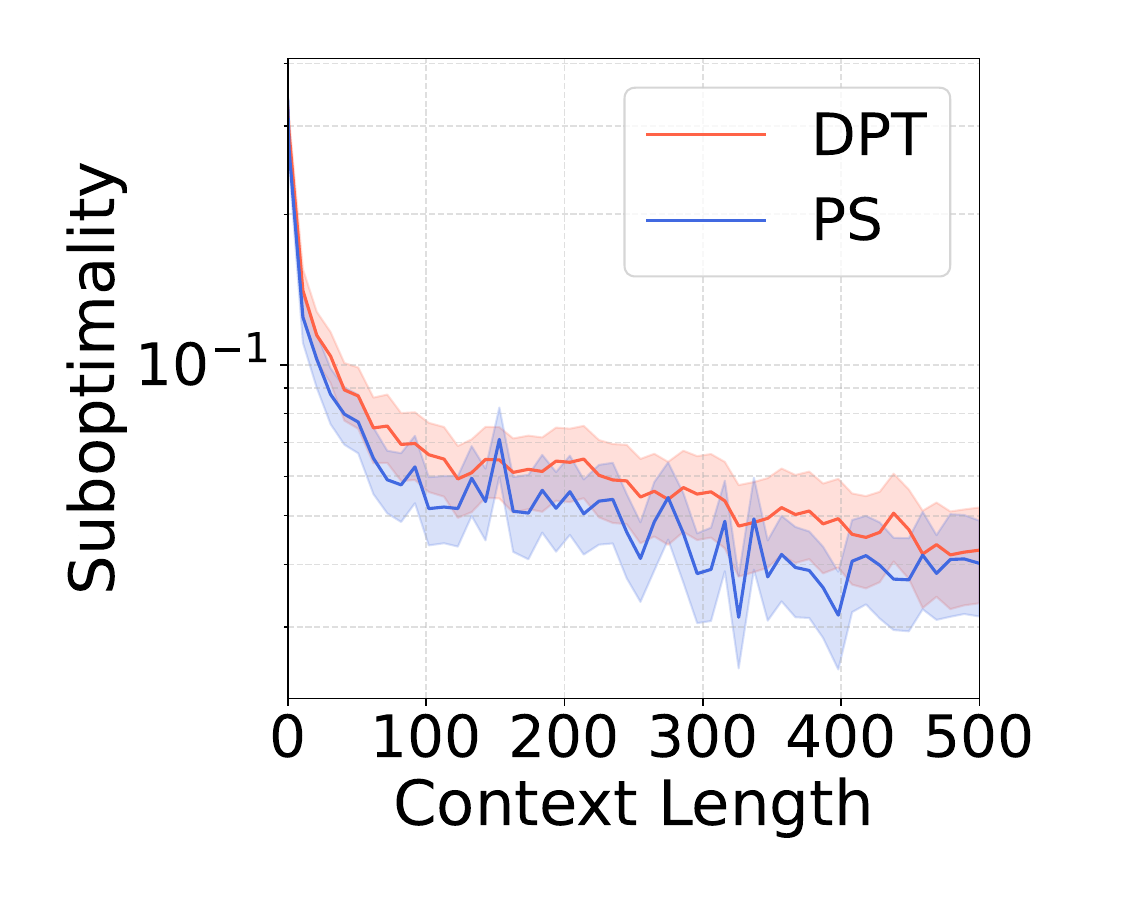}
\caption{Offline suboptimality}
\end{subfigure}
\hfill
\begin{subfigure}{0.245\textwidth}
\includegraphics[width=\textwidth]{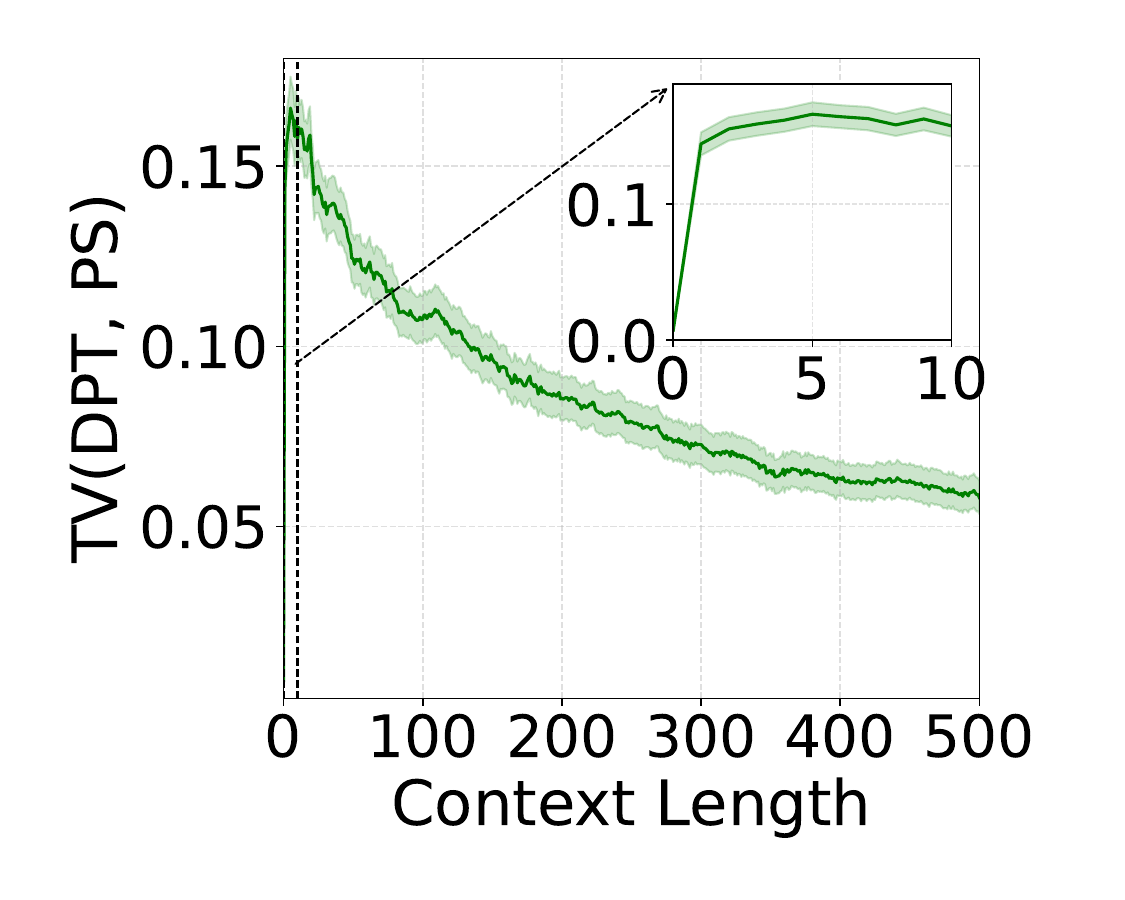}
\caption{Offline TV distance}
\end{subfigure}
\caption{Performance comparison of permutation-invariant DPT and Thompson Sampling in Gaussian bandits under online and offline deployment.}
\label{fig_evaluation_gaussian_bandit_DPT_ablation}
\end{figure}

\subsubsection{DPT\texorpdfstring{$'$}{'} sufficient-statistics and martingale tables and evaluation figures}
\begin{table}[H]
\centering
\small
\setlength{\tabcolsep}{8pt}
\renewcommand{\arraystretch}{1.15}
\caption{$p$-values for sufficient-statistic invariance tests on permutation-invariant DPT$'$ in Gaussian bandits ($N=5000$, $\alpha=0.05$). ``Equal'' denotes identical action probabilities and no $p$-value is reported.}
\label{tab:ss_ablation_dpt_prime}
\begin{tabular}{@{}r rrrr@{}}
\toprule
& \multicolumn{4}{c}{Gaussian} \\
\cmidrule(lr){2-5}
& \multicolumn{2}{c}{Offline}
& \multicolumn{2}{c}{Online} \\
$n$
& \multicolumn{1}{c}{Sort}
& \multicolumn{1}{c}{Replace}
& \multicolumn{1}{c}{Sort}
& \multicolumn{1}{c}{Replace} \\
\midrule
$50$  & Equal & $2.46\times 10^{-233}$
    & Equal & $<10^{-300}$ \\
$100$ & Equal & $1.09\times 10^{-293}$
    & Equal & $<10^{-300}$ \\
$200$ & Equal & $<10^{-300}$
    & Equal & $<10^{-300}$ \\
$300$ & Equal & $<10^{-300}$
    & Equal & $<10^{-300}$ \\
$400$ & Equal & $<10^{-300}$
    & Equal & $<10^{-300}$ \\
$500$ & Equal & $<10^{-300}$
    & Equal & $<10^{-300}$ \\
\bottomrule
\end{tabular}
\end{table}

\begin{table}[H]
\centering
\small
\setlength{\tabcolsep}{8pt}
\renewcommand{\arraystretch}{1.15}
\caption{$p$-values for martingale tests on permutation-invariant DPT$'$ in Gaussian bandits ($N=5000$, $\alpha=0.05$). Here, $n$ and $m$ denote the starting-context and continuation lengths, respectively.}
\label{tab:martingale_ablation_dpt_prime}
\begin{tabular}{@{}c r rr@{}}
\toprule
& & \multicolumn{2}{c}{Gaussian} \\
\cmidrule(lr){3-4}
$n$ & $m$
& \multicolumn{1}{c}{Offline}
& \multicolumn{1}{c}{Online} \\
\midrule
\multirow[c]{5}{*}{$50$}
 & $100$ & $<10^{-300}$ & $<10^{-300}$ \\
 & $200$ & $<10^{-300}$ & $<10^{-300}$ \\
 & $300$ & $<10^{-300}$ & $<10^{-300}$ \\
 & $400$ & $<10^{-300}$ & $<10^{-300}$ \\
 & $450$ & $<10^{-300}$ & $<10^{-300}$ \\
\addlinespace[0.6ex]
\multirow[c]{4}{*}{$100$}
 & $100$ & $<10^{-300}$ & $1.22\times 10^{-37}$ \\
 & $200$ & $<10^{-300}$ & $1.17\times 10^{-159}$ \\
 & $300$ & $<10^{-300}$ & $<10^{-300}$ \\
 & $400$ & $<10^{-300}$ & $<10^{-300}$ \\
\addlinespace[0.6ex]
\multirow[c]{3}{*}{$200$}
 & $100$ & $<10^{-300}$ & $1.16\times 10^{-74}$ \\
 & $200$ & $<10^{-300}$ & $1.53\times 10^{-256}$ \\
 & $300$ & $<10^{-300}$ & $<10^{-300}$ \\
\addlinespace[0.6ex]
\multirow[c]{2}{*}{$300$}
 & $100$ & $<10^{-300}$ & $<10^{-300}$ \\
 & $200$ & $<10^{-300}$ & $1.92\times 10^{-253}$ \\
\addlinespace[0.6ex]
$400$ & $100$ & $<10^{-300}$ & $1.41\times 10^{-236}$ \\
\bottomrule
\end{tabular}
\end{table}

\begin{figure}[H]
\centering
\begin{subfigure}{0.245\textwidth}
\includegraphics[width=\textwidth]{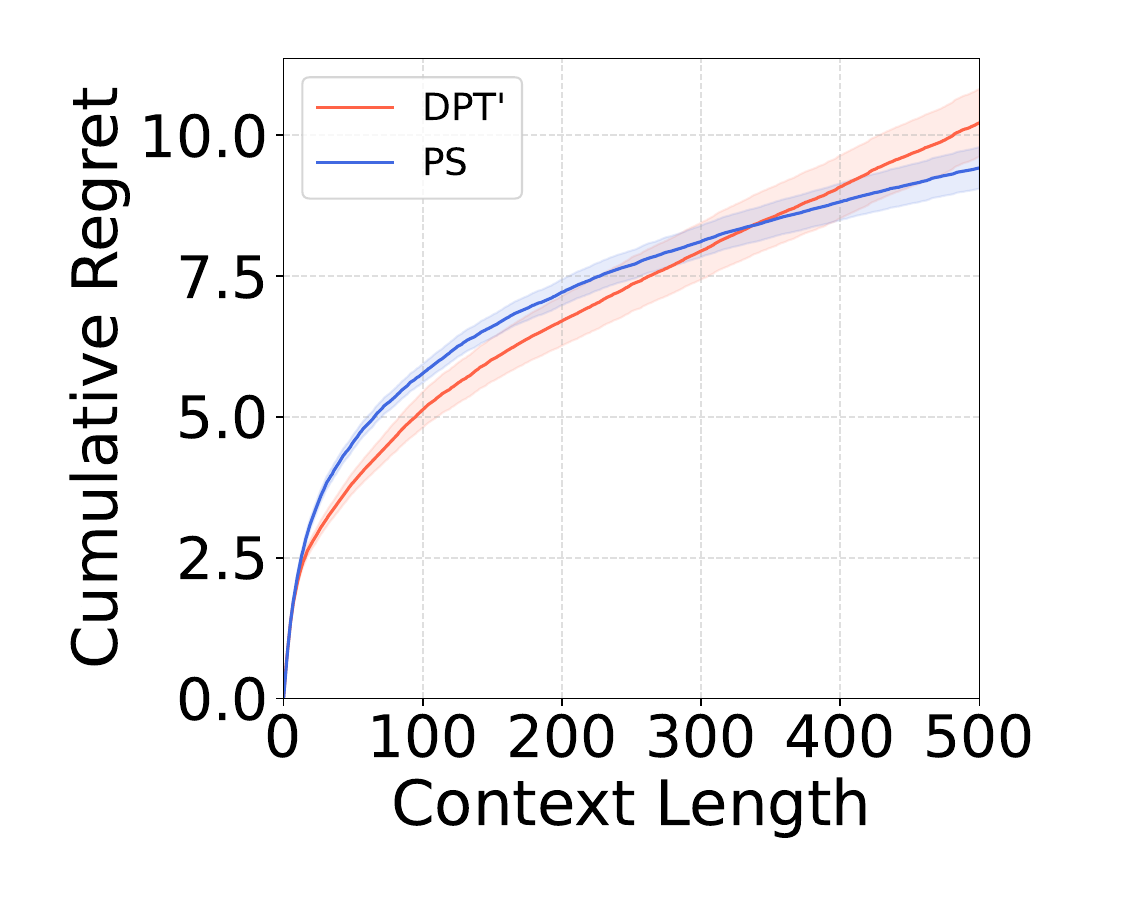}
\caption{Online regret}
\end{subfigure}
\hfill
\begin{subfigure}{0.245\textwidth}
\includegraphics[width=\textwidth]{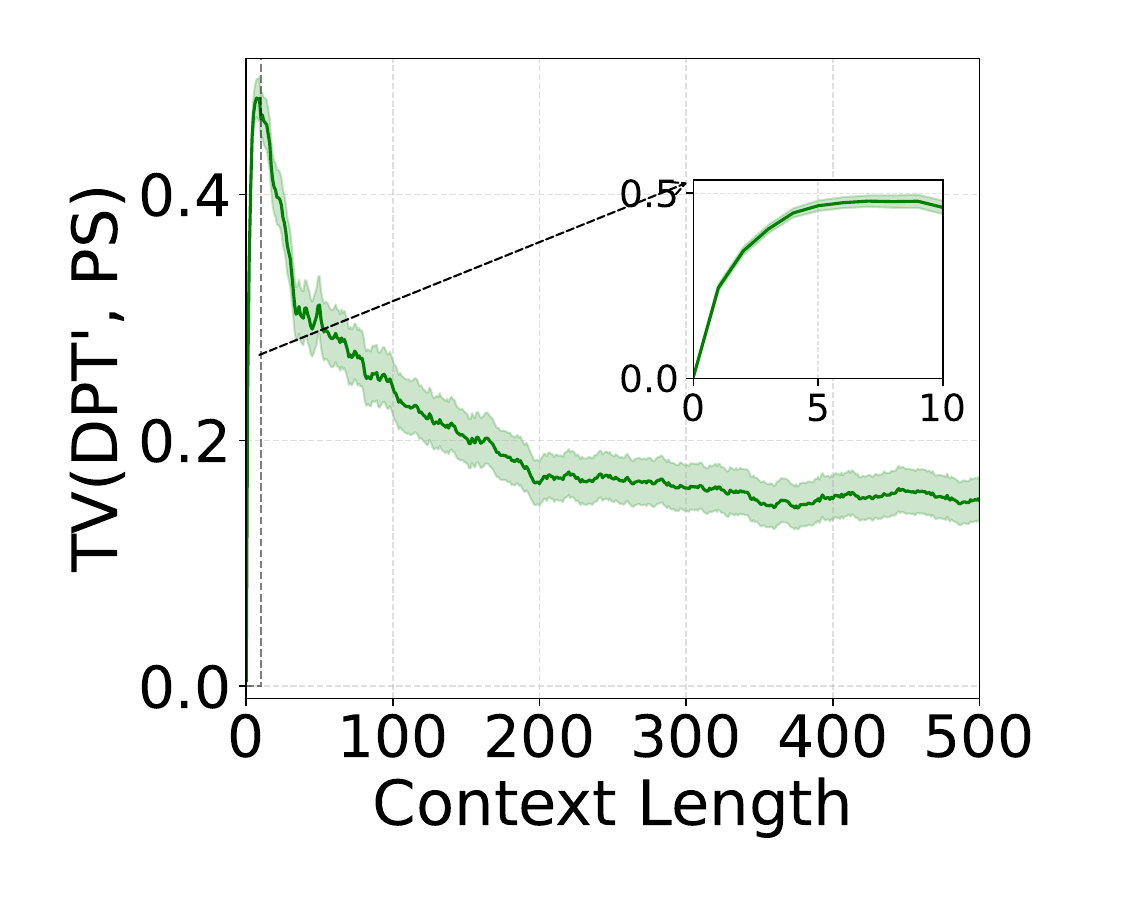}
\caption{Online TV distance}
\end{subfigure}
\hfill
\begin{subfigure}{0.245\textwidth}
\includegraphics[width=\textwidth]{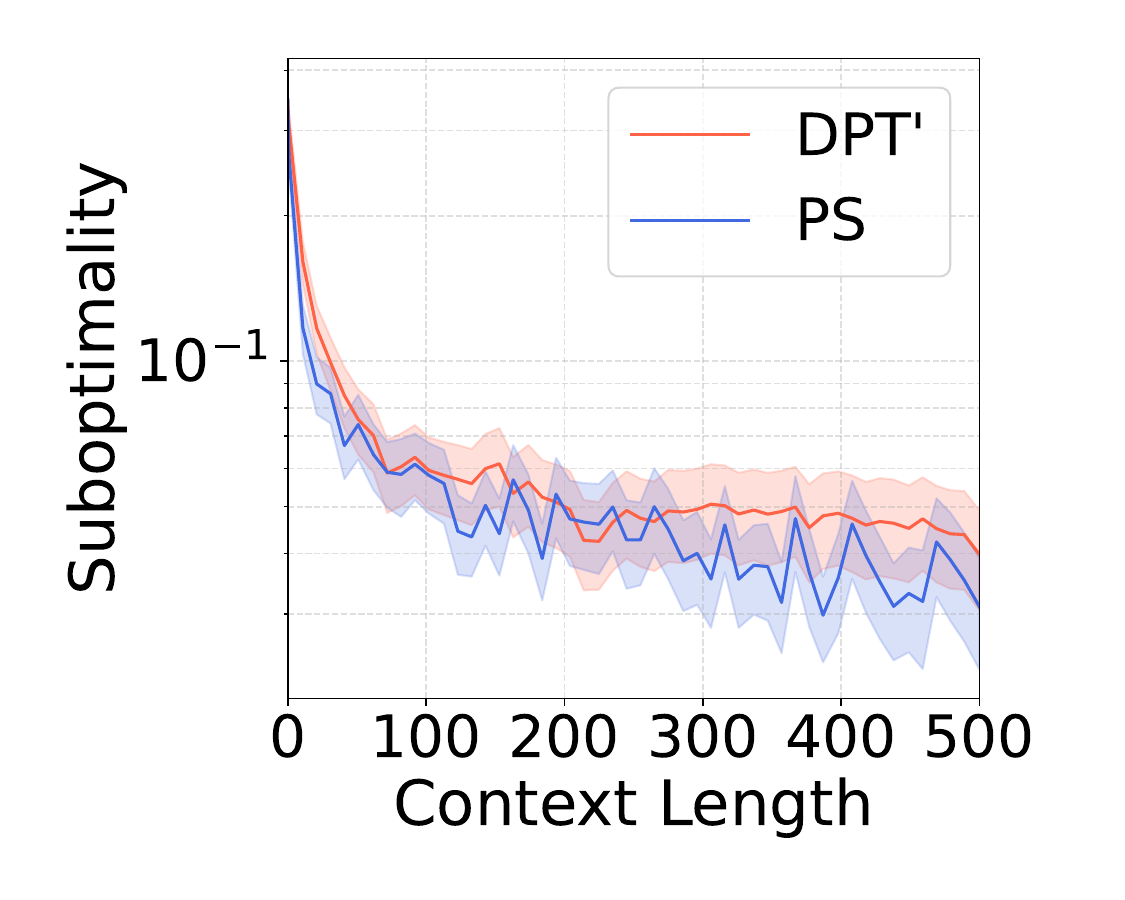}
\caption{Offline suboptimality}
\end{subfigure}
\hfill
\begin{subfigure}{0.245\textwidth}
\includegraphics[width=\textwidth]{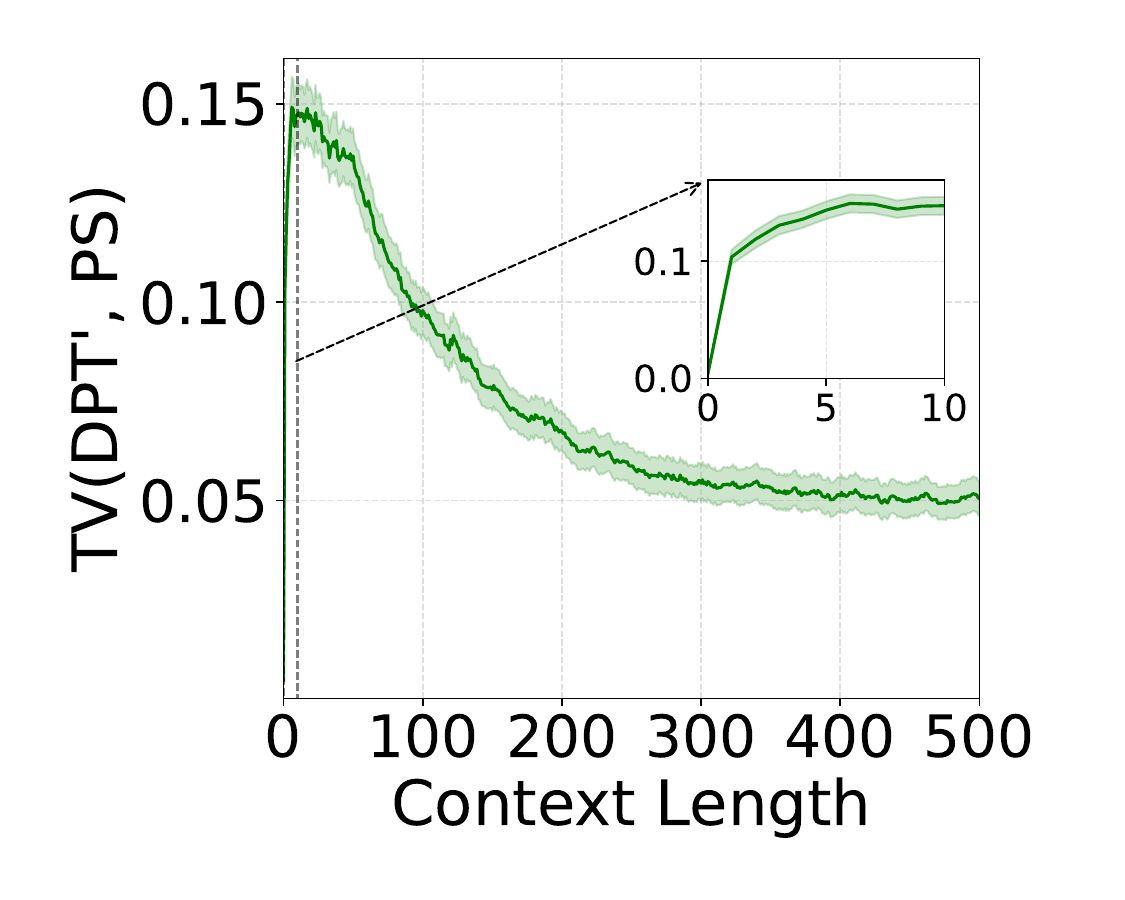}
\caption{Offline TV distance}
\end{subfigure}
\caption{Performance comparison of permutation-invariant DPT$'$ and Thompson Sampling in Gaussian bandits under online and offline deployment.}
\label{fig_evaluation_gaussian_bandit_DPT_prime_ablation}
\end{figure}

\begin{figure}[H]
    \centering
    \begin{subfigure}[b]{0.32\textwidth}
        \centering
        \includegraphics[width=\textwidth]{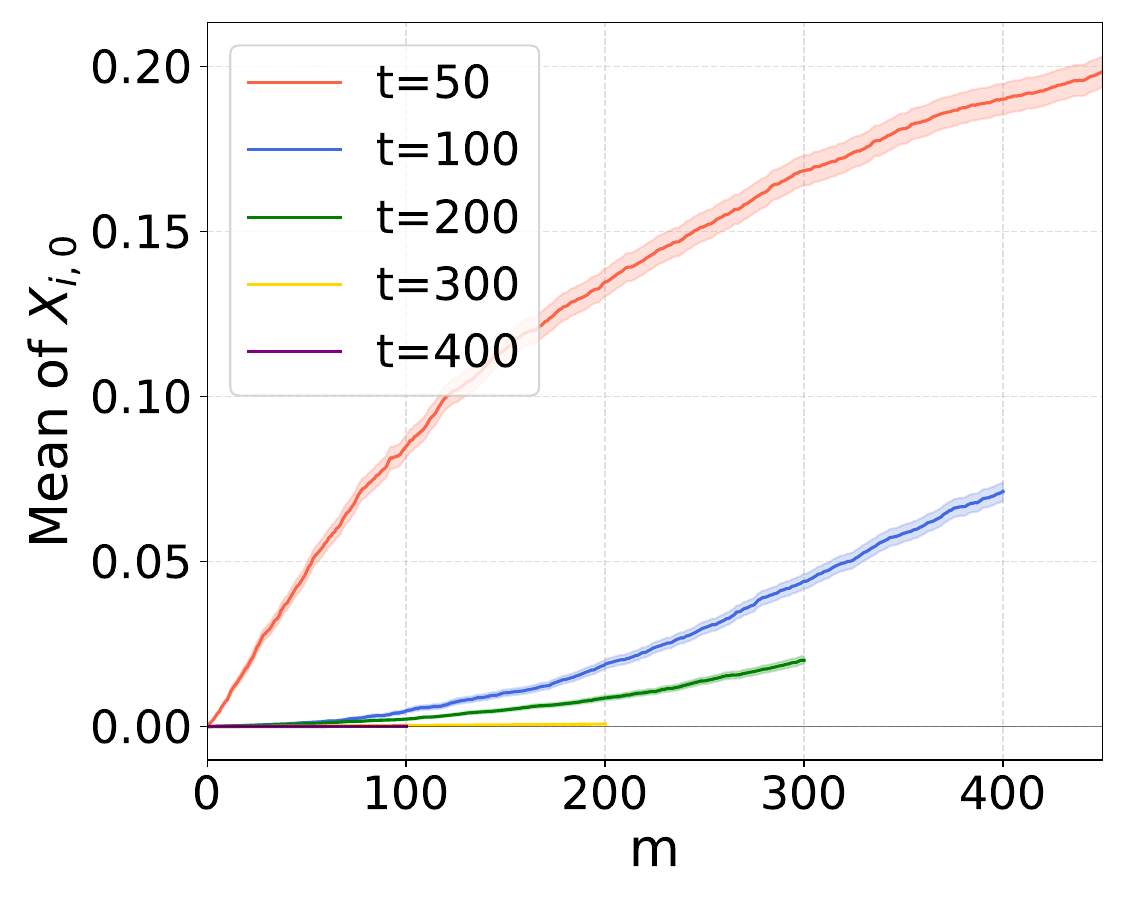}
        \caption{Arm 0}
    \end{subfigure}
    \begin{subfigure}[b]{0.32\textwidth}
        \centering
        \includegraphics[width=\textwidth]{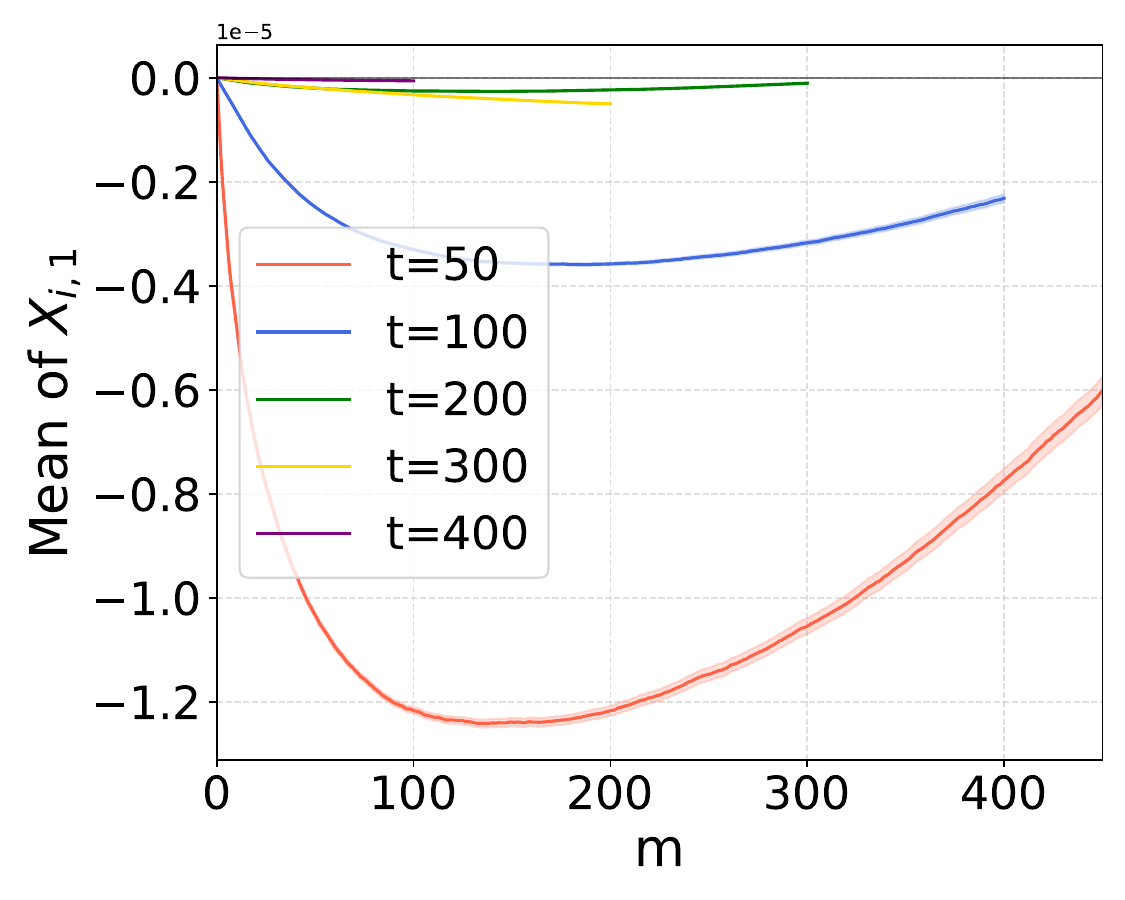}
        \caption{Arm 1}
    \end{subfigure}
    \begin{subfigure}[b]{0.32\textwidth}
        \centering
        \includegraphics[width=\textwidth]{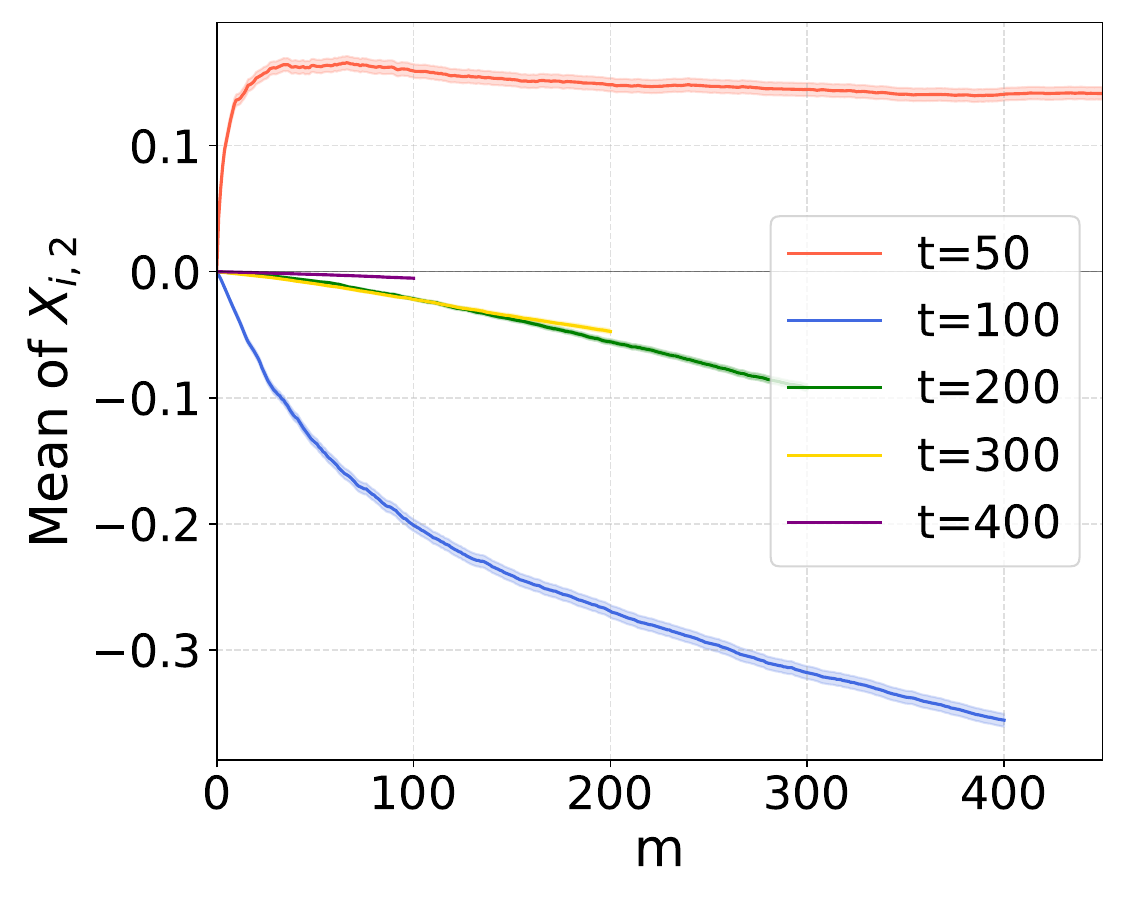}
        \caption{Arm 2}
    \end{subfigure}
    \\
    \begin{subfigure}[b]{0.32\textwidth}
        \centering
        \includegraphics[width=\textwidth]{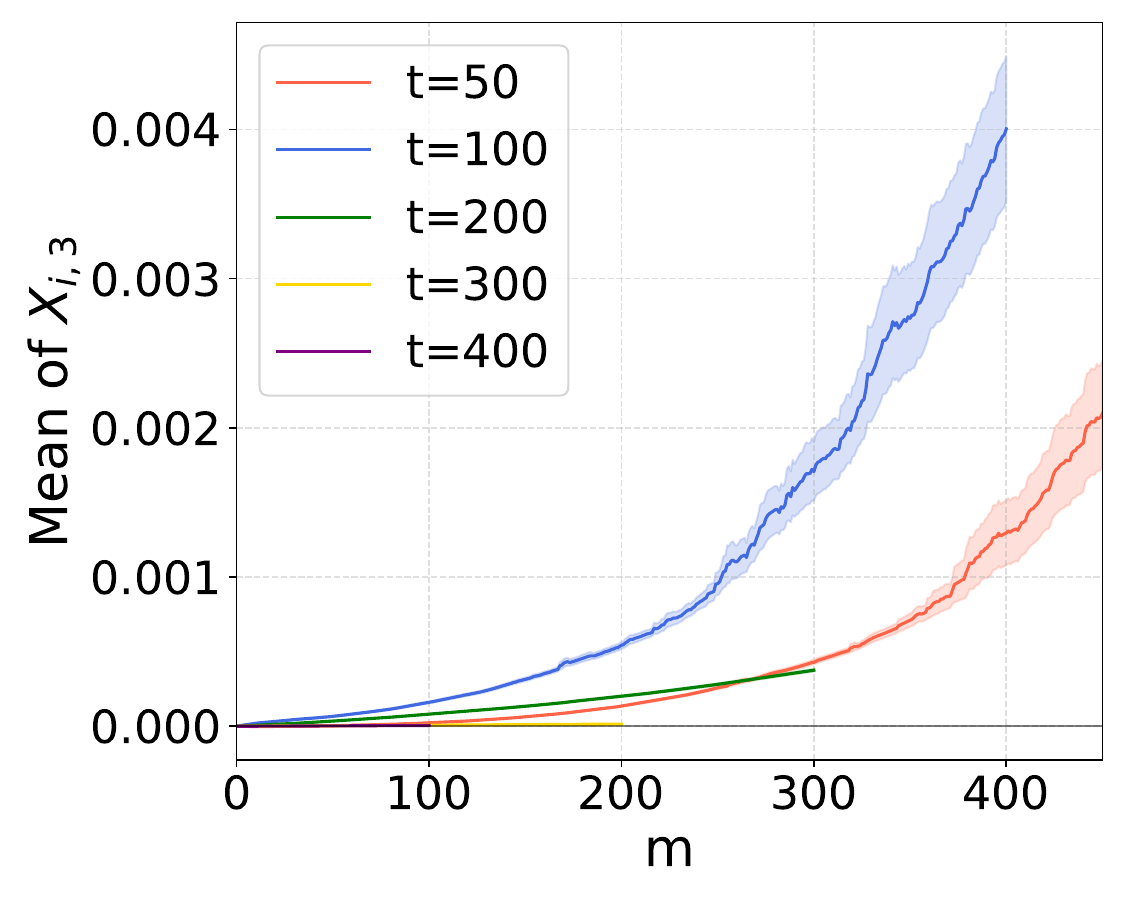}
        \caption{Arm 3}
    \end{subfigure}
    \begin{subfigure}[b]{0.32\textwidth}
        \centering
        \includegraphics[width=\textwidth]{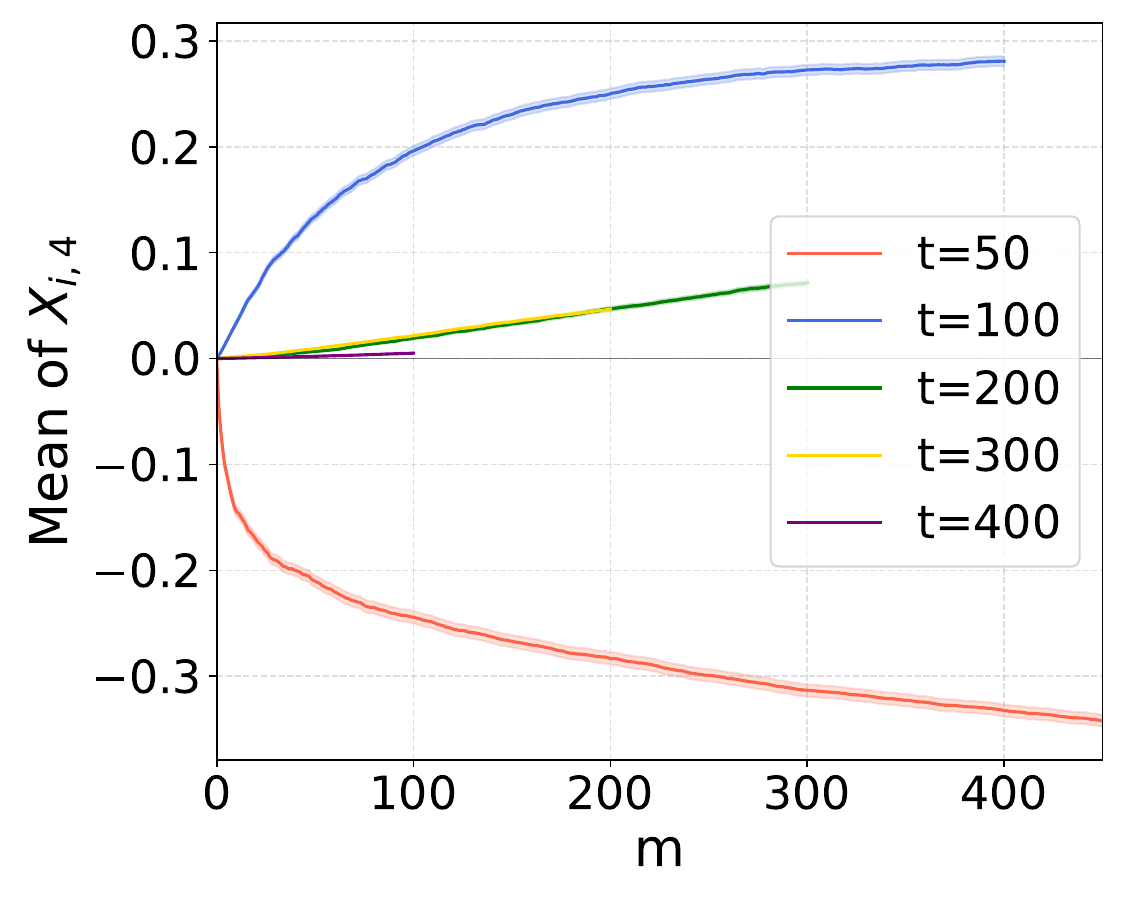}
        \caption{Arm 4}
    \end{subfigure}
    \caption{Per-arm martingale increments for permutation-invariant DPT$'$ in Gaussian bandits with offline starting contexts, where $m$ and $t$ denote the lengths of the martingale continuation and the starting contexts, respectively.}
    \label{fig_martingale_increment_offline_gaussian_DPT_prime_ablation}
\end{figure}

\begin{figure}[H]
    \centering
    \begin{subfigure}[b]{0.32\textwidth}
        \centering
        \includegraphics[width=\textwidth]{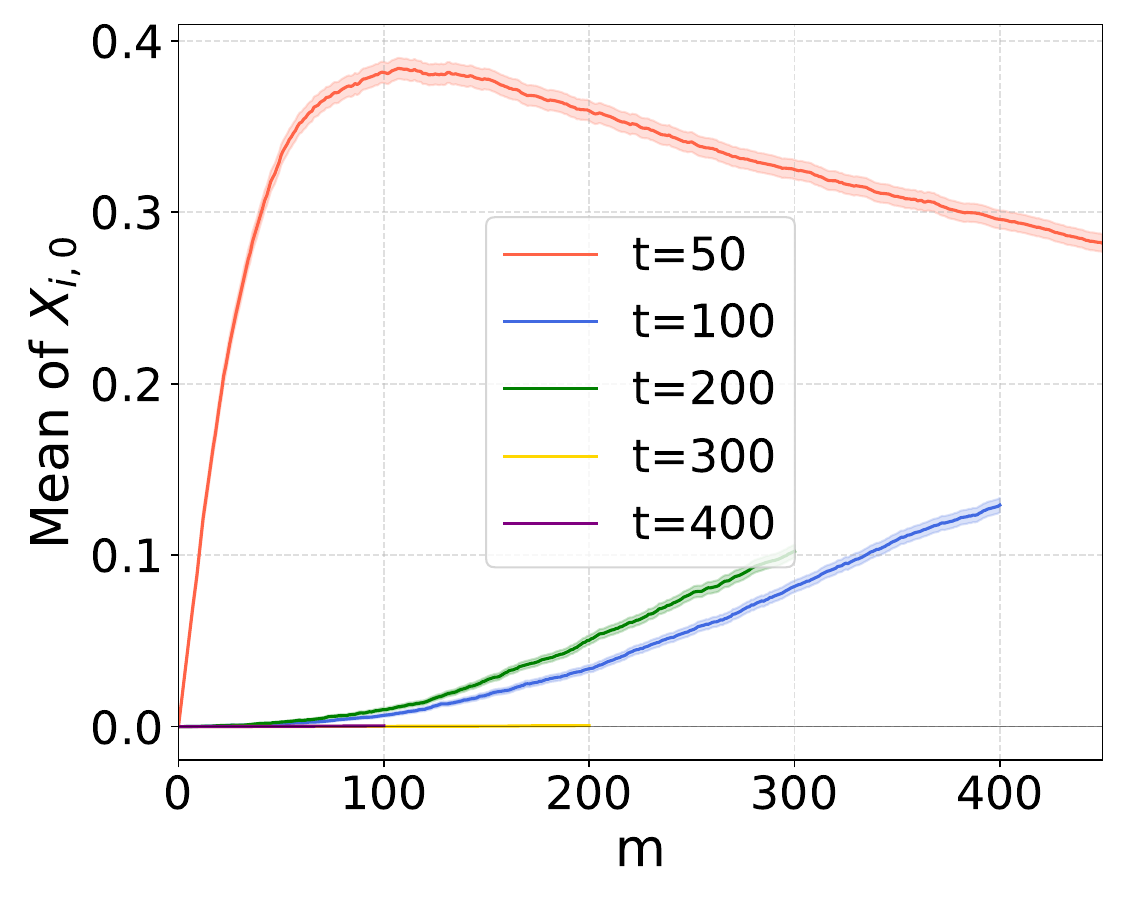}
        \caption{Arm 0}
    \end{subfigure}
    \begin{subfigure}[b]{0.32\textwidth}
        \centering
        \includegraphics[width=\textwidth]{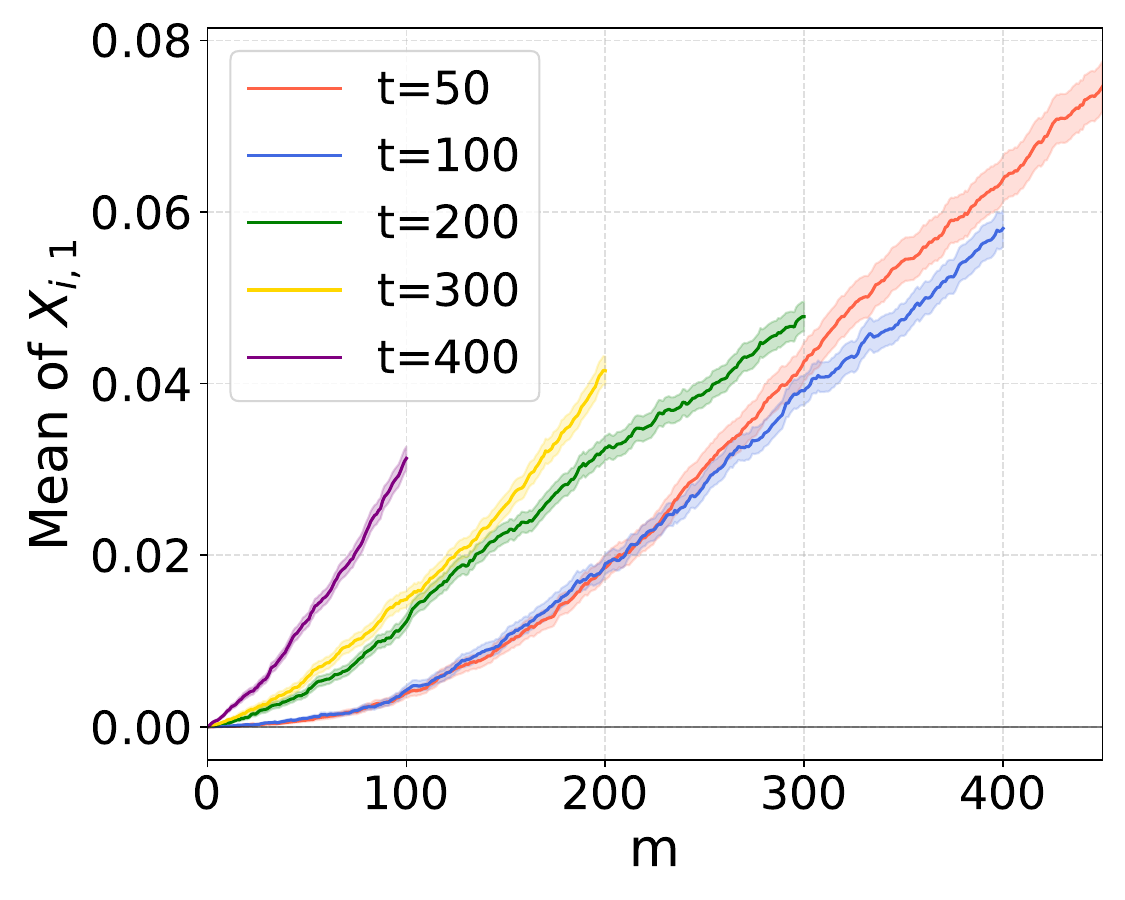}
        \caption{Arm 1}
    \end{subfigure}
    \begin{subfigure}[b]{0.32\textwidth}
        \centering
        \includegraphics[width=\textwidth]{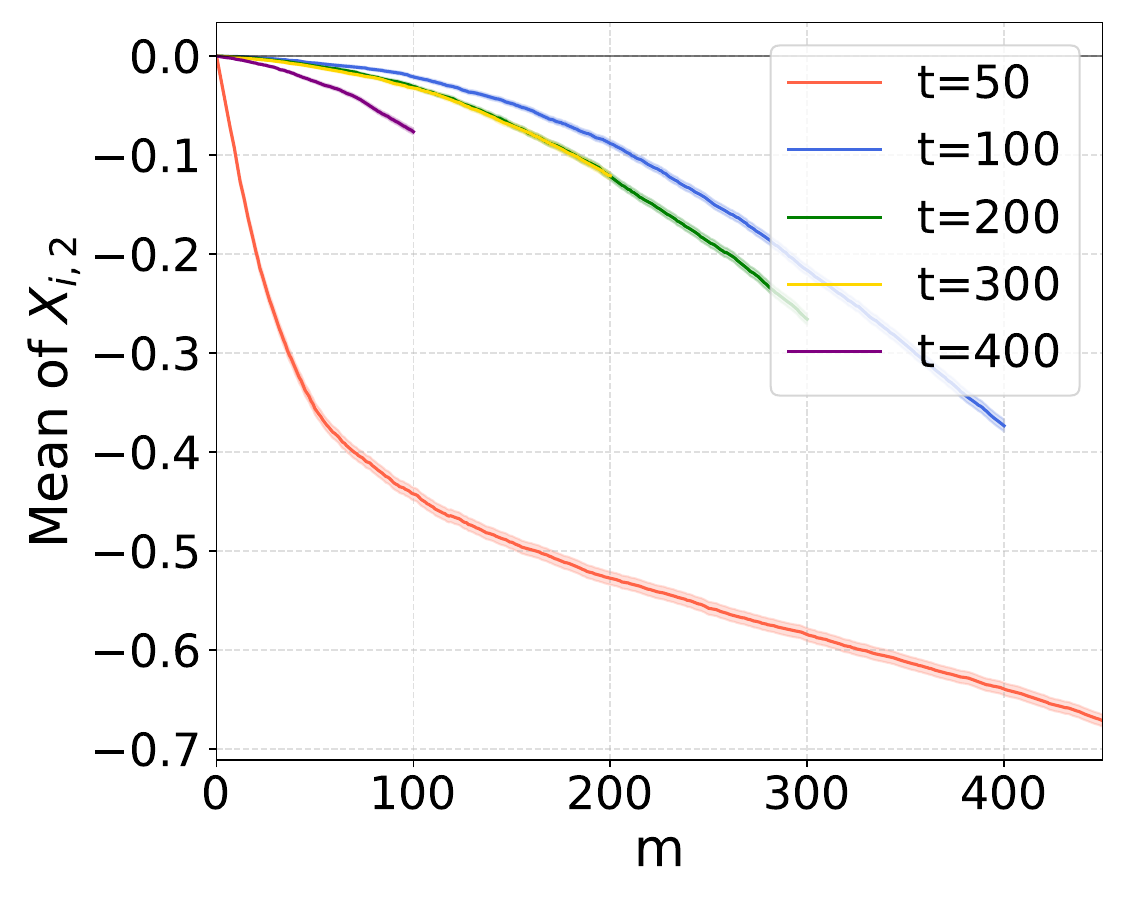}
        \caption{Arm 2}
    \end{subfigure}
    \\
    \begin{subfigure}[b]{0.32\textwidth}
        \centering
        \includegraphics[width=\textwidth]{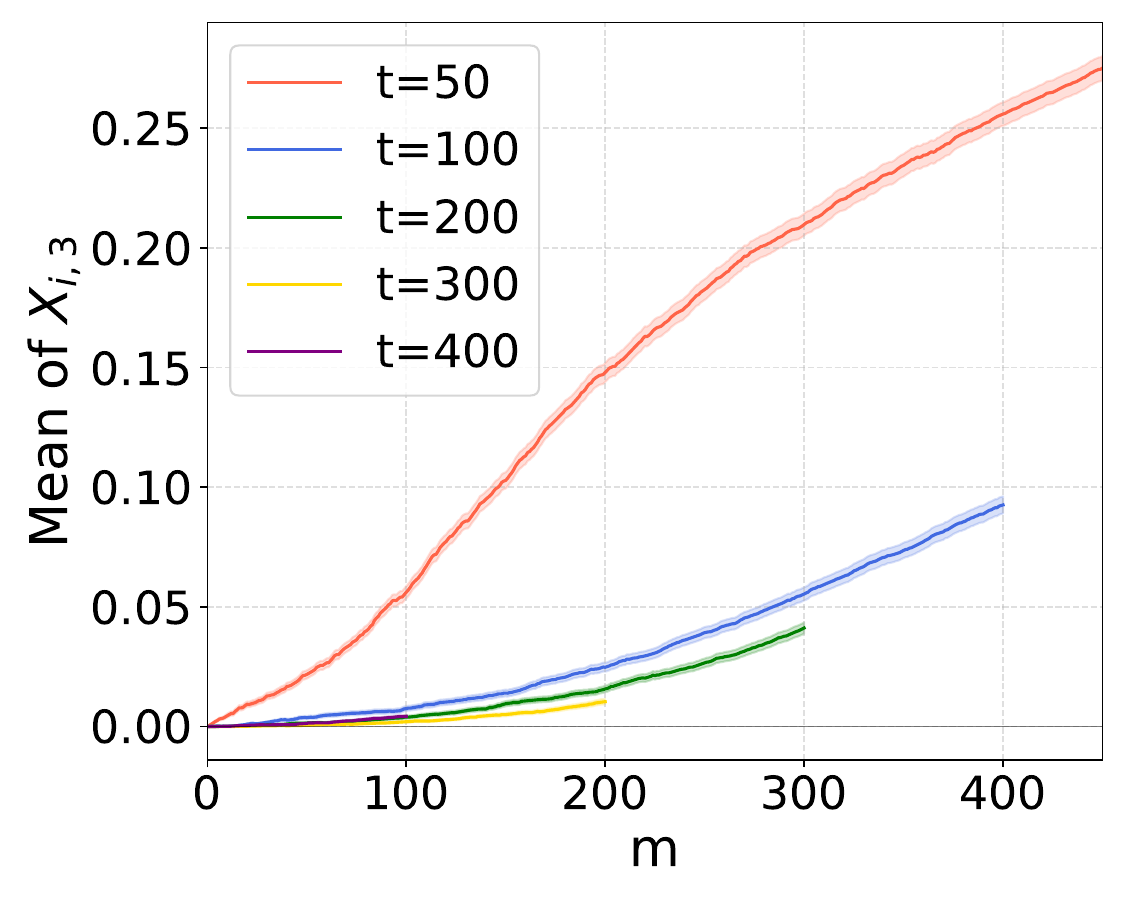}
        \caption{Arm 3}
    \end{subfigure}
    \begin{subfigure}[b]{0.32\textwidth}
        \centering
        \includegraphics[width=\textwidth]{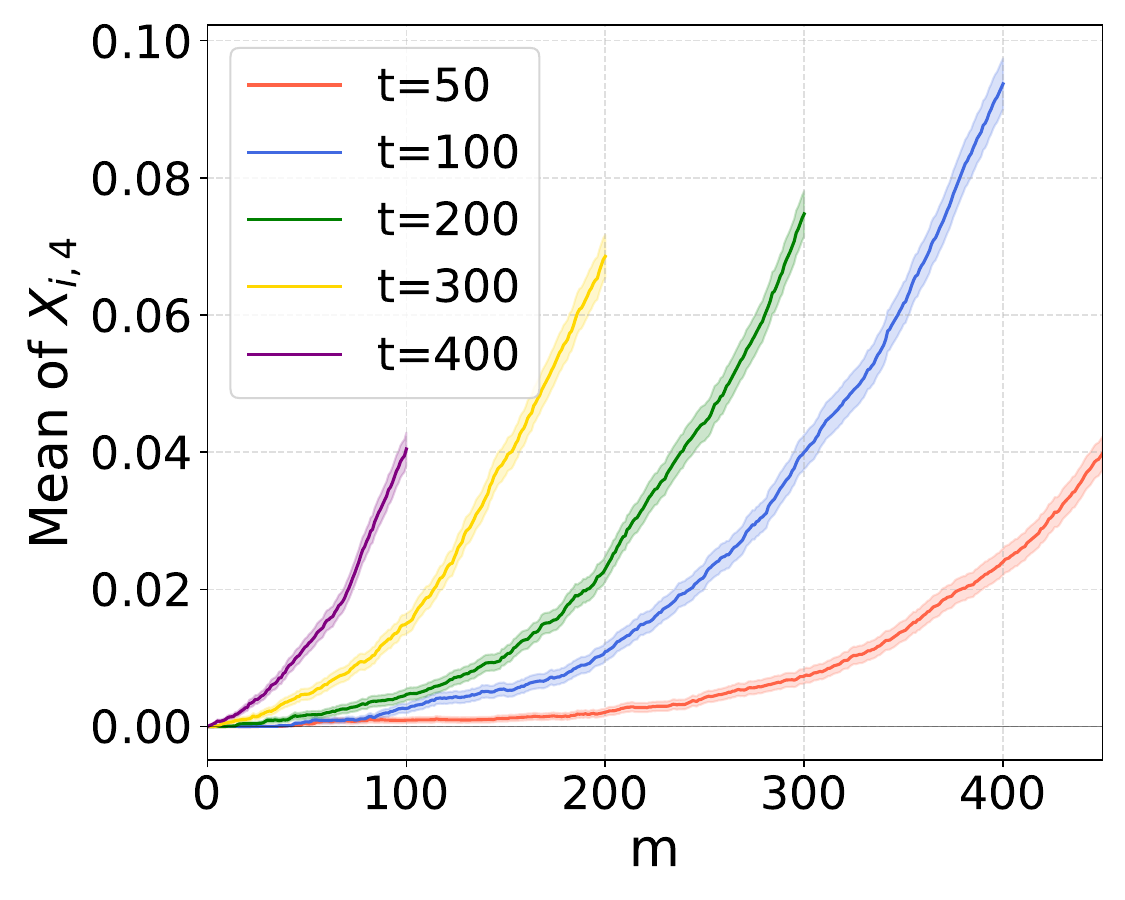}
        \caption{Arm 4}
    \end{subfigure}
    \caption{Per-arm martingale increments for permutation-invariant DPT$'$ in Gaussian bandits with online starting contexts, where $m$ and $t$ denote the lengths of the martingale continuation and the starting contexts, respectively.}
    \label{fig_martingale_increment_online_gaussian_DPT_prime_ablation}
\end{figure}

\end{document}